\documentclass{article} %

\newcommand{\loose}{\looseness=-1}

\usepackage{etoolbox}
\newcommand{\arxiv}[1]{\iftoggle{colt}{}{#1}}
\newcommand{\colt}[1]{\iftoggle{colt}{#1}{}}
\newtoggle{colt}
\global\togglefalse{colt}

\colt{
\usepackage{times}
}

\usepackage[utf8]{inputenc} %
\usepackage[T1]{fontenc}    %
\usepackage{url}            %
\usepackage{booktabs}       %
\usepackage{amsfonts}       %
\usepackage{nicefrac}       %
\usepackage{microtype}      %

\usepackage{tocloft}            %

\usepackage{enumitem}

\usepackage{breakcites}

\usepackage{etoolbox}
\usepackage{comment}
\newtoggle{draft}
\togglefalse{draft}
\newcommand{\draft}[1]{\iftoggle{draft}{#1}{}}

\usepackage{mathrsfs}
\usepackage{algorithm}
\usepackage{verbatim}
\usepackage[noend]{algpseudocode}

\usepackage{multicol}

\usepackage{colortbl}

\usepackage{setspace}

\usepackage{transparent}

\usepackage{inconsolata}
\usepackage[scaled=.90]{helvet}
\usepackage{xspace}

\usepackage{pifont}

\arxiv{
\usepackage[letterpaper, left=1in, right=1in, top=1in, bottom=1in]{geometry}
\PassOptionsToPackage{hypertexnames=false}{hyperref}  %
\usepackage{parskip}

\usepackage[dvipsnames]{xcolor}
\usepackage[colorlinks=true, linkcolor=blue!70!black, citecolor=blue!70!black,urlcolor=black,breaklinks=true]{hyperref}
}

\colt{
\PassOptionsToPackage{dvipsnames}{xcolor}
}

\usepackage{microtype}
\usepackage{hhline}

\makeatletter
\newcommand{\neutralize}[1]{\expandafter\let\csname c@#1\endcsname\count@}
\makeatother

\usepackage{algorithm}

\arxiv{
\usepackage{natbib}
\bibliographystyle{plainnat}
\bibpunct{(}{)}{;}{a}{,}{,}
}

\usepackage{amsthm}
\usepackage{mathtools}
\usepackage{amsmath}
\usepackage{bbm}
\usepackage{amsfonts}
\usepackage{amssymb}

\usepackage{xpatch}

\usepackage{thmtools}
\usepackage{thm-restate}
\declaretheorem[name=Theorem,parent=section]{theorem}
\declaretheorem[name=Lemma,parent=section]{lemma}
\declaretheorem[name=Assumption, parent=section]{assumption}
\declaretheorem[name=Condition, parent=section]{condition}
\declaretheorem[qed=$\triangleleft$,name=Example,style=definition, parent=section]{example}
\declaretheorem[name=Remark,style=definition, parent=section]{remark}
\declaretheorem[name=Proposition, parent=section]{proposition}

\colt{
\newcommand{\jmlrBlackBox}{\rule{1.5ex}{1.5ex}}

\newcommand{\jmlrQED}{\hfill\jmlrBlackBox\par\bigskip}

\makeatletter
  \renewenvironment{proof}[1][Proof]%
  {%
   \par\noindent{\bfseries\upshape {#1.}\ }%
  }%
  {\hfill\jmlrBlackBox\newline}
  \makeatother
}

\arxiv{
\makeatletter
  \renewenvironment{proof}[1][Proof]%
  {%
   \par\noindent{\bfseries\upshape {#1.}\ }%
  }%
  {\qed\newline}
  \makeatother
}

\theoremstyle{definition}  %

\newtheorem{corollary}{Corollary}[section]

\theoremstyle{plain}
\newtheorem{definition}{Definition}[section]

\xpatchcmd{\proof}{\itshape}{\normalfont\proofnameformat}{}{}
\newcommand{\proofnameformat}{\bfseries}

\usepackage[nameinlink,capitalize]{cleveref}

\newcommand{\pref}[1]{\cref{#1}}
\newcommand{\pfref}[1]{Proof of \pref{#1}}

\renewcommand{\eqref}[1]{\texorpdfstring{\hyperref[#1]{Eq. (\ref*{#1})}}{Eq. (\ref*{#1})}}

\crefformat{equation}{#2(#1)#3}
\Crefformat{equation}{#2(#1)#3}

\Crefformat{figure}{#2Figure #1#3}
\Crefname{assumption}{Assumption}{Assumptions}
\Crefformat{assumption}{#2Assumption #1#3}

\Crefname{informal1}{Theorem}{Theorems}
\Crefformat{informal1}{#2Theorem #1#3}

\Crefname{condition}{Condition}{Conditions}
\Crefformat{condition}{#2Condition #1#3}

\newcommand{\aref}[1]{\texorpdfstring{\hyperref[#1]{Appendix
      \ref*{#1}}}{Appendix \ref*{#1}}}

\usepackage{crossreftools}
\newcommand{\creftitle}[1]{\crtcref{#1}}

\usepackage{xparse}

\ExplSyntaxOn
\DeclareDocumentCommand{\XDeclarePairedDelimiter}{mm}
 {
  \__egreg_delimiter_clear_keys: %
  \keys_set:nn { egreg/delimiters } { #2 }
  \use:x %
   {
    \exp_not:n {\NewDocumentCommand{#1}{sO{}m} }
     {
      \exp_not:n { \IfBooleanTF{##1} }
       {
        \exp_not:N \egreg_paired_delimiter_expand:nnnn
         { \exp_not:V \l_egreg_delimiter_left_tl }
         { \exp_not:V \l_egreg_delimiter_right_tl }
         { \exp_not:n { ##3 } }
         { \exp_not:V \l_egreg_delimiter_subscript_tl }
       }
       {
        \exp_not:N \egreg_paired_delimiter_fixed:nnnnn 
         { \exp_not:n { ##2 } }
         { \exp_not:V \l_egreg_delimiter_left_tl }
         { \exp_not:V \l_egreg_delimiter_right_tl }
         { \exp_not:n { ##3 } }
         { \exp_not:V \l_egreg_delimiter_subscript_tl }
       }
     }
   }
 }

\keys_define:nn { egreg/delimiters }
 {
  left      .tl_set:N = \l_egreg_delimiter_left_tl,
  right     .tl_set:N = \l_egreg_delimiter_right_tl,
  subscript .tl_set:N = \l_egreg_delimiter_subscript_tl,
 }

\cs_new_protected:Npn \__egreg_delimiter_clear_keys:
 {
  \keys_set:nn { egreg/delimiters } { left=.,right=.,subscript={} }
 }

\cs_new_protected:Npn \egreg_paired_delimiter_expand:nnnn #1 #2 #3 #4
 {%
  \mathopen{}
  \mathclose\c_group_begin_token
   \left#1
   #3
   \group_insert_after:N \c_group_end_token
   \right#2
   \tl_if_empty:nF {#4} { \c_math_subscript_token {#4} }
 }
\cs_new_protected:Npn \egreg_paired_delimiter_fixed:nnnnn #1 #2 #3 #4 #5
 {
  \mathopen{#1#2}#4\mathclose{#1#3}
  \tl_if_empty:nF {#5} { \c_math_subscript_token {#5} }
 }
\ExplSyntaxOff

\XDeclarePairedDelimiter{\supnorm}{
  left=\lVert,
  right=\rVert,
  subscript=\infty
  }

\DeclarePairedDelimiter{\abs}{\lvert}{\rvert} %
\DeclarePairedDelimiter{\brk}{[}{]}
\DeclarePairedDelimiter{\crl}{\{}{\}}
\DeclarePairedDelimiter{\prn}{(}{)}

\DeclareMathOperator{\En}{\mathbb{E}}

\DeclareMathOperator*{\argmin}{arg\,min} %
\DeclareMathOperator*{\argmax}{arg\,max}

\newcommand{\wt}[1]{\widetilde{#1}}
\newcommand{\wh}[1]{\widehat{#1}}
\newcommand{\wb}[1]{\widebar{#1}}

\def\ddefloop#1{\ifx\ddefloop#1\else\ddef{#1}\expandafter\ddefloop\fi}
\def\ddef#1{\expandafter\def\csname bb#1\endcsname{\ensuremath{\mathbb{#1}}}}
\ddefloop ABCDEFGHIJKLMNOPQRSTUVWXYZ\ddefloop
\def\ddefloop#1{\ifx\ddefloop#1\else\ddef{#1}\expandafter\ddefloop\fi}
\def\ddef#1{\expandafter\def\csname b#1\endcsname{\ensuremath{\mathbf{#1}}}}
\ddefloop ABCDEFGHIJKLMNOPQRSTUVWXYZ\ddefloop
\def\ddef#1{\expandafter\def\csname sf#1\endcsname{\ensuremath{\mathsf{#1}}}}
\ddefloop ABCDEFGHIJKLMNOPQRSTUVWXYZ\ddefloop
\def\ddef#1{\expandafter\def\csname c#1\endcsname{\ensuremath{\mathcal{#1}}}}
\ddefloop ABCDEFGHIJKLMNOPQRSTUVWXYZ\ddefloop
\def\ddef#1{\expandafter\def\csname h#1\endcsname{\ensuremath{\widehat{#1}}}}
\ddefloop ABCDEFGHIJKLMNOPQRSTUVWXYZ\ddefloop
\def\ddef#1{\expandafter\def\csname hc#1\endcsname{\ensuremath{\widehat{\mathcal{#1}}}}}
\ddefloop ABCDEFGHIJKLMNOPQRSTUVWXYZ\ddefloop
\def\ddef#1{\expandafter\def\csname t#1\endcsname{\ensuremath{\widetilde{#1}}}}
\ddefloop ABCDEFGHIJKLMNOPQRSTUVWXYZ\ddefloop
\def\ddef#1{\expandafter\def\csname tc#1\endcsname{\ensuremath{\widetilde{\mathcal{#1}}}}}
\ddefloop ABCDEFGHIJKLMNOPQRSTUVWXYZ\ddefloop
\def\ddefloop#1{\ifx\ddefloop#1\else\ddef{#1}\expandafter\ddefloop\fi}
\def\ddef#1{\expandafter\def\csname scr#1\endcsname{\ensuremath{\mathscr{#1}}}}
\ddefloop ABCDEFGHIJKLMNOPQRSTUVWXYZ\ddefloop

\newcommand{\ind}{\mathbbm{1}}    %

\newcommand{\eps}{\epsilon}
\newcommand{\veps}{\varepsilon}

\newcommand{\ldef}{\vcentcolon=}

\newcommand{\pibar}{\wb{\pi}}

\newcommand{\ma}{MA-DMSO\xspace}
\newcommand{\hr}{HR-DMSO\xspace}
\newcommand{\maf}{\MAFrameworkShort framework\xspace}
\newcommand{\hrf}{\FrameworkShort framework\xspace}

\newcommand{\MAI}{\scrM}

\newcommand{\HRI}{\scrH}

\newcommand{\malong}{Multi-Agent \CompText}

\newcommand{\mashort}{MA-\CompShort}

\newcommand{\creg}{c_{\reg}}
\newcommand{\Creg}{C_{\reg}}

\newcommand{\reg}{\mathrm{reg}}

\newcommand{\vepsl}{\underline{\veps}(T)}
\newcommand{\vepsu}{\wb{\veps}(T)}

\newcommand{\vepslowerT}{\vepsl}
\newcommand{\vepsupperT}{\vepsu}

\newcommand{\Ceff}{C_{\mathrm{prob}}}
\newcommand{\Cprob}{\Ceff}
\newcommand{\Cscale}{C_{\mathrm{prob}}}

\newcommand{\Mtil}{\wt{M}}

\newcommand{\gendec}{\normalfont{\textsf{dec}}}
\newcommand{\rdec}{\normalfont{\textsf{r-dec}}}
\newcommand{\infratio}{\normalfont{\textsf{infr}}}
\newcommand{\explopt}{\normalfont{\textsf{exo}}}

\newcommand{\decopac}[1][\gamma]{\gendec^{\mathrm{o}}_{#1}}
\newcommand{\decoreg}[1][\gamma]{\rdec^{\mathrm{o}}_{#1}}
\newcommand{\deccpac}[1][\veps]{\gendec_{#1}}

\newcommand{\decopacr}[1][\gamma]{\gendec^{\mathrm{o,rnd}}_{#1}}

\newcommand{\infr}[1][\gamma]{\infratio_{#1}}
\newcommand{\exo}[1][\eta]{\explopt_{#1}}

\newcommand{\maexo}{\texttt{MAExO}\xspace}
\newcommand{\exoalg}{\texttt{ExO$^+$}\xspace}

\newcommand{\cMall}{\cM^{+}}
\newcommand{\cMexpl}[1]{\cM^{+}_{#1}}

\newcommand{\RiskDM}{\mathrm{\mathbf{Risk}}(T)}

\renewcommand{\pm}[1][M]{p_{\sss{#1}}}

\renewcommand{\c}{\mathrm{c}}

\newcommand{\sP}{\mathscr{P}}
\newcommand{\sR}{\mathscr{R}}
\newcommand{\sO}{\mathscr{O}}
\newcommand{\sF}{\mathscr{F}}
\newcommand{\sX}{\mathscr{X}}
\newcommand{\sY}{\mathscr{Y}}
\newcommand{\sE}{\mathscr{E}}
\newcommand{\sT}{\mathscr{T}}

\newcommand{\ocirc}{o_{\circ}}
\newcommand{\Ocirc}{\MO_{\circ}}

\newcommand{\qmbar}{
  q_{\sMbar}}

\newcommand{\pmbar}{p_{\sMbar}}

\newcommand{\instpm}{(\MM, \Act, \MO, \{ \fm(\cdot) \}\subs{M})}
\newcommand{\instma}{(\MM, \Act, \MO, \{ \Dev \}_\ag, \{ \Sw \}_\ag)}

\renewcommand{\emptyset}{\varnothing}

\newcommand{\filt}{\mathscr{F}}

\newcommand{\sI}{\scrH}         %
\newcommand{\sJ}{\scrM}         %

\newcommand{\hist}{\mathfrak{H}}

\newcommand{\abscont}{V(\cM)}

\newcommand{\Framework}{Hidden-Reward Decision Making with Structured Observations\xspace}
\newcommand{\FrameworkShort}{HR-DMSO\xspace}
\newcommand{\MAFramework}{Multi-Agent Decision Making with Structured Observations\xspace}
\newcommand{\MAFrameworkShort}{MA-DMSO\xspace}
\newcommand{\dmso}{DMSO\xspace}
\newcommand{\DMSO}{Decision Making with Structured Observations\xspace}

\newcommand{\Ag}{K}
\newcommand{\ag}{k}

\newcommand{\Dev}[1][\ag]{\Pi'_{#1}}
\newcommand{\Devall}{\Pi'}
\newcommand{\dev}[1][\ag]{\pi'_{#1}}
\newcommand{\Sw}[1][\ag]{U_{#1}}

\newcommand{\act}{\pi}
\newcommand{\Act}{\Pi}

\newcommand{\CompText}{Decision-Estimation Coefficient\xspace}
\newcommand{\CompAbbrev}{DEC\xspace}
\newcommand{\CompShort}{\CompAbbrev}

\newcommand{\est}{\mathsf{est}}

\newcommand{\etdppac}{\textsf{E\protect\scalebox{1.04}{2}D}$^{+}$\textsf{ for PAC}\xspace}

\newcommand{\M}[1]{^{{\scriptscriptstyle M}}}  %

\newcommand{\sMbar}{\sss{\Mbar}}
 
\newcommand{\sups}[1]{^{{\scriptscriptstyle#1}}}
\newcommand{\subs}[1]{_{{\scriptscriptstyle#1}}}

\newcommand{\sss}[1]{{\scriptscriptstyle#1}}

\newcommand{\fm}[1][M]{f\sups{#1}}
\newcommand{\fmprime}{\fm[M']}
\newcommand{\pim}[1][M]{\pi_{\sss{#1}}}

\newcommand{\gm}{g\sups{M}}

\newcommand{\gmstar}{g\sups{\Mstar}}

\newcommand{\hm}{h\sups{M}}

\newcommand{\hmstar}{h\sups{\Mstar}}
\newcommand{\hmprime}{h\sups{M'}}

\newcommand{\fmbar}{f\sups{\Mbar}}
\newcommand{\pimbar}{\pi\subs{\Mbar}}

\newcommand{\fmstar}{f\sups{\Mstar}}
\newcommand{\pimstar}{\pi\subs{\Mstar}}

\newcommand{\pistar}{\pi^{\star}}

\newcommand{\pihat}{\wh{\pi}}

\newcommand{\Mbar}{\wb{M}}

\newcommand{\PiRNS}{\Pi_{\mathrm{RNS}}} %

\newcommand{\Est}{\mathrm{\mathbf{Est}}_{\mathsf{H}}}
\newcommand{\EstBar}{\widebar{\mathrm{\mathbf{Est}}}_{\mathsf{H}}}

\newcommand{\EstHel}{\mathrm{\mathbf{Est}}_{\mathsf{H}}(T)}

  \newcommand{\AlgEst}{\mathrm{\mathbf{Alg}}_{\mathsf{Est}}}

\newcommand{\Mhat}{\wh{M}}
\newcommand{\Mstar}{M^{\star}}

\newcommand{\algcommentlight}[1]{\textcolor{blue!70!black}{\transparent{0.5}\footnotesize{\texttt{\textbf{//\hspace{2pt}#1}}}}}

\newcommand{\approxleq}{\lesssim}
\newcommand{\approxgeq}{\gtrsim}

\renewcommand{\ind}[1]{^{{\scriptscriptstyle#1}}}

\newcommand{\bigoh}{O}
\newcommand{\bigoht}{\wt{O}}
\newcommand{\bigom}{\Omega}
\newcommand{\bigomt}{\wt{\Omega}}

\newcommand{\indic}{\mathbb{I}}
\newcommand{\Indic}[1]{\indic_{#1}}

\newcommand{\poly}{\mathrm{poly}}
\newcommand{\polylog}{\mathrm{polylog}}

\newcommand{\kl}[2]{D_{\mathsf{KL}}\prn*{#1\,\|\,#2}}
\newcommand{\Dkl}[2]{D_{\mathsf{KL}}\prn*{#1\,\|\,#2}}

\newcommand{\Dhel}[2]{D_{\mathsf{H}}\prn*{#1,#2}}
\newcommand{\Dgen}[2]{D\prn*{#1\dmid{}#2}}

\newcommand{\Dhels}[2]{D^{2}_{\mathsf{H}}\prn*{#1,#2}}
\newcommand{\hell}[2]{D^{2}_{\mathsf{H}}\prn*{#1,#2}}
\newcommand{\tvd}[2]{D_{\mathsf{TV}}\prn*{#1,#2}}

\newcommand{\Dchis}[2]{D_{\chi^2}\prn*{#1,#2}}

\newcommand{\Dtv}[2]{D_{\mathsf{TV}}\prn*{#1,#2}}

\newcommand{\Dphi}[2]{D_{\phi}(#1\dmid{}#2)}
\newcommand{\Dphishort}{D_{\phi}}

\newcommand{\DhelsX}[3]{D^{2}_{\mathsf{H}}\prn[#1]{#2,#3}}

\newcommand{\Ber}{\mathrm{Ber}}

\newcommand{\dmid}{\;\|\;}

\newcommand{\conv}{\mathrm{co}}
\newcommand{\diam}{\mathrm{diam}}

\newcommand{\mathand}{\quad\text{and}\quad}

\def\multiset#1#2{\ensuremath{\left(\kern-.3em\left(\genfrac{}{}{0pt}{}{#1}{#2}\right)\kern-.3em\right)}}

\newcommand{\grad}{\nabla}

\renewcommand{\emptyset}{\varnothing}

\newcommand{\BR}{\bbR}
\newcommand{\BN}{\bbN}

\newcommand{\single}[2][\ag]{{#2}|_{{#1}}}
\newcommand{\powerset}[1]{\mathcal{P}({#1})}

\newcommand{\NE}{{\scriptscriptstyle{\mathrm{NE}}}}
\newcommand{\CE}{{\scriptscriptstyle{\mathrm{CE}}}}
\newcommand{\CCE}{{\scriptscriptstyle{\mathrm{CCE}}}}

\newcommand{\nc}{\newcommand}
\nc{\DMO}{\DeclareMathOperator}
\newcount\Comments  %
\DMO{\prox}{prox}
\DMO{\Span}{span}
\DMO{\UCB}{UCB}
\DMO{\LCB}{LCB}
\nc{\br}[2]{{\rm br}^{#1}({#2})}
\nc{\depth}[1]{{\rm d}({#1})}
\nc{\child}[2]{{\rm ch}_{#1}({#2})}
\nc{\parent}[1]{{\rm pa}({#1})}
\nc{\dg}{\dagger}
\nc{\indsig}[2]{\mathcal{I}_{#1}({#2})}
\nc{\total}{{\rm fin}}
\nc{\early}{{\rm pre}}
\nc{\zsink}{z_{\rm sink}}
\nc{\lowv}{{\rm low}}
\nc{\ol}{\wb}
\nc{\madec}[3]{\texttt{ma-dec}_{#1}({#2}, {#3})}
\nc{\madeco}[1]{\texttt{ma-dec}_{#1}}
\nc{\madecd}[3]{\texttt{ma-dec}^{\texttt{d}}_{#1}({#2}, {#3})}
\nc{\mainf}{\texttt{ma-inf}}
\nc{\dec}{\texttt{dec}}
\nc{\decc}{\texttt{dec}^{\texttt{c}}}
\nc{\deccp}{\texttt{dec}^{\texttt{c-pac}}}
\nc{\deccr}{\texttt{dec}^{\texttt{c-reg}}}
\nc{\Alg}{{\rm\bf Alg}}
\nc{\co}{{\rm co}}
\nc{\BV}{\mathbb{V}}
\nc{\ham}[2]{d_{\rm Ham}({#1}, {#2})}

\nc{\gamvec}{\gamma}
\nc{\til}{\widetilde}
\nc{\td}{\tilde}
\nc{\todo}[1]{\ifnum\Comments=1 {\color{red}  [TODO: #1]}\fi}
\nc{\old}[1]{\ifnum\Comments=1 {\color{brown}  [OLD: #1]}\fi}
\newcommand{\noah}[1]{\ngcomment{#1}}
\nc{\BP}{\mathbb{P}}
\nc{\BI}{\mathbb{I}}

\nc{\fools}[3]{\MF_{#3}({#1}, {#2})}
\nc{\fool}[2]{\MF({#1},{#2})}
\nc{\clip}[2]{{\rm clip}\left[ \left. {#1} \right| {#2} \right]}
\nc{\imax}{\omega}
\nc{\CF}{\mathscr{F}}
\nc{\CG}{\mathscr{G}}
\nc{\CA}{\mathscr{A}}
\nc{\MH}{\mathcal{H}}
\nc{\MV}{\mathcal{V}}
\nc{\MC}{\mathcal{C}}
\nc{\MI}{\mathcal{I}}
\nc{\MQ}{\mathcal{Q}}
\nc{\st}{\star}
\nc{\lng}{\langle}
\nc{\rng}{\rangle}
\DMO{\OOPT}{opt}
\nc{\dopt}[2]{\ell_{\OOPT}({#1},{#2})}
\nc{\MG}{\mathcal{G}}
\nc{\MP}{\mathcal{P}}
\nc{\PP}{\mathbb{P}}
\nc{\TT}{\mathbb{T}}
\nc{\TTmax}{\TT_{\max}}
\DMO{\REG}{Reg}
\DMO{\WREG}{wReg}
\nc{\wreg}[2]{{\Delta}^{\rm w}_{{#1}}({#2})}
\nc{\wReg}[2]{{\WREG}_{{#1}}({#2})}
\DMO{\Ham}{Ham}
\DMO{\Gap}{Gap}
\DMO{\GD}{GD}
\DMO{\GDA}{GDA}
\DMO{\EG}{EG}
\DMO{\OGDA}{OGDA}
\DMO{\Unif}{Unif}
\DMO{\Tr}{Tr}
\nc{\Qu}{\ul{Q}}
\nc{\Qo}{\ol{Q}}
\nc{\Ro}{\ol{R}}
\nc{\Vu}{\ul{V}}
\nc{\Vo}{\ol{V}}
\nc{\RanQ}{\Delta Q}
\nc{\RanV}{\Delta V}
\nc{\clipQ}{\Delta \breve{Q}}
\nc{\frzQ}{\Delta \mathring{Q}}
\nc{\clipV}{\Delta \breve{V}}
\nc{\clipdelta}{\breve{\delta}}
\nc{\cliptheta}{\breve{\theta}}
\nc{\delmin}{\Delta_{{\rm min}}}
\nc{\delmins}[1]{\Delta_{{\rm min},{#1}}}
\nc{\gapfinal}[1]{\max \left\{ \frac{\frzQ_{{#1}}^{k^\st}(x,a)}{2H}, \frac{\delmin}{4H} \right\}}
\nc{\post}[2]{R({#1}; {#2})}
\nc{\posts}[3]{R_{#3}({#1}; {#2})}
\nc{\pstr}{{\rm po}}
\nc{\prior}{{\rm pr}}

\nc{\algnst}[1]{\begin{align*}#1\end{align*}}
\nc{\algn}[1]{\begin{align}#1\end{align}}
\nc{\matx}[1]{\left(\begin{matrix}#1\end{matrix}\right)}
\renewcommand{\^}[1]{^{\scriptscriptstyle#1}}

\nc{\nuu}{\nu}

\nc{\bel}[1]{\mathbf{b}({#1})}
\nc{\nbel}[1]{\bar{\mathbf{b}}({#1})}
\nc{\sbel}[2]{\mathbf{b}'_{#1}({#2})}
\nc{\nsbel}[2]{\bar{\mathbf{b}}'_{#1}({#2})}

\nc{\bone}{\mathbf{1}}

\nc{\MO}{\mathcal O}
\nc{\MU}{\mathcal{U}}
\nc{\ME}{\mathcal{E}}
\nc{\MN}{\mathcal{N}}
\nc{\MK}{\mathcal{K}}
\nc{\MM}{\mathcal{M}}
\nc{\ML}{\mathcal{L}}
\nc{\MS}{\mathcal{S}}
\nc{\MT}{\mathcal{T}}
\nc{\BF}{\mathbb F}
\nc{\BQ}{\mathbb Q}
\nc{\MX}{\mathcal{X}}
\nc{\MA}{\mathcal{A}}
\nc{\MD}{\mathcal{D}}
\nc{\MB}{\mathcal{B}}
\nc{\MZ}{\mathcal{Z}}
\nc{\MJ}{\mathcal{J}}
\nc{\MW}{\mathcal{W}}
\nc{\MR}{\mathcal{R}}
\nc{\MY}{\mathcal{Y}}
\nc{\BZ}{\mathbb Z}
\nc{\ep}{\epsilon}
\nc{\vep}{\varepsilon}
\nc{\gapfn}[1]{\varepsilon_{#1}}
\nc{\ggapfn}[2]{\varphi_{#1}({#2})}
\nc{\epsahk}{\gapfn{0}}
\nc{\BH}{\mathbb H}
\nc{\BG}{\mathbb{G}}
\nc{\D}{\Delta}
\nc{\MF}{\mathcal{F}}
\nc{\One}[1]{\indic\left\{{#1}\right\}}
\nc{\bOne}{\mathbf{1}}
\nc{\Aopt}{\mathcal{A}^{\rm opt}}
\nc{\Amul}{\mathcal{A}^{\rm mul}}
\nc{\CM}{\mathscr{M}}
\nc{\CO}{\mathscr{O}}
\nc{\CR}{\mathsscr{R}}

\nc{\SP}{\mathsf P}
\nc{\SQ}{\mathsf Q}
\nc{\SC}{\mathscr{C}}
\nc{\SD}{\mathscr{D}}
\nc{\SE}{\mathscr{E}}
\nc{\SG}{\mathscr{G}}

\nc{\DO}{\accentset{\circ}{\D}}
\nc{\mf}{\mathfrak}
\nc{\mfp}{\mathfrak{p}}
\nc{\mfq}{\mf{q}}
\nc{\Sp}{\mbox{Spec}}
\nc{\Spm}{\mbox{Specm}}
\nc{\hookuparrow}{\mathrel{\rotatebox[origin=c]{90}{$\hookrightarrow$}}}
\nc{\hookdownarrow}{\mathrel{\rotatebox[origin=c]{-90}{$\hookrightarrow$}}}
\nc{\hra}{\hookrightarrow}
\nc{\tra}{\twoheadrightarrow}
\nc{\sgn}{{\rm sgn}}
\nc{\aut}{{\rm Aut}}
\nc{\Hom}{{\rm Hom}}
\nc{\img}{{\rm Im}}
\DMO{\id}{Id}
\DMO{\KL}{KL}
\nc{\kld}{\kl}
\nc{\ren}[3]{D_{#3}({#1}||{#2})}
\nc{\chisq}[2]{\chi^2({#1},{#2})}
\nc{\dvg}[2]{D({#1} \| {#2})}
\DMO{\BSS}{BSS}
\DMO{\BES}{BES}
\DMO{\BGS}{BGS}
\nc{\indep}{\perp}
\DMO{\sink}{sink}
\nc{\fp}[1]{\MP_1({#1})}
\nc{\BO}{\mathbb{O}}
\nc{\BT}{\mathbb{T}}

\nc{\RR}{\mathbb{R}}
\nc{\Gradient}{\nabla}
\nc{\norm}[1]{\left \lVert #1 \right \rVert}
\nc{\EE}{\mathbb{E}}

\DMO{\PR}{Pr}

\nc{\E}{\mathbb{E}}
\nc{\ra}{\rightarrow}

\nc{\opo}{\texttt{opo}}

\input{widebar}

\usepackage[suppress]{color-edits}
 \addauthor{df}{ForestGreen}
 \addauthor{dk}{magenta}
 \addauthor{sr}{yellow!60!black}
 \addauthor{ng}{purple}
 \draft{

 \usepackage[textsize=tiny]{todonotes}
 
 }

\makeatletter
\let\OldStatex\Statex
\renewcommand{\Statex}[1][3]{%
  \setlength\@tempdima{\algorithmicindent}%
  \OldStatex\hskip\dimexpr#1\@tempdima\relax}
\makeatother

 \usepackage{accents}

 \addtocontents{toc}{\protect\setcounter{tocdepth}{0}}

\let\oldparagraph\paragraph
\renewcommand{\paragraph}[1]{\oldparagraph{#1.}}

\arxiv{
\title{On the Complexity of Multi-Agent Decision Making:\\ From Learning in Games to Partial Monitoring}
\author{%
    Dylan J. Foster\\{\small \texttt{dylanfoster@microsoft.com}} \and Dean P. Foster\\{\small \texttt{dean@foster.net}} \and Noah Golowich\thanks{Work done in part while interning at Microsoft Research.}\\{\small \texttt{nzg@mit.edu}}\\ \and Alexander Rakhlin \\{\small \texttt{rakhlin@mit.edu}}
  }
  \date{May 1, 2023}
}

\colt{
\title[Complexity of Multi-Agent Decision Making]{On the Complexity of Multi-Agent Decision Making:\\ From Learning in Games to Partial Monitoring}

\coltauthor{%
 \Name{Author Name1} \Email{abc@sample.com}\\
 \addr Address 1
 \AND
 \Name{Author Name2} \Email{xyz@sample.com}\\
 \addr Address 2%
}
}

\colt{\date{}}

\begin{document}
\maketitle
\begin{abstract}

A central problem in the theory of multi-agent reinforcement learning (MARL) is to understand what structural conditions and algorithmic principles lead to sample-efficient learning guarantees, and how these considerations change as we move from few to many agents. We study this question in a general framework for interactive decision making with multiple agents, encompassing Markov games with function approximation and normal-form games with bandit feedback. We focus on equilibrium computation, in which a centralized learning algorithm aims to compute an equilibrium by controlling multiple agents that interact with an (unknown) environment. Our main contributions are:\loose
\begin{itemize}[leftmargin=3em,rightmargin=3em,topsep=2pt,itemsep=2pt]
\item %
  We provide upper and lower bounds on the optimal sample complexity for multi-agent decision making based on a multi-agent generalization of the  \emph{\CompText}, a complexity measure introduced by \citet{foster2021statistical} in the single-agent counterpart to our setting. Compared to the best results for the single-agent setting, our upper and lower bounds have additional gaps. We show that no ``reasonable'' complexity measure can close these gaps, highlighting a striking separation between single and multiple agents.
\item %
  We show that characterizing the statistical complexity for multi-agent decision making is equivalent to characterizing the statistical complexity of \emph{single-agent} decision making, but with \emph{hidden (unobserved) rewards}, a framework that subsumes variants of the partial monitoring problem. As a consequence of this connection, we characterize the statistical complexity for hidden-reward interactive decision making to the best extent possible.%
\end{itemize}
 Building on this development,
 we provide several new structural results, including 1) conditions under which the statistical complexity of multi-agent decision making can be reduced to that of single-agent, and 2) conditions under which the so-called \emph{curse of multiple agents} can be avoided.

\end{abstract}

\arxiv{
  \addtocontents{toc}{\protect\setcounter{tocdepth}{2}}
  {\hypersetup{hidelinks}
    \tableofcontents
  }
  \newpage
}

  \section{Introduction}
\label{sec:intro}
\noah{TODO check there's no overlap in text after theorems in sec 1 and in later sections where theorems stated formally}

Many of the most exciting frontiers for artificial intelligence are game-theoretic in nature, and involve multiple agents with differing incentives interacting and making decisions in dynamic environments, either in cooperation or in competition. Numerous recent approaches, adopting the framework of \emph{multi-agent reinforcement learning} (MARL), have achieved human-level performance in multi-agent game-playing domains \citep{silver2016mastering,brown2018superhuman,perolat2022mastering,kramar2022negotiation,bakhtin2022human}, and while there is great potential to apply MARL further in domains such as cybersecurity
\citep{malialis2015distributed}, autonomous driving
\citep{shalevshwartz2016safe}, and economic policy
\citep{zheng2022ai}, sample-efficiency and reliability are obstacles for real-world deployment. Consequently, a central question is to understand what modeling assumptions and algorithm design principles lead to robust, sample-efficient learning
guarantees. This issue is particularly salient in domains with high-dimensional feedback and decision spaces, where the use of flexible models such as neural networks is critical.

For reinforcement learning in single-agent settings, an extensive line of research identifies modeling assumptions (or, structural conditions) under which sample-efficient learning is possible  \citep{russo2013eluder,jiang2017contextual,sun2019model,wang2020provably,du2021bilinear,jin2021bellman,foster2021statistical}.
Notably, \citet{foster2021statistical,foster2022complexity,foster2023tight} provide a notion of statistical complexity, the \emph{\CompText} (\CompShort), which is both necessary and sufficient for low sample complexity, and leads to unified principles for algorithm design.
For multi-agent reinforcement learning, structural conditions for sample-efficient learning have also received active investigation \citep{chen2022almost,li2022learning,xie2020learning,jin2022power,huang2021towards,zhan2022decentralized,liu2022sample}, drawing inspiration from the single agent setting. However, insights from single agents do not always transfer to multiple agents in intuitive ways \citep{daskalakis2022complexity}, and development has largely proceeded on a case-by-case basis. As such, the problem of developing a unified understanding or \emph{necessary} conditions for sample-efficient multi-agent reinforcement learning remained open.

\paragraph{Contributions}
We consider a general framework, \emph{Multi-Agent Decision Making with Structured Observations} (\MAFrameworkShort), which generalizes the single-agent \dmso framework of \citet{foster2021statistical} and subsumes multi-agent reinforcement learning with general function approximation, as well as normal-form games with bandit feedback and structured action spaces. We focus on \emph{centralized} equilibrium computation, where a centralized learning algorithm with control of all agents aims to compute an equilibrium by interacting with the (unknown) environment.
Our main results are:%
\begin{itemize}
\item \textbf{Complexity of multi-agent decision making.} We introduce a new complexity measure, the \emph{\malong}, generalizing the \CompText of \citet{foster2021statistical,foster2023tight}, and show that it leads to upper and lower bounds on the optimal sample complexity for multi-agent decision making. Compared to the best results for the single-agent setting \citep{foster2023tight}, our upper and lower bounds have additional gaps, which we show that \emph{no (reasonable) complexity measure can close}. 
\item \textbf{Complexity of hidden-reward decision making.}
  We show that characterizing the statistical complexity for multi-agent decision making is equivalent to characterizing the statistical complexity of \emph{single-agent decision making}, but with \emph{hidden (unobserved) rewards}, a framework that we refer to as \emph{\Framework} (\FrameworkShort). Leveraging this connection, we characterize the statistical complexity of the \FrameworkShort framework, which encompasses PAC variants of the stochastic partial monitoring problem \citep{bartok2014partial}, to the best extent possible (for any reasonable complexity measure).\loose
\item \textbf{Additional insights for multiple agents.} Building on the results above, we provide a number of new structural results and algorithmic insights for multi-agent decision making and RL, including 1) general conditions under which the complexity of multi-agent decision making can be reduced to that of single agent decision making, and 2) general conditions under which the so-called \emph{curse of multiple agents} \citep{jin2021v} can be removed.%
\end{itemize}
\arxiv{
  Our results provide a foundation on which to develop a unified understanding of multi-agent reinforcement learning and decision making, and highlight a number of exciting open problems.
}
\colt{
  \paragraph{Organization}
  \emph{Due to space constraints, the main body presents an informal overview of our results, with detailed statements deferred to \pref{part:main} of
    the appendix. Examples are given in \pref{part:examples}. See \pref{sec:organization} for an overview of appendix organization.
  }
  }

\subsection{Multi-agent interactive decision making (\MAFrameworkShort)}
We introduce a multi-agent generalization of the \emph{\DMSO} framework of \cite{foster2021statistical}, which we refer to as \emph{\MAFramework} (\MAFrameworkShort).  The framework consists of $T$ rounds of interaction between $\Ag$ agents and the environment. For each round $t = 1, 2, \ldots, T$:
\begin{enumerate}
\item The agents collectively select a \emph{joint decision} $\act\^t \in \Act$, where $\Act$ is the \emph{joint decision space}.
\item Each agent $\ag\in\brk{\Ag}$ receives a \emph{reward} $r_\ag\^t \in \MR \subseteq \bbR$ and a \emph{pure observation} $\ocirc\^t \in \Ocirc$ sampled via $(r_1\^t, \ldots, r_\Ag\^t, \ocirc\^t) \sim \Mstar(\pi\^t)$, where $\Mstar : \Pi \ra \Delta(\MR^\Ag \times \Ocirc)$ is the underlying \emph{model}. We refer to $\MR$ as the \emph{reward space} and to $\Ocirc$ as the \emph{pure observation space}. We call the tuple $(r_1\^t, \ldots, r_K\^t, \ocirc\^t)$ consisting of all information revealed to agents on round $t$ the \emph{full observation}. 
\end{enumerate}
After the $T$ rounds of interaction, the agents collectively output a joint decision $\wh\act \in \Act$, which may be chosen in an arbitrary fashion based on the data observed over the  $T$ rounds, and may be randomized according to a distribution $p \in \Delta(\Act)$. Their goal, which we formalize in the sequel, is to choose $\pihat$ to be an equilibrium (e.g., Nash or CCE) for the average reward function induced by $\Mstar$. The model $M^\st$, which is formalized as a probability kernel from decisions to full observations (\cref{sec:add-prelim}), is unknown to the agents, and is to be interpreted as the underlying environment. 

The DMSO framework captures most online decision making problems in which a \emph{single agent} interacts with an unknown environment, and the \MAFrameworkShort framework further generalizes it to capture a wide variety of problems in \emph{multi-agent} reinforcement learning. Examples include learning in normal-form games with bandit feedback \citep{Rakhlin2013Optimization,foster2016learning,heliou2017learning,wei2018more,giannou2021rate}, where $\Mstar$ represents the distribution over rewards for each entry in the game, and learning in Markov games with function approximation \citep{chen2022almost,li2022learning,xie2020learning,jin2022power,huang2021towards,zhan2022decentralized,liu2022sample}, where $\Mstar$ represents the underlying Markov game. Additional examples include normal-form games with structured (e.g., convex-concave) rewards and high-dimensional action spaces \citep{bravo2018bandit,maheshwari2022zeroth,lin2021optimal}.

\paragraph{Realizability}
While the model $\Mstar$ is unknown, we make a standard \emph{realizability} assumption.\loose
    \begin{assumption}[Realizability for \MAFrameworkShort]
  \label{ass:realizability-ma}
  The agents have access to a model class $\cM$ consisting of probability kernels $M : \Pi \ra \Delta(\MR^\Ag \times \Ocirc)$ that contains the true model $\Mstar$.
\end{assumption}
For normal-form games, the class $\cM$ encodes
structure in the rewards (e.g., linearity or convexity) or decision space, and for Markov games it encodes structure in  transition probabilities or value functions. See \cref{part:examples} of the appendix for examples, as well as \citet{foster2021statistical} for \arxiv{in the single-agent case where}\colt{with} $K=1$.
\loose

\subsubsection{Equilibria}
The goal of the agents in the \maf is to produce an equilibrium for the underlying game/model $\Mstar$. We formalize the notion of equilibrium in a general fashion which encompasses several standard game-theoretic equilibria. %
To keep notation compact, we define $\MO := \MR^\Ag \times \Ocirc$ to be the \emph{full  observation space}, and will write $o\^t := (r_1\^t, \ldots, r_\Ag\^t, \ocirc\^t)$ to denote the (full) observation.  For $M \in \MM$ and $\act \in \Act$, let $\E\sups{M, \pi}[\cdot]$ denote expectation under the process $(r_1, \ldots, r_\Ag, \ocirc) \sim M(\pi)$; in light of our notation $\MO = \MR^\Ag \times \Ocirc$, we will sometimes denote this process as $o \sim M(\pi)$. For each $\ag \in [\Ag]$ and $M \in \MM$, define the mapping $\fm_\ag : \Act \ra \bbR$ by $\fm_\ag(\act) = \E\sups{M, \pi}[r_\ag]$, which denotes agent $\ag$'s expected reward under $M$ when the joint decision $\act$ is played. \loose

For each agent $\ag$, we assume they are given a \emph{deviation space} $\Dev$, together with a \emph{switching function}, $\Sw : \Dev \times \Act \ra \Act$. Given a joint decision $\act \in \Act$, each agent $\ag$ can choose a deviation $\dev \in \Dev$, which will have the effect that the joint policy played by agents is $\Sw(\dev, \act)$ instead of $\act$.
We aim for the output policy $\wh \pi \sim p$ produced in the \MAFrameworkShort setup to have the property that no agent can significantly increase their value by deviating. We quantify this via
\begin{align}
\RiskDM := \E_{\wh \pi \sim p} \left[ \sum_{k=1}^K \sup_{\dev \in \Dev} \fmstar_\ag(\Sw(\dev, \act)) - \fmstar_\ag(\act) \right].\label{eq:risk}
\end{align}
For $M \in \MM$ and $\act \in \Act$, we abbreviate $\hm(\act) = \sum_{k=1}^K \sup_{\dev \in \Dev} \fm_\ag(\Sw(\dev, \act)) - \fm_\ag(\act)$, so that $\RiskDM := \E_{\wh \pi \sim p} [\hmstar(\wh \pi)]$. The quantity  $\hm(\pi)$ measures the sum of players' incentives to deviate from the joint decision $\pi$ under $M$; we say that $\act$ is an \emph{equilibrium} for $M$ if $\hm(\act) = 0$.\loose

The notion \cref{eq:risk} captures standard notions of equilibria, including Nash equilibria, correlated equilibria (CE), and coarse correlated equilibria (CCE). As we have strived to make the setup in this section as general as possible, we make two regularity assumptions to rule out other, potentially pathological notions of equilibria. \arxiv{The first posits that equilibria exist, and the second asserts that each agent can always choose a deviation that does not decrease their value.}
\begin{assumption}[Existence of equilibria]
  \label{ass:existence-eq}
For \arxiv{any model}\colt{all} $M \in \MM$, there exists $\act \in \Act$ with $\hm(\act) = 0$. 
\end{assumption}
\begin{assumption}[Monotonicity of the optimal deviation]
  \label{ass:nonneg-dev}
  For any model $M \in \MM$, agent $\ag \in [\Ag]$, and joint decision $\act \in \Act$, there is some deviation $\dev \in \Dev$ such that $\fm_\ag(\Sw(\dev, \act)) \geq \fm_\ag(\act)$. 
\end{assumption}
\cref{ass:nonneg-dev} implies that, up to a factor of $\Ag$, the notion of risk in \pref{eq:risk} is equivalent to the maximal gain any agent can achieve by deviating. Both assumptions are satisfied by Nash equilibria, CE, and CCE (see \cref{def:ne-instance,def:cce-instance,def:ce-instance}). %

Summarizing, the \MAFrameworkShort framework captures the problem of equilibrium computation: the agents aim to find an ($\veps$-approximate) equilibrium $\pihat$ so that $\RiskDM \leq \veps$, but the underlying game is unknown, so they must gather information by interacting with it and exploring. We refer to the tuple $\sJ = (\MM, \Act, \MO, \{ \Dev \}_\ag, \{ \Sw \}_\ag)$ as an \emph{instance} for the \MAFrameworkShort framework. The instance $\sJ$ specifies all information known a-priori to the agents before the learning process begins.

\begin{remark}
As described, the \maf allows \emph{centralized} learning protocols, in which a single learning algorithm may control all agents in a centralized fashion (equivalently, unlimited communication and coordination is permitted amongst agents throughout the learning process). Lower bounds against centralized learning algorithms certainly apply to decentralized algorithms, being a special case of the former.  However, in general there may be gaps between the minimax sample complexity for centralized and decentralized algorithms, and we leave a detailed investigation of decentralized multi-agent interactive decision-making  for future work.
\end{remark}

\begin{remark}
  Our presentation of the \MAFrameworkShort framework captures settings in which (multi-agent) learning algorithms are evaluated only on the proximity of output decision $\wh \act$ to equilibrium, as opposed to, say, the average proximity to equilibrium for the decisions played throughout the $T$ rounds of learning. In the single-agent setting, such guarantees are often referred as PAC (Probability Approximately Correct) guarantees, as opposed to regret guarantees \citep{foster2023tight}. It is fairly straightforward to extend many of our results to the regret setting. %
\end{remark}

\subsubsection{Examples of instances for \MAFrameworkShort}

We now highlight basic multi-agent bandit and MARL problems captured by the \MAFrameworkShort framework. We describe \arxiv{the structure of the decision space, deviation space, and switching functions that allow us to}\colt{how to} capture concrete notions of equilibria, then give examples of instances $\scrM$.\loose%

\paragraph{Examples of equilibria}
In \cref{def:ne-instance,def:cce-instance} below, we specify the decision spaces, deviation spaces, and switching functions that can be used to capture \emph{Nash equilibria} and \emph{coarse correlated equilibria (CCE)}; see \cref{app:ma_examples_equilibria} for further examples, including \emph{correlated equilibria (CE)} and variants of CCE and CCE which have been studied in the context of Markov games.

\begin{definition}[Nash equilibrium instance]
  \label{def:ne-instance}
  An \ma instance $\sJ = \instma$ is a \emph{Nash equilibrium (NE) instance} if the following holds:
  \begin{enumerate}
  \item For sets $\Pi_1, \ldots, \Pi_\Ag$, we have $\Pi = \Pi_1 \times \cdots \times \Pi_\Ag$. \colt{For each $\ag \in [\Ag]$, we have $\Pi_\ag' = \Pi_\ag$.}
  \arxiv{\item For each $\ag \in [\Ag]$, we have $\Pi_\ag' = \Pi_\ag$.}
  \item For each $\ag \in [\Ag]$, $\pi \in \Pi$, and $\dev \in \Dev$, it holds that $\Sw(\dev, \pi) = (\pi_k', \pi_{-k})$.\footnote{{We adopt the convention that $\pi_{-k}=(\pi_1,\ldots,\pi_{k-1},\pi_{k+1},\ldots)$ and $(\pi_k, \pi_{-k}) = (\pi_1, \ldots, \pi_k, \ldots, \pi_K)$.}}
  \end{enumerate}
  We say that the NE instance $\MAI$ is a \emph{two-player zero-sum NE instance} if $\Ag = 2$, and for all $M \in \MM, \act \in \Act$, it holds that $\fm_1(\act) + \fm_2(\act) = 0$. 
\end{definition}
The notion of Nash equilibrium in \cref{def:ne-instance} encompasses, but goes well beyond the standard notion of mixed Nash equilibria in normal-form games (e.g., \citep{Nisan2007Algorithmic}). In particular, \cref{def:ne-instance} does not assume that the decision spaces $\Pi_k$ are distributions over a pure action space of player $k$. Therefore, it captures refined solution concepts including \emph{pure} Nash equilibria in normal-form games \citep{daskalakis2006computing} and Markov Nash equilibria in Markov games (\cref{ex:mne}). As a result of this generality, an NE instance per \cref{def:ne-instance} is not  guaranteed to satisfy \cref{ass:existence-eq}, i.e., to have equilibria; nevertheless, we will ensure that all examples of NE instances we consider are constructed in such a way so that \cref{ass:existence-eq} is satisfied. 

\arxiv{\cref{def:cce-instance} gives an analogue of \cref{def:ne-instance} which can capture the notion of (normal-form) coarse correlated equilibria. }
\begin{definition}[Coarse correlated equilibrium instance]
  \label{def:cce-instance}
  An instance $\sJ = \instma$ for \MAFrameworkShort is a \emph{coarse correlated equilibrium (CCE) instance} if the following holds:
  \begin{enumerate}
  \item For some sets $\Sigma_1, \ldots, \Sigma_K$ (called \emph{pure decisions}), we have $\Pi = \Delta(\Sigma_1\times \cdots \times \Sigma_\Ag)$. We will write $\Sigma := \Sigma_1 \times \cdots \times \Sigma_K$.
    \colt{For each $\ag \in [\Ag]$, we have $\Pi_\ag' = \Sigma_\ag \cup \{ \perp \}$.}
  \item \label{it:mpisigma} For each $\pi \in \Pi$ and $M \in \MM$, it holds that $M(\pi) = \E_{\sigma \sim \pi}[M(\sigma)]$.      %
  Further, there is a measurable
    function $\varphi : \MO \ra \Sigma$ so that $\BP_{o \sim M(\sigma)}(\varphi(o) = \sigma) = 1$ for each $M \in \MM$ and $\sigma \in \Sigma$ (i.e., $M(\sigma)$ reveals $\sigma$).\colt{\footnote{When clear from context, we associate singleton distributions $\indic_{\sigma}$ with the element $\sigma$ the distribution places its mass on.}}
  \arxiv{\item For each $\ag \in [\Ag]$, we have $\Pi_\ag' = \Sigma_\ag \cup \{ \perp \}$.}
  \item For each $\ag \in [\Ag]$, $\pi \in \Pi$, and $\dev \in \Dev$, it holds that
\begin{align}
      \textstyle
      \Sw(\dev, \pi) =\begin{cases}
        \indic_{\dev} \times \pi_{-\ag} &: \dev \neq \perp \\
\pi &: \dev = \perp
      \end{cases}\nonumber,
\end{align}
    where $\indic_{\dev} \times \pi_{-\ag} \in \Pi$ denotes the product distribution whereby agent $\ag$ plays $\dev$ and the other agents play according to their joint marginal under $\pi \in \Pi$.
  \end{enumerate}
\end{definition}
In \cref{def:cce-instance}, the inclusion of $\perp \in \Dev$ corresponds to player $k$ choosing not to deviate. This is necessary to satisfy \cref{ass:nonneg-dev} since there can be distributions $\pi \in \Pi$ so that if player $k$ deviates to any fixed option in $\Sigma_k$, their value decreases.%
\footnote{In some contexts, \arxiv{coarse correlated equilibria}\colt{CCE} are defined without such an option $\perp \in \Dev$; in settings where the only goal is to establish \emph{upper bounds}, the addition of $\perp$ does not make a material difference (since its only effect is to guarantee that the suboptimality of a decision is non-negative), but since we aim to prove \emph{lower bounds} as well, it is crucial to have the option $\perp \in \Dev$. } We also remark that \cref{def:cce-instance} captures the notion of CCE in \emph{normal-form} games (with pure action sets $\Sigma_k$); in \cref{app:ma_examples_equilibria} we give an example of an instance capturing a slightly different notion of CCE in Markov games.

\arxiv{
\begin{remark}
We use the following convention throughout the paper, including in \cref{it:mpisigma} of the above definition: when convenient, we associate any singleton distribution with the element that the distribution places its mass on. For instance, for a pure decision $\sigma = (\sigma_1, \ldots, \sigma_\Ag) \in \Sigma_1 \times \cdots \times \Sigma_\Ag$ in the context  of \cref{def:cce-instance}, we will denote its corresponding singleton distribution $\indic_\sigma \in \Delta(\Sigma) = \Pi$ as just $\sigma \in \Pi$. In addition, when possible, we use the convention that $\Sigma$ denotes a pure decision set, whereas $\Pi$ denotes a decision set that may be pure or mixed (this will be clear from context).
\end{remark}
}

\paragraph{Examples of equilibria}
We now provide concrete examples for the NE and CCE instances in \cref{def:ne-instance,def:cce-instance}; see \cref{app:ma_examples} for additional examples (including CE) and discussion.

\begin{example}[Learning Nash, and CCE in normal-form games]
  \label{ex:cce}
 We begin by describing the problem of learning in normal-form games with bandit feedback. Suppose that each player $k \in [K]$ has a finite \emph{action set} $\MA_k$, with joint action set denoted by $\MA = \MA_1 \times \cdots \times \MA_K$. %
  Upon playing a joint action profile $a \in \MA$, the (unknown) ground truth model $M^\st$ samples $(r_1, \ldots, r_K) \sim M^\st(a)$, where $r_k$ denotes the reward received by player $k$. The goal is to compute a distribution over joint action profiles which is some type of \emph{equilibrium} of the game whose payoffs are given by expected rewards under $\Mstar$.\loose \arxiv{Below we formally describe the \MAFrameworkShort instances corresponding to the problems of computing Nash equilibria and coarse correlated equilibria:}
  \arxiv{\begin{itemize}}
    \colt{\begin{itemize}[leftmargin=*]}
  \item To express the problem of Nash equilibrium computation, set $\Pi_k := \Delta(\MA_k)$ for each $k$, let $\Pi = \Pi_1 \times \cdots \Pi_K$ be the space of product distributions on $\MA$, and define $\Dev, \Sw$ as in \cref{def:ne-instance}. Moreover, let $\MR = [0,1]$ and $\Ocirc = \MA$, $\MO = \MR^K \times \Ocirc$. Let $\MM$ be the class of models so that: (a) for all singleton distributions $\indic_a = \indic_{a_1} \times \cdots \times \indic_{a_K} \in \Pi$, $M(\indic_a) \in \Delta(\MR^K) \times \{\indic_a\}$, and (b) for all $\pi \in \Pi$, $M(\pi) = \E_{a \sim \pi}[M(\indic_a)]$. In words, $M(\pi)$ samples an action profile $a \sim \pi$ (in particular, $a_k \sim \pi_k$ for each $k$), reveals the action profile $a$ sampled,\footnote{We assume that the model reveals the action profile played for technical reasons (see \cref{ass:convexity_pols}); this is a very mild assumption, satisfied in essentially all (centralized) settings, since agents know which action they play.} as well as $K$ $[0,1]$-valued rewards drawn from an arbitrary distribution. Then the instance $\sJ = \instma$ is an NE instance per \cref{def:ne-instance}. For $\wh\pi \in \Pi$, $\hmstar(\wh\pi)$ measures the sum of the players' incentives to deviate from $\wh\pi$ under the true model $M^\st$; in particular, $\hmstar(\wh\pi) = 0$ if and only if $\wh \pi$ is a Nash equilibrium of the game whose payoff functions are given by $a \mapsto \fmstar_k(a) := \E\sups{\Mstar, a}[r_k]$. 
  \item To express the problem of CCE computation, set $\Pi = \Delta(\MA_1 \times \cdots \times \MA_K)$, and define $\Dev, \Sw$ as in \cref{def:cce-instance} with $\Sigma_k = \MA_k$ for each $k$. Moreover, let $\MR = [0,1]$, and $\Ocirc = \MA$, $\MO = \MR^K \times \Ocirc$. Let $\MM$ be the class of models so that: (a) for all singleton distributions $\indic_a \in \Pi$, $M(\indic_a) \in \Delta(\MR^K) \times \{\indic_a\}$, and (b), for $\pi \in \Pi$, $M(\pi) = \E_{a \sim \pi}[M(\indic_a)]$. Then the instance $\sJ := \instma$ is a CCE instance per \cref{def:cce-instance}. For $\wh\pi \in \Pi$, $\hmstar(\wh \pi)$ measures the sum of players' \arxiv{non-negative }incentives to deviate from $\wh \pi$  under the true model $M^\st$; in particular, $\hmstar(\wh\pi) = 0$ if and only if $\wh \pi$ is a CCE of the game whose payoff functions are given by $a \mapsto \fmstar_k(a) := \E\sups{\Mstar, a}[r_k]$.\loose
    \arxiv{  \item The \MAFrameworkShort framework can also express the problem of \emph{correlated equilibrium computation}. We use the same CCE instance as described in the previous point, but define $\Dev, \Sw$ slightly differently; see \cref{def:ce-instance} in \cref{app:ma_examples}.
      }
  \end{itemize}
  In the most basic (``finite-action'') version of the normal-form game setup, we allow $M^\st(a)$ to be arbitrary, subject to the constraint that  $r_k \in [0,1]$, but assume that $A_k\ldef{}\abs{\cA_k}<\infty$ for all $k$. Beyond finite-action normal-form games, the \MAFrameworkShort framework captures structured normal-form games with bandit feedback (equivalently, multi-agent variants of the structured bandit problem), in which the players' action spaces are large or infinite, but rewards have additional structure. Examples include linear, convex, or concave payoffs (generalizing bandit convex optimization) \citep{bravo2018bandit,maheshwari2022zeroth,lin2021optimal}, and many others \citep{cui2022learning}. 
\end{example}

\colt{
\begin{example}[Learning Nash equilibria in Markov games]
  \label{ex:mne}
  For reinforcement learning, each model $M \in \MM$ is a Markov game $M = (H, \{ \MS_h\}_{h \in [H]}, \{ \MA_k \}_{k \in [K]},\ \{P_h\sups{M}\}_{h\in [H]}, \{ R_{k,h}\sups{M} \}_{k \in [K], h \in [H]}, d_1)$, where $H \in \BN$ denotes the \emph{horizon}, $\MS_h$ denotes the state space for layer $h$, $\MA_k$ denotes the action space for player $k$, $\MA := \MA_1 \times \cdots \times \MA_K$ denotes the joint action space,  $P_h\sups{M} : \MS_h \times \MA \ra \Delta(\MS_{h+1})$ denotes the probability transition kernel for layer $h$, $R_{k,h}\sups{M} : \MS_h \times \MA \ra \Delta(\BR)$ denotes player $k$'s reward distribution for layer $h$, and $d_1 \in \Delta(\MS_1)$ denotes the initial state distribution. Each agent's decision space $\Pi_k$ is the space of their \emph{randomized Markov policies} $\pi_k = (\pi_{k,1}, \ldots, \pi_{k,H})$, where $\pi_{k,h} : \MS_h \ra \Delta(\MA_k)$, and the joint decision space is $\Pi = \Pi_1 \times \cdots \times \Pi_K$. Given a joint decision $\pi \in \Pi$, an observation is drawn from $M(\pi)$ according to the following process, called an \emph{episode}. First, an initial state is drawn according to $s_1 \sim d_1$. Then, for $h \in [H]$, the state evolves via:
  \begin{enumerate}[leftmargin=*]
  \item $a_{k,h} \sim \pi_{k,h}(s_h)$, and $r_{k,h} \sim R_{k,h}\sups{M}(s_h, (a_{1,h}, \ldots, a_{K,h}))\;\forall{}k\in\brk{K}$.\quad 2. $s_{h+1} \sim P_h\sups{M}(\cdot | s_h, (a_{1,h}, \ldots, a_{K,h}))$
\end{enumerate}
  The distribution of $(r_1, \ldots, r_K, \ocirc) \sim M(\pi)$ is given by $\ocirc = \{ (s_h, (a_{1,h}, \ldots, a_{K,h}), (r_{1,h}, \ldots, r_{K,h})\}_{h \in [H]}$ and $r_k = \sum_{h=1}^H r_{k,h}$.  With  $\Dev, \Sw$ defined as in \cref{def:ne-instance}, $\sJ := \instma$ is an NE instance of \MAFrameworkShort, and $\hmstar(\wh \pi) = 0$ if and only if $\wh \pi$ is a \emph{Markov Nash equilibrium} of $\Mstar$ (e.g., \citet{daskalakis2022complexity}). By restricting $\cM$ appropriately, this formulation captures complex settings with function approximation \citep{chen2022almost,li2022learning,xie2020learning,jin2022power,huang2021towards,zhan2022decentralized,liu2022sample}; see \cref{app:ma_examples}.
\end{example}
}

\arxiv{
\begin{example}[Learning Nash equilibria in Markov games]
  \label{ex:mne}
  Next, we consider an episodic multi-agent finite-horizon reinforcement learning setting, in which the unknown ground truth model $M^\st$ is a \emph{Markov game}. We focus on the problem of computing a Markov Nash equilibrium; the problems of computing variants of CCE and CE are discussed in \cref{app:ma_examples_equilibria}.

  Formally, each model $M \in \MM$ defines a Markov game of the form $M = (H, \{ \MS_h\}_{h \in [H]}, \{ \MA_k \}_{k \in [K]},\\ \{P_h\sups{M}\}_{h\in [H]}, \{ R_{k,h}\sups{M} \}_{k \in [K], h \in [H]}, d_1)$, where $H \in \BN$ denotes the \emph{horizon}, $\MS_h$ denotes the state space for layer $h$, $\MA_k$ denotes the action space for player $k$, $\MA := \MA_1 \times \cdots \times \MA_K$ denotes the joint action space,  $P_h\sups{M} : \MS_h \times \MA \ra \Delta(\MS_{h+1})$ denotes the probability transition kernel for layer $h$, $R_{k,h}\sups{M} : \MS_h \times \MA \ra \Delta(\BR)$ denotes player $k$'s reward distribution for layer $h$, and $d_1 \in \Delta(\MS_1)$ denotes the initial state distribution. The transition kernel and reward distributions are allowed to vary across models in $\MM$, but we assume that the state and action spaces, horizon, and initial state distribution are the same for all models in $\MM$.

  Each agent's decision space $\Pi_k$ is the space of their \emph{randomized Markov policies} $\pi_k = (\pi_{k,1}, \ldots, \pi_{k,H})$, where $\pi_{k,h} : \MS_h \ra \Delta(\MA_k)$, and the joint decision space is $\Pi = \Pi_1 \times \cdots \times \Pi_K$. Given a joint decision $\pi \in \Pi$, an observation is drawn from $M(\pi)$ according to the following process, called an \emph{episode}. First, an initial state is drawn according to $s_1 \sim d_1$. Then, for $h \in [H]$, the following random variables are sampled in sequence:
  \begin{itemize}
  \item For all $k \in [K]$, $a_{k,h} \sim \pi_{k,h}(s_h)$, and $r_{k,h} \sim R_{k,h}\sups{M}(s_h, (a_{1,h}, \ldots, a_{K,h}))$.
  \item $s_{h+1} \sim P_h\sups{M}(\cdot | s_h, (a_{1,h}, \ldots, a_{K,h}))$. 
  \end{itemize}
  The sequence $\tau = \{ (s_h, (a_{1,h}, \ldots, a_{K,h}), (r_{1,h}, \ldots, r_{K,h})\}_{h \in [H]}$ of all states, actions, and rewards is called a \emph{trajectory}. The distribution of $(r_1, \ldots, r_K, \ocirc) \sim M(\pi)$ is given by $\ocirc = \tau$ and $r_k = \sum_{h=1}^H r_{k,h}$. In particular, the pure observation space $\Ocirc$ is the space of trajectories. We assume that $\sum_{h=1}^H r_{k,h} \in [0,1]$ with probability 1, meaning that $\MR = [0,1]$, and write $\MO = \MR^K \times \Ocirc$.

  let $\Dev, \Sw$ be defined as in \cref{def:ne-instance}. Then the instance $\sJ := \instma$ is an NE instance of \MAFrameworkShort. For $\wh \pi \in \Pi$, the value $\hmstar(\wh \pi)$ measures the sum of players' incentives to deviate from $\wh \pi$ under the true model $\Mstar$, where each agent can choose an arbitrary non-stationary Markov policy as their deviation. In particular, $\hmstar(\wh \pi) = 0$ if and only if $\wh \pi$ is a \emph{Markov Nash equilibrium} of $\Mstar$ (e.g., \citet{daskalakis2022complexity}).

  A key question in (multi-agent) online reinforcement learning is to understand what structural properties of the model class $\MM$ permit efficient learnability. In the simplest case (known as the \emph{tabular} case), the state and action spaces $\MS_h, \MA$ are all finite, and $\MM$ consists of all models specified by arbitrary transitions $P_h\sups{M}$ and reward distributions $R_{k,h}\sups{M}$ with uniformly bounded support. By restricting $\cM$, our formulation also captures a more complex settings that incorporate function approximation \citep{chen2022almost,li2022learning,xie2020learning,jin2022power,huang2021towards,zhan2022decentralized,liu2022sample}; see \cref{app:ma_examples}.

\end{example}
}
\arxiv{We refer to \cref{app:ma_examples} for additional examples and exposition.}

\subsection{\MAFrameworkShort: Overview of results}

We provide upper and lower bounds on the minimax sample complexity for the \MAFrameworkShort framework using a new complexity measure, the \emph{\malong}, which generalizes the \emph{Constrained Decision-Estimation Coefficient} introduced by \cite{foster2023tight} in the single agent setting.\loose

\paragraph{The \malong}
For \arxiv{probability }measures $\bbP$ and $\bbQ$ with a common
dominating measure $\nu$, define squared Hellinger distance by \colt{$  \Dhels{\bbP}{\bbQ}=\int\prn[\big]{\sqrt{\frac{d\bbP}{d\nu}}-\sqrt{\frac{d\bbQ}{d\nu}}}^{2}d\nu$.}
\arxiv{\[
  \Dhels{\bbP}{\bbQ}=\int\prn[\bigg]{\sqrt{\frac{d\bbP}{d\nu}}-\sqrt{\frac{d\bbQ}{d\nu}}}^{2}d\nu.
\]}
Consider an instance $\sJ = (\MM, \Act, \MO, \{ \Dev\}_\ag, \{ \Sw \}_\ag )$ for the \MAFrameworkShort framework, as well as a \emph{reference model} $\Mbar: \Act \ra \Delta(\MO)$.\footnote{The reference model $\Mbar$ may be arbitrary, and is not required to lie in $\MM$.}  For a scale parameter $\vep > 0$, the \malong for the instance $\sJ$ with reference model $\Mbar$ at scale $\vep$ is defined by
\begin{align}
\deccpac[\vep](\sJ, \Mbar) := \inf_{p,q \in \Delta(\Act)} \sup_{M \in \MM} \left\{ \E_{\act \sim p} [\hm(\act)] \ | \ \E_{\act \sim q} [\hell{M(\act)}{\Mbar(\act)}] \leq \vep^2 \right\};\label{eq:ma_dec}
\end{align}
whenever the set
\colt{$\MH_{q, \vep}(\Mbar) \ldef \{ M \in \MM \ | \ \E_{\act \sim q} [\hell{M(\pi)}{\Mbar(\pi)}] \leq \vep^2 \}$}
is empty, we adopt the convention that $\deccpac[\vep](\sJ, \Mbar) = 0$.
\arxiv{\begin{align}
\MH_{q, \vep}(\Mbar) \ldef \{ M \in \MM \ | \ \E_{\act \sim q} [\hell{M(\pi)}{\Mbar(\pi)}] \leq \vep^2 \}\label{eq:define-hellset}
       \end{align}
       }
In addition, we define \colt{$\deccpac[\vep](\sJ) := \sup_{\Mbar\in\conv(\cM)}\deccpac(\MAI,\Mbar)$,}
\arxiv{\begin{align}
\deccpac[\vep](\sJ) := \sup_{\Mbar\in\conv(\cM)}\deccpac(\MAI,\Mbar),\label{eq:ma_dec_full}
       \end{align}
       }
where $\co(\MM)$ denotes the convex hull of the class $\MM$.

The interpretation of the definition \pref{eq:ma_dec}, which is a min-max game, is as follows. The model $M\in\cM$ selected by max-player represents a worst-case choice for the underlying model. The joint distributions $p,q\in\Delta(\Pi)$ selected by the min-player represent strategies for a centralized learning algorithm controlling all agents. The distribution $q\in\Delta(\Pi)$ is an \emph{exploration distribution} which acts as a strategy for acquiring information, with the quantity $\E_{\act \sim q} [\hell{M(\act)}{\Mbar(\act)}]$ acting as their average ``information gain'' (that is, the amount information that allows to distinguish between $M\in\cM$ and the reference model $\Mbar$). The distribution $p\in\Delta(\Pi)$ is an \emph{exploitation distribution} which aims to be near equilibrium for the model $M\in\cM$ selected by the max-player, with $\En_{\pi\sim{}p}\brk*{\hm(\pi)}$ representing the distance from equilibrium. Thus, to summarize, the value \pref{eq:ma_dec} captures, for a best-case choice of $p,q\in\Delta(\Pi)$, the worst-case distance to equilibrium for $p$ for models $M\in\cM$ that are ``close'' to $\Mbar$ in the sense that their information gain under $q$ is small.

For familiar readers, we recall that the (single-agent) constrained \CompShort generalizes the earlier \emph{offset} \CompShort of \citet{foster2021statistical} (which acts as a Lagrangian relaxation), and always leads to tighter guarantees \citep{foster2023tight}. Our definition \pref{eq:ma_dec} generalizes the so-called PAC variant of the constrained \CompShort in \citet{foster2023tight}, as opposed the \emph{regret} variant, which restricts to $p=q$.

\paragraph{Main results}

The first of our results gives upper and lower bounds on the minimax sample complexity for the \MAFrameworkShort framework based on the \malong. To state the result in the simplest form, we assume that $\abs{\cM}<\infty$; see \cref{sec:bounds} for more general results.\loose

\begin{theorem}[Informal version of \pref{cor:ma-constrained-upper,cor:ma-constrained-lower}]
  \label{thm:minimax-intro}
    For any instance $\sJ = (\MM, \Act, \MO, \{ \Dev\}_\ag, \{ \Sw \}_\ag )$ for the \MAFrameworkShort framework and $T\in\bbN$:
    \begin{itemize}
    \item Upper bound: Under \pref{ass:realizability-ma}, there exists an algorithm
      that achieves
      \begin{flalign}
        &\En\brk*{\RiskDM} \leq \bigoht(1)\cdot\deccpac[\vepsu](\MAI),\quad\text{where}\quad\vepsu\leq \wt{\Theta}\prn[\big]{\sqrt{\log\abs{\cM}/T}}. &&
        \label{eq:upper_informal}
      \end{flalign}
  \item Lower bound: For a worst-case model $M\in\cM$, any algorithm must have
    \begin{flalign}
      &\En\brk*{\RiskDM} \geq \bigomt(1)\cdot\deccpac[\vepsl](\MAI),
      \quad\text{where}\quad \text{$\vepslowerT\;$ solves $\;\deccpac[\veps](\MAI)\geq\bigomt\prn*{\veps^2KT}$}.
      &&
      \label{eq:lower_informal}
    \end{flalign}
  \end{itemize}
  \end{theorem}
  This result shows that the \mashort is a fundamental limit for equilibrium computation in the \MAFrameworkShort framework, and is sufficient for low sample complexity whenever $\log\abs{\cM}<\infty$. The upper bound is an immediate corollary of an upper bound given by \citet{foster2023tight} in the single-agent setting, while the lower bound requires a new approach; this is due to fundamental differences between the single and multiple agents, which we highlight in the sequel.\loose

To build intuition, let us start with a basic example. Suppose that $\MAI$ is a CCE instance consisting of two-player $A_1\times{}A_2$ normal-form games (that is, $\abs{\cA_1}=A_1$ and $\abs{\cA_2}=A_2$) with bandit feedback (\cref{ex:cce}) and Bernoulli noise. In this case, one can show that $\deccpac(\scrM)\propto{}\veps\cdot\sqrt{A_1+A_2}$, so that the upper bound \pref{eq:upper_informal} gives \colt{$  \En\brk*{\RiskDM} \approxleq \sqrt{\frac{(A_1+A_2)\log\abs{\cM}}{T}}$,}
\arxiv{\begin{align*}
  \En\brk*{\RiskDM} \approxleq \sqrt{\frac{(A_1+A_2)\log\abs{\cM}}{T}},
       \end{align*}
       }
or equivalently, $\frac{(A_1+A_2)\log\abs{\cM}}{\veps^2}$ rounds of interaction are sufficient to find an $\veps$-CCE. For this class, one can take $\log\abs{\cM}\approxleq{}\bigoht(A_1\cdot{}A_2)$. We give more refined results (\pref{sec:curse}) which allow one to replace $\log\abs{\cM}$ by $\max_{k}\log\abs{\Pi'_k}\approxleq\log(A_1+A_2)$, so that we achieve sample complexity $\bigoht\prn[\big]{\frac{A_1+A_2}{\veps^2}}$, which is optimal.\loose

Turning to lower bounds, for the same normal-form game instance $\MAI$, one can choose $\vepslowerT\approxgeq{}\frac{\sqrt{A_1+A_2}}{T}$, so that \pref{eq:lower_informal} gives \colt{$  \En\brk*{\RiskDM} \approxgeq \frac{A_1+A_2}{T}$,}
\arxiv{\begin{align*}
  \En\brk*{\RiskDM} \approxgeq \frac{A_1+A_2}{T},
       \end{align*}
       }
or equivalently, $\bigomt\prn[\big]{\frac{A_1+A_2}{\veps}}$ rounds of interaction are necessary to find an $\veps$-CCE. Comparing the upper and lower bounds, there are two gaps. The first is the term $\log\abs{\cM}$ appearing in the upper bound, which represents the sample complexity required to perform statistical estimation with the class $\cM$, {and in general scales poorly with the number of agents}. This can be refined (cf. \pref{sec:curse}), but is not possible to completely remove in general, even in the single-agent setting; see \citet{foster2021statistical,foster2023tight} and \cref{sec:bounds} for further discussion.

The second gap is the difference between the values $\vepsupperT$ and $\vepslowerT$ appearing in the upper and lower bound; we set $\vepsupperT\propto1/\sqrt{T}$, while $\vepslowerT$ is chosen to solve the fixed-point equation  $\deccpac[\veps](\MAI)\geq\bigomt\prn*{\veps^2T}$ (we focus on the case of constant $K$ in this discussion). For normal-form games, this causes the lower bound to scale with $\frac{1}{\veps}$ instead of $\frac{1}{\veps^2}$. This gap is not present in the single-agent setting \citep{foster2023tight}, where the best upper and lower bounds based on the constrained \CompShort have $\vepsupperT\approx\vepslowerT$ (up to dependence on $\log\abs{\cM}$). We show (\pref{prop:gap-bounding}) that for most parameter regimes,
\[
\deccpac[\vepsupperT](\MAI) \lesssim \prn*{{K^2}\log\abs{\cM}\cdot\deccpac[\vepslowerT](\MAI)}^{1/2},
\]
i.e., the gap between the upper and lower bounds is no worse than quadratic generically.
This gap turns out to be fundamental: We show (\cref{prop:gap-ub-easy,prop:gap-inherent}) that there exist instances for which each bound (upper and lower) is tight, and---somewhat surprisingly---the following result shows that \emph{no complexity measure} satisfying fairly general conditions can fully characterize the sample complexity\arxiv{ of multi-agent decision }making beyond a quadratic gap, even when $\log\abs{\cM}=\bigoht(1)$.

\begin{theorem}[Informal version of \cref{thm:ma-fdiv-separation}]
  \label{thm:separation-intro}
  For any $\vep \in \BN$, there exist two-player zero-sum Nash equilibrium \ma instances
  $\MAI_1 = (\MM_1, \Act, \MO, \{ \Dev\}_\ag, \{ \Sw \}_\ag )$ and $\MAI_2 = (\MM_2, \Act, \MO, \{ \Dev\}_\ag, \{ \Sw \}_\ag )$ \colt{with $\log\abs{\cM_1}=\log\abs{\cM_2}=\bigoht(1)$ }and a one-to-one mapping $\sE : \MM_1 \ra \MM_2$ satisfying:
  \begin{enumerate}
  \item \label{it:fm-equal-intro} For all $M \in \MM_1$, $\fm_k \equiv f\sups{\sE(M)}_k$ for all $k\in\brk{2}$.
  \item \label{it:dphi-equal-intro} For all $M, M' \in \MM_1$ and all $\pi \in \Pi$, $\Dhels{M(\pi)}{M'(\pi)} = \Dhels{\sE(M)(\pi)}{\sE(M')(\pi)}$.
  \item \label{it:minimax-separation-intro1} There exists an algorithm that finds an $\veps$-NE for any model in $\MAI_1$ using $\bigoht\prn*{\frac{1}{\veps}}$ rounds, yet any algorithm requires $\bigomt\prn*{\frac{1}{\veps^2}}$ rounds %
    to find an $\veps$-NE for a worst-case model in $\MAI_2$.\loose
  \end{enumerate}
  \arxiv{In addition, $\log\abs{\cM_1}=\log\abs{\cM_2}=\bigoht(1)$.}
\end{theorem}
Informally, this result states that if a complexity measure depends on the instance $\MAI$ only through 1) reward functions and 2) pairwise Hellinger distances for models in $\cM$, then it cannot characterize the optimal sample complexity for every instance beyond the gap in the prequel. In addition, the full result is not limited to Hellinger distance, and applies to general $f$-divergences including KL- and $\chi^2$-divergence. This rules out tighter guarantees based on various variants of the \CompShort, as well as most other general-purpose complexity measures\arxiv{ for interactive decision making}; see \arxiv{\pref{sec:optimality}}\colt{\aref{sec:optimality}} for details.\footnote{{Directly applying \cref{thm:separation-intro} to the constrained \CompShort presents complications due to $\Mbar\in\conv(\cM)$; see App. \ref{sec:optimality}.}}

\arxiv{\cref{thm:separation-intro} (and \cref{prop:gap-ub-easy,prop:gap-inherent}) highlight}\colt{We emphasize} a fundamental separation between the single and multi-agent frameworks. In the single-agent setting, the constrained \CompShort characterizes, up to logarithmic factors, the optimal number of samples required to learn an $\veps$-optimal decision, as long as $\log\abs{\cM}=\bigoht(1)$ \citep{foster2023tight}. For two or more agents, \cref{thm:separation-intro,prop:gap-ub-easy,prop:gap-inherent} rule out such a characterization. 

\subsection{Hidden-reward interactive decision making (\FrameworkShort)}
To prove the results in the prequel, we establish a certain equivalence between the \MAFrameworkShort framework and another \emph{single-agent} setting we refer to as \emph{\Framework} (\FrameworkShort), which generalizes the single-agent \dmso framework (\MAFrameworkShort with $K=1$) by allowing rewards to be hidden from the agent. This setting is of interest in its own right, and can be thought of as a stochastic, PAC variant of the partial monitoring problem \citep{bartok2014partial}. In what follows, we introduce the framework, then show that 1) \MAFrameworkShort can be viewed as a special case of the \FrameworkShort framework via a simple reduction, and 2) a converse holds, thus showing a sort of equivalence. We then discuss implications for minimax rates in both frameworks.\loose

Formally, the \FrameworkShort framework proceeds in $T$ rounds, where for each round $t = 1, 2, \ldots, T$:
\begin{enumerate}
  \colt{\setlength\itemsep{.3em}}
\item The learner selects a \emph{decision} $\act\^t \in \Act$, where $\Act$ is the \emph{decision space}, and gains (but does not observe) reward $\fmstar(\pi\ind{t})$.
\item The learner receives an observation $o\^t \in \MO$ sampled via $o\^t \sim \Mstar(\pi\^t)$, where $M^\st : \Act \ra \Delta(\MO)$ is the underlying \emph{model}. We refer to $\MO$ as the \emph{observation space}. 
\end{enumerate}
After this process finishes, the learner uses the data collected throughout the $T$ rounds of interaction to produce an output decision $\wh \act \in \Act$, which may be randomized according to a distribution $p \in \Delta(\Act)$. The learner's goal is to choose the decision $\wh\act$ so as to maximize its (unobserved) reward $\fmstar(\wh\act)$.  Formally, writing $\pim := \argmax_{\act \in \Act} \fm(\act)$, we define the \emph{risk} of an algorithm as: \colt{$\RiskDM := \E_{\wh \pi \sim p} [ \fmstar(\pimstar) - \fmstar(\wh \act)]$.}
\arxiv{\begin{align}
\RiskDM := \E_{\wh \pi \sim p} [ \fmstar(\pimstar) - \fmstar(\wh \act)]\nonumber.
\end{align}}

We assume that every model $M$ is associated a (known) function $\fm : \Act \ra \bbR$, where $\fm(\act)$ specifies the learner's value under decision $\act \in \Act$ when the underlying model is $M$.
We make the following realizability assumption, analogous to \cref{ass:realizability-ma}.
    \begin{assumption}[Realizability for \FrameworkShort]
  \label{ass:realizability}
  The learner has access to a model class $\cM$ consisting of probability kernels $M : \Pi \ra \Delta(\MO)$ that contains the true model $\Mstar$.
\end{assumption}
\arxiv{
  We refer to the tuple $\sI = (\MM, \Act, \MO, \{ \fm(\cdot) \}_{M\in\cM} )$ as an \emph{instance} for the \FrameworkShort framework. It specifies all of the information known to a learner a-priori before interacting with \arxiv{the model }$\Mstar \in \MM$.\loose
  }

\arxiv{
\begin{remark}
  An equivalent formulation of the \FrameworkShort framework would be to consider models $M:\Pi\to\Delta(\cO\times\cR)$ that specify joint distributions over observations and rewards and define $\fm(\pi)=\En^{\sss{M},\pi}\brk{r}$, but only allow $o$ to be observed by the learner under $(o,r)\sim{}M(\pi)$.
\end{remark}
}
  We refer to the tuple $\sI = (\MM, \Act, \MO, \{ \fm(\cdot) \}_{M\in\cM} )$ as an \emph{instance} for the \FrameworkShort framework. We extend the constrained Decision-Estimation Coefficient of \cite{foster2023tight} to \hr as follows. For an instance $\sI = (\MM, \Act, \MO, \{ \fm(\cdot ) \}_M )$, reference model $\Mbar : \Act \ra \Delta(\MO)$, and scale parameter $\vep > 0$, the constrained Decision-Estimation Coefficient is given by\footnote{Note that we use the same notation for the DEC in the \FrameworkShort and \MAFrameworkShort settings; we will typically use the letter $\sI$ to denote \FrameworkShort instances and $\sJ$ to denote \MAFrameworkShort instances to avoid ambiguity.}
\begin{align}
\hspace{-0.4cm} \deccpac[\vep](\sI, \Mbar) = \inf_{p,q \in \Delta(\Act)} \sup_{M \in \MM} \left\{ \E_{\act \sim p} [\fm(\pim) - \fm(\act) \ | \ \E_{\act \sim q}[\hell{M(\act)}{\Mbar(\act)}] \leq \vep^2 \right\}\label{eq:decc-pm}.
\end{align}
We define the Decision-Estimation Coefficient (DEC) of the instance $\sI$ at scale $\vep$ to be \colt{$\deccpac[\vep](\sI) = \sup_{\Mbar \in \co(\MM)} \deccpac[\vep](\sI, \Mbar)$.}
\arxiv{\begin{align}
\deccpac[\vep](\sI) = \sup_{\Mbar \in \co(\MM)} \deccpac[\vep](\sI, \Mbar)\label{eq:decc-m}.
\end{align}}
This definition is identical to the constrained PAC DEC \citep{foster2023tight}; this is natural, as the only difference between the \FrameworkShort framework and the DMSO framework \citep{foster2023tight} is that we relax the constraint that the agent observes its reward. \loose

\begin{remark}
  The \FrameworkShort framework is related to the partial monitoring problem \citep{bartok2014partial}\colt{, which typically considers regret guarantees; in contrast, we consider PAC guarantees.}
  \arxiv{. While most  work in partial monitoring considers regret guarantees (that is, cumulative suboptimality for $\pi\ind{1},\ldots,\pi\ind{T}$), we consider PAC guarantees (i.e., final suboptimality for $\pihat$).}  %
  \arxiv{An additional difference between the two settings is that partial monitoring typically considers finite decision and observation spaces, while we allow for large, structured spaces (formalized via the model class $\cM$), and aim for sample complexity guarantees that reflect the intrinsic complexity of these spaces.}
\end{remark}

\arxiv{
\begin{remark}[Contrast with \emph{reward-free} DMSO]
  Despite the similar name, the \FrameworkShort framework is distinct from the ``reward-free'' DMSO framework considered in the recent work of \citet{chen2022unified}; in the latter framework, which is specialized to Markov decision processes, a reward-function is given to the learner explicitly, but only after the learning process ends.
\end{remark}
}

\subsection{\FrameworkShort: Overview of results}
It is fairly immediate to see that the \FrameworkShort framework generalizes the \MAFrameworkShort framework. For any \MAFrameworkShort instance $\MAI = (\MM, \Act, \MO, \{\Dev\}_{\ag}, \{\Sw\}_{\ag})$ satisfying \cref{ass:nonneg-dev} and \cref{ass:existence-eq}, by choosing the value function $\fm(\cdot)=-\hm(\cdot)$, the instance of the \FrameworkShort framework specified by the tuple $\HRI = (\MM, \Act, \MO, \{ f\sups{M} \}_M)$ (recalling that $\MO = \Ocirc \times \MR^\Ag$) is statistically equivalent to $\MAI$.\footnote{It is essential for this reduction that the rewards in $\sI$ be hidden, since it is in general impossible to simulate a reward whose mean is $-\hm(\pi)$ using samples from $M(\pi)$.} In particular, letting $\mathfrak{M}(\MAI,T)$ denote the minimax risk for an instance $\MAI$ in the \MAFrameworkShort framework, and let $\mathfrak{M}(\HRI,T)$ denote the minimax risk for the corresponding \hr instance $\HRI$  (see \pref{sec:add-prelim} for formal definitions), we have:
\colt{   \begin{enumerate}[leftmargin=*]
   \item For all $\Mbar$ and $\vep > 0$, $\deccpac(\sI, \Mbar) = \deccpac(\sJ, \Mbar)$.\quad 2. For all $T$, $\mf M(\sI, T) = \mf M(\sJ, T)$. 
   \end{enumerate}
   }
\arxiv{   \begin{enumerate}
   \item For all models $\Mbar$ and $\vep > 0$, $\deccpac(\sI, \Mbar) = \deccpac(\sJ, \Mbar)$.
   \item For all $T \in \BN$, $\mf M(\sI, T) = \mf M(\sJ, T)$. 
   \end{enumerate}
   }

  It is natural to ask whether the \FrameworkShort framework is \emph{strictly} more general than the \MAFrameworkShort framework. Indeed, by allowing rewards to be hidden, one might imagine that \FrameworkShort can capture problems outside of  \MAFrameworkShort, which forces rewards to be observed. The next result shows that this is not the case: any \FrameworkShort instance can be embedded in a two-player zero-sum NE instance for \MAFrameworkShort, with minimal increase in statistical complexity. 
\begin{theorem}[Informal version of \cref{prop:pm-to-ma}]
  \label{thm:equivalence-intro}
  Consider any \hr instance specified by the tuple $\HRI = \instpm$. For any $\delta>0$, there exists a two-player zero-sum NE \ma instance $\MAI = (\til \MM, \til \Pi, \til \MO, \Dev, \Sw)$  (\cref{def:ne-instance}) such that: 
  \begin{enumerate}
  \item For all $\vep > 0$, $\deccpac[\vep](\HRI) \leq \deccpac[\vep](\MAI) \leq \delta+ \deccpac[\vep + \delta](\HRI)$.
  \item For all $T \in \BN$, it holds that $\mf M(\HRI, T) \leq \mf M(\MAI, T) \leq  \mf M(\HRI, T) + \delta$.
  \item If $\MM$ is finite, then $\log |\til \MM| \leq \log |\MM| + \polylog(T,\delta^{-1})$. 
  \end{enumerate}
\end{theorem}
This result establishes that the \MAFrameworkShort and \FrameworkShort frameworks satisfy a sort of equivalence, and shows that characterizing the minimax sample complexity for \MAFrameworkShort is no easier than characterizing the minimax sample complexity for the \FrameworkShort framework. The proof proceeds by embedding a given instance for the \FrameworkShort framework into a two-player game: the first of the two agents in the game plays the role of the \FrameworkShort agent, and the second agent selects actions to ensure that optimal actions for the original \FrameworkShort instance are Nash equilibria for the new instance, and vice-versa. The key idea is that even though rewards in the game are observed, by making the game polynomially large, we can ensure that discovering them requires a prohibitively large amount of exploration, rendering them effectively hidden.

    \paragraph{\FrameworkShort: Minimax rates}
    \label{sec:c-dec-pm}
    To prove the multi-agent minimax rates in \cref{thm:minimax-intro,thm:separation-intro}, we first prove analogous bounds for the \FrameworkShort framework, then use the equivalence above to extend them to \MAFrameworkShort. \colt{These results can be found in \cref{thm:constrained-upper-finitem,thm:constrained-lower,prop:fdiv-separation} in \cref{sec:bounds}.\loose}
 \arxiv{   
    In particular, the following result provides our main sample complexity bounds for \hr, generalizing     \cref{thm:minimax-intro}.
\begin{theorem}[Informal version of \pref{thm:constrained-upper-finitem,thm:constrained-lower,prop:fdiv-separation}]
  For any instance $\sI = (\MM, \Act, \MO, \{ \fm(\cdot ) \}_M )$ for the \FrameworkShort framework and $T\in\bbN$:
    \begin{itemize}
    \item Upper bound: Under \pref{ass:realizability-ma}, there exists an algorithm
      that achieves
      \begin{align}
        \En\brk*{\RiskDM} \leq \bigoht(1)\cdot\deccpac[\vepsu](\HRI)
      \end{align}
    for all $M\in\cM$, where $\vepsu\leq \wt{\Theta}\prn[\big]{\sqrt{\log\abs{\cM}/T}}$.
  \item Lower bound: For a worst-case model $M\in\cM$, any algorithm must have
    \begin{align}
      \En\brk*{\RiskDM} \geq \bigomt(1)\cdot\deccpac[\vepsl](\HRI),
    \end{align}
    where $\vepslowerT$ is the largest value $\veps>0$ such that  $\deccpac[\veps](\HRI)\geq\bigomt\prn*{\veps^2T}$.
    \end{itemize}
In addition, no complexity measure that depends on the instance $\HRI$ only through the reward functions $\crl{\fm(\cdot)}_{M\in\cM}$ and pairwise Hellinger distances for models $M,M'\in\cM$ can characterize the optimal sample complexity for every instance, beyond a quadratic gap.  
\end{theorem}
}

\subsection{\MAFrameworkShort: Additional results}

Beyond minimax rates, we provide a number of structural results for the \MAFrameworkShort framework\arxiv{ that we believe to be of independent interest}, including: (1) conditions under which the multi-agent \CompShort can be controlled by the single-agent \CompShort, and (2) conditions under which the so-called \emph{curse of multiple agents} can be avoided. 
\arxiv{We now highlight these results.}

\paragraph{From multi-agent to single-agent}
We show that it is generically possible to upper bound the \mashort in terms of the single-agent \CompShort for each player $k$. This result is most easily stated in terms of a multi-agent analogue of the \emph{offset} version of the DEC introduced in \cite{foster2021statistical}. Specifically, we consider a \emph{regret} variant of the offset DEC that restricts $p=q$, coupling exploration and exploitation: For an instance $\sJ$, reference model $\Mbar$, and scale parameter $\gamma > 0$, we define
\begin{align}
\decoreg[\gamma](\sJ, \Mbar) := \inf_{p \in \Delta(\Act)} \sup_{M \in \MM} \left\{ \E_{\act \sim p} [\hm(\act)] - \gamma \cdot \E_{\act \sim p} [\hell{M(\act)}{\Mbar(\act)}] \right\}\label{eq:define-deco}.
\end{align}
It follows immediately from the results of \cite{foster2023tight} (see \cref{prop:constrained-to-offset}) that 
$\deccpac[\veps](\sJ, \Mbar)
\leq \inf_{\gamma>0}\crl*{\decoreg[\gamma](\MAI,\Mbar)\vee{}0 + \gamma\veps^2}$,
so upper bounds on $\decoreg(\sJ)$ yield upper bounds on $\deccpac(\sJ)$, which can in turn be inserted into \cref{thm:minimax-intro} to yield upper bounds on minimax risk. 
While it is also possible to directly upper bound $\deccpac(\sJ, \Mbar)$ without going through $\decoreg(\sJ, \Mbar)$, using $\decoreg(\cdot)$ is more convenient and does not lead to any significant quantitative loss in the resulting upper bounds.\loose

We prove an upper bound on the multi-agent DEC of the instance $\sJ$, in terms of the (single-agent) DEC of $\Ag$ different model classes $\til \MM_k$, defined in terms of $\sJ$. To define these model classes, for $M \in \MM$ and $\ag \in [\Ag]$, we first define an induced \emph{single-agent} model $\single{M}$ as follows: a pure observation drawn from $\single{M}(\pi)$ has the distribution of the pure observation $\ocirc$ when $\ocirc \sim M(\pi)$, and the reward drawn from $\single{M}(\pi)$ has the distribution of $r_\ag$ when $(r_1, \ldots, r_\Ag) \sim M(\pi)$. In short, the model $\single{M}$ is identical to $M$ but simply ignores the rewards of all agents except $\ag$. 
Next, the model class $\til \MM_\ag$ is defined to have policy space $\Pi_k$, so that models in $\til \MM_\ag$ are mappings $\til M : \Pi_k \ra \Delta(\MR \times \Ocirc)$. Finally, we define the class $\til \MM_\ag$, which is indexed by $\Pi_{-\ag} \times \MM$\arxiv{, as follows:
\begin{align}
\til\MM_\ag = \left\{ \pi_\ag \mapsto \single{M}(\pi_\ag, \pi_{-\ag}) \ : \ \pi_{-\ag} \in \Pi_{-\ag},\ M \in \MM\right\}.
\end{align}}\colt{, via $\til\MM_\ag = \left\{ \pi_\ag \mapsto \single{M}(\pi_\ag, \pi_{-\ag}) \ : \ \pi_{-\ag} \in \Pi_{-\ag},\ M \in \MM\right\}$.}
\noah{remove informal tag}
\begin{theorem}%
  \label{thm:single-multiple-ch-informal}
  Let $\sJ = \instma$ be a NE \MAFrameworkShort instance satisfying \cref{ass:convexity_pols}. Then for any $\gamma > 0$, it holds that\footnote{Here, the notation $\decoreg[\gamma/K](\til \MM_k, \Mbar_k)$ refers to the single-agent DEC for  the model class $\til \MM_k$; see \cref{sec:add-prelim}.}
    \begin{align}
\sup_{\Mbar \in \co(\MM)} \decoreg[\gamma](\sJ, \Mbar) \leq & \sum_{\ag=1}^\Ag \sup_{\wb M_\ag \in \co(\til \MM_\ag)} \decoreg[\gamma/\Ag](\til \MM_\ag, \wb M_\ag)\nonumber.
  \end{align}
\end{theorem}
This result allows us to bound the \mashort using standard bounds on the single-agent \CompShort \citep{foster2021statistical}. For example, for normal-form games\arxiv{ with bandit feedback}, where each player has $A_k$ actions, it yields $\arxiv{\sup_{\Mbar \in \co(\MM)} }\decoreg[\gamma](\sJ, \Mbar)\approxleq{}K\cdot\sum_{k=1}^{K}\frac{A_k}{\gamma}$. See \cref{sec:single-multiple} for refinements concerning Markov games.

The proof of \cref{thm:single-multiple-ch} employs a novel fixed-point argument: For each agent $k$, if all other agents commit to some joint distribution, this induces a single-agent \dmso instance, and it is natural for \arxiv{agent }$k$ to play the strategy that minimizes the single-agent \CompShort for this instance. Using Kakutani's fixed point theoerem, we show that it is possible for all $K$ agents to apply this strategy simultaneously.\loose

\paragraph{On the curse of multiple agents}
In multi-agent reinforcement learning, the \emph{curse of multiple agents} refers to the situation in which the sample complexity required to learn an equilibrium scales exponentially in the number of players \citep{jin2021v}. In general, our upper bounds on sample complexity for the \MAFrameworkShort framework (\cref{thm:minimax-intro}) suffer from the curse of multiple agents due to the presence of the estimation complexity term $\log\abs{\cM}$. For example, in a $K$-player normal-form game with $A$ actions per player, one has $\log\abs{\cM}\approx{} A^{K}$ (using an appropriate discretization of $\MM$). Our final result shows that it is possible to avoid the curse of multiple agents by replacing the estimation complexity $\log\abs{\cM}$ with the maximum size $\max_{k}\log\abs{\Dev}$ for each player's deviation set, which is usually polynomial in the number of agents; the tradeoff is that the result scales with the \mashort for the \ma instance in which the model class $\cM$ is \emph{convexified} via $\cM\gets\conv(\cM)$.
\begin{theorem}[Informal version of \cref{thm:curse_ub}]
  \label{thm:curse-ub-intro}
  Let $\sJ = \instma$ be a CCE instance (\cref{def:cce-instance}) or a CE instance (\cref{def:ce-instance}) of the \MAFrameworkShort framework. Then, for any $T \in \BN$, \cref{alg:maexo} outputs $\wh \pi \in \Pi$ such that with probability at least $1-\delta$,
  \begin{align}
\RiskDM = \hmstar(\wh \pi) \leq \bigoht(K) \cdot \inf_{\gamma > 0} \left\{\decoreg[\gamma](\co(\sJ)) + \frac{\gamma}{T} \cdot \log \left( \frac{\max_k |\Dev|}{ \delta} \right) \right\}\nonumber,
  \end{align}
  where we adopt the convention that $\conv(\MAI) \equiv (\conv(\MM), \Act, \MO, \{\Dev\}_{\ag}, \{\Sw\}_{\ag})$.
\end{theorem}
In normal-form games with $K$ players and $A$ actions per player, we have $\decopac(\conv(\MAI))\approxleq\frac{A}{\gamma}$ and $\max_{k}\log\abs{\Dev}=\log(A)$, so this result gives
\colt{$\RiskDM\approxleq{} \sqrt{\poly(K)A/T}$.}
\arxiv{\[
  \RiskDM\approxleq{} \sqrt{\frac{\poly(K)\cdot{}A}{T}}.
\]}
More broadly, \cref{thm:curse-ub-intro} shows that it is generically possible to avoid the curse of multiple agents for convex classes, including structured classes of normal-form games with bandit feedback such as games with linear or convex payoffs. In general though, it does not lead to tight guarantees for non-convex classes such as Markov games, and is best thought of as complementary to results for this setting \citep{jin2021v,song2021can,mao2022provably}.
The result is proven by adapting the powerful \emph{exploration-by-optimization} algorithm from the single-agent setting \citep{lattimore2022minimax,foster2022complexity} so as to exploit the unique feedback structure of the multi-agent setting. We refer to \cref{sec:curse} for details, as well as additional results which highlight settings in which the curse of multiple agents cannot be avoided in the sense of  \cref{thm:curse-ub-intro}.
\loose

\arxiv{
\subsection{Preliminaries}
\label{sec:add-prelim}

Below we provide additional technical preliminaries which will be used throughout our proofs.

\paragraph{Probability kernels}
For probability spaces $(\MX, \sX)$ and $(\MY, \sY)$, a \emph{probability kernel} $P(\cdot |\cdot)$ from $(\MX, \sX)$ to $(\MY, \sY)$ is a mapping $P : \sY \times \MX \ra [0,1]$ which satisfies (1) for all $x \in \MX$, $P(\cdot | x)$ is a probability measure on $(\MY, \sY)$, and (2) for all $Y \in \sY$, the mapping $x \mapsto P(Y | x)$ is measurable with respect to $\sX$. To simplify notation we often denote probability kernels as $P : \MX \ra \Delta(\MY)$.

\paragraph{\MAFrameworkShort framework}
We adopt the same formalism for probability spaces as in
\cite{foster2021statistical,foster2023tight}. Decisions are associated
with a measure space $(\Act, \sP)$, and observations are associated
with the measure space $(\MO, \sO)$. In the \MAFrameworkShort
framework, pure observations are associated with the measure space
$(\Ocirc, \sO_\circ)$ and  rewards are associated with a measure space
$(\MR, \sR)$, and furthermore, we have $\MO = \Ocirc \times \MR^\Ag$
and  $\sO = \sO_\circ \otimes \sR^{\otimes \Ag}$. Formally, a
\emph{model} $M(\cdot \mid \cdot)$ is a probability kernel from $(\Pi,
\sP)$ to $(\MO, \sO)$. We denote the set of all models as
$\cMall$. Note that $\cMall$ depends on the measure spaces $(\Pi,
\sP), (\MO, \sO)$; when we wish to make this dependence explicit, we
will write $\cMexpl{\Pi, \MO}$.  The \emph{history} up to time $t$ is given by $\hist\^t = (\pi\^1, o\^1), \ldots, (\pi\^t, o\^t)$. We define
\begin{align}
\Omega\^t = \prod_{i=1}^t (\Pi \times \MO), \qquad \qquad \sF\^t = \bigotimes_{i=1}^T (\sP \otimes \sO)\nonumber,
\end{align}
so that $\hist\^t$ is associated with the space $(\Omega\^t,
\sF\^t)$.

We assume throughout the paper that $\MR = [0,1]$ (which implies in particular that $\hm(\pi) \in [0,K]$ for all $M, \pi$) unless otherwise stated. To simplify notation, for each $\pi\in \Pi$ and $M \in \MM$, we write $\hm_k(\pi) := \sup_{\pi_k' \in \Pi_k} \fm_k(\Sw(\dev, \pi)) - \fm_k(\pi)$, so that $\hm(\pi) = \sum_{k=1}^K \hm_k(\pi)$. %

\colt{
\begin{remark}[Notation for \maf]
We use the following convention throughout the paper: When convenient, we associate any singleton distribution with the element that the distribution places its mass on. For instance, for a pure decision $\sigma = (\sigma_1, \ldots, \sigma_\Ag) \in \Sigma_1 \times \cdots \times \Sigma_\Ag$ in the context  of \cref{def:cce-instance}, we will denote its corresponding singleton distribution $\indic_\sigma \in \Delta(\Sigma) = \Pi$ as just $\sigma \in \Pi$. In addition, when possible, we use the convention that $\Sigma$ denotes a pure decision set, whereas $\Pi$ denotes a decision set that may be pure or mixed (this will be clear from context).
\end{remark}
}

\paragraph{The canonical single-agent instance}
Given a decision space $\Pi$, an observation space $\MO = \Ocirc \times \MR$, and a model class $\MM \subset (\Pi \ra \Delta(\MO))$, there is a canonical single-agent instance $\sJ = \instma$ corresponding to the model class $\MM$: we take $\Dev[1] = \Pi$ and $\Sw[1](\dev[1], \pi) = \dev[1]$, which ensures that $\hm(\pi) = \max_{\pi' \in \Pi} \fm(\pi') - \fm(\pi)$ for all $\pi \in \Pi, M \in \MM$. The single-agent instance $\sJ$ of the 1-player \MAFrameworkShort framework exactly captures the DMSO framework in \cite{foster2021statistical,foster2023tight} for the model class $\MM$. Furthermore, for any model $\Mbar$, we will write $\deccpac(\MM, \Mbar) = \deccpac(\sJ, \Mbar)$ (and similarly we will write $\decoreg[\gamma](\MM, \Mbar) = \decoreg[\gamma](\sJ, \Mbar)$ for regret variant of the offset DEC introduced in \cref{sec:single-multiple}); the quantity $\deccpac(\MM, \Mbar) = \deccpac(\sJ, \Mbar)$ is identical to the constrained (PAC) DEC of the model class $\MM$ as defined in \cite{foster2023tight}, and the quantity $\decoreg(\MM, \Mbar) = \decoreg(\sJ, \Mbar)$ is identical to the offset (regret) DEC of the model class $\MM$ as defined in \cite{foster2021statistical}. 

\paragraph{\FrameworkShort framework}
As in the \MAFrameworkShort framework, decisions are associated
with a measure space $(\Act, \sP)$, observations are associated
with the measure space $(\MO, \sO)$, and models $M(\cdot \mid \cdot)$ are
probability kernels from $(\Pi,
\sP)$ to $(\MO, \sO)$. The history up to time $t$ is given by
$\hist\^t = (\pi\^1, o\^1), \ldots, (\pi\^t, o\^t)$, and is associated
with the space $(\Omega\^t,
\sF\^t)$ given by
\begin{align}
\Omega\^t = \prod_{i=1}^t (\Pi \times \MO), \qquad \qquad \sF\^t = \bigotimes_{i=1}^T (\sP \otimes \sO)\nonumber.
\end{align}
We denote the set of all models as
$\cMall$. Unless stated otherwise, we will assume throughout that $\fm(\act) \in
[0,1]$ for all $M\in\cMall$ and $\act\in\Pi$.

For a model $M$ and decision $\act \in \Act$, $\E\sups{M, \pi}[\cdot]$
denotes expectation under the process $o\sim M(\pi)$. To simplify
notation, we often abbreviate $\gm(\pi) := \fm(\pim) - \fm(\pi)$, so
that $\RiskDM = \E_{\wh \pi \sim p}[\gmstar(\wh\pi)]$.

\colt{
\begin{remark}[Alternative formulation for \hr]
  An equivalent formulation of the \FrameworkShort framework would be to consider models $M:\Pi\to\Delta(\cO\times\cR)$ that specify joint distributions over observations and rewards and define $\fm(\pi)=\En^{\sss{M},\pi}\brk{r}$, but only allow $o$ to be observed by the learner under $(o,r)\sim{}M(\pi)$.
\end{remark}
}

\colt{
\begin{remark}[Contrast with \emph{reward-free} DMSO]
  Despite the similar name, the \FrameworkShort framework is distinct from the ``reward-free'' DMSO framework considered in the recent work of \citet{chen2022unified}; in the latter framework, which is specialized to Markov decision processes, a reward-function is given to the learner explicitly, but only after the learning process ends.
\end{remark}
}

\paragraph{Density ratios}
For both the \MAFrameworkShort and \FrameworkShort, we define
\begin{align}
\abscont \ldef \sup_{M, M' \in \MM} \sup_{\act \in \Act} \sup_{A \in \sO} \left\{ \frac{M(A \mid \act)}{M'(A \mid \act)} \right\} \vee e\label{eq:define-abscont}.
\end{align}
Finiteness of $\abscont$ is not necessary for our results to hold, but
improves several of our bounds by a $\log(T)$ factor.

\paragraph{Divergences}Total variation distance is given by
  \[
    \Dtv{\bbP}{\bbQ}=\sup_{A\in\filt}\abs{\bbP(A)-\bbQ(A)}
    = \frac{1}{2}\int\abs{d\bbP-d\bbQ},
  \]
and the Kullback Leibler divergence is given by
  \[
    \Dkl{\bbP}{\bbQ} =\left\{
    \begin{array}{ll}
\int\log\prn[\big]{
      \frac{d\bbP}{d\bbQ}
      }d\bbP,\quad{}&\bbP\ll\bbQ,\\
      +\infty,\quad&\text{otherwise.}
    \end{array}\right.
\]

\paragraph{Minimax sample complexity}
Formally, for $T \in \BN$, an \emph{algorithm} (for either the \FrameworkShort or \MAFrameworkShort frameworks) is a collection of probability kernels $(p,q) = \left( p(\cdot \mid \cdot), \{ q\^t ( \cdot \mid \cdot ) \}_{t=1}^T \right)$, where each $q\^t : \Omega\^{t-1} \ra \Delta(\Pi)$ is a probability kernel from $(\Omega\^{t-1}, \sF\^{t-1})$ to $(\Pi, \sP)$, and $p : \Omega\^T \ra \Delta(\Pi)$ is a probability kernel from $(\Omega\^T, \sF\^T)$ to $(\Pi, \sP)$. We let $\BP\sups{M, (p,q)}$ denote the law of $(\hist\^T, \wh \pi)$ under the process:
\begin{align}
\pi\^t \sim q\^t (\cdot \mid \hist\^{t-1}), \ o\^t \sim M(\cdot \mid \pi\^t),\ \forall t \in [T], \qquad \wh \pi \sim p(\cdot \mid \hist\^T)\nonumber,
\end{align}
and we use $\E\sups{M, (p,q)}$ to denote the corresponding expectation. 
Our main goal is to characterize the \emph{minimax PAC sample complexity} of an instance $\sJ = \instma$ of the \MAFrameworkShort framework or $\sI = \instpm$ of the \FrameworkShort framework. The minimax sample complexities for both cases are defined in an identical manner, spelled out below:
\begin{align}
  \mf M(\sJ, T) &\ldef %
                    \inf_{(p,q)} \sup_{\Mstar \in \MM} \E\sups{\Mstar, (p,q)} \E_{\wh \pi \sim p(\cdot \mid \hist\^T)}\left[ \sum_{k=1}^K \sup_{\dev \in \Dev} \fm_\ag(\Sw(\dev, \wh\act)) - \fm_\ag(\wh\act)\right],\nonumber\\
  \mf M(\sI, T) %
                     &\ldef \inf_{(p,q)} \sup_{\Mstar \in \MM} \E\sups{\Mstar, (p,q)} \E_{\wh \pi \sim p(\cdot \mid \hist\^T)} [\fmstar(\pimstar) - \fmstar(\wh\pi)]\nonumber.
\end{align}

\arxiv{
  \subsection{Organization}
  \colt{
  \cref{part:main} of the appendix presents our main results, with
  preliminaries in \cref{sec:add-prelim}. In
  particular, \cref{sec:relations} and
  \cref{sec:bounds} present our main sample complexity guarantees for
  the \MAFrameworkShort and \FrameworkShort frameworks.
\begin{itemize}
\item \cref{sec:relations} establishes a certain equivalence between
  the \MAFrameworkShort and \FrameworkShort frameworks.
\item \cref{sec:bounds} establishes upper and lower bounds on the
  minimax rates for both frameworks based on the \CompText, and
  highlights barriers to obtaining sharper guarantees analogous to
  those found in the single-agent, reward-observed \dmso framework \citep{foster2023tight}.
\end{itemize}
\cref{sec:single-multiple} and \cref{sec:curse} present additional results concerning the \MAFrameworkShort framework.
\begin{itemize}
\item \pref{sec:single-multiple} gives general conditions under which it is possible to bound the \mashort in terms of the single-agent \CompShort.
\item \pref{sec:curse} gives conditions under which one can obtain
  sample complexity guarantees in the \MAFrameworkShort framework that
  avoid the so-called \emph{curse of multiple agents}.
\end{itemize}

\noindent\cref{part:examples} provides examples (instances, as well as
upper
and lower bounds on the DEC and minimax rates) and discussion 
for the \maf, and \cref{part:proofs} contains proofs for all results.
}

\arxiv{
   This paper is organized as follows. First,\cref{sec:relations} and \cref{sec:bounds} present our main results:
\begin{itemize}
\item \cref{sec:relations} establishes a certain equivalence between the \MAFrameworkShort and \FrameworkShort.
\item \cref{sec:bounds} establishes upper and lower bounds on the minimax rates for both frameworks based on the \CompText, and highlights barriers to obtaining sharper guarantees analogous to those found in the basic \dmso framework \citep{foster2023tight}.
\end{itemize}
\cref{sec:single-multiple} and \cref{sec:curse} then present additional results concerning the \MAFrameworkShort framework:
\begin{itemize}
\item \pref{sec:single-multiple} gives general conditions under which it is possible to bound the \mashort in terms of the single-agent \CompShort.
\item \pref{sec:curse} gives conditions under which one can obtain sample complexity guarantees in the \MAFrameworkShort framework that avoid the so-called \emph{curse of multiple agents}, as well as examples in which this is not possible.
\end{itemize}
All proofs are deferred to the appendix. Further examples for both
frameworks are given in \cref{app:ma_examples}.
}

  }

\paragraph{Additional notation}
For an integer $n\in\bbN$, we let $[n]$ denote the set
  $\{1,\dots,n\}$. For a set $\cX$, we let
        $\Delta(\cX)$ denote the set of all probability distributions
        over $\cX$. For $x\in\cX$, we use $\indic_x\in\Delta(\cX)$ to
        denote the distribution which places probability mass $1$ on $x$. We adopt standard
        big-oh notation, and write $f=\bigoht(g)$ to denote that $f =
        \bigoh(g\cdot{}\max\crl*{1,\mathrm{polylog}(g)})$. We use $\approxleq$ only in informal statements to emphasize the most relevant aspects of an inequality. For a set $\MX$, let $\powerset{\MX}$ denote the power set of i.e., the set of all subsets of $\MX$.

\section{Equivalence of \MAFrameworkShort and \FrameworkShort frameworks}
\label{sec:relations}

In this section, which forms the starting point for our main results, we show that the \MAFrameworkShort and \FrameworkShort frameworks satisfy a certain statistical equivalence.
First, in \pref{prop:ma-to-pm}, we formalize the trivial direction of this equivalence: namely, any instance of the \maf can be viewed as an instance of the \hrf. To state the result, recall that per our convention, the full observation space in a \MAFrameworkShort instance is denoted by $\MO = \Ocirc \times \MR^\Ag$. 
\begin{theorem}[Reducing \MAFrameworkShort to \FrameworkShort]
  \label{prop:ma-to-pm}
  Consider any instance of the \MAFrameworkShort framework satisfying \cref{ass:nonneg-dev} and \cref{ass:existence-eq} and specified by the tuple $\MAI = (\MM, \Act, \MO, \{\Dev\}_{\ag}, \{\Sw\}_{\ag})$. %
  Then for some choice of value functions $\til{f}\sups{M}$, the instance of the \FrameworkShort framework specified  by the tuple $\HRI = (\MM, \Act, \MO, \{ \til{f}\sups{M} \}\subs{M})$, satisfies:
   \begin{enumerate}
   \item For all models $\Mbar$ and $\vep > 0$, $\deccpac(\sI, \Mbar) = \deccpac(\sJ, \Mbar)$.
   \item For all $T \in \BN$, $\mf M(\sI, T) = \mf M(\sJ, T)$. 
   \end{enumerate}
 \end{theorem}
This result proceeds by choosing the value function $\til{f}\sups{M}(\pi)=K-\hm(\pi)$. Note that for this reduction to be admissible, it is critical that rewards are hidden: the function $\hm(\pi)$ is not observed directly in the \maf, and as we will see, this is a source of fundamental hardness.

\pref{prop:ma-to-pm} is a fairly immediate result, and it is natural to imagine that the \FrameworkShort framework might truly be more general than the \MAFrameworkShort framework, especially since rewards \emph{are} observed in the latter.
The following result, which is the formal version of \pref{thm:equivalence-intro}, shows that if one allows for small approximation, any instance of the \FrameworkShort framework can be embedded in a two-player, zero-sum NE instance for \MAFrameworkShort with minimal increase in complexity.

\begin{theorem}[Reducing \FrameworkShort to \MAFrameworkShort]
  \label{prop:pm-to-ma}
  Consider any instance of the \FrameworkShort framework specified by the tuple $\sI = \instpm$. Then for any $V \in \BN$, there is a two-player zero-sum NE instance $\sJ = (\til \MM, \til \Pi, \til \MO, \Dev, \Sw)$ for the \MAFrameworkShort framework (\cref{def:ne-instance}) such that:
  \begin{enumerate}
  \item For all $\vep > 0$, $\deccpac[\vep](\sJ) \leq \deccpac[\vep](\sI) \leq 6/\sqrt{V} + \deccpac[\vep + (6/V)^{-1/2}](\sJ)$.
  \item For all $T \in \BN$, it holds that $\mf M(\sJ, T) \leq \mf M(\sI, T) \leq  \mf M(\sJ, T) + O((T \log(T)/V)^{1/4})$.
  \item $\til \MM$ is indexed by tuples $(M, i) \in \MM \times [V]$. In particular, if $\MM$ is finite, then $\log |\til \MM| = \log |\MM| + \log V$.
  \end{enumerate}
\end{theorem}
The main consequence of this result is that characterizing the minimax sample complexity for the \MAFrameworkShort is no easier than characterizing the minimax sample complexity for the \FrameworkShort framework; this will allow us to restrict our attention to the latter task for the results that follow. Let us make some additional remarks.
\begin{itemize}
\item As we increase the parameter $V$, the approximation to the minimax rate in \pref{prop:pm-to-ma} improves. Choosing $V=\poly(T)$ suffices for all settings of interest, the only tradeoff is that the size of the model class $\cM$ increases from $\log\abs{\cM}$ to $\log\abs{\cM}+\log{}V$. For the results we consider in subsequent sections, this increase will be inconsequential (beyond $\log(T)$ factors).
\item Beyond preserving the minimax risk, both reductions preserve the value of the \CompText, which is a consequence of preserving rewards and Hellinger distances for models in the class. This will become relevant for our results in the sequel (\pref{sec:bounds}), where we show that the \CompShort is closely connected to minimax risk, yet not completely equivalent.
\item Both reductions are algorithmic in nature. For example, suppose that we start with a \FrameworkShort instance $\HRI$ and produce a \MAFrameworkShort instance $\MAI$ via the reduction in \pref{prop:pm-to-ma}. Then any algorithm that achieves low risk for every model in $\MAI$ can be efficiently lifted to an algorithm for the original class $\HRI$.
  \arxiv{\dfcomment{I wonder if we should actually make this aspect explicit in both theorems. This strengthens the story somewhat IMO.}\noah{maybe punt to after colt, these results are pretty straightforward...}\dfcomment{agree--let's punt}}
\end{itemize}
\pref{prop:pm-to-ma} is proven by embedding a given instance $\HRI$ for the \FrameworkShort framework into a two-player zero game instance $\MAI$, where the first of the two agents plays the role of the \FrameworkShort agent. The key properties of the embedding are that:
\begin{enumerate}
\item The second agent selects actions to ensure that near-optimal decisions for the original \FrameworkShort instance form Nash equilibria for the new instance, and vice-versa.
\item Even though rewards in the game instance $\MAI$ are observed, by increasing the size of the game (as a function of the parameter $V$), we can ensure that discovering an action with non-zero reward requires a prohibitively large amount of exploration, rendering them hidden (up to small approximation error).
\end{enumerate}

\section{Upper and lower bounds on minimax rates}
\label{sec:bounds}
This section presents our results regarding minimax rates for the \MAFrameworkShort and \FrameworkShort frameworks. We work in the \FrameworkShort framework for the majority of the section, and give implications for the \MAFrameworkShort at the end, using the equivalence from \pref{sec:relations}. In more detail:
\begin{itemize}
\item In \cref{sec:upper}, we give upper and lower bounds on the minimax rates for interactive decision making in the \FrameworkShort framework, which scale with the constrained DEC.
\item Next, we establish in \cref{sec:gap} that, under mild regularity assumptions on the constrained DEC, the upper and lower bounds on the minimax rate are separated by at most a polynomial factor (ignoring the estimation error term); for most parameter regimes, the gap between the bounds is at most quadratic. We then show---perhaps surprisingly---that neither the upper or lower bounds can be improved, in that there are instances where each is nearly tight. 
  In other words, in contrast to the DMSO framework \citep{foster2023tight}, in the \FrameworkShort framework, the constrained DEC cannot not give a characterization of the minimax sample complexity which is tight beyond a quadratic factor.
We show further that this gap is not limited to the constrained DEC, and in fact holds for an entire family of complexity measures based on pairwise $f$-divergences between models. %
As a result, any characterization of the minimax rate for \FrameworkShort which is tight up to polylogarithmic factors must use a complexity measure substantially different from those considered in recent works \citep{foster2021statistical,foster2022complexity,foster2023tight}. 
\item  Finally, using the equivalence shown in the previous section, we establish (\cref{sec:ma-analogues}) that all of the results above hold verbatim in the \MAFrameworkShort framework. %
\end{itemize}
All of the results in this section are presented in a general form. We refer to \cref{part:examples} of the appendix for applications to specific instances of interest.

\subsection{\FrameworkShort: Upper and lower bounds on minimax rates}
\label{sec:upper}

We now give upper and lower bounds on the minimax risk for the \hrf. We obtain upper bounds as an immediate corollary of regret bounds for the \emph{Estimation-to-Decisions$^{+}$} (\etdppac) algorithm from recent work of \citet{foster2023tight}. The \etdppac algorithm was introduced in the (single-agent/non-hidden-reward) DMSO framework, where it leads to tight upper bounds on minimax risk based on the constrained DEC \citep{foster2023tight}. We observe that it provides identical guarantees for the more general \hrf without modification; this can be seen by inspecting the proof of correctness of the \etdppac algorithm in \citet{foster2023tight} and noting that it  does not make use of the fact that the learning agent observes the rewards $r\ind{1},\ldots,r\^T$. Further background on the algorithm  may be found in \cref{sec:upper-proof}. %

Our main upper bound is stated for the case in which $\cM$ is finite ($\abs{\cM}<\infty$); more general guarantees for infinite classes are given in \cref{sec:upper-proof}.
\begin{theorem}[Minimax upper bound for \hr \citep{foster2023tight}]
  \label{thm:constrained-upper-finitem}
  Fix $\delta \in \left(0,\frac{1}{10}\right)$ and $T \in \BN$, and consider any instance $\sI = \instpm$. Suppose that \cref{ass:realizability} holds. Letting $\vepsupperT := 16 \sqrt{\frac{\lceil \log 2/\delta\rceil}{T} \cdot \log \frac{|\MM|}{\delta}}$, the \etdppac algorithm, when configured appropriately, guarantees that with probability at least $1-\delta$,
  \begin{align}
    \RiskDM  \leq %
    \deccpac[\vepsupperT](\sI)\nonumber.
  \end{align}

  In addition, if $\fm(\cdot) \in [0,R]$ for all $M \in \MM$ and some $R > 0$, then the expected risk is bounded as $\E[\RiskDM] \leq \deccpac[\vepsupperT](\sI) + \delta R$. 
\end{theorem}
Before interpreting this result, we complement it with our main lower bound, \cref{thm:constrained-lower}, which shows that the minimax risk for any algorithm is lower bounded by the constrained DEC for an appropriate choice of the scale parameter $\vep > 0$. The statement of this result uses the definition $C(T) := \log(T \wedge \abscont)$. In addition, we recall that $\gm(\pi)\ldef\fm(\pim)-\fm(\pi)$.
\begin{theorem}[Minimax lower bound for \hr]
  \label{thm:constrained-lower}
  Consider any instance $\sI = \instpm$ and write $R := \sup_{\pi \in \Pi, M \in \MM} \gm(\pi)$.  Given $T \in \bbN$, let $\vepslowerT > 0$ be chosen as large as possible such that
  \begin{align}
\vepslowerT^2 \cdot C(T) \cdot R \cdot T \leq \frac 18 \cdot {\deccpac[\vepslowerT](\sI)}.\label{eq:lower_fixed_point}
  \end{align}
  Then for any algorithm, there exists a a model in $\MM$ for which
  \begin{align}
\E[\RiskDM] \geq \frac{1}{6} \cdot \deccpac[\vepslowerT](\sI)\nonumber.
  \end{align}
\end{theorem}

\paragraph{Understanding the bounds} We now give a sense for the behavior of the lower bound of \cref{thm:constrained-lower} and the upper bound of \cref{thm:constrained-upper-finitem} through several examples. For simplicity we consider the case that $R = \sup_{\pi \in \Pi, M \in \MM} \gm(\pi) = 1$ (in the context of \cref{thm:constrained-lower}). %
\begin{itemize}
\item \emph{$\sqrt{T}$-rates.} Most of the classes studied in the literature on bandits and reinforcement learning have the property that the optimal rate is $O(\sqrt{T})$. Many of these problems have the property that rewards are observed (i.e., they lie in the DMSO framework), but such rates also arise for problems in \FrameworkShort for which rewards are not observed; a notable example is \emph{locally observable finite partial monitoring problems} \citep{bartok2014partial}. For such classes, it holds that $\deccpac[\vep](\sI) \propto \vep \cdot \sqrt{\Cprob}$, for some problem-dependent constant $\Cprob > 0$ reflecting the complexity of the model class $\MM$ (see \citet{foster2021statistical,foster2023tight} for examples).
  In this case, by choosing a failure probability of $\delta = 1/T$, we have $\vepsupperT \lesssim \sqrt{\log(T) \log(T|\MM|)/T}$, so that \cref{thm:constrained-upper-finitem} gives an upper bound of \[
    \En\brk*{\RiskDM} \leq \til O \left( \sqrt{\frac{\Cprob \log |\MM|}{T}}\right)\] on the minimax risk. For lower bounds, if $\deccpac[\vep](\sI) \propto \vep \cdot \sqrt{\Cprob}$, then the solution to the fixed point equation \pref{eq:lower_fixed_point} is $\vepslowerT \gtrsim \sqrt{\Cprob}/(T \cdot C(T))$. This translates, via \cref{thm:constrained-lower}, into a lower bound of \[
    \En\brk{\RiskDM} \geq \til\Omega \left(\frac{\Cprob}{T }\right)
  \] on the minimax risk, which differs from the upper bound by a quadratic factor (ignoring the $\log |\MM|$ factor). By the results of \cite{foster2023tight}, for the special case where rewards are observed (i.e., the DMSO framework), the upper bound of $\til O \left( \sqrt{\frac{\Cprob \log |\MM|}{T}}\right)$ is the correct rate (up to the $\log |\MM|$ factor and $\log T$ factors). We will show in the sequel that for general settings where rewards are not observed, this is not necessarily the case, and the lower bound can be tight.

\item \emph{Nonparametric rates.}  For nonparametric model classes, for which the optimal regret is $\omega(\sqrt T)$, it is typically the case that $\deccpac[\vep](\sI) \propto \vep^{1-\rho}$ for some $\rho \in (0,1)$. For such problems, \cref{thm:constrained-upper-finitem} yields an upper bound of $\En\brk*{\RiskDM}\leq\til O \left( (\log |\MM| / T)^{(1-\rho)/2}\right)$ on the minimax risk. In contrast, the best possible solution to the fixed point equation in \pref{eq:lower_fixed_point} is $\vepslowerT \gtrsim 1/(T \cdot C(T))^{\frac{1}{1+\rho}}$, which translates, via \cref{thm:constrained-lower}, into a lower bound of $\En\brk{\RiskDM}\geq\til \Omega\left(1/T^{\frac{1-\rho}{1+\rho}}\right)$ on the minimax risk. Here the lower bound is off from the  upper bound (ignoring the $\log |\MM|$ factor) by a power of $\frac{2}{1+\rho} \leq 2$. By the results of \cite{foster2023tight}, for the special case where rewards are observed, the upper bound of $\til O \left( (\log |\MM| / T)^{(1-\rho)/2}\right)$ is the correct rate (up to the $\log |\MM|$ factor and $\log T$ factors).
\end{itemize}
We refer to \citet{foster2023tight} for concrete examples exhibiting the growth rates sketched above for the special case where rewards are observed (DMSO), and to \pref{part:examples} of the appendix for examples arising from \ma.

\subsection{\FrameworkShort: Gaps between bounds and impossibility of tight characterizations}
\label{sec:gap}

We now investigate the nature of the gap between the upper and lower bounds in \cref{thm:constrained-upper-finitem,thm:constrained-lower}. We first give a generic bound on the gap, then show that it is not possible---in a fairly strong sense---to close the gap further.

\subsubsection{On the gap between the upper and lower bounds}
Ignoring constant factors, the only difference between the upper and lower bounds of \cref{thm:constrained-upper-finitem,thm:constrained-lower} is the scale $\vep$ at which the DEC is computed. The upper bound of \cref{thm:constrained-upper-finitem} uses scale $\vepsupperT = 8 \sqrt{\frac{\lceil \log 2/\delta\rceil}{T} \cdot \log\abs{\cM}}$, whereas the lower bound of \cref{thm:constrained-lower} (with $R=1$) uses the scale $\vepslowerT$, which is defined implicitly to be as large as possible subject to the constraint $\vepslowerT^2 \cdot C(T) \cdot T \leq \frac{1}{8} \cdot \deccpac[\vepslowerT](\sI)$. Thus, the size of the gap between $\vepsupperT$ and $\vepslowerT$ controls the degree of tightness of these upper and lower bounds. In what follows, we give a bound on the size of this gap that holds whenever the constrained DEC satisfies the following regularity assumption.
\begin{assumption}[Regularity]
  \label{ass:regularity}
  An instance $\sI$ (of either \FrameworkShort or \MAFrameworkShort) is said to satisfy the regularity condition with constants $\Creg, \creg > 1$ at scale $\vep \in (0,2)$ if
  \begin{align}
\deccpac[\vep](\sI) \leq \creg^2 \cdot \deccpac[\vep/\Creg](\sI)\nonumber.
  \end{align}
\end{assumption}
Most natural classes satisfy \cref{ass:regularity} for some constants $\creg, \Creg$ (in particular, the condition is satisfied whenever $\deccpac(\sI)\propto\veps^{p}$ for $p<2$). We note that a similar assumption used in \cite{foster2023tight} to give upper bounds on the optimal rates attainable in the DMSO framework.

Under \cref{ass:regularity}, the following result shows that our upper bound on minimax risk, which scales with $\deccpac[\vepsupperT](\sI)$, is bounded above by a quantity that is a polynomial of our lower bound, namely $\deccpac[\vepslowerT](\sI)$. 
\begin{proposition}
  \label{prop:gap-bounding}
  Suppose that an instance $\sI$ (for either \FrameworkShort or \MAFrameworkShort) satisfies \cref{ass:regularity} for some values $\Creg > \creg > 1$ and for all $\vep \in (\vepslowerT \cdot \frac{\creg}{\Creg}, \vepsupperT)$. Choose any $\beta \geq \frac{\log \creg}{\log (\Creg/\creg)}$. Then for any $T \in \BN$, %
  \begin{align}
    \deccpac[\vepsupperT](\sI)\leq \left( C \log 1/\delta \cdot \log\abs{\cM} \cdot C(T) \cdot \Creg/\creg\right)^{\frac{\beta}{1+\beta}} \cdot \deccpac[\vepslowerT](\sI)^{\frac{1}{1+\beta}}\nonumber.
    \end{align}
  \end{proposition}
We remark that \cref{prop:gap-bounding} is a purely algebraic fact that makes no use of the structure of the DEC, and in particular holds for instances of both the \FrameworkShort and \MAFrameworkShort frameworks. To make the result concrete, we consider, we revisit each of the situations we discussed in \cref{sec:upper}, and describe how applying \cref{prop:gap-bounding} allows us to conclude that our upper and lower bounds are related by a polynomial factor.

\begin{itemize}
\item \emph{$\sqrt{T}$-rates.}  Suppose that $\deccpac[\vep](\sI) \propto \vep \cdot \sqrt{\Cprob}$, for some problem-dependent constant $\Cprob > 0$. Then, for any constant $\beta > 1$, there is a sufficiently large absolute constant $\Creg > 1$ so that, for all $\vep > 0$, $\deccpac[\vep](\sI) \leq \Creg^{\beta} \cdot \deccpac[\vep/\Creg](\sI)$. It follows that \cref{ass:regularity} is satisfied with the constants $\Creg$ and $\creg := \Creg^{\beta/2}$ (which satisfy $\beta \geq \frac{\log \creg}{\log(\Creg/\creg)}$), and \cref{prop:gap-bounding} gives that
  \begin{align}
\deccpac[\vepsupperT](\sI) \leq \til O(\log\abs{\cM})^{\frac{\beta}{1+\beta}} \cdot \deccpac[\vepslowerT](\sI)^{\frac{1}{1+\beta}}\nonumber.
  \end{align}
  Disregarding the estimation error and taking $\beta \ra 1$, we conclude that $\deccpac[\vepsupperT](\sI) \lesssim \deccpac[\vepslowerT](\sI)^{1/2 - o(1)}$, i.e., there is a (roughly) quadratic gap between our upper and lower bounds.

\item \emph{Nonparametric rates.} %
  Suppose that $\deccpac[\vep](\sI) \propto \vep^{1-\rho}$ for some $\rho \in (0,1)$. Then for any constant $\beta > \frac{1-\rho}{1+\rho}$, there is a sufficiently large constant $\Creg > 1$ so that, for all $\vep > 0$, $\deccpac[\vep](\sI) \leq \Creg^{\beta  (1+\rho)} \cdot \deccpac[\vep/\Creg](\sI)$. Thus, \cref{ass:regularity} is satisfied with the constants $\Creg$ and $\creg := \Creg^{\beta(1+\rho)/2}$, which satisfy $\beta \geq \frac{\log \creg}{\log(\Creg/\creg)}$, and \cref{prop:gap-bounding} gives that
    \begin{align}
\deccpac[\vepsupperT](\sI) \leq \til O(\log\abs{\cM})^{\frac{\beta}{1+\beta}} \cdot \deccpac[\vepslowerT](\sI)^{\frac{1}{1+\beta}}\nonumber.
    \end{align}
    Disregarding the estimation error and taking $\beta \ra \frac{1-\rho}{1+\rho}$ (so that $\frac{1}{1+\beta} \ra \frac{1+\rho}{2}$), we conclude that $\deccpac[\vepsupperT](\sI) \lesssim \deccpac[\vepslowerT](\sI)^{\frac{1+\rho}{2} - o(1)}$, i.e., the gap between the upper and lower bounds is smaller than quadratic. 
  \end{itemize}
  Of course, the arguments in \cref{sec:upper} already allowed us to draw these conclusions directly; the purpose here is to exhibit how this conclusion can obtained as a special case of the more general \cref{prop:gap-bounding}.

\subsubsection{On tight characterizations for the minimax risk}
\label{sec:optimality}

It is natural to wonder whether the polynomial gap between our upper and lower bounds can be tightened to give a characterization of the minimax risk up that is only loose by polylogarithmic factors. In this section, we show that this is not possible in several senses.

\paragraph{Tightness of the upper and lower bounds}
In \cref{prop:gap-ub-easy,prop:gap-inherent}, we give two instances $\sI_1$ $\sI_2$, so that, up to $\log \frac{1}{\vep}$ factors, we have both $\deccpac[\vep](\sI_1) \asymp \vep$ and $ \deccpac[\vep](\sI_2) \asymp \vep$. Despite having the same behavior for the \CompShort, the minimax rates for the instances are different: For the instance $\sI_1$, the upper bound from \cref{thm:constrained-upper-finitem} is tight ($\mf M(\sI_1, T) \gtrsim 1/\sqrt{T}$), yet for $\sI_2$, the lower bound from \cref{thm:constrained-lower} is tight ($\mf M(\sI_2, T)\lesssim\log(T)/T$).

\begin{proposition}[An instance where the upper bound is tight]
  \label{prop:gap-ub-easy}
  For any sufficiently $L, A \in \bbN$, there is an instance $\sI_1 = \instpm$ with $\log |\MM| \leq \log(LA)$ and which satisfies the following properties:
  \begin{enumerate}
  \item For all $T \leq 2^{L/2}$, the minimax rate for $\sI_1$ is given by $ \mf M(\sI_1, T) = \Theta(\sqrt{A/T})$.
  \item For all $\vep \in (2^{-L},1/\sqrt{A})$, it holds that $c \cdot \vep \sqrt{A} \leq \deccpac(\sI_1) \leq C \cdot \vep \sqrt{A}$, for some constants $c, C > 0$. 
  \end{enumerate}
\end{proposition}
The instance $\sI_1$ in \cref{prop:gap-ub-easy} has model class given by a subclass of multi-armed bandit problems with $A$ arms and Bernoulli rewards, and the bounds in the proposition are an immediate consequence of prior work. We provide a proof in \cref{sec:gap} for completeness. 

\begin{proposition}[An instance where the lower bound is tight]
  \label{prop:gap-inherent}
  For any sufficiently large $L \in \BN$ and any $\Cscale \geq 1$, there exists an instance $\sI_2 = \instpm$ with $\log |\MM| \leq L^2$, satisfying the following properties:
  \begin{enumerate}
  \item For all $T \leq 2^L$, the minimax rate for the instance $\sI_2$ is bounded as
$ 
\mf M(\sI_2, T) \leq \frac{8 \Cscale^2 \log T}{T}.
$
  \item For all $\vep \geq \frac{\sqrt{2}}{\Cscale \cdot 2^L}$, we have
$ 
\frac{\Cscale}{\sqrt{8} \cdot L} \cdot \vep \leq \deccpac[\vep](\sI_2) \leq 2 \Cscale \cdot \vep.
$
    In particular, $\vepslowerT \geq \Omega \left( \frac{\Cprob}{T \log(T) \cdot L} \right)$ as long as $T \leq 2^L / L^3$.
  \end{enumerate}
  In particular, for any $T \in \bbN$, by choosing $L = 100\log T$, we have that for all $\vep \geq \Omega\left( \frac{1}{\Cprob \cdot T^{100}}\right)$, the instance $\sI_2$  satisfies, $\Omega(\vep \cdot \Cprob/\log \frac{1}{\vep}) \leq \deccpac(\sI_2) \leq O(\vep \cdot \Cprob)$, yet the minimax risk is bounded as $\mf M(\sI_2, T) \leq O(\Cprob^2 \log(T)/T)$.
\end{proposition}
Let us compare the instances for \cref{prop:gap-ub-easy} and \cref{prop:gap-inherent}. First, note that for both instances, the  estimation complexity $\log\abs{\cM}$ scales as $\bigoht(1)$. Thus:
\begin{itemize}
\item \cref{thm:constrained-upper-finitem}, using the radius $\vepsupperT$, yields an upper bound on the minimax risk of $\til O(1/\sqrt T)$, which is tight for $\sI_1$.
\item \cref{thm:constrained-lower}, using the radius $\vepslowerT$, yields a lower bound on the minimax risk of $\til \Omega(1/T)$, which is tight for $\sI_2$.
\end{itemize}
That is, the instance $\sI_1$ establishes that our upper bound cannot be improved to use the radius $\vepslowerT$, and the instance $\sI_2$ establishes that our lower bound cannot be improved to use the radius $\vepsupperT$. More generally, since $\deccpac(\sI_1)$ and $\deccpac(\sI_2)$ have the same behavior, yet $\sI_1$ and $\sI_2$ have different minimax rates, the constrained DEC cannot give a tight characterization of the minimax risk for the \FrameworkShort framework. This contrasts the situation for the (reward-observed) DMSO framework in \citet{foster2023tight}, where the constrained \CompShort characterizes the minimax rates up to logarithmic factors whenever $\log\abs{\cM}=\bigoht(1)$.

We remark in passing that the instances constructed in \cref{prop:gap-ub-easy,prop:gap-inherent} satisfy the regularity condition of \cref{ass:regularity} for $\Creg, \creg \leq  O(\log T)$ and all $\vep \geq \vepslowerT$. Thus, the regularity condition is not sufficient to close the gap between the upper and lower bounds.

\paragraph{Ruling out more general characterizations}
We now show that the gaps highlighted above are not limited to the \CompShort, and are in fact intrinsic to a broad class of complexity measures.
Our main result, \cref{prop:fdiv-separation} shows that for any $f$-divergence $\Dgen{\cdot}{\cdot}$ satisfying a mild assumption, it is possible to construct two \FrameworkShort instances $\sI_1$ and $\sI_2$ for which the minimax risk differs by a polynomial factor, yet 1)  the value functions associated with $\sI_1$ and $\sI_2$ are identical, and 2) the pairwise $\Dgen{\cdot}{\cdot}$-divergences between all models in $\sI_1$ and $\sI_2$ are identical. In other words:
\begin{center}
\emph{It is impossible to obtain a tight characterization for minimax risk that depends only on value functions and pairwise $f$-divergences.}
\end{center}
\cref{def:fdiv-bounded} gives our main technical assumption regarding $f$-divergences: roughly speaking, it states that the function defining the $f$-divergence exhibits at most polynomial growth near $0$ and $\infty$.
\begin{definition}[Bounded $f$-divergence]
  \label{def:fdiv-bounded}
  Consider a convex function $\phi : [0,\infty) \ra [0, \infty]$ so that $\phi(1) = 0$ and $\phi(x)$ is finite for all $x > 0$, \dfedit{
    and let 
    \begin{align}
      \label{eq:fdiv}
      \Dphi{\bbP}{\bbQ}\ldef{} \En_{\bbQ}\brk*{\phi\prn*{\frac{d\bbP}{d\bbQ}}}
    \end{align}
    denote the associated $f$-divergence for probability measures $\bbP$ and $\bbQ$ with $\bbP\ll\bbQ$.
    }
  For constants $\alpha, \beta \geq 0$, we say that $\phi$ is \emph{$(\alpha, \beta)$-bounded} if, for all $x \geq 1$,
  \begin{align}
\phi(1/x) + \frac{\phi(x)}{x} \leq \beta \cdot x^\alpha\nonumber.
  \end{align}
  In such a case, we say that the $f$-divergence $\Dphishort$ is $(\alpha, \beta)$-bounded.
\end{definition}
Essentially all commonly used $f$-divergences satisfy \cref{def:fdiv-bounded} for small values of $\alpha$ and $\beta$.  For the Hellinger divergence, we have $\phi(x) = (\sqrt{x} - 1)^2$, so that $\hell{\cdot}{\cdot}$ is $(0, 2)$-bounded; for the KL-divergence, we have $\phi(x) = x \ln(x) + 1-x$, so that $\kld{\cdot}{\cdot}$ is $(0,2)$-bounded; and for the $\chi^2$-divergence, we have $\phi(x) = (x-1)^2$, so that $\Dchis{\cdot}{\cdot}$ is  $(1,1)$-bounded. 

\begin{remark}[Non-negativity of $\phi$]
  We remark that often, when $f$-divergences are presented, it is assumed that the function $\phi$ maps to $[-\infty, \infty]$ (as opposed to $[0,\infty]$). Assuming that $\phi$ maps to $[0,\infty]$ is without loss of generality, for the following reason. It is well-known that for any $c \in \bbR$, and for any convex function $\phi$ satisfying $\phi(1) = 0$, letting $\til \phi(x) = \phi(x) + c \cdot (x-1)$, we have $\Dphishort = D_{\til \phi}$. Thus, given any $\phi : [0,\infty) \ra [-\infty, \infty]$, we may choose any $c \in -\partial \phi(1)$, so that $0 \in \partial(\phi(x) + c\cdot (x-1))$,  which in particular implies that $\phi(x) + c\cdot (x-1) \geq 0$ for all $x$, and the $f$-divergence induced by $\phi(x) + c\cdot (x-1)$ is equivalent to $\Dphishort$.
\end{remark}

\begin{theorem}
  \label{prop:fdiv-separation}
  For some constants $\alpha, \beta \geq 0$, suppose $\Dphishort$ is an $(\alpha,\beta)$-bounded $f$-divergence. Then for any $T \in \BN$, $\ep\in(0,1)$, and $\Cprob \geq 1$, there are instances $\sI_1 = (\MM_1, \Pi_1, \MO_1, \{ \fm_1(\cdot) \}\subs{M \in \MM_1})$, $\sI_2 = (\MM_2, \Pi_2, \MO_2, \{ \fm_2(\cdot) \}\subs{M \in \MM_2})$ of the \FrameworkShort framework, so that $\Pi_1 = \Pi_2,\ \MO_1 = \MO_2$, and there is a one-to-one mapping $\sE : \MM_1 \ra \MM_2$ satisfying:
  \begin{enumerate}
  \item \label{it:fm-equal} For all $M \in \MM_1$, $\fm_1 \equiv f\sups{\sE(M)}_2$.
  \item \label{it:dphi-equal} For all $M, M' \in \MM_1$, and $\pi \in \Pi_1$, $\Dphi{M(\pi)}{M'(\pi)} = \Dphi{\sE(M)(\pi)}{\sE(M')(\pi)}$.
  \item \label{it:minimax-separation} There is some constant $C_\phi$ depending only on $\phi$ so that for all $T'$ with $T \leq T' \leq T^{3/2 - 2\ep} \cdot (C_\phi \Cprob^{1/2 + \ep} \ln T)^{-1}$, it holds that %
    \begin{align}
\mf M(\sI_1, T') \leq  \frac 1T + 2 \cdot \left( \frac{\Cprob}{T} \right)^{1/2 + \ep/(2\alpha)}, \qquad \mf M(\sI_2, T') \geq 2^{-2-2/\ep} \cdot \left( \frac{\Cprob}{T} \right)^{1/2} \nonumber.
    \end{align}
  \end{enumerate}
\end{theorem}

In the event that $\alpha = 0$, the quantity $(\Cprob/T)^{1/2 + \ep/(2\alpha)}$ in the statement of \cref{prop:fdiv-separation} is to be interpreted as 0. 
In particular, if  $\Dgen{\cdot}{\cdot}$ is the Hellinger divergence or the KL divergence, then we have $\mf M(\sI_1, T') \leq 1/T$ in \cref{it:minimax-separation}, giving a quadratic separation. If $\Dgen{\cdot}{\cdot}$ is the $\chi^2$-divergence, then we have $\mf M(\sI_1, T') \leq O(1/T^{1/2 + \ep/2})$, which leads to a smaller, yet still polynomial separation for any choice of the constant $\ep > 0$. 

Several variants of the DEC and related complexity measures depend only on the value functions $\fm(\cdot)$ (for $M \in \MM$) and pairwise $f$-divergences between models in the class $\MM$, and thus cannot provide a characterization for minimax risk in the \FrameworkShort framework that is tight up to polylogarithmic factors. Below, we highlight a few notable examples.
\begin{itemize}
\item The distributional offset DEC \citep{foster2021statistical,chen2022unified,foster2023tight}, is defined for $\sI = \instpm$ as:\footnote{We consider the PAC variant of the offset DEC here \citep{foster2023tight}, but it is clear that our argument applies identically to the regret version of the DEC \citep{foster2021statistical}.}
  \begin{align}
\decopacr(\sI)=
 \sup_{\nu \in \Delta(\MM)} \inf_{p,q\in\Delta(\Pi)}\sup_{M\in\cM}\En_{\pi\sim{}p}\brk*{
  \fm(\pim) - \fm(\pi)}
  - \gamma \cdot \E_{\pi \sim q} \brk*{ \En_{\Mbar\sim\nu}\brk*{\Dhels{M(\pi)}{\Mbar(\pi)}}
  }\nonumber.
  \end{align}
Clearly, this definition depends only on value functions $\crl{\fm}_{M\in\cM}$ and pairwise Hellinger distances for models in $\cM$, and hence can only characterize minimax risk up to a quadratic factor.
  
\item The offset DEC \citep{foster2021statistical,foster2023tight} is defined for $\sI = \instpm$ as:
  \begin{align}
\decopac(\sI) = \sup_{\Mbar \in \co(\MM)} \inf_{p,q \in \Delta(\Pi)} \sup_{M \in \MM} \E_{\pi \sim p} [\fm(\pim) - \fm(\pi)] - \gamma \cdot \E_{\pi \sim q}[\hell{M(\pi)}{\Mbar(\pi)}]\nonumber.
  \end{align}
  Note that $\decopac(\sI)$ depends on the divergence between models in $\MM$ and those in $\co(\MM)$, which is not covered by \cref{prop:fdiv-separation}. However, \citet[Proposition D.2]{foster2023tight} show that $\decopac(\sI) \leq \decopacr[\gamma/4](\sI) \leq \decopac[\gamma/4](\sI)$ (that is, $\decopac(\sI)$ and $\decopacr(\sI)$ are equivalent up to constant factors), so it follows from the previous bullet point that this complexity measure can only characterize minimax risk up to a quadratic factor.
\item \citet{foster2021statistical,foster2023tight} consider variants of the \CompShort that are applied to \emph{localized} subsets of the model class $\cM$. In particular, the following two notions of localization have been considered in \cite{foster2021statistical}: for some localization radius $\alpha > 0$, a model class $\MM$, and a reference model $\Mbar$, 
  \begin{align}
    \MM_\alpha(\Mbar) :=&  \{ M \in \MM \ : \ \fm(\pim) \leq \fmbar(\pimbar) + \alpha \},\mathand \nonumber\\
    \MM_\alpha^\infty(\Mbar) :=&  \left\{ M \in \MM \ : \ \left|(\fm(\pim) - \fm(\pi)) - (\fmbar(\pimbar) - \fmbar(\pi))\right| \leq \alpha  \ \forall \pi \in \Pi \right\}\nonumber.
  \end{align}
  Since these definitions only depend on the value functions $\crl{\fm}_{M\in\cM}$, \cref{prop:fdiv-separation} implies that incorporating localization into the variants of the DEC considered above cannot help to provide a characterization of the minimax risk.
\item The \emph{information ratio} \citep{russo2014learning,russo2018learning,lattimore2021mirror} was introduced to bound the Bayesian regret for posterior sampling and a more general algorithm known as \emph{information-directed sampling}. The information ratio of a model class $\MM$ is closely related to the DEC of the convex hull of $\MM$; in particular, \cite{foster2022complexity} showed that a parametrized version of the information ratio of $\MM$ is equivalent to the DEC of the convex hull of $\MM$, up to constant factors. As the DEC of $\co(\MM)$ involves pairwise Hellinger distances between models in the \emph{convex hull of $\MM$}, \cref{prop:fdiv-separation} does not definitively rule it out as providing a characterization of minimax risk. However, the DEC of $\co(\MM)$ is known to be exponentially larger than the minimax risk for many natural examples (e.g., tabular reinforcement learning \citep{foster2022complexity}), so it seems unlikely to provide a tight characterization.

  There are also variants of the information ratio which \cref{prop:fdiv-separation} \emph{does} rule out: given a reference model $\Mbar \in \MM$ and a distribution $\mu \in \Delta(\MM)$, one can define \citep{foster2021statistical}
  \begin{align}
\mathcal{I}(\sI, \Mbar, \mu) := \argmin_{p,q \in \Delta(\Pi)} \frac{(\E_{\pi \sim p} \E_{M \sim \mu}[\fm(\pim) - \fm(\pi)])^2}{\E_{\pi \sim q} \E_{M \sim \mu} [\kld{M(\pi)}{\Mbar(\pi)}]}\nonumber.
  \end{align}
  As this definition depends only on value functions and pairwise KL-divergences for models in $\cM$, \cref{prop:fdiv-separation}, no function of $\mathcal{I}(\sI, \Mbar, \mu)$ (such as a worst-case version of the information ratio defined by $\max_{\Mbar \in \MM} \max_{\mu \in \Delta(\MM)} \mathcal{I}(\sI, \Mbar, \mu)$) can provide a characterization of minimax risk.%
\item Note that in general, the constrained \CompShort  $\deccpac(\sI) = \sup_{\Mbar \in \co(\MM)} \deccpac(\sI, \Mbar)$ depends on Hellinger divergences between models in $\MM$ and those in $\co(\MM)$, so \cref{prop:fdiv-separation} does not directly rule out a characterization in terms of $\deccpac(\sI)$. However, we have already ruled out such a characterization separately in \cref{prop:gap-ub-easy,prop:gap-inherent}. Of course, the variant $\sup_{\Mbar \in \MM} \deccpac[\vep](\sI, \Mbar)$, which restricts to $\Mbar \in \MM$, only depends on the value functions and pairwise Hellinger divergences of models in $\MM$, and hence is covered by \cref{prop:fdiv-separation}.
\end{itemize}
Let us remark that one complexity measure not currently ruled out by our results is the generalized information ratio considered in the work of \citet{lattimore2022minimax} on adversarial partial monitoring, which uses an unnormalized KL-like divergence based on the logarithmic barrier, and cannot be written in terms of $f$-divergences. The upper and lower bounds on regret given by \citet{lattimore2022minimax} are loose by $\poly(\abs{\Pi})$ factors, and as such we find it to be unlikely that this complexity measure can give tight guarantees in the ``large decision-space/model class'' regime where $T \ll \min \{ |\MM|, |\Pi|\}$, which is the focus of our work.

\begin{remark}
  While this is out of scope for the present paper, we remark that it is possible to establish similar impossibility results for the regret (as opposed to PAC) framework.
\end{remark}

\subsection{Implications for \MAFrameworkShort framework}
\label{sec:ma-analogues}
Up to this point, all of the results in this section concerned the \FrameworkShort framework. Using \cref{prop:ma-to-pm,prop:pm-to-ma}, we can immediately derive analogous results for the \MAFrameworkShort framework. In what follows, we state these analogues (in particular, upper and lower bounds on minimax risk, and impossibility of tighter results), all of which are corollaries the results in the prequel. We refer to \cref{part:examples} of the appendix for applications of these results.

\paragraph{Upper and lower bounds on minimax risk} We begin by stating upper and lower bounds for the minimax risk for instance of \MAFrameworkShort in terms of the Multi-Agent DEC; these results are corollaries of \cref{thm:constrained-upper-finitem,thm:constrained-lower}.
\begin{corollary}[Minimax upper bound for \MAFrameworkShort]
  \label{cor:ma-constrained-upper}
  Fix $\delta \in \left(0,\frac{1}{10}\right)$ and $T \in \BN$, and consider any $K$-player \ma instance $\sJ = \instma$. Suppose that %
  $\fm_k(\cdot) \in [0,1]$ for all $k \in [K]$ and $M \in \MM$, and let $\vepsupperT := 16 \sqrt{\frac{\lceil \log 2/\delta \rceil}{T} \cdot \frac{\log |\MM|}{\delta}}$. Then we have
  \begin{align}
\mf M(\sJ, T) \leq \deccpac[\vepsupperT](\sJ ) + K \delta\label{eq:multiagent-dec-ub}.
  \end{align}
\end{corollary}
\begin{proof}[\pfref{cor:ma-constrained-upper}]
  Given an instance $\sJ$ of \MAFrameworkShort, consider the instance $\sI = (\MM, \Pi, \MO, \{ \til{f}\sups{M} \}\subs{M})$ as per \cref{prop:ma-to-pm}. We have that $\hm(\cdot) \in [0,K]$ for all $M \in \MM$, meaning that $\til{f}\sups{M}(\cdot) \in [-K+1,1]$ for all $M \in \MM$ under the construction in the proof of \cref{prop:ma-to-pm}. By rescaling $\til{f}\sups{M}(\cdot)$, the guarantee from \cref{thm:constrained-upper-finitem} ensures that $\mf M(\sI, T) \leq \deccpac[\vepsupperT](\sI) + K \delta$, from which \cref{eq:multiagent-dec-ub} follows using \cref{prop:ma-to-pm}. We have also used here that both $\mf M(\sI, T)$ and  $\deccpac[\vep](\sI)$ scale linearly under rescaling of the value functions $\til{f}\sups{M}(\cdot)$.
\end{proof}

As we discuss further in \cref{rem:ma-hr-rescaling}, the high-probability guarantee from \cref{thm:constrained-upper} applies also in the \MAFrameworkShort setting, i.e., in the contex of \cref{cor:ma-constrained-upper}.

\begin{corollary}[Minimax lower bound for \MAFrameworkShort]
  \label{cor:ma-constrained-lower}
  Consider any instance $\sJ = \instma$ for the \maf with $\MR = [0,1]$.  Given $T \in \bbN$, let $\vepslowerT > 0$ be chosen as large as possible such that $\vepslowerT^2 \cdot C(T) \cdot K \cdot T \leq \frac 18 \cdot {\deccpac[\vepslowerT](\sJ)}$. Then
  \begin{align}
\mf M(\sJ, T) \geq \frac{1}{6} \cdot \deccpac[\vepslowerT](\sJ).\nonumber
  \end{align}
\end{corollary}
\begin{proof}[\cref{cor:ma-constrained-lower}]
Given an instance $\sJ$ of \MAFrameworkShort, consider the instance $\sI = (\MM, \Pi, \MO, \{ \til{f}\sups{M} \}\subs{M})$ as per \cref{prop:ma-to-pm}. By definition of $\til{f}\sups{M}$, we have that $\sup_{\pi \in \Pi, M \in \MM} \sup_{\pi' \in \Pi} \til f\sups{M}(\pi') - \til f\sups{M}(\pi) \leq K$. Then we have $\mf M(\sJ, T) = \mf M(\sI, T) \geq \frac{1}{6} \cdot \deccpac[\vepslowerT](\sI) = \frac{1}{6} \cdot \deccpac[\vepslowerT](\sJ)$, where the two equalities use \cref{prop:ma-to-pm} and the inequality uses \cref{thm:constrained-lower}. 
\end{proof}

As we have already remarked, \cref{prop:gap-bounding}, which bounds the gap between our upper and lower bounds based on the \CompShort, already applies to instances of \MAFrameworkShort whenever \cref{ass:regularity} is satisfied. In particular, this means that whenever $\deccpac[\veps](\MAI)\propto\veps^{1-\rho}$ for $\rho\in[0,1)$, we have
    \begin{align}
\deccpac[\vepsupperT](\MAI) \leq \til O(\dfedit{K^{\frac{1-\rho}{1+\rho}}}\log^{\frac{1-\rho}{2}}\abs{\cM}) \cdot \deccpac[\vepslowerT](\MAI)^{\frac{1+\rho}{2}}\nonumber.
    \end{align}

\paragraph{Tightness of the gaps}
Next, we provide analogues of \cref{prop:gap-ub-easy,prop:gap-inherent} for the \MAFrameworkShort. The results construct \ma instances $\sJ_1$ (\cref{prop:gap-ub-easy-ma}) and $\sJ_2$ (\cref{prop:gap-inherent-ma}) that exhibit the same DEC behavior, in that $\deccpac(\sJ_1) \asymp \vep$ and $ \deccpac(\sJ_2) \asymp \vep$, yet have minimax rates: $\mf M(\sJ_1, T) \gtrsim 1/\sqrt{T}$ and $\mf M(\sJ_2, T) \lesssim \log(T)/T$. In particular, \cref{prop:gap-ub-easy-ma} below shows that in the upper bound \cref{cor:ma-constrained-upper}, the scale $\vepsupperT$ cannot be decreased, and \cref{prop:gap-inherent-ma} below shows that in the lower bound \cref{cor:ma-constrained-lower}, the scale $\vepsupperT$ cannot be increased.
\begin{proposition}
  \label{prop:gap-ub-easy-ma}
  For any sufficiently large $L, A \in \bbN$, there is an instance $\sJ_1 = \instma$ with $\log |\MM| \leq \log(LA)$ and which satisfies the following properties:
  \begin{enumerate}
  \item For all $T \leq 2^{L/2}$, the minimax rate for the instance $\sJ_1$ is given by $\mf M(\sJ_1, T) = \Theta(\sqrt{A/T})$.
  \item For all $\vep \in (2^{-L}, 1/\sqrt{A})$, it holds that $c \cdot \vep \sqrt{A} \leq \deccpac(\sJ_1) \leq C \cdot \vep \sqrt{A}$, for some constants $c, C > 0$. 
  \end{enumerate}
\end{proposition}
\begin{proof}[\pfref{prop:gap-ub-easy-ma}]
We observe that the instance $\sI = \instpm$ used to prove \cref{prop:gap-ub-easy} immediately yields the 1-player instance of \MAFrameworkShort given by $\sJ_1 = (\MM, \Pi, \MO,\Dev[1], \Sw[1])$, with $\Dev[1] = \Pi$ and $\Sw[1](\dev[1], \pi) = \dev[1]$, since rewards are observed under all models in $\MM$. The result then follows immediately from \cref{prop:gap-ub-easy}.
\end{proof}

\begin{proposition}
  \label{prop:gap-inherent-ma}
  For any sufficiently large $L \in \BN$ and any $\Cscale \geq 1$, there exists an instance $\sJ_2 = \instma$ with $\log |\MM| \leq O(L^2 + \log \Cprob)$, satisfying the following properties:
  \begin{enumerate}
  \item For all $T \leq 2^L$, the minimax rate for the instance $\sJ_2$ is bounded as
$ 
\mf M(\sJ_2, T) \leq \frac{8 \Cscale^2 \log T}{T}.
$
  \item For all $\vep \geq \frac{\sqrt{2}}{\Cscale \cdot 2^L}$, we have
$ 
\frac{\Cscale}{\sqrt{8} \cdot L} \cdot \vep \leq \deccpac[\vep](\sJ_2) \leq 2 \Cscale \cdot \vep.
$
    In particular, $\vepslowerT \geq \Omega \left( \frac{\Cprob}{T \log(T) \cdot L} \right)$ as long as $T \leq 2^L / L^3$.
  \end{enumerate}
\end{proposition}
\begin{proof}[\pfref{prop:gap-inherent-ma}]
  Given $L$ and $\Cprob$, let $\sI = \instpm$ be the instance given per \cref{prop:gap-inherent}. Next, let $\sJ_2 = (\til \MM, \til \Pi, \til \MO, \{ \Dev\}_k, \{ \Sw \}_k)$ be the instance constructed per \cref{prop:pm-to-ma} for the instance $\sI$ with $V = 100 \cdot \Cprob^2 \cdot 2^{2L}$. We have $\log |\til \MM| = \log |\MM| + \log V \leq O(L^2 + \log \Cprob)$. Using the guarantees of \cref{prop:gap-inherent} and \cref{prop:pm-to-ma}, we have $\mf M(\sJ_2, T) \leq \mf M(\sI, T) \leq \frac{8\Cprob^2 \log T}{T}$ for $T \leq 2^L$, and for all $\vep \geq \frac{2}{\Cprob \cdot 2^L}$ (which ensures that $\vep - \sqrt{6/V} \geq \frac{\sqrt 2}{\Cprob \cdot 2^L}$),
  \begin{align}
\frac{\Cprob}{\sqrt{8} \cdot L} \cdot \left(\vep - \sqrt{6/V}\right)\leq \deccpac[\vep - (6/V)^{-1/2}](\sI) - 6/\sqrt{V} \leq \deccpac[\vep](\sJ_2) \leq \deccpac[\vep](\sI) \leq 2\Cprob \cdot \vep\nonumber.
  \end{align}
  Since $\vep \geq \frac{1}{\Cprob \cdot 2^L}$ implies that $\sqrt{6/V} \leq \vep/2$, it follows that $\frac{\Cprob}{2\sqrt{8}L} \cdot \vep \leq \deccpac[\vep](\sJ_2) \leq 2\Cprob \cdot \vep$. 
\end{proof}

\paragraph{Ruling out more general characterizations} Finally, we state an analogue of \cref{prop:fdiv-separation} for the \MAFrameworkShort framework, which shows that any complexity measure that dependence on the instance $\MAI$ only through value functions and pairwise $f$-divergences can only characterize the minimax risk up to polynomial factors.  %
\begin{theorem}
  \label{thm:ma-fdiv-separation}
  For some constants $\alpha, \beta \geq 0$, suppose that $\Dphishort$ is an $(\alpha,\beta)$-bounded $f$-divergence (\cref{def:fdiv-bounded}). Then for any $T \in \BN$, $\ep > 0$, and $\Cprob \geq 1$, there are instances $\sJ_1 = (\MM_1, \Pi, \MO, \{ \Dev\}_\ag, \{ \Sw \}_\ag)$, $\sJ_2 = (\MM_2, \Pi, \MO, \{ \Dev\}_\ag, \{ \Sw \}_\ag)$ of the \MAFrameworkShort framework, so that there is a one-to-one mapping $\sE : \MM_1 \ra \MM_2$ satisfying:
  \begin{enumerate}
  \item \label{it:fm-equal-ma} For all $M \in \MM_1$, $\fm_1 \equiv f\sups{\sE(M)}_2$.
  \item \label{it:dphi-equal-ma} For all $M, M' \in \MM_1$, and $\pi \in \Pi$, $\Dphi{M(\pi)}{M'(\pi)} = \Dphi{\sE(M)(\pi)}{\sE(M')(\pi)}$.
  \item \label{it:minimax-separation-ma} There is some constant $C_\phi$ depending only on $\phi$ so that for all $T'$ with $T \leq T' \leq T^{3/2 - 2\ep} \cdot (C_\phi \Cprob^{1/2 + \ep} \ln T)^{-1}$, it holds that %
    \begin{align}
\mf M(\sJ_1, T') \leq  \frac 1T + 2 \cdot \left( \frac{\Cprob}{T} \right)^{1/2 + \ep/(2\alpha)}, \quad\text{yet}\quad \mf M(\sJ_2, T') \geq 2^{-3-2/\ep} \cdot \left( \frac{\Cprob}{T} \right)^{1/2} \nonumber.
    \end{align}
  \end{enumerate}
\end{theorem}
The proof uses the equivalence of \cref{prop:pm-to-ma} to translate the construction of \FrameworkShort instances in \cref{prop:fdiv-separation} to the \MAFrameworkShort framework. Since \cref{prop:fdiv-separation} makes a claim about pairwise $f$-divergences as opposed to the constrained DEC of the instance, $\deccpac[\vep](\sI)$, we cannot apply \cref{prop:pm-to-ma} in an entirely black-box manner, yet most of the reasoning from the proof of \cref{prop:pm-to-ma} carries over.

\section{\MAFrameworkShort: From multi-agent to single-agent}
\label{sec:single-multiple}
Having established upper and lower bounds on the minimax risk for the \maf based on the \malong, we spend the remainder of the paper providing structural results which can be used to apply our main risk bounds to concrete settings of interest. To this end, in section we provide generic results which allow the conditions under which the multi-agent \CompShort can be controlled by the single-agent \CompShort, thereby allowing one to lift the plethora of existing results for the single-agent setting \citep{foster2021statistical,foster2023tight} to multiple agents.

\paragraph{Induced single-agent model classes}
Consider a Nash equilibrium instance $\sJ = \instma$ for the \maf (\cref{def:ne-instance}),
recalling that $\Pi = \Pi_1 \times \cdots \times \Pi_K$. %
We will prove upper bounds on the multi-agent DEC of the instance $\sJ$ in terms of the single-agent DEC for a collection of \emph{induced single-agent} model classes $\til \MM_k$ defined based on $\sJ$. To define the model classes $\til \MM_k$,  for $M \in \MM$ and $\ag \in [\Ag]$, we first define a \emph{single-agent model} $\single{M}$ as follows: a pure observation drawn from $\single{M}(\pi)$ has the distribution of the pure observation $\ocirc$ when $\ocirc \sim M(\pi)$, and the reward drawn from $\single{M}(\pi)$ has the distribution of $r_\ag$ when $(r_1, \ldots, r_\Ag) \sim M(\pi)$. In other words, the model $\single{M}$ is identical to $M$ but ignores the rewards of all agents except $\ag$. 

The single-agent model class $\til \MM_\ag$ is defined to have policy space $\Pi_k$, so that models in $\til \MM_\ag$ are mappings $\til M : \Pi_k \ra \Delta(\Ocirc \times \MR)$. In addition, $\til \MM_\ag$ is indexed by $\Pi_{-\ag} \times \MM$ and its models are given as follows:
\begin{align}
\til\MM_\ag = \left\{ \pi_\ag \mapsto \single{M}(\pi_\ag, \pi_{-\ag}) \ : \ \pi_{-\ag} \in \Pi_{-\ag},\ M \in \MM\right\}.\label{eq:define-til-mk}
\end{align}
The intuition behind this definition is that for each agent $k$, if other agents commit to playing $\pi_{-\ag}$, this induces a ``single-agent'' environment for $k$. If $M\in\cM$ is the original environment, then the model $\single{M}(\cdot, \pi_{-\ag})\in\til\MM_\ag$ is precisely the induced single-agent environment for $k$ (in a
decentralized protocol in which each agent observes its own reward but not the reward of other agents).

\paragraph{Offset \CompText}
The results in this section are most naturally stated in terms of the \emph{offset} variant of the DEC introduced in \citet{foster2021statistical}---specifically, the \emph{regret} variant which restricts to $p=q$ (that is, exploration and exploitation are coupled). For an instance $\sJ$, reference model $\Mbar$, and scale parameter $\gamma > 0$, we define
\begin{align}
  \decoreg[\gamma](\sJ, \Mbar) := \inf_{p\in \Delta(\Act)} \sup_{M \in \MM} \left\{ \E_{\act \sim p} [\hm(\act)] - \gamma \cdot \E_{\act \sim p} [\hell{M(\act)}{\Mbar(\act)}] \right\}.\label{eq:regret-dec}
\end{align}
We remark, via \citet{foster2023tight}, that this notion can be related to the constrained (PAC) \CompShort as follows.
\begin{proposition}[\cite{foster2023tight}]
  \label{prop:constrained-to-offset}
  For all $\Mbar\in\cMall$ and $\vep>0$,
\begin{align}
\deccpac[\veps](\sJ, \Mbar)
\leq \inf_{\gamma>0}\crl*{\decoreg[\gamma](\MAI,\Mbar)\vee{}0 + \gamma\veps^2}.\label{eq:constrained-to-offset}
\end{align}
\end{proposition}
\cref{prop:constrained-to-offset} suffices to derive tight bounds on the constrained \CompShort for all of the examples we will consider. It is also possible to relate the two complexity measures in the opposite direction, but this can lead to loose results \citep{foster2023tight}; this will not be necessary for our purposes.

\subsection{Bounding the \mashort for convex decision spaces}
\label{sec:madec-single-convex}
Our first result considers a general class of instances in which agents' decision spaces $\Pi_k$ satisfy a \emph{convexity} property, formally stated as \cref{ass:convexity_pols}. 
\begin{assumption}[Convexity of decision spaces]
  \label{ass:convexity_pols}
  For each $\ag \in [\Ag]$, there is a finite set $\MA_\ag$ (called the \emph{pure decision set}) so that $\Pi_\ag = \Delta(\MA_\ag)$. Furthermore, the following holds:
  \begin{enumerate}
  \item  Each $M \in \MM$ is linear in $\pi$, i.e., for $\pi \in \Pi$, $M(\pi) = \E_{a_\ag \sim \pi_\ag \forall \ag}[M(a)]$, where we write $a = (a_1, \ldots, a_\Ag)$.
  \item There is a measurable function $\varphi : \MO \ra \MA$ so that, for all $a \in \MA$ and $M \in \MM$, $\BP_{o \sim M(a)}(\varphi(o) = a) = 1$, i.e., $M(a)$ reveals $a$. 
    \end{enumerate}
  \end{assumption}
  This assumption is quite mild, and is satisfied whenever players 1) are allowed to randomize their actions, and 2) observe the resulting actions that are sampled at each round. In particular, this encompasses (structured) normal-form games with bandit feedback (see examples in \cref{app:examples_dec}).
  To simplify notation, we will write $\MA = \MA_1 \times \cdots \times \MA_\Ag$ and $\MA_{-\ag} = \prod_{\ag' \neq \ag} \MA_{\ag'}$.
  \arxiv{\dfcomment{i feel we may still be under-selling how mild this assumption is}\noah{I agree (open to suggestions to how to sell further), though I feel there is actually an important distinction to make in terms of when this is satisfied or not:  e.g., in the multi-player MAB regret setting where you do not see the actions, there may be a regret-PAC gap (i.e., an explore-exploit tradeoff), in that the best regret bound could be $T^{2/3}$. This would be neat if true, since most of the other such tradeoffs we know about are somewhat contrived, i.e., exploring arm. (Also, this isn't really relevant for this paper, but the assumption is certainly not satisfied in the decentralized setting since you don't see others' actions.)
\dfcomment{i just tweaked the sentence. i am not really selling it harder per se, but i try to illustrate what is happening more vividly}
    }
    }

Our main result for this subsection, \cref{thm:single-multiple-ch}, shows that for any $\Mbar \in \Delta(\MM)$, 
we can bound the multi-agent DEC $\decoreg[\gamma](\sJ, \Mbar)$ in terms of the single-agent DECs $\decoreg[\gamma/\Ag](\til \MM_\ag, \wb M_\ag)$, of the $\Ag$ model classes $\til \MM_\ag$ and reference models $\wb M_\ag$. 
\begin{theorem}[Restatement of \cref{thm:single-multiple-ch-informal}]
  \label{thm:single-multiple-ch}
  Suppose that $\sJ = \instma$ is an NE instance of the \MAFrameworkShort framework satisfying \cref{ass:convexity_pols}. Then for any $\gamma > 0$, it holds that
  \begin{align}
\sup_{\Mbar \in \co(\MM)} \decoreg[\gamma](\sJ, \Mbar) \leq & \sum_{\ag=1}^\Ag \sup_{\wb M_\ag \in \co(\til \MM_\ag)} \decoreg[\gamma/\Ag](\til \MM_\ag, \wb M_\ag)\nonumber.
  \end{align}
\end{theorem}
This result is quite intuitive: It shows that the complexity of centralized equilibrium computation is no larger than the complexity required for each agent to optimize their own reward in the face of a worst-case environment induced by the other players. It is proven using the following fixed-point argument: For a given agent $k$, if all other agents commit to a joint distribution, this induces a single-agent \dmso class $\til\cM_k$, and it is natural for agent $k$ to play the strategy that minimizes the single-agent \CompShort for this class. This is not enough to bound the \mashort as-is, because we need to specify a strategy for all agents, but by applying Kakutani's fixed point theoerem, we show that it is possible for all $K$ agents to simultaneously minimize their respective single-agent DECs with respect to the other agents' strategies. Furthermore, we remark that an immediate consequence of \cref{thm:single-multiple-ch} is that the same upper bound on $\decoreg(\sJ)$ holds also when $\sJ$ is a CCE or a CE instance, since Nash equilibria are always (coarse) correlated equilibria (see \cref{app:examples_dec}).

As a concrete example, for the multi-armed bandit problem with $A$ actions, we have $\decoreg(\cM)\leq\bigoh\prn[\big]{\frac{A}{\gamma}}$ \citep{foster2021statistical}. Using \pref{thm:single-multiple-ch}, it follows that if $\MAI$ is the class of $K$-player normal-form games with bandit feedback and $A_k$ actions per player, then
\[
  \sup_{\Mbar \in \co(\MM)} \decoreg[\gamma](\sJ, \Mbar)
  \leq{} \bigoh(K)\cdot\frac{\sum_{k=1}^{K}A_k}{\gamma}.
\]
Using \cref{prop:constrained-to-offset}, we conclude that
$\deccpac(\MAI) \leq \bigoh\prn[\Big]{
    \veps\cdot\sqrt{K\sum_{k=1}^{K}A_k}
    }$. We refer to \cref{app:examples_dec} for details, as well as additional examples, including structured normal-form games with linear or concave payoffs. %
For many of these examples, the application of \cref{thm:single-multiple-ch} leads to nearly tight bounds on $\decoreg(\sJ)$. However, this is not always true: In \cref{prop:multi-single-separation} (\aref{app:single_multiple_separation}), we show that there are instances $\sJ$ for which $\decoreg(\til \MM_k)$ is much larger than $\decoreg(\sJ)$. 

\arxiv{\noah{TODO -- look into connection with quantal response equilibrium; punt to arxiv}}

\subsection{Bounding the \mashort for Markov games}
While \cref{ass:convexity_pols} is quite general, and holds for most standard normal-form game setups, a notable setting that it does not capture is that of \emph{Markov games}, where the joint decision space $\Pi$ consists of randomized non-stationary policies (formalized in \cref{ass:pi-mg} below).\footnote{One might try to satisfy \cref{ass:convexity_pols} by convexifying each agent's decision space $\Pi_k$; however, in the setting of Markov games, this will lead the model classes $\til \MM_k$ defined in \cref{eq:define-til-mk} to be prohibitively large, since the policies $\pi_{-k}$ will now be \emph{mixtures} of non-stationary Markov policies. In particular, the DEC of the induced model classes $\til \MM_k$ that result will in general scale with the DEC of the class of mixtures of MDPs, which is exponential even in the tabular setting \citep{foster2022complexity}.} In this section, we provide an analogous result specialized to this general, non-convex setting.

\begin{assumption}[Markov game instance]
  \label{ass:pi-mg}
  The instance $\sJ = \instma$ is such that for some $H \in \BN$, finite state space $\MS$, and finite joint action space $\MA = \prod_{k=1}^K \MA_k$, each model $M \in \MM$ is a $K$-player, horizon-$H$ Markov game with state space $\MS$ and joint action space $\MA$ (see \cref{ex:mne}). In addition, for each $k$, the class $\Pi_k$ consists of non-stationary, randomized Markov policies, i.e.,
  \begin{align}
\Pi_k = \left\{ (\pi_{k,1}, \ldots, \pi_{k,H}) \mid \pi_{k,h} : \MS \ra \Delta(\MA_k)\;\;\forall{}h\in\brk{H}\right\}\nonumber.
  \end{align}
\end{assumption}
The finiteness of $\MS$ and $\MA$ in \cref{ass:pi-mg} is made for technical reasons, so as to enable the application of fixed point theorems; our bounds in this section will not depend quantitatively on $|\MS|$ or $|\MA|$, and we anticipate that this assumption can be relaxed.

 Under \cref{ass:pi-mg}, we provide the following analogue of \cref{thm:single-multiple-ch}.
\begin{theorem}
  \label{thm:single-multiple-mg}
There is a constant $C > 0$ so that the following holds.  Suppose that $\sJ = \instma$ is an NE instance of the \MAFrameworkShort framework satisfying \cref{ass:pi-mg}. Then for any $\gamma > 0$, it holds that
  \begin{align}
\sup_{\Mbar \in \MM} \decoreg[\gamma](\sJ, \Mbar) \leq \frac{CKH \log H}{\gamma} + \sum_{k=1}^K \sup_{\Mbar_k \in \til \MM_k} \decoreg[\gamma/(C KH\log H)](\til \MM_k, \Mbar_k) \nonumber.
  \end{align}
\end{theorem}
As an example, for when $\cM$ is a class of tabular MDPs with $\abs{\cS}=S$ and $\abs{\cA}=A$, we have $\decoreg(\cM)\leq\frac{\poly(S,A,H)}{\gamma}$ \citep{foster2021statistical}. \cref{thm:single-multiple-mg} then implies that for tabular Markov games with $\abs{\cS}\leq{}S$ and $\abs{\cA_k}\leq{}A$, we have $\sup_{\Mbar \in \MM} \decoreg[\gamma](\sJ, \Mbar)\leq\frac{\poly(S,A,H,K)}{\gamma}$ and via \cref{eq:constrained-to-offset},
\[
\sup_{\Mbar \in \MM} \deccpac[\veps](\sJ, \Mbar)\leq\veps\cdot\sqrt{\poly(S,A,H,K)}.
\]
We remark that while \cref{thm:single-multiple-ch} allows for improper reference models $\Mbar\in\conv(\cM)$, \cref{thm:single-multiple-mg} is restricted to proper reference models $\Mbar\in\cM$, and hence is mainly useful in settings (such as tabular MGs) in which proper estimators are available. See \cref{app:examples_dec} examples, as well as further details.
\arxiv{\dfcomment{i haven't looked at the proof for this one yet---does working with the distributional DEC help at all?}\noah{I tried for a bit but didn't see how to use distributional DEC. The issue is that the exploration distribution depends heavily on $\Mbar$, and if $\Mbar \sim \mu$ it's not clear if it's possible to combine these exploration distributions together in an appropriate sense. This actually could be somewhat important, since, e.g., I think it also is a roadblock if we were to try to use the technique from the note on model-free DEC together with this argument to get something for, e.g., linear MGs. (to do that you also have to redo the arguments for hellinger here for bilinear divergence, which I'm not sure actually works either)}

\noah{TODO: currently everything is stated for Regret DEC -- do we want to add PAC as well? (I think the statements are interesting for regret, so think it's good to keep regret in the paper, but not sure what the least repetitive way of getting PAC is...). 
  One option is to accept that we're loose by upper bounding PAC in MA by regret in single-agent, but then just saying the gap between pac and regret dec is a polynomial, and we only care about getting results correct up to a polynomial (though in many cases of interest we do recover the right $T$-rates)}
}

\section{\MAFrameworkShort: On the curse of multiple agents}
\label{sec:curse}
A nuisance encountered frequently in the study of multi-agent reinforcement learning %
is poor scaling of sample complexity with respect to the number of agents $K$. In particular, algorithms which directly estimate the model $\Mstar$ or agents' $Q$-value functions typically incur sample complexity exponential in $K$, due to the fact that both the model and agents' $Q$-value functions require at least $\exp(K)$ parameters to specify; this phenomenon has been called the \emph{curse of multiple agents} \citep{jin2021v}. In this section, we investigate the curse of multiple agents in the \maf through the lens of the \malong.

We first remark that the upper bound on the minimax risk in terms of the DEC in our upper bound, \cref{thm:constrained-upper-finitem} (as well as the more general version, \cref{thm:constrained-upper}), does indeed suffer from the curse of multiple agents: even for very simple model classes such as $K$-player normal-form games, the estimation error $\log\abs{\cM}$ in \cref{thm:constrained-upper-finitem} will scale exponentially in $K$ (see examples in \cref{app:examples_dec} for details and discussion), and therefore the upper bound in \cref{thm:constrained-upper-finitem} will also scale exponentially in $K$, even though the \mashort is not itself exponential. Note that our lower bound (\cref{thm:constrained-lower}) does \emph{not} have exponential dependence on $K$, since (a) the DEC typically scales as $\deccpac(\sJ) \asymp \Cprob \cdot \vep$, where the problem-dependent constant $\Cprob$ depends only on the size of agents' individual action sets, thus avoiding scaling exponential in $K$, and (b) the bound of \cref{thm:constrained-lower} does not include any term involving model estimation error (in particular, it does not multiply the scale $\vepslowerT$ at which the DEC is evaluated).\footnote{We recall that even in the single-agent setting, the appearance of the estimation error term in the upper bound, but not in the lower bound, leads to a gap between them. \cite{foster2023tight} emphasize that narrowing this gap is an important open problem.}

\paragraph{Evading the curse of multiple agents} Celebrated results in multi-agent (bandit) learning imply that the curse of multiple agents is not necessary, at least for multi-player normal-form games with bandit feedback: if each player runs an adversarial bandit no-regret algorithm, then the empirical average of their joint action profiles over $T$ time steps approaches a (coarse) correlated equilibrium for the game at a rate of $\poly(K,\max_k A_k)/\sqrt{T}$ (e.g., \citet{Rakhlin2013Optimization}), where $A_k$ is the number of actions for player $k$. Furthermore, a sequence of recent works has extended these results to the setting of Markov games \citep{jin2021v,song2021can,mao2022provably}.

It is natural to wonder if it is possible to capture these results, which avoid exponential scaling with $K$, through our framework and the \malong. In light of the discussion above, this question translates to asking whether the $\log |\MM|$ term in \cref{thm:constrained-upper-finitem} (more generally, the term $\EstHel$ in \cref{thm:constrained-upper}, which can be controlled in terms of covering numbers), which results from estimation error, can be decreased. Note that in general, as observed in \cite{foster2021statistical}, the estimation error term $\log\abs{\cM}$ appearing in \cref{thm:constrained-upper-finitem} cannot be removed completely, even in single-agent settings, but one might hope to replace it with a weaker quantity. One possible avenue, if possible, would be to replace $\log\abs{\cM}$ with $\log |\MF_{\cM}|$, where $\MF_{\cM}$ denotes the induced class of \emph{value functions}; this approach was explored for the single-agent setting in \citep{foster2022note}, where it leads to tighter guarantees for model-free reinforcement learning settings. However, this approach is insufficient for the purpose of avoiding the curse of multiple agents, since (an $\vep$-cover of) the value function class $\MF$ typically has size whose logarithm scales exponentially in $K$, even for normal-form games with bandit feedback (\cref{ex:cce}). 

In light of this discussion, perhaps most promising approach for evading the curse of multiple agents is to aim for bounds that are analogous to \cref{thm:constrained-upper-finitem}, but replace the factor $\log\abs{\cM}$ with the logarithm of the size of the agents' \emph{decision sets}. Indeed, the logarithm of the size of the joint (pure) decision set typically does not scale exponentially in $K$. For instance, for $K$-player normal-form games in which each player has $A$ actions, the number of pure action profiles is $A^K$, so its logarithm is only \emph{linear} in $K$; equivalently, one can look for bounds which scale as the sum of the logarithms of the agents' individual decision sets. In the single-agent DMSO setting, \cite{foster2021statistical,foster2022complexity} indeed obtain bounds that scale with $\log |\Pi|$, as opposed to $\log |\MM|$. There is a cost to pay for this improvement, however: the upper bounds of \citet{foster2021statistical,foster2022complexity} that replace $\log\abs{\cM}$ with $\log |\Pi|$ depend on the DEC of the \emph{convex hull of $\MM$}, as opposed to the DEC of $\MM$ itself. 

\paragraph{Our upper bound} In \cref{thm:curse_ub} below, we provide an upper bound that replaces the factor $\log\abs{\cM}$ appearing in \cref{thm:constrained-upper-finitem} with $\max_k \log |\Dev|$, at the cost of scaling with the \mashort for a convexified version of the instance $\MAI$. The quantity $\max_k \log |\Dev|$ is equal to $\max_k \log (|\Sigma_k|+1)$ in the special case of CCE instances  (\cref{def:cce-instance}), but is also small for CE instances (\cref{def:ce-instance}), as well as the following more general notion of correlated equilibrium, which we refer to as a ``generalized correlated equilibrium''. %
\begin{assumption}[Generalized correlated equilibrium]
  \label{ass:ce-convexity}
  We say that an \ma instance $\sJ = \instma$ satisfies the \emph{generalized correlated equilibrium} assumption if the following holds: we have $\Pi = \Delta(\Sigma_1 \times \cdots \times \Sigma_K)$, for finite sets $\Sigma_1, \ldots, \Sigma_K$, called \emph{pure decision sets}. Furthermore, writing $\Sigma := \Sigma_1 \times \cdots \times \Sigma_K$, the instance $\sJ$ satisfies:
  \begin{enumerate}
  \item Each $M \in \MM$ is linear in $\pi$, i.e., for $\pi \in \Pi$, $M(\pi) = \E_{\sigma \sim \pi}[M(\sigma)]$.
  \item The deviation functions $\Sw$ respect linearity in the sense that for all $k \in [K]$, $M \in \MM$, and $\pi \in \Pi, \dev \in \Dev$, we have $\fm_k(\Sw(\dev, \pi)) = \E_{\sigma \sim \pi}[\fm_k(\Sw(\dev, \sigma))]$.
  \end{enumerate}
\end{assumption}
It is straightforward to check that both CCE instances (\cref{def:cce-instance}) and CE instances (\cref{def:ce-instance}) satisfy \cref{ass:ce-convexity} as long as the pure decision sets $\Sigma_k$ are all finite.

To state our result, for an instance $\sJ = \instma$ of the \MAFrameworkShort framework, we define the \emph{convex hull} of the instance $\sJ$ to be the instance $\co(\sJ) := (\co(\MM), \Pi, \MO, \{ \Dev \}_k, \{ \Sw \}_k)$. We with the results in the previous section, our guarantees are most naturally stated in terms of the regret variant of the \mashort ($\decoreg$; cf. \cref{eq:regret-dec}).
\begin{theorem}
  \label{thm:curse_ub}
  Suppose that $\sJ = \instma$ is an \ma instance satisfying \cref{ass:ce-convexity}. Then, for any $T \in \BN$ and $\delta \in (0,1)$, there exists an algorithm (\maexo; \cref{alg:maexo} in \cref{sec:proofs_curse}) which produces $\wh \pi \in \Pi$ such that with probability at least $1-\delta$,
  \begin{align}
\RiskDM = \hmstar(\wh \pi) \leq O(K) \cdot \inf_{\gamma > 0} \left\{\decoreg[\gamma](\co(\sJ)) + \frac{\gamma}{T} \cdot \log \left( \frac{K \cdot \max_k |\Dev|}{ \delta} \right) \right\}\nonumber.
  \end{align}
\end{theorem}
We view this result as extending guarantees that replace $\log\abs{\cM}$ by $\log\abs{\Pi}$ in the single-agent setting \citep{foster2021statistical,foster2022complexity}; as with those prior results, the cost is that the \CompShort is applied to the convex hull of the instance. For the problem of computing CCE in normal form games with $K$ players and $A$ actions per player, we have $\decopac(\conv(\MAI))\approxleq\frac{A}{\gamma}$ and $\max_{k}\log\abs{\Dev}=\log(A)$, so this result gives
\[
  \RiskDM\approxleq{} \sqrt{\frac{\poly(K)\cdot{}A}{T}};
\]
see \cref{app:examples_dec} for details and further examples. %
\cref{thm:curse_ub} shows that it is possible to avoid the curse of multiple agents for convex classes, and leads to tight guarantees for structured classes of normal-form games with bandit feedback, such as games with linear or convex payoffs. In general though, it does not lead to tight guarantees non-convex classes such as Markov games. We prove the result by adapting the powerful \emph{exploration-by-optimization} algorithm from the single-agent setting \citep{lattimore2022minimax,foster2022complexity} in a way that exploits the unique feedback structure of the multi-agent setting. One might wonder how the guarantee of \cref{thm:curse_ub} compares to what one would obtain by having each agent $k$ run the (single-agent) exploration-by-optimization algorithm of \cite{foster2022complexity} separately (applied to the model class $\til \MM_k$ defined in \cref{eq:define-til-mk}) and using the resulting regret bound of \cite{foster2022complexity} for each agent to obtain an approximate CCE. As we show in \cref{prop:multi-single-separation}, the guarantee of \cref{thm:curse_ub} can be arbitrarily better than this alternative approach, since it involves the multi-agent DEC, $\decoreg(\co(\sJ))$, which can be arbitrarily smaller than the DEC for the single-agent classes, $\decoreg(\co(\til \MM_k))$. 

\dfcomment{Add remark that by \cref{prop:multi-single-separation}, this improves upon what we get by running single-agent exo.}\noah{added}

\paragraph{Extending the result to infinite decision sets}
We next explain how to extend the guarantee of \cref{thm:curse_ub} to the setting where the pure decision sets $\Sigma_k$ and deviation sets $\Dev$ are not finite. We will focus on CCE instances: consider a \ma instance $\sJ = \instma$ satisfying \cref{ass:ce-convexity}. Consider subsets $\til \Sigma_k \subseteq \Sigma_k$ and $\til \Pi_k' \subseteq \Dev$ for each $k$, and write $\til \Pi = \Delta(\til \Sigma_1 \times \cdots \times \til \Sigma_K) \subset \Pi$. (As an example, if $\sJ$ is a CCE instance, we will often take $\til\Pi_k' = \til \Sigma_k \cup \{ \perp\}$.) It is straightforward to see that the instance $\til \sJ = (\MM, \til \Pi, \MO, \{ \til \Pi_k' \}_k, \{  \Sw \}_k)$ satisfies \cref{ass:ce-convexity} (with pure decision sets $\til \Sigma_k$). We now define a sense in which the instance $\til \sJ$ is a good cover for $\sJ$.
\begin{definition}
  \label{def:decision-cover}
  Let $\sJ, \til \sJ$ be defined as above. For $\vep \geq 0$, we say that that $\til \sJ$ is an \emph{$\vep$-decision space cover} for $\sJ$ if
  \begin{align}
\forall M \in \MM, \quad \forall k \in [K], \quad \forall \til \pi \in \til \Pi, \quad \exists \til \pi_k' \in \til \Pi_k'\ \  \mbox{ s.t. }  \ \ \max_{\dev \in \Dev} \fm_k(\Sw(\dev, \til\pi)) - \fm_k(\Sw(\til \pi_k',\til \pi)) \leq \frac{\vep}{K}\nonumber.
  \end{align}
  We let $\MN_\Pi(\sJ , \vep) := \max_{k \in [K]} |\til\Pi_k'|$ denote the size of the largest deviation set in the smallest such cover, and define, for $T \in \bbN$,
  \begin{align}
\est_\Pi(\sJ, T) = \inf_{\vep \geq 0} \left\{ \log \MN_\Pi(\sJ, \vep) + \vep T \right\}\nonumber.
  \end{align}
\end{definition}
Let $\til \sJ = (\til \MM, \til \Pi, \MO, \{ \til \Pi_k' \}_k, \{ \Sw \}_k)$ be an $\vep$-decision space cover for $\sJ$. Note that, for any $\wh \pi \in \til \Pi$, it follows from \cref{def:decision-cover} that
\begin{align}
\hm(\wh \pi) = \sum_{k=1}^K \max_{\dev \in \Dev} \fm_k(\Sw(\dev, \wh\pi)) - \fm_k(\wh\pi) \leq \sum_{k=1}^K \max_{\til \pi_k' \in \til \Pi_k} \fm_k(\Sw(\til \pi_k', \wh\pi)) - \fm_k(\wh\pi) + \vep.\nonumber
\end{align}
Therefore, applying the algorithm of \cref{thm:curse_ub} to an appropriate decision space cover for the instance $\sJ$ (for an appropriate choice of $\vep$), we get the following result as an immediate corollary:
\begin{corollary}
  \label{cor:curse_ub_infinite}
    Suppose that $\sJ = \instma$ is a \ma instance satisfying \cref{ass:ce-convexity}. Then, for any $T \in \BN$ and $\delta \in (0,1)$, there exists an algorithm which produces $\wh \pi \in \Pi$ such that with probability at least $1-\delta$,
  \begin{align}
\RiskDM = \hmstar(\wh \pi) \leq O(K) \cdot \inf_{\gamma > 0} \left\{\decoreg[\gamma](\co(\sJ)) + \frac{\gamma}{T} \cdot \left( \est_\Pi(\sJ, T) + \log(K/\delta) \right) \right\}\nonumber.
  \end{align}
\end{corollary}

\paragraph{Lower bounds for Nash equilibrium instances}
\cref{thm:curse_ub} relies on the assumption that $\MAI$ is a generalized correlated equilibrium instance (\cref{ass:ce-convexity}). To close the section, we complement this result by showing that it is not possible to achieve analogous guarantees for Nash equilibria. First, in \cref{prop:ne-query-complexity} we show such an impossibility result for $K$-player NE instances: We give an instance for which the upper bound in \cref{thm:curse_ub} is polynomial in $K$, yet the minimax risk is exponential in $K$.
\begin{proposition}
  \label{prop:ne-query-complexity}
  There is a constant $c_0 > 0$ so that the following holds. For any $K \in \bbN$, there is a $K$-player NE instance $\sJ = \instma$ so that:
  \begin{enumerate}
  \item $\max_k |\Dev[k]| =2$.
  \item For all $\gamma > 0$, $\decoreg[\gamma](\co(\sJ)) \leq O(K/\gamma)$.
  \item There is no algorithm that adaptively draws $2^{o(K)}$ samples and outputs a policy with expected risk at most $c_0 \cdot K$. 
  \end{enumerate}
\end{proposition}
For the instance $\sJ$ in \cref{prop:ne-query-complexity}, we have $|\Pi'_k| = \bigoh(1)$ so a bound of the form in \cref{thm:curse_ub} would imply that $\til O(\poly(K) / \ep^2)$ samples suffice to learn an $\ep$-approximate Nash equilibrium; the lower bound on sample complexity of $2^{\Omega(K)}$ from \cref{prop:ne-query-complexity} rules this out. 
The proof of \cref{prop:ne-query-complexity} follows directly from well-known lower bounds on the query complexity of $K$-player Nash equilibria \citep{rubinstein2016settling,babichenko2016query,chen2017wellsupported}.

For our last result \cref{thm:2p0s-poldep}, we go even further, and show that the impossibility of proving any variant of \cref{thm:curse_ub} for NE instances persists even in the case when $K=2$ and the game is zero-sum. 
\begin{theorem}
  \label{thm:2p0s-poldep}
 There is a constant $C_0 > 0$ so that the following holds. Fix any $N \in \BN$ with $N \geq C_0$ and $\ep \in (1/N, 1)$. There is a two-player zero-sum NE instance $\sJ = \instma$ such that the following holds:
  \begin{enumerate}
  \item $\max\{ |\Dev[1]|, |\Dev[2]| \} \leq |\Pi| \leq C_0 \cdot N^2/\ep^2$.
  \item For all $\gamma \geq C_0$, $\decoreg[\gamma](\co(\sJ)) \leq \ep$.
  \item There is no algorithm that adaptively draws $\sqrt{N}/C_0$ samples and outputs a policy with expected risk at most $1/C_0$. 
  \end{enumerate}
\end{theorem}
\dfcomment{Probably worth mentioning: This theorem constructs an NE instance, but since $\Pi$ is finite, it does not correspond to the standard notion of mixed Nash in normal-form games (cf discussion after definition 1.1)---this is salient since the latter is a special case of CCE and should be taken care of by theorem 5.1}\noah{mentioned} 

Observe that for the instance $\MAI$ in \cref{thm:2p0s-poldep}, we have $\log\abs{\Pi}\approxleq\log(N/\eps)$, so a bound of the form in \cref{thm:curse_ub} would imply that roughly $\frac{\log(N/\eps)}{\eps}$ samples suffice to learn an $\eps$-approximate equilibrium. The lower bound on sample complexity in \cref{thm:2p0s-poldep}, which shows that $\bigom(\sqrt N)$ samples are required, thus rules out a guarantee of this type in a fairly strong sense. 

We remark that the instance $\sJ$ constructed in \cref{thm:2p0s-poldep}, while an NE instance per \cref{def:ne-instance}, does not correspond to the standard notion of mixed Nash equilibrium in normal-form games (see the discussion following \cref{def:ne-instance}). Since the marginals of coarse correlated equilibria in two-player zero-sum games constitute mixed Nash equilibria, \cref{thm:curse_ub} rules out a strengthening of \cref{thm:2p0s-poldep} which constructs an NE instance corresponding to the standard notion of mixed Nash equilibrium. 

The proof of \cref{thm:2p0s-poldep} is significantly more challenging (given prior work) than that of \cref{prop:ne-query-complexity}. It uses the classical support estimation problem
(e.g., \citet{paninski2008coincidence,canonne2020survey}) to construct an instance for which the DEC is small but the minimax risk is large. This idea is natural, because the support estimation problem has large model-estimation error, and the upper bound of \cref{thm:constrained-upper}, which involves the model estimation error, must be respected by the instance $\sJ$. Using the support estimation problem as a building block, we construct a class of two-player zero-sum games, which bears some resemblance to the construction used in the proof of \cref{prop:pm-to-ma}. However, the construction in the latter result does not ensure that $\decoreg(\co(\sJ))$ remains small, necessitating a more sophisticated approach. To ensure that $\decoreg(\co(\sJ))$ is small  while maintaining a lower bound on minimax risk, we need to embed a few additional components in the construction, namely the composition of a Reed-Solomon code and a randomness extractor. We refer the reader to \cref{sec:proofs_lbs_curse} for further details.

}

\arxiv{
\subsection*{Acknowledgements}
We thank Rob Schapire, Yunzong Xu, and Yanjun Han for helpful
comments and discussions. NG is supported at MIT by a Fannie \& John Hertz Foundation Fellowship and an NSF Graduate Fellowship. AR acknowledges support from ONR under grant N00014-20-1-2336 and ARO through award W911NF-21-1-0328. 
}

\clearpage

\bibliography{refs} 

\begin{thebibliography}{72}
\providecommand{\natexlab}[1]{#1}
\providecommand{\url}[1]{\texttt{#1}}
\expandafter\ifx\csname urlstyle\endcsname\relax
  \providecommand{\doi}[1]{doi: #1}\else
  \providecommand{\doi}{doi: \begingroup \urlstyle{rm}\Url}\fi

\bibitem[Abernethy et~al.(2008)Abernethy, Hazan, and
  Rakhlin]{abernethy2008competing}
Jacob Abernethy, Elad Hazan, and Alexander Rakhlin.
\newblock Competing in the dark: An efficient algorithm for bandit linear
  optimization.
\newblock In \emph{Proc. of the 21st Annual Conference on Learning Theory
  (COLT)}, 2008.

\bibitem[Anagnostides et~al.(2022)Anagnostides, Farina, and
  Sandholm]{anagnostides2022near}
Ioannis Anagnostides, Gabriele Farina, and Tuomas Sandholm.
\newblock Near-optimal $phi$-regret learning in extensive-form games, 2022.

\bibitem[Audibert and Bubeck(2009)]{audibert2009minimax}
Jean-Yves Audibert and S{\'e}bastien Bubeck.
\newblock Minimax policies for adversarial and stochastic bandits.
\newblock In \emph{COLT}, volume~7, pages 1--122, 2009.

\bibitem[Babichenko(2016)]{babichenko2016query}
Yakov Babichenko.
\newblock Query complexity of approximate nash equilibria.
\newblock \emph{J. ACM}, 63\penalty0 (4), oct 2016.
\newblock ISSN 0004-5411.

\bibitem[Bai et~al.(2020)Bai, Jin, and Yu]{bai2020near}
Yu~Bai, Chi Jin, and Tiancheng Yu.
\newblock Near-optimal reinforcement learning with self-play.
\newblock In \emph{Proceedings of the 34th International Conference on Neural
  Information Processing Systems}, NIPS'20, Red Hook, NY, USA, 2020. Curran
  Associates Inc.
\newblock ISBN 9781713829546.

\bibitem[Bakhtin et~al.(2022)Bakhtin, Brown, Dinan, Farina, Flaherty, Fried,
  Goff, Gray, Hu, Jacob, Komeili, Konath, Kwon, Lerer, Lewis, Miller, Mitts,
  Renduchintala, Roller, Rowe, Shi, Spisak, Wei, Wu, Zhang, and
  Zijlstra]{bakhtin2022human}
Anton Bakhtin, Noam Brown, Emily Dinan, Gabriele Farina, Colin Flaherty, Daniel
  Fried, Andrew Goff, Jonathan Gray, Hengyuan Hu, Athul~Paul Jacob, Mojtaba
  Komeili, Karthik Konath, Minae Kwon, Adam Lerer, Mike Lewis, Alexander~H.
  Miller, Sasha Mitts, Adithya Renduchintala, Stephen Roller, Dirk Rowe, Weiyan
  Shi, Joe Spisak, Alexander Wei, David Wu, Hugh Zhang, and Markus Zijlstra.
\newblock Human-level play in the game of <i>diplomacy</i> by combining
  language models with strategic reasoning.
\newblock \emph{Science}, 378\penalty0 (6624):\penalty0 1067--1074, 2022.
\newblock \doi{10.1126/science.ade9097}.
\newblock URL \url{https://www.science.org/doi/abs/10.1126/science.ade9097}.

\bibitem[Bart{\'o}k et~al.(2014)Bart{\'o}k, Foster, P{\'a}l, Rakhlin, and
  Szepesv{\'a}ri]{bartok2014partial}
G{\'a}bor Bart{\'o}k, Dean~P Foster, D{\'a}vid P{\'a}l, Alexander Rakhlin, and
  Csaba Szepesv{\'a}ri.
\newblock Partial monitoring—classification, regret bounds, and algorithms.
\newblock \emph{Mathematics of Operations Research}, 39\penalty0 (4):\penalty0
  967--997, 2014.

\bibitem[Beer(1993)]{beer1993topologies}
Gerald Beer.
\newblock \emph{Topologies on Closed and Closed Convex Sets}.
\newblock Kluwer Academic Publishers, 1993.

\bibitem[Bravo et~al.(2018)Bravo, Leslie, and Mertikopoulos]{bravo2018bandit}
Mario Bravo, David Leslie, and Panayotis Mertikopoulos.
\newblock Bandit learning in concave n-person games.
\newblock \emph{Advances in Neural Information Processing Systems}, 31, 2018.

\bibitem[Brown and Sandholm(2018)]{brown2018superhuman}
Noam Brown and Tuomas Sandholm.
\newblock Superhuman ai for heads-up no-limit poker: Libratus beats top
  professionals.
\newblock \emph{Science}, 359\penalty0 (6374):\penalty0 418--424, 2018.
\newblock \doi{10.1126/science.aao1733}.
\newblock URL \url{https://www.science.org/doi/abs/10.1126/science.aao1733}.

\bibitem[Bubeck(2015)]{bubeck2015convex}
S{\'e}bastien Bubeck.
\newblock Convex optimization: Algorithms and complexity.
\newblock \emph{Foundations and Trends{\textregistered} in Machine Learning},
  8\penalty0 (3-4):\penalty0 231--357, 2015.

\bibitem[Bubeck et~al.(2012)Bubeck, Cesa-Bianchi, and
  Kakade]{bubeck2012towards}
S{\'e}bastien Bubeck, Nicolo Cesa-Bianchi, and Sham~M Kakade.
\newblock Towards minimax policies for online linear optimization with bandit
  feedback.
\newblock In \emph{Conference on Learning Theory}, pages 41--1. JMLR Workshop
  and Conference Proceedings, 2012.

\bibitem[Bubeck et~al.(2017)Bubeck, Lee, and Eldan]{bubeck2017kernel}
S{\'e}bastien Bubeck, Yin~Tat Lee, and Ronen Eldan.
\newblock Kernel-based methods for bandit convex optimization.
\newblock In \emph{Proceedings of the 49th Annual ACM SIGACT Symposium on
  Theory of Computing}, pages 72--85, 2017.

\bibitem[Canonne(2020)]{canonne2020survey}
Cl{\'e}ment~L Canonne.
\newblock A survey on distribution testing: Your data is big. but is it blue?
\newblock \emph{Theory of Computing}, pages 1--100, 2020.

\bibitem[Chen et~al.(2022{\natexlab{a}})Chen, Mei, and Bai]{chen2022unified}
Fan Chen, Song Mei, and Yu~Bai.
\newblock Unified algorithms for rl with decision-estimation coefficients:
  No-regret, pac, and reward-free learning.
\newblock \emph{arXiv preprint arXiv:2209.11745}, 2022{\natexlab{a}}.

\bibitem[Chen et~al.(2017)Chen, Cheng, and Tang]{chen2017wellsupported}
Xi~Chen, Yu~Cheng, and Bo~Tang.
\newblock {Well-Supported vs. Approximate Nash Equilibria: Query Complexity of
  Large Games}.
\newblock In Christos~H. Papadimitriou, editor, \emph{8th Innovations in
  Theoretical Computer Science Conference (ITCS 2017)}, volume~67 of
  \emph{Leibniz International Proceedings in Informatics (LIPIcs)}, pages
  57:1--57:9, Dagstuhl, Germany, 2017. Schloss Dagstuhl--Leibniz-Zentrum fuer
  Informatik.
\newblock ISBN 978-3-95977-029-3.
\newblock \doi{10.4230/LIPIcs.ITCS.2017.57}.
\newblock URL \url{http://drops.dagstuhl.de/opus/volltexte/2017/8163}.

\bibitem[Chen et~al.(2022{\natexlab{b}})Chen, Zhou, and Gu]{chen2022almost}
Zixiang Chen, Dongruo Zhou, and Quanquan Gu.
\newblock Almost optimal algorithms for two-player zero-sum linear mixture
  markov games.
\newblock In \emph{International Conference on Algorithmic Learning Theory},
  pages 227--261. PMLR, 2022{\natexlab{b}}.

\bibitem[Cui et~al.(2022)Cui, Xiong, Fazel, and Du]{cui2022learning}
Qiwen Cui, Zhihan Xiong, Maryam Fazel, and Simon~S Du.
\newblock Learning in congestion games with bandit feedback.
\newblock \emph{arXiv preprint arXiv:2206.01880}, 2022.

\bibitem[Dani et~al.(2007)Dani, Hayes, and Kakade]{dani2007price}
Varsha Dani, Thomas~P Hayes, and Sham Kakade.
\newblock The price of bandit information for online optimization.
\newblock 2007.

\bibitem[Daskalakis and Papadimitriou(2006)]{daskalakis2006computing}
Constantinos Daskalakis and Christos~H Papadimitriou.
\newblock Computing pure nash equilibria in graphical games via markov random
  fields.
\newblock In \emph{Proceedings of the 7th ACM Conference on Electronic
  Commerce}, pages 91--99, 2006.

\bibitem[Daskalakis et~al.(2022)Daskalakis, Golowich, and
  Zhang]{daskalakis2022complexity}
Constantinos Daskalakis, Noah Golowich, and Kaiqing Zhang.
\newblock The complexity of markov equilibrium in stochastic games.
\newblock \emph{arXiv preprint arXiv:2204.03991}, 2022.

\bibitem[Du et~al.(2021)Du, Kakade, Lee, Lovett, Mahajan, Sun, and
  Wang]{du2021bilinear}
Simon~S Du, Sham~M Kakade, Jason~D Lee, Shachar Lovett, Gaurav Mahajan, Wen
  Sun, and Ruosong Wang.
\newblock Bilinear classes: A structural framework for provable generalization
  in {RL}.
\newblock \emph{International Conference on Machine Learning}, 2021.

\bibitem[Even-Dar et~al.(2009)Even-Dar, Mansour, and
  Nadav]{even2009convergence}
Eyal Even-Dar, Yishay Mansour, and Uri Nadav.
\newblock On the convergence of regret minimization dynamics in concave games.
\newblock In \emph{Proceedings of the forty-first annual ACM symposium on
  Theory of computing}, pages 523--532, 2009.

\bibitem[Flaxman et~al.(2005)Flaxman, Kalai, and McMahan]{flaxman2005online}
Abraham~D Flaxman, Adam~Tauman Kalai, and H~Brendan McMahan.
\newblock Online convex optimization in the bandit setting: gradient descent
  without a gradient.
\newblock In \emph{Proceedings of the sixteenth annual ACM-SIAM symposium on
  Discrete algorithms}, pages 385--394, 2005.

\bibitem[Foster et~al.(2016)Foster, Li, Lykouris, Sridharan, and
  Tardos]{foster2016learning}
Dylan~J Foster, Zhiyuan Li, Thodoris Lykouris, Karthik Sridharan, and Eva
  Tardos.
\newblock Learning in games: Robustness of fast convergence.
\newblock \emph{Advances in Neural Information Processing Systems}, 29, 2016.

\bibitem[Foster et~al.(2021)Foster, Kakade, Qian, and
  Rakhlin]{foster2021statistical}
Dylan~J Foster, Sham~M Kakade, Jian Qian, and Alexander Rakhlin.
\newblock The statistical complexity of interactive decision making.
\newblock \emph{arXiv preprint arXiv:2112.13487}, 2021.

\bibitem[Foster et~al.(2022{\natexlab{a}})Foster, Golowich, Qian, Rakhlin, and
  Sekhari]{foster2022note}
Dylan~J Foster, Noah Golowich, Jian Qian, Alexander Rakhlin, and Ayush Sekhari.
\newblock A note on model-free reinforcement learning with the
  decision-estimation coefficient.
\newblock \emph{arXiv preprint arXiv:2211.14250}, 2022{\natexlab{a}}.

\bibitem[Foster et~al.(2022{\natexlab{b}})Foster, Rakhlin, Sekhari, and
  Sridharan]{foster2022complexity}
Dylan~J Foster, Alexander Rakhlin, Ayush Sekhari, and Karthik Sridharan.
\newblock On the complexity of adversarial decision making.
\newblock \emph{arXiv preprint arXiv:2206.13063}, 2022{\natexlab{b}}.

\bibitem[Foster et~al.(2023)Foster, Golowich, and Han]{foster2023tight}
Dylan~J. Foster, Noah Golowich, and Yanjun Han.
\newblock Tight guarantees for interactive decision making with the
  decision-estimation coefficient.
\newblock \emph{arXiv preprint arXiv:2301.08215}, 2023.

\bibitem[Giannou et~al.(2021)Giannou, Vlatakis-Gkaragkounis, and
  Mertikopoulos]{giannou2021rate}
Angeliki Giannou, Emmanouil-Vasileios Vlatakis-Gkaragkounis, and Panayotis
  Mertikopoulos.
\newblock On the rate of convergence of regularized learning in games: From
  bandits and uncertainty to optimism and beyond.
\newblock \emph{Advances in Neural Information Processing Systems},
  34:\penalty0 22655--22666, 2021.

\bibitem[Gordon et~al.(2008)Gordon, Greenwald, and Marks]{gordon2008noregret}
Geoffrey~J. Gordon, Amy Greenwald, and Casey Marks.
\newblock No-regret learning in convex games.
\newblock In \emph{Proceedings of the 25th International Conference on Machine
  Learning}, ICML '08, page 360–367, New York, NY, USA, 2008. Association for
  Computing Machinery.
\newblock ISBN 9781605582054.

\bibitem[Guruswami et~al.(2022)Guruswami, Ruda, and
  Sudan]{guruswami2022essential}
Venkatesan Guruswami, Atri Ruda, and Madhu Sudan.
\newblock \emph{Essential Coding Theory}.
\newblock 2022.

\bibitem[Heliou et~al.(2017)Heliou, Cohen, and
  Mertikopoulos]{heliou2017learning}
Am{\'e}lie Heliou, Johanne Cohen, and Panayotis Mertikopoulos.
\newblock Learning with bandit feedback in potential games.
\newblock \emph{Advances in Neural Information Processing Systems}, 30, 2017.

\bibitem[Huang et~al.(2021)Huang, Lee, Wang, and Yang]{huang2021towards}
Baihe Huang, Jason~D Lee, Zhaoran Wang, and Zhuoran Yang.
\newblock Towards general function approximation in zero-sum markov games.
\newblock \emph{arXiv preprint arXiv:2107.14702}, 2021.

\bibitem[Jiang et~al.(2017)Jiang, Krishnamurthy, Agarwal, Langford, and
  Schapire]{jiang2017contextual}
Nan Jiang, Akshay Krishnamurthy, Alekh Agarwal, John Langford, and Robert~E
  Schapire.
\newblock Contextual decision processes with low {Bellman} rank are
  {PAC}-learnable.
\newblock In \emph{International Conference on Machine Learning}, pages
  1704--1713, 2017.

\bibitem[Jin et~al.(2021{\natexlab{a}})Jin, Liu, and
  Miryoosefi]{jin2021bellman}
Chi Jin, Qinghua Liu, and Sobhan Miryoosefi.
\newblock Bellman eluder dimension: New rich classes of {RL} problems, and
  sample-efficient algorithms.
\newblock \emph{Neural Information Processing Systems}, 2021{\natexlab{a}}.

\bibitem[Jin et~al.(2021{\natexlab{b}})Jin, Liu, Wang, and Yu]{jin2021v}
Chi Jin, Qinghua Liu, Yuanhao Wang, and Tiancheng Yu.
\newblock V-learning--a simple, efficient, decentralized algorithm for
  multiagent rl.
\newblock \emph{arXiv preprint arXiv:2110.14555}, 2021{\natexlab{b}}.

\bibitem[Jin et~al.(2022)Jin, Liu, and Yu]{jin2022power}
Chi Jin, Qinghua Liu, and Tiancheng Yu.
\newblock The power of exploiter: Provable multi-agent rl in large state
  spaces.
\newblock In \emph{International Conference on Machine Learning}, pages
  10251--10279. PMLR, 2022.

\bibitem[Kleinberg(2004)]{kleinberg2004nearly}
Robert Kleinberg.
\newblock Nearly tight bounds for the continuum-armed bandit problem.
\newblock \emph{Advances in Neural Information Processing Systems},
  17:\penalty0 697--704, 2004.

\bibitem[Kramár et~al.(2022)Kramár, Eccles, Gemp, Tacchetti, McKee,
  Malinowski, Graepel, and Bachrach]{kramar2022negotiation}
János Kramár, Tom Eccles, Ian Gemp, Andrea Tacchetti, Kevin~R. McKee, Mateusz
  Malinowski, Thore Graepel, and Yoram Bachrach.
\newblock Negotiation and honesty in artificial intelligence methods for the
  board game of {Diplomacy}.
\newblock \emph{Nature Communications}, 13\penalty0 (1):\penalty0 7214,
  December 2022.
\newblock ISSN 2041-1723.
\newblock \doi{10.1038/s41467-022-34473-5}.
\newblock URL \url{https://www.nature.com/articles/s41467-022-34473-5}.
\newblock Number: 1 Publisher: Nature Publishing Group.

\bibitem[Lattimore(2020)]{lattimore2020improved}
Tor Lattimore.
\newblock Improved regret for zeroth-order adversarial bandit convex
  optimisation.
\newblock \emph{Mathematical Statistics and Learning}, 2\penalty0 (3):\penalty0
  311--334, 2020.

\bibitem[Lattimore(2022)]{lattimore2022minimax}
Tor Lattimore.
\newblock Minimax regret for partial monitoring: Infinite outcomes and
  rustichini's regret.
\newblock \emph{arXiv preprint arXiv:2202.10997}, 2022.

\bibitem[Lattimore and Gy\"{o}rgy(2021)]{lattimore2021mirror}
Tor Lattimore and Andras Gy\"{o}rgy.
\newblock Mirror descent and the information ratio.
\newblock In \emph{Conference on Learning Theory}, pages 2965--2992. PMLR,
  2021.

\bibitem[Li et~al.(2022)Li, Zhou, Gu, and Jordan]{li2022learning}
Chris~Junchi Li, Dongruo Zhou, Quanquan Gu, and Michael~I Jordan.
\newblock Learning two-player mixture markov games: Kernel function
  approximation and correlated equilibrium.
\newblock \emph{arXiv preprint arXiv:2208.05363}, 2022.

\bibitem[Lin et~al.(2021)Lin, Zhou, Ba, and Zhang]{lin2021optimal}
Tianyi Lin, Zhengyuan Zhou, Wenjia Ba, and Jiawei Zhang.
\newblock Optimal no-regret learning in strongly monotone games with bandit
  feedback.
\newblock \emph{arXiv preprint arXiv:2112.02856}, 2021.

\bibitem[Liu et~al.(2021)Liu, Yu, Bai, and Jin]{liu2021sharp}
Qinghua Liu, Tiancheng Yu, Yu~Bai, and Chi Jin.
\newblock A sharp analysis of model-based reinforcement learning with
  self-play.
\newblock In Marina Meila and Tong Zhang, editors, \emph{Proceedings of the
  38th International Conference on Machine Learning}, volume 139 of
  \emph{Proceedings of Machine Learning Research}, pages 7001--7010. PMLR,
  18--24 Jul 2021.

\bibitem[Liu et~al.(2022)Liu, Szepesv{\'a}ri, and Jin]{liu2022sample}
Qinghua Liu, Csaba Szepesv{\'a}ri, and Chi Jin.
\newblock Sample-efficient reinforcement learning of partially observable
  markov games.
\newblock \emph{arXiv preprint arXiv:2206.01315}, 2022.

\bibitem[Maheshwari et~al.(2022)Maheshwari, Chiu, Mazumdar, Sastry, and
  Ratliff]{maheshwari2022zeroth}
Chinmay Maheshwari, Chih-Yuan Chiu, Eric Mazumdar, Shankar Sastry, and Lillian
  Ratliff.
\newblock Zeroth-order methods for convex-concave min-max problems:
  Applications to decision-dependent risk minimization.
\newblock In \emph{International Conference on Artificial Intelligence and
  Statistics}, pages 6702--6734. PMLR, 2022.

\bibitem[Malialis and Kudenko(2015)]{malialis2015distributed}
Kleanthis Malialis and Daniel Kudenko.
\newblock Distributed response to network intrusions using multiagent
  reinforcement learning.
\newblock \emph{Engineering Applications of Artificial Intelligence},
  41:\penalty0 270--284, 2015.
\newblock ISSN 0952-1976.
\newblock \doi{https://doi.org/10.1016/j.engappai.2015.01.013}.
\newblock URL
  \url{https://www.sciencedirect.com/science/article/pii/S095219761500024X}.

\bibitem[Mansour et~al.(2022)Mansour, Mohri, Schneider, and
  Sivan]{mansour2022strategizing}
Yishay Mansour, Mehryar Mohri, Jon Schneider, and Balasubramanian Sivan.
\newblock Strategizing against learners in bayesian games.
\newblock In Po-Ling Loh and Maxim Raginsky, editors, \emph{Proceedings of
  Thirty Fifth Conference on Learning Theory}, volume 178 of \emph{Proceedings
  of Machine Learning Research}, pages 5221--5252. PMLR, 02--05 Jul 2022.

\bibitem[Mao and Basar(2022)]{mao2022provably}
Weichao Mao and Tamer Basar.
\newblock Provably efficient reinforcement learning in decentralized
  general-sum markov games.
\newblock \emph{Dynamic Games and Applications}, pages 1--22, 2022.

\bibitem[Nisan et~al.(2007)Nisan, Roughgarden, Tardos, and
  Vazirani]{Nisan2007Algorithmic}
Noam Nisan, Tim Roughgarden, Eva Tardos, and Vijay~V Vazirani.
\newblock \emph{Algorithmic game theory}, volume~1.
\newblock Cambridge University Press Cambridge, 2007.

\bibitem[Osborne and Rubinstein(1994)]{osborne1994course}
Martin~J Osborne and Ariel Rubinstein.
\newblock \emph{A course in game theory}.
\newblock MIT press, 1994.

\bibitem[Paninski(2008)]{paninski2008coincidence}
Liam Paninski.
\newblock A coincidence-based test for uniformity given very sparsely sampled
  discrete data.
\newblock \emph{IEEE Transactions on Information Theory}, 54\penalty0
  (10):\penalty0 4750--4755, 2008.

\bibitem[Perolat et~al.(2022)Perolat, Vylder, Hennes, Tarassov, Strub, de~Boer,
  Muller, Connor, Burch, Anthony, McAleer, Elie, Cen, Wang, Gruslys, Malysheva,
  Khan, Ozair, Timbers, Pohlen, Eccles, Rowland, Lanctot, Lespiau, Piot,
  Omidshafiei, Lockhart, Sifre, Beauguerlange, Munos, Silver, Singh, Hassabis,
  and Tuyls]{perolat2022mastering}
Julien Perolat, Bart~De Vylder, Daniel Hennes, Eugene Tarassov, Florian Strub,
  Vincent de~Boer, Paul Muller, Jerome~T. Connor, Neil Burch, Thomas Anthony,
  Stephen McAleer, Romuald Elie, Sarah~H. Cen, Zhe Wang, Audrunas Gruslys,
  Aleksandra Malysheva, Mina Khan, Sherjil Ozair, Finbarr Timbers, Toby Pohlen,
  Tom Eccles, Mark Rowland, Marc Lanctot, Jean-Baptiste Lespiau, Bilal Piot,
  Shayegan Omidshafiei, Edward Lockhart, Laurent Sifre, Nathalie Beauguerlange,
  Remi Munos, David Silver, Satinder Singh, Demis Hassabis, and Karl Tuyls.
\newblock Mastering the game of stratego with model-free multiagent
  reinforcement learning.
\newblock \emph{Science}, 378\penalty0 (6623):\penalty0 990--996, 2022.
\newblock \doi{10.1126/science.add4679}.
\newblock URL \url{https://www.science.org/doi/abs/10.1126/science.add4679}.

\bibitem[Polyanskiy and Wu(2014)]{polyanskiy2014lecture}
Yury Polyanskiy and Yihong Wu.
\newblock Lecture notes on information theory.
\newblock 2014.

\bibitem[Rakhlin and Sridharan(2013)]{Rakhlin2013Optimization}
Alexander Rakhlin and Karthik Sridharan.
\newblock Optimization, learning, and games with predictable sequences.
\newblock In \emph{Advances in Neural Information Processing Systems (NIPS)},
  pages 3066--3074, 2013.

\bibitem[Rosen(1965)]{rosen1965existence}
J.~B. Rosen.
\newblock Existence and uniqueness of equilibrium points for concave n-person
  games.
\newblock \emph{Econometrica}, 33\penalty0 (3):\penalty0 520--534, 1965.

\bibitem[Rubinstein(2016)]{rubinstein2016settling}
Aviad Rubinstein.
\newblock Settling the complexity of computing approximate two-player {N}ash
  equilibria.
\newblock In \emph{Annual Symposium on Foundations of Computer Science (FOCS)},
  pages 258--265. IEEE, 2016.

\bibitem[Russo and Van~Roy(2013)]{russo2013eluder}
Daniel Russo and Benjamin Van~Roy.
\newblock Eluder dimension and the sample complexity of optimistic exploration.
\newblock In \emph{Advances in Neural Information Processing Systems}, pages
  2256--2264, 2013.

\bibitem[Russo and Van~Roy(2014)]{russo2014learning}
Daniel Russo and Benjamin Van~Roy.
\newblock Learning to optimize via posterior sampling.
\newblock \emph{Mathematics of Operations Research}, 39\penalty0 (4):\penalty0
  1221--1243, 2014.

\bibitem[Russo and Van~Roy(2018)]{russo2018learning}
Daniel Russo and Benjamin Van~Roy.
\newblock Learning to optimize via information-directed sampling.
\newblock \emph{Operations Research}, 66\penalty0 (1):\penalty0 230--252, 2018.

\bibitem[Shalev{-}Shwartz et~al.(2016)Shalev{-}Shwartz, Shammah, and
  Shashua]{shalevshwartz2016safe}
Shai Shalev{-}Shwartz, Shaked Shammah, and Amnon Shashua.
\newblock Safe, multi-agent, reinforcement learning for autonomous driving.
\newblock \emph{CoRR}, abs/1610.03295, 2016.
\newblock URL \url{http://arxiv.org/abs/1610.03295}.

\bibitem[Silver et~al.(2016)Silver, Huang, Maddison, Guez, Sifre, Van
  Den~Driessche, Schrittwieser, Antonoglou, Panneershelvam, Lanctot,
  et~al.]{silver2016mastering}
David Silver, Aja Huang, Chris~J Maddison, Arthur Guez, Laurent Sifre, George
  Van Den~Driessche, Julian Schrittwieser, Ioannis Antonoglou, Veda
  Panneershelvam, Marc Lanctot, et~al.
\newblock Mastering the game of go with deep neural networks and tree search.
\newblock \emph{nature}, 529\penalty0 (7587):\penalty0 484, 2016.

\bibitem[Song et~al.(2021)Song, Mei, and Bai]{song2021can}
Ziang Song, Song Mei, and Yu~Bai.
\newblock When can we learn general-sum markov games with a large number of
  players sample-efficiently?
\newblock \emph{arXiv preprint arXiv:2110.04184}, 2021.

\bibitem[Sun et~al.(2019)Sun, Jiang, Krishnamurthy, Agarwal, and
  Langford]{sun2019model}
Wen Sun, Nan Jiang, Akshay Krishnamurthy, Alekh Agarwal, and John Langford.
\newblock Model-based {RL} in contextual decision processes: {PAC} bounds and
  exponential improvements over model-free approaches.
\newblock In \emph{Conference on learning theory}, pages 2898--2933. PMLR,
  2019.

\bibitem[Vadhan(2012)]{vadhan2012pseudorandomness}
Salil Vadhan.
\newblock Pseudorandomness.
\newblock \emph{Foundations and Trends in Theoretical Computer Science},
  7:\penalty0 1--336, 2012.

\bibitem[Wang et~al.(2020)Wang, Salakhutdinov, and Yang]{wang2020provably}
Ruosong Wang, Russ~R Salakhutdinov, and Lin Yang.
\newblock Reinforcement learning with general value function approximation:
  Provably efficient approach via bounded eluder dimension.
\newblock \emph{Advances in Neural Information Processing Systems}, 33, 2020.

\bibitem[Wei and Luo(2018)]{wei2018more}
Chen-Yu Wei and Haipeng Luo.
\newblock More adaptive algorithms for adversarial bandits.
\newblock In \emph{Conference On Learning Theory}, pages 1263--1291. PMLR,
  2018.

\bibitem[Xie et~al.(2020)Xie, Chen, Wang, and Yang]{xie2020learning}
Qiaomin Xie, Yudong Chen, Zhaoran Wang, and Zhuoran Yang.
\newblock Learning zero-sum simultaneous-move markov games using function
  approximation and correlated equilibrium.
\newblock In \emph{Conference on learning theory}, pages 3674--3682. PMLR,
  2020.

\bibitem[Zhan et~al.(2022)Zhan, Lee, and Yang]{zhan2022decentralized}
Wenhao Zhan, Jason~D Lee, and Zhuoran Yang.
\newblock Decentralized optimistic hyperpolicy mirror descent: Provably
  no-regret learning in markov games.
\newblock \emph{arXiv preprint arXiv:2206.01588}, 2022.

\bibitem[Zheng et~al.(2022)Zheng, Trott, Srinivasa, Parkes, and
  Socher]{zheng2022ai}
Stephan Zheng, Alexander Trott, Sunil Srinivasa, David~C. Parkes, and Richard
  Socher.
\newblock The ai economist: Taxation policy design via two-level deep
  multiagent reinforcement learning.
\newblock \emph{Science Advances}, 8\penalty0 (18):\penalty0 eabk2607, 2022.
\newblock \doi{10.1126/sciadv.abk2607}.
\newblock URL \url{https://www.science.org/doi/abs/10.1126/sciadv.abk2607}.

\end{thebibliography}

\clearpage

\appendix

\colt{
  \addtocontents{toc}{\protect\setcounter{tocdepth}{2}}
  \renewcommand{\contentsname}{{\Large Contents of Appendix}}
  {\hypersetup{hidelinks}
    \tableofcontents
    }
}

\newpage

\arxiv{

\part{Examples}
\label{part:examples}
\section{\MAFrameworkShort: Examples of instances}
\label{app:ma_examples}

In this section of the appendix, we give examples of instances for the \maf, and apply our results to derive upper and lower bounds on the minimax risk.
\begin{itemize}
\item In \cref{app:ma_examples_equilibria} we give additional examples equilibria that can be captured in the \maf, focusing on correlated equilibria and variants.
\item In \cref{app:ma_examples_instances} we give detailed examples of \ma instances, including normal-form games with linear or concave payoffs (\aref{sec:bandit_instances}) and Markov games (\aref{sec:marl_instances}).
\item Finally, in \cref{app:examples_dec}, we give bounds on the \malong and minimax risk for variance instances, including finite-action normal-form games (\aref{sec:normal-form-games}), structured normal-form games (\aref{sec:linear-bandit-games}, \aref{sec:concave_dec}), and tabular Markov games (\aref{sec:tabular_dec}). In addition, in \aref{app:single_multiple_separation}, we give an instance which shows that the multi-agent to single-agent reduction in \cref{thm:single-multiple-ch} can be loose in general.
\end{itemize}

\subsection{Additional examples of equilibria}
\label{app:ma_examples_equilibria}

\cref{def:ce-instance} below shows how we can use the \MAFrameworkShort framework to capture the problem of (normal-form)  \emph{correlated equilibrium} computation in games. The definition is similar to that of CCE instances (\cref{def:cce-instance}), except players' deviation sets consist of mappings from their pure decision set to itself; these mappings describe how the player deviates as a function of their pure decision.
\begin{definition}[Correlated equilibrium instance]
  \label{def:ce-instance}
  We say that an \ma instance $\sJ = \instma$ is a \emph{correlated equilibrium (CE) instance} if the following holds:
  \begin{enumerate}
  \item For some finite sets $\Sigma_1, \ldots, \Sigma_K$ (called \emph{pure decisions}), we have $\Pi = \Delta(\Sigma_1\times \cdots \times \Sigma_\Ag)$. We write $\Sigma = \Sigma_1 \times \cdots \times \Sigma_\Ag$.
      \item For each $\pi \in \Pi$ and $M \in \MM$, it holds that $M(\pi) = \E_{\sigma \sim \pi}[M(\sigma)]$. %
  \item For $\ag \in [\Ag]$, we have $\Dev = \Sigma_\ag^{\Sigma_\ag}$, i.e., $\Dev$ is the set of functions $\phi : \Sigma_\ag \ra \Sigma_\ag$. 
  \item For each $\ag \in [\Ag]$, $\pi \in \Pi$, and $\phi \in \Dev$, $\Sw(\phi, \pi) \in \Delta(\Sigma)$ is the distribution whose probability mass function is given as follows:
    \begin{align}
\forall \sigma \in \Sigma, \quad \Sw(\phi, \pi)(\sigma) = \pi\left(\left\{ (\sigma_k', \sigma_{-k}) \in \Sigma \ : \ \phi(\sigma_k') = \sigma_k \right\}\right).\nonumber
    \end{align}
    In words, $\Sw(\phi, \pi)$ is the distribution of $(\phi(\sigma_k), \sigma_{-k})$, for $\sigma \sim \pi$. 
  \end{enumerate}
\end{definition}

Our next example considers notions of equilibria specialized to Markov games. Recall that \cref{def:cce-instance,def:ce-instance} describe instances that capture the notions of (coarse) correlated equilibria in \emph{normal-form games}, in which the pure actions belong to $\Sigma = \Sigma_1 \times \cdots \times \Sigma_K$. In the setting of Markov games, often a slightly different notion of (coarse) correlated equilibrium is used, whch we show is captured by \cref{ex:mce} below. 
\begin{example}[Markov (coarse) correlated equilibria in Markov games]
  \label{ex:mce}
  In \cref{ex:mne},
  We will show how to capture the problem of computing \emph{Markov coarse correlated equilibria (CCE)} and \emph{Markov correlated equilibria (CE)} (e.g., \citet{bai2020near,liu2021sharp,daskalakis2022complexity}) in the \maf, generalizing the notion of Markov Has equilibrium from \cref{ex:mne}. As in \cref{ex:mne}, we assume that the class $\cM$ consists of finite-horizon Markov games with horizon $H \in \BN$, state spaces $\MS_h$ for $h \in [H]$, action spaces $\MA_k$ for $k \in [K]$, and distribution $d_1 \in \Delta(\MS_1)$, all of which are identical across all models in the model class. The pure observation space $\Ocirc$ consists of trajectories, and the reward space is  $\MR = [0,1]$. For both Markov CE and Markov CCE, the joint decision space is the set $\Pi$ of \emph{Markov correlated policies}, namely policies $\pi = (\pi_1, \ldots, \pi_H)$, where each $\pi_h : \MS_h \ra \Delta(\MA_1 \times \cdots \times \MA_K)$ specifies a mapping from states to \emph{joint} distributions over actions. For a model $M$ and a joint decision $\pi \in \Pi$, an observation (trajectory) $o = \{ (s_h, (a_{1,h}, \ldots, a_{K,h}), (r_{1,h}, \ldots, r_{K,h}))\}_{h\in[H]}$ is drawn as follows: first, $s_1 \sim d_1$, and then for $h \in [H]$:
  \begin{itemize}
  \item $(a_{1,h}, \ldots, a_{K,h}) \sim \pi_h(s_h)$ and $r_{k,h} \sim R\sups{M}_k(s_h, (a_{1,h}, \ldots, a_{K,h}))$.
  \item $s_{h+1} \sim P_h\sups{M}(\cdot | s_h, (a_{1,h}, \ldots, a_{K,h}))$. 
  \end{itemize}
  It remains to specify the deviation sets $\Dev$ and switching functions $\Sw$:
  \begin{itemize}
  \item For the case of Markov CCE, for each $k \in [K]$, the deviation set $\Dev$ is the set of deterministic Markov policies for player $k$, which take the form $\pi_k' = (\pi_{k,1}', \ldots, \pi_{k,H}')$, where $\pi_{k,h}' : \MS_h \ra \MA_k$. For a joint policy $\pi \in \Pi$, $\Sw(\dev, \pi) \in \Pi$ is the Markov correlated policy where player $k$ plays according to $\pi_{k,h}'$ at each state and all other players play according to $\pi$. In particular, denoting $\til \pi := \Sw(\pi_k', \pi)$, we have that $\til \pi_h(s_h) = \pi_{k,h}'(s_h) \times \pi_{-k,h}(s_h)$, where $\pi_{-k,h}(s_h)$ denotes the marginal of $\pi_h(s_h)$ on the actions of all players but $k$. Summarizing, for the \ma instance $\sJ = \instma$, we have that $\wh \pi \in \Pi$, $\hmstar(\wh \pi) = 0$ if and only if $\wh \pi$ is a \emph{Markov CCE} of $\Mstar$.
  \item For the case of Markov CE, for each $k \in [K]$, the deviation set $\Dev$ is simply the set of tuples $\phi =(\phi_{k,h,s})_{h \in [H], s \in \MS_h}$, where each $\phi_{k,h,s} : \MA_k \ra \MA_k$ is a function from $\MA_k$ to itself. For a joint policy $\pi \in \Pi$, $\Sw(\phi, \pi)$ is the Markov correlated policy $\til \pi$ defined as follows: the joint action distribution of $\til \pi$ at step $h$ and state $s_h \in \MS_h$ is the distribution given by:
    \begin{align}
\til \pi_{h}(s)(a) = \pi_h(s)(\{ (a_k', a_{-k}) \in \MA \ : \ \phi(a_k') = a_k\})\nonumber,
    \end{align}
    for joint actions $a \in \MA$. In words, $\til \pi_h(s)$ is the distribution of $(\phi(a_k'), a_{-k})$, for $a \sim \pi_h(s)$. Summarizing, for the \ma instance $\sJ = \instma$ , we have that  $\wh \pi \in \Pi$, $\hmstar(\wh\pi) = 0$ if and only if $\wh \pi$ is a \emph{Markov CE} of $\Mstar$. 
  \end{itemize}
  Note that the instances constructed above are not special cases of the CCE or CE instances ( \cref{def:cce-instance,def:ce-instance}) we consider for normal-form games. This is because the notions of Markov (C)CE discussed above are more restrictive, forcing the joint decision $\wh \pi$ to be a (joint) Markov policy, as opposed to an arbitrary distribution over joint policies. Nevertheless, as \cref{ex:mce} shows, the \MAFrameworkShort framework is sufficiently general to capture all of these notions of equilibria. 
\end{example}

\subsection{Additional examples of instances}
\label{app:ma_examples_instances}
In this section, we give additional examples of instances that capture standard equilibrium learning problems found in the literature. We begin by describing examples of structured normal-form games in \aref{sec:bandit_instances}, and then consider multi-agent reinforcement learning problems in \aref{sec:marl_instances}. 

\subsubsection{Instances for bandits}
\label{sec:bandit_instances}
In this section, we describe several instances of structured normal-form games, which may be thought of as multi-agent generalization of structured bandit problem found in the single-agent setting. For each example we consider, the models will have the following common structure (paralleling that of \cref{ex:cce}).
\begin{itemize}
\item Each agent $k \in [K]$ will have a  set $\MA_k$, referred to as its \emph{pure action set}, and the joint policy space $\Pi$ will be a subset of $\Delta(\MA) = \Delta(\MA_1 \times \cdots \times \MA_K)$ which contains all singleton distributions $\indic_a$.
\item We will take $\MR := [-1,1]$ as the reward space and $\Ocirc := \MA$ as the pure observation space.
\item Let a class of mean reward functions $\MF \subseteq (\MA \ra \MR^K)$ be given. We define the model class $\MM_\MF$ as the set of models $M : \Pi \ra \Delta(\MR^K \times \Ocirc)$ for which there is some $(f_1, \ldots, f_K) \in \MF$ so that: (a) for all singleton distributions $\indic_a \in \Pi$, the distribution of $(r_1, \ldots, r_K, \ocirc) \sim M(\indic_a)$ satisfies $\ocirc = a$ a.s. and $ \E\sups{M,\indic_a}[r_k]=\fm(\indic_a) = f_k(a)$, and (b) for all $\pi \in \Pi$, $M(\pi) = \E_{a \sim \pi}[M(\indic_a)]$. 
\end{itemize}
In words, $\MM_\MF$ consists of models $M$ where (i) value functions $\fm_k(\cdot)$ are given by some element of $\MF$, and (ii) observations reveal the action played (via the pure observation).

First, in \cref{ex:linear-bandits}, we consider a normal-form game with linearly structured rewards, generalizing the single-agent linear bandit problem \citep{dani2007price,abernethy2008competing,bubeck2012towards}.
\arxiv{\noah{cite single-player, any multi-player ref out?}.\dfcomment{not sure for multi-player.}}
This example generalizes \cref{ex:cce}, which can be thought of as the special case where each player's action set is the set of standard unit vectors. 
\begin{example}[Normal-form games with linear rewards]
  \label{ex:linear-bandits}
 Fix $K \in \BN$; for each player $k \in [K]$, 
$\MA_k \subset \BR^{d_k}$ for some $d_k \in \bbN$. Write $d = d_1 d_2 \cdots d_K$. Suppose that $\Theta_1, \ldots, \Theta_K \subset \BR^{d}$ are convex sets so that $|\lng a_1 \otimes \cdots \otimes a_K, \theta_k \rng | \leq 1$ for all $a_1 \in \MA_1, \ldots, a_K \in \MA_K$, $k \in [K]$, and $\theta_k \in \Theta_k$. Define $\MF \subset (\MA \ra \BR^K)$ by $\MF = \{ (a_1, \ldots, a_K) \mapsto (\lng a_1 \otimes \cdots \otimes a_K, \theta_k \rng)_{k \in [K]} \ : \ \theta_1 \in \Theta_1, \ldots, \theta_K \in \Theta_K \}$. We can now consider the instances corresponding to finding Nash equilibria, CE, and CCE for the class of games whose payoffs are given by functions in $\MF$:
\begin{itemize}
\item We first treat Nash equilibria: suppose we set $\Pi_k = \Delta(\MA_k)$ for each $k \in [K]$ and $\Pi = \Pi_1 \times \cdots \times \Pi_k$, and define $\Dev, \Sw$ as in \cref{def:ne-instance}. We define $\MM = \MM_\MF$. Then the instance $\sJ = \instma$ captures the problem of finding Nash equilibria in an unknown linear bandit game. 
\item Next we treat (C)CE: we set $\Pi = \Delta(\MA_1 \times \cdots \times \MA_K)$ and define $\Dev, \Sw$ as in \cref{def:ce-instance} (respectively, \cref{def:cce-instance}). We define $\MM = \MM_\MF$, so that the instance $\sJ = \instma$ captures the problem of finding (coarse) correlated equilibria in an unknown linear bandit game.
\end{itemize}
\end{example}

Next, \cref{ex:convex-bandits} treats the setting of \emph{concave games} (with bandit feedback), which has received extensive attention in the game theory literature \citep{rosen1965existence,even2009convergence}, as well as machine learning \citep{bravo2018bandit,maheshwari2022zeroth,lin2021optimal}. It can also be viewed as a generalization of the problem of single-player \emph{concave bandits}     \citep{kleinberg2004nearly,flaxman2005online,bubeck2017kernel,lattimore2020improved}.\footnote{Often referred to as \emph{convex bandits}, or \emph{zeroth-order convex optimization}, since it is typically phrased in the form of loss minimization, whereas we consider reward maximization.}
\begin{example}[Concave games]
  \label{ex:convex-bandits}
  Given $K \in \BN$, for each $k \in [K]$, let $d_k \in \BN$ and $\MA_k \subset \BR^{d_k}$ be a convex and compact subset with nonempty interior. Set $\MA := \MA_1 \times \cdots \times \MA_K \subset \BR^d$, where $d = d_1 + \cdots + d_K$. Define $\MF \subset (\MA \ra \BR^K)$ by
  \begin{align}
\MF = \left\{ f : \MA \ra [0,1]^K \ | \ \forall k \in [K],\ \forall a_{-k} \in \MA_{-k}, \ \  \MA_k \ni a_k \mapsto f_k(a_k, a_{-k}) \ \mbox{ is concave and 1-Lipschitz}\right\}\nonumber.
  \end{align}
  Above, 1-Lipschitzness is with respect to the $\ell_2$ norm. 
  We consider the following Nash and CCE instances:
  \begin{itemize}
  \item We first consider Nash equilibria: define $\Dev, \Sw$ as in \cref{def:ne-instance}, and set $\Pi = \MA$, $\MM = \MM_\MF$. Then the instance $\sJ = \instma$ captures the problem of finding Nash equilibria in concave games, a classical problem \citep{rosen1965existence}. \arxiv{\noah{check about other refs}}  In the two-player zero-sum case (namely, when $K = 2$ and $f_1(a) + f_2(a) = 0$ for all $a \in \Pi$), the problem of bandit feedback which we cover has received extensive attention \citep{bravo2018bandit,maheshwari2022zeroth,lin2021optimal}.
  \item We next consider coarse correlated equilibria. Define $\Pi := \Delta(\MA_1 \times \cdots \times \MA_K)$, namely the space of Borel measures on the compact set $\MA_1 \times \cdots \times \MA_K \subset \BR^d$, and set $\MM = \MM_\MF$. Furthermore define $\Dev, \Sw$ as in \cref{def:cce-instance}. Then the instance $\sJ = \instma$ captures the problem of finding coarse correlated equilibria in concave games; this has received less attention than Nash equilibria in concave games.\arxiv{, but has been studied recently in \noah{find anything?} }
    \end{itemize}
    Since the action sets $\MA_k$ are infinite in this setting, it is not particularly natural to define a CE instance in the sense of \cref{def:ce-instance}. 
\end{example}

\subsubsection{Instances for multi-agent reinforcement learning}
\label{sec:marl_instances}

We now give concrete examples of Markov game classes $\cM$. The first example considers the special case of the instances for computing Markov Nash equilibria and Markov (coarse) correlated equilibria described in \cref{ex:mne,ex:mce} in which the Markov game under consider is \emph{tabular} (i.e., has finite states and actions).
\begin{example}[Equilibria in tabular Markov games]
  \label{ex:tabular-mg}
  Fix parameters $K, H \in \bbN$ representing the number of players and the horizon, finite action spaces $\MA_k$ (of size $A_k \in \BN$) for each player $k \in [K]$, and finite state spaces $\MS_h$ (each of size $S\in \BN$) at each step $h \in [H]$. The instances for each of the three types of equilibria (Nash, CE, CCE) share the same observation space $\MO$: in particular, their pure observation space is $\Ocirc$, the space of all possible $H$-step trajectories over the state and action spaces $\cS_1,\ldots,\cS_H$ and $\cA$, and the reward space is $\MR = [0,1]$.

  We refer to the \emph{tabular setting} as the model class $\MM$ parametrized by all possible $K$-player Markov games with horizon $H$, state spaces $\MS_h$, and action spaces $\MA_k$, so that the sum of each player's rewards is bounded in $[0,1]$ on any positive-probability trajectory.\footnote{This assumption allows us to take $\MR = [0,1]$.} Then for the deviation and switching functions $\Dev, \Sw$ as described in \cref{ex:mne}, the instance $\sJ = \instma$ captures the problem of computing Markov Nash equilibrium in an unknown tabular Markov game, and for $\Dev, \Sw$ as described in \cref{ex:mce} corresponding to the notions of Markov CCE or Markov CE, respectively, the instance $\sJ = \instma$ captures the problem of computing Markov CCE or Markov CE, respectively, in an unknown tabular Markov game. 
\end{example}

\colt{
  A key question in (multi-agent) online reinforcement learning is to understand what structural properties of the model class $\MM$ permit efficient learnability. In the simplest case (known as the \emph{tabular} case), the state and action spaces $\MS_h, \MA$ are all finite, and $\MM$ consists of all models specified by arbitrary transitions $P_h\sups{M}$ and reward distributions $R_{k,h}\sups{M}$ with uniformly bounded support. By restricting $\cM$, our formulation also captures a more complex settings that incorporate function approximation \citep{chen2022almost,li2022learning,xie2020learning,jin2022power,huang2021towards,zhan2022decentralized,liu2022sample}. In what follows, we give one such example.
  }

  \begin{example}[Equilibria in linear mixture Markov games]
    \label{ex:linmix-mg}
    Fix parameters $K, H \in \BN$ representing the number of players and the horizon, finite action spaces $\MA_k$ for each $k \in [K]$, and finite state spaces $\MS_h$ for each $h \in [H]$.\arxiv{\footnote{We require the state spaces to be finite for technical reasons, but our bounds will not depend on the size of the state spaces.}} For a \emph{dimension} parameter $d \in \bbN$, we are given mappings $\phi_h : \MS \times \MA \times \MS \ra \BR^d$, $\psi_{k,h} : \MS \times \MA \ra \BR^d$ such that for all $h \in [H]$ and $k \in [K]$
    \[
      \sum_{s_{h+1} \in \MS_{h+1}} \phi_h(s_{h+1} | s_h,a) = \mathbf{1} \in \BR^d,\mathand \|\psi_{k,h}(s_h,a) \|_2 \leq 1
    \]
    for all $s_h \in \MS_h, a \in \MA, s_{h+1} \in \MS_{h+1}$.\footnote{The values of $\phi_H(s_{h+1} | s_h,a)$ will not matter, so we may take $\MS_{H+1}$ to be, e.g., the set consisting of single state.} The instances we construct have pure observation space $\Ocirc$ given by the set of all possible $H$-step trajectories over the action and state spaces $\MA$ and $\MS_h$, and have reward space $\MR = [0,1]$.

    For some $B \in \bbN$, the set of \emph{linear mixture Markov games} is the model class $\MM$ consisting of all $K$-player Markov games $M$ with horizon $H$, state spaces $\MS_h$, and action spaces $\MA_k$, for which there are vectors $\theta_h\sups{M} \in \BR^{d}$ satisfying $\| \theta_h\sups{M} \|_2 \leq B$ and 
  \begin{align}
P_h\sups{M}(s_{h+1} | s_h, a) = \lng \theta_h\sups{M}, \phi_h(s_{h+1} | s_h, a) \rng, \qquad R_{k,h}\sups{M}(s_h, a) = \lng \theta_h\sups{M}, \psi_{k,h}(s_h, a)\rng\nonumber
  \end{align}
  for all $h \in [H], k \in [K]$, $s_h \in \MS_h, a \in \MA, s_{h+1} \in \MS_{h+1}$, and for which under any positive-probability trajectory, $\sum_{h=1}^H r_{k,h} \in[0,1]$.\arxiv{ (For simplicity, we assume the rewards are deterministic and equal to the quantity $R_{k,h}\sups{M}(s_h, a)$ defined above.)}

  For the deviation and switching functions $\Dev, \Sw$ as described in \cref{ex:mne}, the instance $\sJ = \instma$ captures the problem of computing Markov Nash equilibrium in an unknown linear mixture Markov game, and for $\Dev, \Sw$ as described in \cref{ex:mce} corresponding to the notions of Markov CCE or Markov CE, respectively, the instance $\sJ = \instma$ captures the problem of computing Markov CCE or Markov CE, respectively, in an unknown linear mixture Markov game. 

\end{example}

\subsection{Computing bounds on the DEC and minimax risk of multi-agent instances}
\label{app:examples_dec}
In this section, we apply our results from \cref{sec:bounds,sec:single-multiple,sec:curse} to (a) give bounds on the DEC of various  \MAFrameworkShort instances, and (b) use these bounds on the DEC to derive bounds on the minimax risk for learning equilibria in multi-agent interactive decision making.

\subsubsection{Normal-form games with finite action spaces}
\label{sec:normal-form-games}
We begin with perhaps the simplest example: finite-action normal-form games with bandit feedback. We consider Nash, CE, and CCE instances, as described in \cref{ex:cce}. Let us fix $K \in \BN$ along with action sets $\MA_1, \ldots, \MA_K$ for each of the $K$ players, with joint action set $\MA := \MA_1 \times \cdots \times \MA_K$. We write $A_k := |\MA_k|$ for $k \in [K]$. Let $\sJ^\NE, \sJ^\CE, \sJ^\CCE$ denote the NE, CE, and CCE instances, respectively, constructed in \cref{ex:cce}. In this section, we bound the DEC of these instances; we begin with an upper bound on the offset DEC, which immediately yields an upper bound on the constrained DEC via \cref{prop:constrained-to-offset}.
\begin{proposition}
  \label{prop:nf-dec-bound}
  For any $\gamma > 0$, the instances $\sJ^\NE, \sJ^\CE, \sJ^\CCE$ defined above satisfy
  \begin{align}
\decoreg[\gamma](\sJ^\CCE) \leq \decoreg[\gamma](\sJ^\CE) \leq \decoreg[\gamma](\sJ^\NE) \leq \frac{K \cdot \sum_{k=1}^K A_k}{\gamma}\nonumber.
  \end{align}
\end{proposition}
\begin{proof}[\pfref{prop:nf-dec-bound}]
  Note that the instances $\sJ^\NE, \sJ^\CE, \sJ^\CCE$ share the same observation space $\MO$, i.e., they have pure observation space $\Ocirc = \MA$ and reward space $\MR = [0,1]$.\footnote{Technically, the model class for the instance $\sJ^\NE$ only acts on product distributions in $\Pi^\NE = \Delta(\MA_1) \times \cdots \times \Delta(\MA_K)$, as opposed to $\Pi^\CCE = \Pi^\CE = \Delta(\MA) \supset \Pi^\NE$; we will formally interpret the domain of $\MM$ for the instance $\sJ^\NE$ as $\Pi^\NE$ to avoid cluttering notation.} 
  Thus, let us write $\sJ^\NE = (\MM, \Pi^\NE, \MO, \{ (\Dev)^\NE \}_k, \{ \Sw^\NE\}_k)$, $\sJ^\CE = (\MM, \Pi^\CE, \MO, \{ (\Dev)^\CE \}_k, \{ \Sw^\CE \}_k)$, and $\sJ^\CCE = (\MM, \Pi^\CCE, \MO, \{ (\Dev)^\CCE \}_k, \{ \Sw^\CCE\}_k)$. To distinguish between the three different settings, we augment the functions $\fm(\cdot)$ and $\hm(\cdot)$ with the superscripts NE/CE/CCE. For example, for the instance $\sJ^\NE$, we have, for $M \in \MM, \pi \in \Pi^\NE$, 
  \begin{align}
    f_k\sups{M, \NE}(\pi) :=  \E\sups{M, \pi}[r_k], \mathand
h\sups{M, \NE}(\pi) =& \sum_{k=1}^K \max_{\dev \in (\Dev)^\NE} f_k\sups{M, \NE}(\Sw^\NE(\dev, \pi)) - f_k\sups{M,\NE}(\pi)\nonumber.
  \end{align}
 The functions $h\sups{M,\CE} : \Pi^\CE \ra \bbR$ and $h\sups{M,\CCE} : \Pi^\CCE \ra \bbR$ are defined analogously.

  It holds that $\Pi^\CE = \Pi^\CCE$; furthermore, for any $M \in \MM$ and $\pi \in \Pi^\CE = \Pi^\CCE$, we have that $h\sups{M, \CCE}(\pi) \leq h\sups{M, \CE}(\pi)$. It immediately follows that $\decoreg[\gamma](\sJ^\CCE) \leq \decoreg[\gamma](\sJ^\CE)$. Next, note that $\Pi^\NE \subset \Pi^\CE$, and for any $\pi \in \Pi^\NE$ and $M \in \MM$, we have that $h\sups{M, \NE}(\pi) = h\sups{M,\CE}(\pi)$.  Hence, for $\Mbar \in \co(\MM)$, 
  \begin{align}
    \decoreg[\gamma](\sJ^\NE, \Mbar) =& \inf_{p \in \Delta(\Pi^\NE)} \sup_{M \in \MM} \E_{\pi \sim p} \left[ h\sups{M,\NE}(\pi) - \gamma \cdot \hell{M(\pi)}{\Mbar(\pi)} \right]\nonumber\\
    \geq & \inf_{p \in \Delta(\Pi^\CE)} \sup_{M \in \MM} \E_{\pi \sim p} \left[ h\sups{M,\CE}(\pi) - \gamma \cdot \hell{M(\pi)}{\Mbar(\pi)} \right] = \decoreg(\sJ^\CE, \Mbar)\nonumber.
  \end{align}
  This establishes that
  \[
\decoreg[\gamma](\sJ^\CCE) \leq \decoreg[\gamma](\sJ^\CE) \leq \decoreg[\gamma](\sJ^\NE).
    \]
  It remains to upper bound $\decoreg(\sJ^\NE)$. For $k \in [K]$, we write $\Pi_k := \Delta(\MA_k)$ and $\Pi_{-k} := \prod_{k' \neq k} \Pi_{k'}$.  For each $k \in [K]$, define the model class $\til \MM_k \subset (\Pi_k \ra \Delta(\MR \times \Ocirc))$ as in \eqref{eq:define-til-mk}; in particular:
  \begin{align}
\til \MM_k = \{ \pi_k \mapsto \single{M}(\pi_k, \pi_{-k}) \ : \ \pi_{-k} \in \Pi_{-k}, M \in \MM \}\nonumber.
  \end{align}
  Next define the model class $\MM_k' \subset (\MA_k \to \Delta(\MR \times \{ \perp\}))$ by \[\MM_k' = \{ M \ : \ M(a_k) \in \Delta(\MR \times \{ \perp \}) \ \forall a_k \in \MA_k \},\] i.e., $M(a_k)$ is allowed to be an arbitrary distribution over $\MR \times \{ \perp\}$ for each $a_k$. Proposition 5.2 of \cite{foster2021statistical} shows that $\decoreg[\gamma](\MM_k') \leq \frac{A_k}{\gamma}$. Next, fix $\Mbar \in \co(\til \MM_k)$, and let $\Mbar' \in \co(\MM_k')$ be the unique model so that the reward $r \sim \Mbar'(a_k)$ is distributed identically to the reward $r \sim \Mbar(a_k)$ for all $a_k \in \MA_k$. Then we have
  \begin{align}
    & \decoreg(\til \MM_k, \Mbar) \nonumber\\
    =&  \inf_{p \in \Delta(\Pi_k)} \sup_{M \in \MM, \pi_{-k} \in \Pi_{-k}} \E_{\pi_k \sim p} \left[\max_{\pi_k' \in \Pi_k} f\sups{M, \NE}_k(\pi_k', \pi_{-k}) - f_k\sups{M, \NE}(\pi_k, \pi_{-k}) - \gamma \cdot \hell{M(\pi_k, \pi_{-k})}{\Mbar(\pi_k)} \right]\nonumber\\
    \leq & \inf_{p \in \Delta(\MA_k)} \sup_{M \in \MM, \pi_{-k} \in \Pi_{-k}} \E_{a_k \sim p} \left[ \max_{a_k' \in \MA_k} f_k\sups{M,\NE}(a_k', \pi_{-k}) - f_k\sups{M,\NE}(a_k, \pi_{-k}) - \gamma \cdot \hell{M(a_k, \pi_{-k})}{\Mbar(a_k)} \right]\nonumber\\
    \leq & \inf_{p \in \Delta(\MA_k)} \sup_{M' \in \MM_k'} \E_{a_k \sim p} \left[ \max_{a_k' \in \MA_k} f\sups{M'}_k(a_k') - f\sups{M'}_k(a_k) - \gamma \cdot \hell{M'(a_k)}{\Mbar'(a_k)}\right] = \decoreg(\MM_k', \Mbar')\nonumber,
  \end{align}
  where the first inequality follows since $\MA_k \subset \Pi_k$ (by identifying each action $a_k \in \MA_k$ with its indicator distribution $\indic_{a_k} \in \Pi_k$), and the second inequality follows since for any $M \in \MM, \pi_{-k} \in \Pi_{-k}$, there is a model $M' \in \MM_k'$ so that for all $a_k \in \MA_k$, the distribution of the reward $r \sim M(a_k, \pi_{-k})$ is identical to the distribution of $r \sim M'(a_k) $. Note that in the display above we have associated actions $a_k \in \MA_k$ with their singleton distribution $\indic_{a_k} \in \Pi_k$, per our convention. %
  It follows that $\decoreg(\til \MM_k) \leq \decoreg(\MM_k')$ for all $\gamma > 0$. Finally, by \cref{thm:single-multiple-ch} applied to the instance $\sJ^\NE$, we have that
  \begin{align}
\decoreg(\sJ^\NE) \leq \sum_{k=1}^K \decoreg[\gamma/K](\til \MM_k) \leq \sum_{k=1}^K \decoreg[\gamma/K](\MM_k') \leq \frac{K\cdot  \sum_{k=1}^K A_k}{\gamma}\nonumber.
  \end{align}
  Note that our application of \cref{thm:single-multiple-ch} is valid since \cref{ass:convexity_pols} is satisfied by the definition of $\sJ^\NE$ in \cref{ex:cce} (in particular, our assumption that $M(\indic_a) \in \Delta(\MR^K) \times \{ \indic_a \}$, i.e., that $M$ reveals $a$, satisfies the second point of \cref{ass:convexity_pols}). 
\end{proof}

Using \cref{prop:nf-dec-bound}, we now bound the minimax rates for the instances $\sJ^\NE, \sJ^\CE, \sJ^\CCE$. To simplify matters slightly, we consider slightly simplified special cases of these instances in which the model class is constrained to models which output rewards according to the Bernoulli distribution (i.e., the rewards are $\{0,1\}$-valued).\footnote{This restriction of the model class is essentially without loss of generality: given any model class with general reward distributions in $[0,1]$, we can simulate samples from a model class with the same value functions $\fm_k(\cdot)$ and Bernoulli reward distributions by, upon receiving rewards $(r_1, \ldots, r_K) \sim M(\pi)$, replacing each $r_k$ with a sample $r_k' \sim \Ber(r_k)$.} Furthermore, we assume for simplicity that $A_k \geq 2$ for all $k$. We denote the corresponding \ma instances with Bernoulli rewards by $\sJ^\NE_0, \sJ^\CE_0, \sJ^\CCE_0$. First, we bound the minimax rate for $\sJ^\NE_0$:
\begin{proposition}
  \label{prop:m0ne-risk}
There is an algorithm for the instance $\sJ^\NE_0$ which guarantees that with probability at least $1-\delta$, $\RiskDM \leq \sqrt{(\max_k A_k) \cdot A \cdot T^{-1}} \cdot \polylog(T, A,\delta^{-1})$, where $A = A_1 A_2 \cdots A_K$.
\end{proposition}
It is evident that the same upper bound on risk for $\sJ_0^\NE$ in \cref{prop:m0ne-risk} applies to $\sJ_0^\CE, \sJ_0^\CCE$ since for any decision $\wh \pi \in \Pi^\NE \subset \Pi^\CE = \Pi^\CCE$, we have $h\sups{M, \CCE}(\wh \pi) \leq h\sups{M, \CE}(\wh \pi) \leq h\sups{M, \NE}(\wh \pi) $ (recall the definition of $h\sups{M, \NE}, h\sups{M, \CE}, h\sups{M, \CCE}$ in the proof of \cref{prop:nf-dec-bound}). 
\begin{proof}[\pfref{prop:m0ne-risk}]
  The combination of \cref{prop:nf-dec-bound} and \cref{prop:constrained-to-offset} yields that $\deccpac[\vep](\sJ_0^\NE) \leq \vep \cdot 2\sqrt{K \cdot \sum_{k=1}^K A_k}$. Since, $|\MO| \leq 2^K A$ (as rewards are assumed to be Bernoulli), the class $\MM$ satisfies \cref{ass:kernel-B} with $B = 2^KA$, and therefore \cref{prop:estimation-infinite} gives that $\Est(T, \delta) = O(\est(\MM, T)+ \log \delta^{-1}) \cdot K^2 \cdot \log^2(\max_k A_k)$. Finally, by discretizing the reward means into multiples of $\vep^2$, we see that $\MN(\MM, \vep) \leq (1/\vep^2)^{AK}$, which implies that $\est(\MM, T) \leq O(AK \cdot \log(T))$. Therefore, \cref{thm:constrained-upper} combined with \cref{prop:ma-to-pm} gives that there is an algorithm with
  \begin{align}
\RiskDM \leq \sqrt{K (A_1 + \cdots + A_K)} \cdot \sqrt{\frac{\Est(T, \delta)}{T}} \cdot \polylog(T, 1/\delta) \leq \sqrt{\frac{\max_k A_k \cdot A}{T}} \cdot \polylog(T, 1/\delta, A)\nonumber.
  \end{align}
\end{proof}
Note that the upper bound of \cref{prop:m0ne-risk} suffers from the curse of multiple agents: the number of joint action profiles $A$ is exponential in the number of agents $K$. It is a well-known result that such exponential dependence on $K$ is necessary for learning (e.g., \citet{rubinstein2016settling}; see \cref{prop:ne-query-complexity}), while it is not necessary for learning (coarse) correlated equilibria. We next show that our results in \cref{sec:curse} allow us to recover this improved (polynomial) bound for (coarse) correlated equilibria:
\begin{proposition}
  \label{prop:nf-curse-ub}
  Fix any $T \in \bbN, \delta \in (0,1)$.   There is an algorithm for the instance $\sJ_0^\CE$ which produces $\wh \pi \in \Pi^\CE$ such that with probability at least $1-\delta$,
  \begin{align}
\RiskDM \leq \sqrt{\frac{K^4 \max_k A_k^2}{T}} \cdot \polylog\left(K, \max_k A_k, \delta^{-1}\right).\nonumber
  \end{align}
  Furthermore, there is an algorithm for the instance $\sJ_0^\CCE$ which produces $\wh \pi \in \Pi^\CCE$ such that with probability at least $1-\delta$,
  \begin{align}
\RiskDM \leq \sqrt{\frac{K^3 \sum_{k=1}^K A_k}{T}} \cdot \polylog\left(K, \max_k A_k, \delta^{-1}\right)\nonumber.
  \end{align}
\end{proposition}
\begin{proof}[\pfref{prop:nf-curse-ub}]
  The statement of the proposition is an immediate consequence of \cref{thm:curse_ub}. For the instance $\sJ_0^\CCE$, we have that $\decoreg(\co(\sJ_0^\CCE)) \leq  \decoreg(\co(\sJ^\CCE)) = \decoreg(\sJ^\CCE) \leq \frac{K \sum_{k=1}^K A_k}{\gamma}$, where we have used that the model class $\MM$ is convex and \cref{prop:nf-dec-bound}. Therefore, \cref{thm:curse_ub} gives that there is an algorithm achieving
  \begin{align}
    \RiskDM \leq &  O \left( K \cdot \inf_{\gamma > 0} \left\{\frac{K \sum_{k=1}^K A_k}{\gamma} + \frac{\gamma}{T} \cdot \log \left( \frac{K \cdot \max_k A_k}{\delta} \right)\right\} \right)\nonumber\\
    \leq & \sqrt{\frac{K^3 \sum_{k=1}^K A_k}{T}} \cdot \polylog\left(K, \max_k A_k, \delta^{-1}\right)\nonumber.
  \end{align}
  Next, for the instance $\sJ_0^\CE$, the same upper bound on the DEC of $\co(\sJ_0^\CE)$ holds, but the deviation sets are larger: we have $\max_k |\Dev| = \max_k |A_k^{A_k}|$, and so \cref{thm:curse_ub} gives
    \begin{align}
    \RiskDM \leq &  O \left( K \cdot \inf_{\gamma > 0} \left\{\frac{K \sum_{k=1}^K A_k}{\gamma} + \frac{\gamma}{T} \cdot \log \left( \frac{K \cdot \max_k A_k^{A_k}}{\delta} \right)\right\} \right)\nonumber\\
    \leq & \sqrt{\frac{K^4 \max_k A_k^2}{T}} \cdot \polylog\left(K, \max_k A_k, \delta^{-1}\right)\nonumber.
  \end{align}
\end{proof}
\paragraph{Lower bounds} Next we discuss lower bounds for the instances $\sJ_0^\CCE, \sJ_0^\CE, \sJ_0^\NE$. It is straightforward to see that each of them embeds an instance of single-player $\max_k A_k$-armed bandits, by restricting the model class $\MM$ to models for which the reward distribution depends only on the action taken by any single player $k$. It then follows from the proof of Proposition 5.3 of \cite{foster2021statistical} that  $\deccpac[\vep](\sJ_0^\NE) \geq \deccpac[\vep](\sJ_0^\CE) \geq \deccpac[\vep](\sJ_0^\CCE) \geq \Omega(\vep \sqrt{\max_k A_k})$ for $\vep > 0$; in fact, these lower bounds are obtained by subclasses of $\MM$ which have $C(T) = \log(T \wedge V(\MM)) = O(1)$. Therefore, \cref{thm:constrained-lower} (with $\vepslowerT = \frac{c\sqrt{\max_k A_k}}{KT}$, for sufficiently small $c > 0$) together with \cref{prop:ma-to-pm} gives that for any of the instances $\sJ_0^\CCE, \sJ_0^\CE, \sJ_0^\NE$, and any algorithm, there is a model for which $\E[\RiskDM] \geq \Omega({\max_k A_k} / (KT))$ under any of these three instances.

For the instance $\sJ_0^\CCE$, in the learnable regime $T > \max_k A_k$, this lower bound is off from the upper bound of \cref{prop:nf-curse-ub} by a factor of $\sqrt{T/\max_k A_k} \cdot \poly(K, \max_k \log A_k, \log T)$; for $\sJ_0^\CE$, the gap increases to $\sqrt{T/\max_k A_k} \cdot \max_k \sqrt{A_k} \cdot \poly(K, \max_k \log A_k, \log T)$, and for $\sJ_0^\NE$, the gap increases further to $\sqrt{T/\max_k A_k} \cdot \sqrt{A} \cdot \polylog(T, A)$. In all these cases, the factor of $\sqrt{T}$ in the gap is due to the impossibility results discussed in \aref{sec:optimality}, and the remaining terms are due to model estimation error appearing in the upper bound but not the lower bound. In particular (up to a $O(K)$ factor), there is no gap in the upper and lower bounds we have computed on the \mashort for these instances.

\subsubsection{Normal-form games with linear payoffs}
\label{sec:linear-bandit-games}
In this section we bound the DEC and minimax regret for the linearly structured normal-form game instances defined in \cref{ex:linear-bandits}.
In particular, fix action sets $\MA_k \subset \BR^{d_k}$ for each $k \in [K]$, as well as convex sets $\Theta_1, \ldots, \Theta_K \subset \BR^d$ (with $d = d_1 \cdots d_K$) so that $| \lng a_1 \otimes \cdots \otimes a_K, \theta_k \rng | \leq 1$ for all $(a_1, \ldots, a_K) \in  \MA_1 \times \cdots \times \MA_K, \ k \in [K],\ \theta_k \in \Theta_k$.  Let $\sJ^\NE, \sJ^\CE, \sJ^\CCE$ denote the NE, CE, and CCE instances constructed given the sets $\MA_k, \Theta_k$ as in \cref{ex:linear-bandits}. The below proposition bounds the (regret) offset DEC of these instances:
\begin{proposition}
  \label{prop:dec-linear-bandits}
  For any $\gamma > 0$, the instances $\sJ^\NE, \sJ^\CE, \sJ^\CCE$ defined above satisfy
  \begin{align}
\decoreg(\sJ^\CCE) \leq \decoreg(\sJ^\CE) \leq \decoreg(\sJ^\NE) \leq \frac{K \cdot \sum_{k=1}^K d_k}{\gamma}\nonumber.
  \end{align}
\end{proposition}
\begin{proof}[\pfref{prop:dec-linear-bandits}]
  The proof is essentially identical to that of \cref{prop:nf-dec-bound}, except that each induced model class $\wt{\cM}_k$ can be viewed as a single-agent linear bandit problem in $d$ dimensions, allowing us to use Proposition 6.1 of \cite{foster2021statistical} to bound the DEC for (single-player) linear bandits.\arxiv{ \noah{need to add more detail?}}
\end{proof}
Using \cref{prop:dec-linear-bandits}, we now bound the minimax rates for the instances $\sJ^\NE, \sJ^\CE, \sJ^\CCE$. As in the previous subsection, to simplify matters, we restrict the instances so that the model class is constrained to models which output (random) rewards that take values in $\{-1,1\}$ (recall that, for the linear bandit instances defined in \cref{ex:linear-bandits}, $\fm_k(a) \in [-1,1]$ for all $M \in \MM, a \in \MA$). Furthermore, we assume that for each $k$, $d_k \geq 2$ and all $\theta_k \in \Theta_k$ satisfy $\| \theta_k\|_2 \leq D$ and all $a_k \in \MA_k$ satisfy $\| a_k \|_2 \leq D$ for some $D > 0$. It follows that $\| a_1 \otimes \cdots \otimes a_K \|_2 \leq D^K$ for all $a_1\in\MA_1, \ldots, a_k\in\MA_k$; our bounds depend only logarithmically on $D$. We denote the corresponding \ma instances with $\{-1,1\}$-valued rewards by $\sJ^\NE_0, \sJ^\CE_0, \sJ^\CCE_0$. First, we bound the minimax rate for $\sJ^\NE_0$:
\begin{proposition}
  \label{prop:linbandit-risk}
For any $T \in \bbN, \delta \in (0,1)$, there is an algorithm for the instance $\sJ^\NE_0$ which guarantees that with probability at least $1-\delta$, $\RiskDM \leq \sqrt{(\max_k d_k) \cdot d \cdot T^{-1}} \cdot \polylog(T, d,\delta^{-1})$, where $d = d_1 d_2 \cdots d_K$.
\end{proposition}
\begin{proof}[\pfref{prop:linbandit-risk}]
  Analogous to our notation for finite-action normal-form  games, let us write $\sJ_0^\NE = (\MM, \Pi^\NE, \MO, \{ (\Dev)^\NE \}_k, \{ \Sw^\NE\}_k)$. The combination of \cref{prop:nf-dec-bound} and \cref{prop:constrained-to-offset} yields that $\deccpac[\vep](\sJ_0^\NE) \leq 2\vep \cdot \sqrt{K \cdot \sum_{k=1}^K d_k}$. For any $\pi \in \Pi^\NE \subset \Delta(\MA)$, the distribution on $\MO = \MR^K \times \Ocirc = \MR^K \times \MA$ defined by $\nu(\cdot | \pi) := \Unif(\{-1,1\}^K) \times \pi$ verifies that $\MM$ satisfies \cref{ass:kernel-B} with $B = 2^K$, and therefore \cref{prop:estimation-infinite} gives that $\Est(T, \delta) = O(\est(\MM, T)+ \log \delta^{-1}) \cdot K^2$. Finally, note that a product of $\vep^2/D^K$-covers of $\Theta_k$, for $k \in [K]$, with respect to the Euclidean norm yields a $\vep$-model class cover of $\MM$ in the sense of \cref{def:model-cover}. Since each $\Theta_k$ has a $\vep^2/D^K$-cover of size $O(D^{K+1}/\vep^2)^d$, it follows that $\MN(\MM, \vep) \leq (D^{K+1}/\vep^2)^{Kd}$, which implies that $\est(\MM, T) \leq O(K^2d \cdot \log(TD))$. Therefore, \cref{thm:constrained-upper} combined with \cref{prop:ma-to-pm} gives that there is an algorithm with
  \begin{align}
\RiskDM \leq \sqrt{K (d_1 + \cdots + d_K)} \cdot \sqrt{\frac{\Est(T, \delta)}{T}} \cdot \polylog(T, 1/\delta) \leq \sqrt{\frac{\max_k d_k \cdot d}{T}} \cdot \polylog(T, 1/\delta, d)\nonumber.
  \end{align}
\end{proof}
As in the case of finite-action normal-form games, the upper bound in \cref{prop:linbandit-risk} (which also applies to the instances $\sJ_0^\CE, \sJ_0^\CCE$) suffers from the curse of multiple agents. For the instance $\sJ_0^\CCE$, we obtain improved bounds with minimax risk scaling only polynomially with $K$ by appealing to our results in \cref{sec:curse}.
\begin{proposition}
  \label{prop:linbandit-curse}
  For any $T \in \bbN, \delta \in (0,1)$, there is an algorithm for the instance $\sJ_0^\CCE$ which guarantees that with probability at least $1-\delta$,
  \begin{align}
\RiskDM \leq \sqrt{\frac{K^5 \cdot \max_k \{ d_k \}}{ T}} \cdot \log(KDT/\delta).\nonumber
  \end{align}
\end{proposition}
One might wonder whether a similar bound can be established for the instance $\sJ_0^\CE$. According to our definition of $\sJ_0^\CE$ (which is a CE instance per \cref{def:ce-instance}) we have $|\Dev| = |\MA_k|^{|\MA_k|}$ for each $k$, meaning that the upper bound of \cref{thm:curse_ub} would yield a risk bound with polynomial dependence on $|\MA_k|$, which is unacceptable in the linear bandit setting since $\MA_k$ is often taken to be exponentially large or infinite. Even if we were to attempt to use \cref{cor:curse_ub_infinite} to decrease the size of the deviation sets, the only choice of deviation set that works generically is $\til \Pi_k' := \til \Sigma_k^{\til \Sigma_k}$, which has logarithm scaling exponentially in the dimension $d_k$. A more promising avenue is to consider notions of equilibria between CCE and CE (sometimes known as $\Phi$-equilibria), as in, e.g., \cite{gordon2008noregret,anagnostides2022near,mansour2022strategizing}; we leave this direction for future work.
\begin{proof}[\pfref{prop:linbandit-curse}]
  The proposition follows as a consequence of \cref{cor:curse_ub_infinite}. Paralleling our notation for normal-form games, let us write $\sJ_0^\CCE = (\MM, \Pi^\CCE, \MO, \{ (\Dev)^\CCE \}_k, \{ \Sw^\CCE \}_k)$. Let us write $\MA_1 \otimes \cdots \otimes \MA_K := \{ a_1 \otimes \cdots \otimes a_K \ : \ a_1 \in \MA_1, \ldots, a_K \in \MA_K \}$. For each $k \in [K]$, there is an $\vep/(KD^K)$-cover with respect to the $\ell_2$-norm of $\MA_k$ of size at most $O(KD^{K+1}/\vep)^{d_k}$. Let us denote such a cover by $\til\MA_k \subseteq \MA_k$. Let us write $\til \MA = \til \MA_1 \times \cdots \times \til \MA_K$, $\til \Pi^\CCE = \Delta(\til \MA)$, and $(\til \Pi_k')^\CCE := \til \MA_k \cup \{ \perp \}$ for each $k \in [K]$. Consider any model $M \in \MM$. Note that, for any $k \in [K]$, and $a_k \in \MA_k$, there is some $\til a_k' \in \til \MA_k$ so that for all $\til a \in \til \MA$, 
  \begin{align}
    | \fm_k(\Sw(a_k, \til a)) - \fm_k(\Sw(\til a_k', \til a)) | =& | \lng \til a_1 \otimes \cdots \otimes a_k \otimes \cdots \otimes \til a_K, \theta_k\sups{M} \rng - \lng \til a_1 \otimes \cdots \otimes \til a_k' \otimes \cdots \otimes \til a_K , \theta_k\sups{M} \rng |\nonumber\\
    \leq & \| \theta_k\sups{M} \|_2 \cdot D^{K-1} \cdot \| \til a_k - \til a_k' \|_2 \leq \vep/K\nonumber,
  \end{align}
  which in particular implies that the instance $\til \sJ_0^\CCE := (\MM, \til \Pi^\CCE, \MO, \{ (\til \Pi_k')^\CCE \}_k, \{ \Sw^\CCE \}_k)$ is a $\vep$-decision space cover for $\sJ^\CCE$ (per \cref{def:decision-cover}). It therefore follows that $\est_\Pi(\sJ_0^\CCE, T) \leq K \cdot \max_k \{ d_k\} \cdot \log(KDT)$. We have $\decoreg(\co(\sJ_0^\CCE)) \leq \decoreg(\co(\sJ^\CCE)) = \decoreg(\sJ^\CCE) \leq \frac{K \cdot \sum_{k=1}^K d_k}{\gamma}$ by \cref{prop:dec-linear-bandits} and convexity of the class $\MM$, which follows since the sets $\Theta_k$ are convex. By \cref{cor:curse_ub_infinite}, we have that there is an algorithm with
  \begin{align}
    \RiskDM \leq & O(K) \cdot \inf_{\gamma > 0} \left\{ \frac{K \cdot \sum_{k=1}^K d_k}{\gamma} + \frac{\gamma}{T} \cdot K \cdot \max_k \{ d_k\} \cdot \log(KDT/\delta) \right\}\nonumber\\
    \leq & \sqrt{\frac{K^5 \cdot \max_k \{ d_k \}}{ T}} \cdot \log(KDT/\delta)\nonumber.
  \end{align}
\end{proof}

\paragraph{Lower bounds} We now derive lower bounds for the instances $\sJ_0^\CCE, \sJ_0^\CE, \sJ_0^\NE$ under the assumption that $\MA_k, \Theta_k$ contain the unit $\ell_2$ ball in their respective spaces.\footnote{Analogous lower bounds can be obtained under alternative action and parameter sets; for instance, if $\MA_k$ each contains the $\ell_1$ ball and $\Theta_k$ each contains the $\ell_\infty$ ball, then we can embed the normal-form game setting from the previous subsection.}
It follows from the proof of Proposition 6.2 of \cite{foster2021statistical} that  $\deccpac(\sJ_0^\NE) \geq \deccpac(\sJ_0^\CE) \geq \deccpac(\sJ_0^\CCE) \geq \Omega(\vep \sqrt{\max_k d_k})$ for $\vep > 0$, using the fact that each of these instances embeds an instance of single-player linear bandits in dimension $\max_k d_k$.
Therefore, \cref{thm:constrained-lower} (with $\vepslowerT = \frac{c\sqrt{\max_k d_k}}{KT}$, for sufficiently small $c > 0$) together with \cref{prop:ma-to-pm} gives that for any of the instances $\sJ_0^\CCE, \sJ_0^\CE, \sJ_0^\NE$, and any algorithm, there is a model for which $\E[\RiskDM] \geq \Omega({\max_k d_k} / (KT))$ under any of these three instances. Similar considerations apply to the gaps between the upper and lower bounds as discussed in \aref{sec:normal-form-games}. 

\subsubsection{Concave (bandit) games}
\label{sec:concave_dec}
We now bound the DEC and minimax regret for the normal-form games wth concave rewards given in \cref{ex:convex-bandits}. Fix sets $\MA_k \subset \BR^{d_k}$ and the class $\MF \subset (\MA \ra \BR^K)$ as described in \cref{ex:convex-bandits}. We assume that $\| a \|_2 \leq D$ for all $a \in \MA$, for some $D > 0$.; our bounds depend only logarithmically on $D$. Let $\sJ^\NE = (\MM, \Pi^\NE, \MO, \{ (\Dev)^\NE \}_k, \{ \Sw^\NE \}_k)$ denote the NE instance constructed in \cref{ex:convex-bandits}, and let $\sJ^\CCE = (\MM, \Pi^\CCE, \MO, \{ (\Dev)^\CCE \}_k, \{ \Sw^\CCE \}_k)$ denote the CCE instance constructed in \cref{ex:convex-bandits}.\footnote{As we have done previously, we use the model class $\MM$ for both instances $\sJ^\NE, \sJ^\CCE$, where it is understood that models have domain appropriate for each instance.} The below proposition bounds the (regret) offset DEC of these instances:
\begin{proposition}
  \label{prop:concave-game-dec}
  For any $\gamma > 0$, the instances $\sJ^\NE, \sJ^\CCE$ defined above satisfy
  \begin{align}
\decoreg(\sJ^\CCE) \leq \decoreg(\sJ^\NE) \leq \frac{K \cdot \sum_{k=1}^K d_k^4}{\gamma} \cdot \polylog\left(\max_k \{d_k\}, D, \gamma\right)\nonumber.
  \end{align}
\end{proposition}
\begin{proof}[\pfref{prop:concave-game-dec}]
  The fact that $\decoreg(\sJ^\CCE) \leq \decoreg(\sJ^\NE)$ follows from the fact that $\Pi^\NE$ may be identified as a subset of $\Pi^\CCE$ (namely, $\Pi^\NE$ consists of singleton distributions in $\Pi^\CCE$), in a similar manner to the proof of \cref{prop:nf-dec-bound}. To prove the second upper bound, we will use \cref{thm:single-multiple-ch} applied to the instance $\sJ^\NE$, which gives that $\decoreg(\sJ^\NE) \leq \sum_{k=1}^K \decoreg[\gamma/K](\til \MM_k)$, for $\til \MM_k$ defined as in \cref{eq:define-til-mk}. In turn, to bound the DEC of $\til \MM_k$, we define the model class $\MM_k' \subset (\MA_k \ra \Delta(\MR \times \{ \perp \})$, by $\MM_k' = \{ M \ : \ \fm(\cdot) \mbox{ is concave} \}$. Since, for any $k \in [K]$, $M \in \MM^\NE$, $a_{-k} \in \MA_{-k}$, there is a model $M_k' \in \MM_k'$ so that, for all $a_k \in \MA_k$, the distribution of $r \sim M_k'(a_k)$ is the same as the distribution of $r_k \sim M(a_k, a_{-k})$, it holds that $\decoreg[\gamma/K](\til \MM_k) \leq \decoreg[\gamma/K](\MM_k')$. Finally, Proposition 6.3 of \cite{foster2021statistical} (which is a restatement of Theorem 3 of \cite{lattimore2020improved}) gives that, for all $\gamma' > 0$, $\decoreg[\gamma'](\MM_k') \leq \frac{d_k^4}{\gamma} \cdot \polylog(d_k, \diam(\MA_k), \gamma)$, which yields that $\decoreg[\gamma](\sJ^\NE) \leq \frac{K \sum_{k=1}^K d_k^4}{\gamma} \cdot \polylog\left( \max_k \{ d_k \}, D, \gamma \right)$.  
\end{proof}

We now turn our attention to bounding the minimax risk. The model classes $\MM^\NE, \MM^\CCE$ for our concave game instances are extremeley large: 
any cover of $\MM^\NE$ or $\MM^\CCE$ in the sense of \cref{def:model-cover} must have logarithm exponential in the dimensions $d_k$, so the model-based guarantee from \cref{thm:constrained-upper} is not particularly interesting, even in the case where $K$ is small. Therefore, we turn directly to the policy-based guarantees given in \cref{sec:curse}, and will prove a minimax risk upper bound for the instance $\sJ^\CCE$. It turns out that such an upper bound will immediately imply upper bounds for the instance $\sJ^\NE$, under the following assumption, specializing \citet{even2009convergence}. %
\begin{assumption}[Zero-sum socially concave]
  \label{ass:socially-concave}
We say that a model class $\MM$ is \emph{zero-sum socially concave} if for all $M \in \MM$, $k \in [K]$ and $a_k \in \MA_k$, the mapping $\MA_{-k} \ni a_{-k} \mapsto f_k\sups{M}(a_k, a_{-k})$ is a convex function and for all $a \in \MA$, $\sum_{k=1}^K \fm_k(a) = 0$.
\end{assumption}
In the special case that the model $M$ is a two-player zero-sum concave game (i.e., $\fm_1 + \fm_2 \equiv 0$), zero-sum social concavity necessarily holds.

\begin{proposition}
  \label{prop:concave-games-risk}
 Then for any $T \in \bbN, \delta \in (0,1)$, there is an algorithm for the instance $\sJ^\CCE$ which guarantees that with probability at least $1-\delta$,
 \begin{align}
   \label{eq:risk-concave-game}
\RiskDM \leq \frac{K^2 \cdot \max_k \{ d_k^{2.5} \}}{\sqrt T} \cdot \polylog\left(D, T, \gamma, \max_k \{ d_k \}, K, 1/\delta\right).
  \end{align}
    Suppose further that the model class $\MM$ is zero-sum socially concave (i.e., it satisfies \cref{ass:socially-concave}). Then there is an algorithm for the instance $\sJ^\NE$ which guarantes the same upper bound on risk in \eqref{eq:risk-concave-game} with probability at least $1-\delta$. 
\end{proposition}
\begin{proof}[\pfref{prop:concave-games-risk}]
  For each $k \in [K]$, there is an $\vep$-cover with respect to the $\ell_2$-norm of $\MA_k$ of size at most $O(D/\vep)^{d_k}$. Let us denote such a cover by $\til \MA_k \subset \MA_k$.  Write $\til \MA := \til \MA_1 \times \cdots \times \til \MA_K$. %
  Let  $\til \Pi^\CCE := \Delta(\MA)$, and $(\til \Pi_k')^\CCE := \til \MA_k \cup \{ \perp\}$. Note that, for any $M \in \MM$, $k \in [K]$, and $a_k \in \MA_k$, there is some $\til a_k' \in \til \MA_k$ so that for all $\til a \in \til \MA$, 
  \begin{align}
| \fm_k(\Sw(a_k, \til a)) - \fm_k(\Sw(\til a_k', \til a))| = |\fm_k(a_k, \til a_{-k}) - \fm_k(\til a_k', \til a_{-k})| \leq \| a_k -\til a_k' \|_2 \leq \vep\nonumber,
  \end{align}
  where the first inequality uses 1-Lipschitzness of $\fm_k(\cdot)$. Hence, the CCE instance $\til \sJ^\CCE := (\MM, \til \Pi^\CCE, \MO, \{ (\til \Pi_k')^\CCE \}_k, \{ \Sw^\CCE \}_k)$ is an $\vep$-decision space cover for $\sJ^\CCE$ (per \cref{def:decision-cover}). It follows that $\est_\Pi(\sJ^\CCE, T) \leq \max_k \{ d_k \} \cdot \log(DT)$.

  Next, we have $\decoreg(\co(\sJ^\CCE)) = \decoreg(\sJ^\CCE) \leq \frac{K \cdot \sum_{k=1}^K d_k^4}{\gamma} \cdot \polylog(\max_k \{ d_k\}, D, \gamma)$ by \cref{prop:concave-game-dec} and convexity of the class $\MM$ (which follows since the convex combination of concave and 1-Lipschitz functions is concave and 1-Lipschitz). By \cref{cor:curse_ub_infinite}, for any $T\in\bbN$ and $\delta>0$, there is an algorithm which outputs $\wh \pi \in \Pi^\CCE$ so that with probability at least $1-\delta$,
  \begin{align}
    \RiskDM \leq & K \cdot \inf_{\gamma > 0} \left\{ \frac{K \cdot \sum_{k=1}^K d_k^4}{\gamma} + \frac{\gamma}{T} \cdot \max_k \{ d_k \} \right\} \cdot \polylog\left(D, T, \gamma, \max_k \{ d_k \}, K, 1/\delta\right)\nonumber\\
    \leq & \frac{K^2 \cdot \max_k \{ d_k^{2.5} \}}{\sqrt T} \cdot \polylog\left(D, T, \gamma, \max_k \{ d_k \}, K, 1/\delta\right)\nonumber.
  \end{align}
  Next we prove the upper bound for $\sJ^\NE$. As we have done previously in this section, for $\wh \pi \in \Pi^\CCE$ and $M \in \MM$, we write $h\sups{M, \CCE}(\wh \pi)$ to denote the suboptimality of $\wh\pi$ with respect to the instance $\sJ^\CCE$, and for $\wh\pi\in \Pi^\NE$, we write $h\sups{M, \NE}(\wh \pi)$ to denote the suboptimality of $\wh\pi$ with respect to the instance $\sJ^\NE$. Given $\wh \pi \in \Pi^\CCE$, define $\wh a := \E_{a \sim \wh \pi}[a] \in \MA$. For each $k \in [K]$ and $M \in \MM$, we have that
  \begin{align}
    h\sups{M, \NE}(\wh a) =& \sum_{k=1}^K \max_{a_k' \in \MA_k} \fm_k(a_k', \wh a_{-k}) - \fm_k(\wh a) \nonumber\\
    \leq & \sum_{k=1}^K \max_{a_k' \in \MA_k} \E_{a_{-k} \sim \wh \pi}[\fm_k(a_k', a_{-k})] - \fm_k(\wh a)\nonumber\\
    = & \sum_{k=1}^K \max_{a_k' \in \MA_k} \E_{a_{-k} \sim \wh \pi}[\fm_k(a_k', a_{-k})] - \E_{a \sim \wh \pi}[\fm_k(a)] = h\sups{M, \CCE}(\wh \pi)\nonumber, 
  \end{align}
  where the first inequality follows from social concavity and the second equality follows from the fact that $\sum_{k=1}^K \fm_k(\wh a) = 0 = \sum_{k=1}^K \E_{a \sim \wh \pi}[\fm_k(a)]$. Thus, given a decision $\wh \pi$ output by our algorithm for the instance $\sJ^\CCE$, we may simply output $\wh a=\En_{a\sim\wh\pi}\brk{a}$, which yields the same upper bound on risk. 
\end{proof}

\paragraph{Lower bounds} Assume that $\MA = \MA_1 \times \cdots \times \MA_K$ contains the unit $\ell_2$-ball. Then the instances $\sJ^\CCE, \sJ^\NE$ each embed a single-player linear bandit instance with dimension $\max_k d_k$ (namely, by taking the subclass of $\MF$ to consist of linear functions in $a_k$ only), and so the lower bounds from \aref{sec:linear-bandit-games} give $\deccpac(\sJ^\NE) \geq \deccpac(\sJ^\CCE) \geq \Omega(\vep \sqrt{\max_k d_k})$ and a minimax risk lower bound of $\Omega({\max_k d_k}/(KT))$. In this setting, even the DEC lower bound (in the single-agent setting) is off from the upper bound implied by \cref{prop:concave-game-dec} and \cref{prop:constrained-to-offset} \citep{foster2021statistical}.

\subsubsection{Tabular Markov games}
\label{sec:tabular_dec}
We next give bounds on minimax risk for the instances corresponding to Markov Nash equilibria, Markov CE, and Markov CCE in tabular Markov games, as described in \cref{ex:tabular-mg}. %
Given $H \in \BN$, state spaces $\MS_h$ each of size $S$, action spaces $\MA_k$ of size $A_k := \MA_k$, and an initial distribution $d^1 \in \Delta(\MS_1)$, let $\sJ^\NE = (\MM, \Pi^\NE, \MO, \{ (\Dev)^\NE \}_k, \{ \Sw^\NE \}_k)$, $\sJ^\CE = (\MM, \Pi^\CE, \MO, \{ (\Dev)^\CE \}_k, \{ \Sw^\CE \}_k)$, and $\sJ^\CCE = (\MM, \Pi^\CCE, \MO, \{ (\Dev)^\CCE \}_k, \{ \Sw^\CCE \}_k)$ be the \MAFrameworkShort instances corresponding to Markov Nash equilibria, Markov CE, and Markov CCE as defined in \cref{ex:tabular-mg}. Technically, the model class  $\MM$ for $\sJ^\NE$ acts on policies in $\Pi^\NE$, whereas the model class $\MM$ for $\sJ^\CE$ and $\sJ^\CCE$ acts on policies in $\Pi^\CE = \Pi^\CCE \neq \Pi^\NE$; we will write the model class for each instance as $\MM$ and formally interpret its domain as the appropriate decision space, to avoid cluttering notation. %

In \cref{prop:tabmg-dec} below, we begin with an upper bound on their offset DEC, which immediately yields an upper bound on the constrained DEC via \cref{prop:constrained-to-offset}.
\begin{proposition}
  \label{prop:tabmg-dec}
  For any $\gamma > 0$, and any $\Mbar \in \MM$, the instances $\sJ^\NE, \sJ^\CE, \sJ^\CCE$ defined above satisfy
  \begin{align}
\decoreg(\sJ^\CCE, \Mbar) \leq \decoreg(\sJ^\CE, \Mbar) \leq \decoreg(\sJ^\NE, \Mbar) \leq  \frac{27 K H^3 \log(H) S \sum_{k=1}^K A_k}{\gamma}\nonumber. %
  \end{align}
\end{proposition}
\begin{proof}[\pfref{prop:tabmg-dec}]
  As in the proof of \cref{prop:nf-dec-bound}, we augment the functions $\fm(\cdot)$ and $\hm(\cdot)$ with the superscripts NE/CE/CCE to distinguish between the value functions for models in the three different instances. For example, for the instance $\sJ^\NE$, we have, for $M \in \MM, \pi \in \Pi^\NE$,
  \begin{align}
f_k\sups{M, \NE}(\pi) := \E\sups{M, \pi}\left[ \sum_{h=1}^H r_{k,h} \right], \qquad h\sups{M, \NE}(\pi) =\sum_{k=1}^K \max_{\dev \in (\Dev)^\NE} f_k\sups{M, \NE}(\Sw^\NE(\dev, \pi)) - f_k\sups{M, \NE}(\pi).\nonumber
  \end{align}
  The functions $h\sups{M, \CE} : \Pi^\CE \ra \BR$ and $h\sups{M, \CCE} : \Pi^\CCE \ra \BR$ are defined similarly.

  We have $\Pi^\CE = \Pi^\CCE$; furthermore, for any $M \in \MM$ and $\pi \in \Pi^\CE = \Pi^\CCE$, we have that $h\sups{M, \CCE}(\pi) \leq h\sups{M, \CE}(\pi)$. Thus $\decoreg(\sJ^\CCE, \Mbar) \leq \decoreg(\sJ^\CE, \Mbar)$. Next, note that we may identify $\Pi^\NE$ as a subset of $\Pi^\CE$ as follows: for $\pi = (\pi_1, \ldots, \pi_K) \in \Pi^\NE$, we associate it to the joint Markov policy $\til \pi = (\til \pi_1, \ldots, \til \pi_H) \in \Pi^\CE$ where $\til \pi_h(s_h)$ is the product distribution $\til \pi_h(s_h) := \pi_{1,h}(s_h) \times \cdots \times \pi_{K,h}(s_h)$. It is straightforward to see that, for such $\pi$ and any model $M \in \MM$, the distributions of $M(\pi)$ and $M(\til \pi)$ are identical. Accordingly, with slight abuse of notation, for $\pi \in \Pi^\NE$, we denote its corresponding policy in $\Pi^\CE$ as $\pi$ as well. Thus we have $h\sups{M,\NE}(\pi) = h\sups{M, \CE}(\pi)$, and for any $\Mbar \in \MM$, we have
    \begin{align}
    \decoreg[\gamma](\sJ^\NE, \Mbar) =& \inf_{p \in \Delta(\Pi^\NE)} \sup_{M \in \MM} \E_{\pi \sim p} \left[ h\sups{M,\NE}(\pi) - \gamma \cdot \hell{M(\pi)}{\Mbar(\pi)} \right]\nonumber\\
    \geq & \inf_{p \in \Delta(\Pi^\CE)} \sup_{M \in \MM} \E_{\pi \sim p} \left[ h\sups{M,\CE}(\pi) - \gamma \cdot \hell{M(\pi)}{\Mbar(\pi)} \right] = \decoreg(\sJ^\CE, \Mbar)\nonumber.
    \end{align}
    It remains to upper bound $\decoreg(\sJ^\NE)$. For $k \in [K]$, let $\Pi_k$ be the class of randomized Markov policies of player $k$ (so that $\Pi^\NE = \Pi_1 \times \cdots \times \Pi_K$). For each $k \in [K]$, define the model class $\til \MM_k \subset (\Pi_k \to \Delta(\MR \times \Ocirc))$ as in \cref{eq:define-til-mk}:
    \begin{align}
      \til\MM_\ag = \left\{ \pi_\ag \mapsto \single{M}(\pi_\ag, \pi_{-\ag}) \ : \ \pi_{-\ag} \in \Pi_{-\ag},\ M \in \MM\right\}\nonumber.
    \end{align}
    Define $\MM_k'$ to be the model class consisting of all horizon-$H$ Markov decision processes with action set $\MA_k$ and state spaces $\MS_1, \ldots, \MS_H$, and so that the sum of rewards under any trajectory that occurs with positive probability is bounded in $[0,1]$. Formally, the pure observation space of $\MM_k'$ is the space $\Ocirc'$ of trajectories $\{ (s_h, a_{k,h}, r_{k,h}) \}_{h \in [H]}$, with $s_h \in \MS_h, a_{k,h} \in \MA_k, r_{k,h} \in \BR$, its reward space is $\MR = [0,1]$, and its decision space is $\Pi_k$. Thus $\MM_k' \subset (\Pi_k \ra \Delta(\MR \times \Ocirc'))$. Proposition 5.4 of \cite{foster2021statistical} shows that for all $\Mbar \in \MM_k'$, $\decoreg(\MM_k', \Mbar) \leq 26 \frac{H^2 S A_k}{\gamma}$.

    Next, fix $\til M \in \til \MM_k$. By definition of $\til \MM_k$, we can find $\Mbar \in \MM$ and $\wb\pi_{-k} \in \Pi_{-k}$ so that $\til M(\pi_k) = \single{\Mbar}(\pi_k, \wb\pi_{-k})$ for all $\pi_k \in \Pi_k$. Let $\Mbar' \in \MM_k'$ be the unique model so that for all $\pi_k \in \Pi_k$, the marginal distribution of $\{ (s_h, a_{k,h}, r_{k,h}) \}_{h \in [H]}$ for a trajectory drawn from $ \til M(\pi_k)$ is identical to the distribution of the pure observation drawn from  $\Mbar'(\pi_k)$. Such a model exists, since for each state $s_h \in \MS_h$ and action $a_{k,h} \in \MA_k$, the transition distribution $P_h\sups{\Mbar'}(\cdot | s_h, a_{k,h}) \in \Delta(\MS_{h+1})$ is defined as $\E_{a_{k',h} \sim \wb\pi_{k',h}(s_h) \ \forall k' \neq k}[P_h\sups{\Mbar}(\cdot | s_h, (a_{k,h}, a_{-k,h}))]$ and the reward distribution $R_{k,h}\sups{\Mbar'}(s_h, a_{k,h}) \in \Delta(\bbR)$ is defined as $\E_{a_{k',h} \sim \wb\pi_{k',h}(s_h) \ \forall k' \neq k}[R_h\sups{\Mbar}(s_h, (a_{k,h}, a_{-k,h}))]$. 
    We now compute
    \begin{align}
      &\decoreg(\til \MM_k, \til M) \\
      &= \inf_{p \in \Delta(\Pi_k)} \sup_{M \in \MM, \pi_{-k} \in \Pi_{-k}} \E_{\pi_k \sim p} \left[ \max_{\pi_k' \in \Pi_k} f_k\sups{M, \NE}(\pi_k', \pi_{-k}) - f_k\sups{M, \NE}(\pi) - \gamma \cdot \hell{M(\pi_k, \pi_{-k})}{\til M(\pi_k)} \right]\nonumber\\
      &\leq  \inf_{p \in \Delta(\Pi_k)} \sup_{M' \in \MM_k'} \E_{\pi_k \sim p} \left[ \max_{\pi_k' \in \Pi_k} f_k\sups{M'}(\pi_k') - f_k\sups{M'}(\pi_k) - \gamma \cdot \hell{M'(\pi_k)}{\Mbar'(\pi_k)}\right]\nonumber\\
      &=  \decoreg(\MM_k', \Mbar')\nonumber,
    \end{align}
    where the inequality follows since, via the same argument used to construct $\Mbar'$, for any $M \in \MM, \pi_{-k} \in \Pi_{-k}$, there is some $M' \in \MM_k'$ so that for any $\pi_k \in \Pi_k$, the marginal distribution of $\{ (s_h, a_{k,h}, r_{k,h})\}_{h \in [H]}$ for a trajectory drawn from $M(\pi_k, \pi_{-k})$ is the same as the distribution of a trajectory drawn from $M'(\pi_k)$. In addition, we have applied the data processing inequality for the Hellinger distance to conclude that $\DhelsX{\big}{M(\pi_k, \pi_{-k})}{\til M(\pi_k)}$ is an upper bound for the squared Hellinger distance between the marginal distributions of $\{ (s_h, a_{k,h}, r_{k,h}) \}_{h \in [H]}$ under $M(\pi_k, \pi_{-k})$ and $\til M(\pi_k)$. Finally, by \cref{thm:single-multiple-mg} applied to the instance $\sJ^\NE$, we have that
    \begin{align}
      \sup_{\Mbar \in \MM} \decoreg(\sJ^\NE, \Mbar) \leq & \frac{CKH \log H}{\gamma} + \sum_{k=1}^K \sup_{\til M_k \in \til \MM_k} \decoreg[\gamma/(CKH\log H)](\til \MM_k, \til M_k) \nonumber\\
      \leq & \frac{CKH\log H}{\gamma} + \sum_{k=1}^K \sup_{\Mbar_k' \in \MM_k'} \decoreg[\gamma/(CKH\log H)](\MM_k', \Mbar_k')\nonumber\\
      \leq & \frac{CKH \log H}{\gamma} + \sum_{k=1}^K 26 H^2 SA_k \cdot \frac{CKH \log H}{\gamma}\nonumber\\
      \leq & \frac{27 K H^3 \log(H) S \sum_{k=1}^K A_k}{\gamma}\nonumber.
    \end{align}
  \end{proof}
    Using \cref{prop:tabmg-dec}, we now bound the minimax rates for the instances $\sJ^\NE, \sJ^\CE, \sJ^\CCE$. To simplify matters, we assume that reward distributions are known. Formally, we fix some functions $R_{k,h}^\st : \MS_h \times \MA \ra \Delta(\MR)$ (for $k \in [K], h \in [H]$) and restrict the model class $\MM$ to models $M \in \MM$ for which $R_{k,h}^\st(s_h, a) \equiv R_{k,h}\sups{M}(s_h, a) \in \Delta([0,1/H])$ for all $M \in \MM$. We also assume that $A_k \geq 2$ for all $k$. With the functions $R_{k,h}^\st$ fixed, let us denote the resulting instances by $\sJ_0^\NE, \sJ_0^\CE, \sJ_0^\CCE$.\footnote{Essentially the same argument in \cref{prop:mg-risk} allows us to upper bound the minimax risk for the original instances $\sJ^\NE, \sJ^\CE, \sJ^\CCE$, for which rewards are not known, but doing so requires a renormalization argument (and the loss of a factor of $H$) to ensure rewards are always bounded in $[0,1]$, which we omit for brevity.}
  \begin{proposition}
    \label{prop:mg-risk}
    There is an algorithm for each of the instances $\sJ_0^\NE, \sJ_0^\CE, \sJ_0^\CCE$ which guarantees that with probability at least $1-\delta$, $\RiskDM \leq \sqrt{\frac{\max_k A_k \cdot AS^3H^4 }{T}} \cdot \polylog(T, \delta^{-1}, A, S, H)$. 
  \end{proposition}
  \begin{proof}[\pfref{prop:mg-risk}]
    Note that it suffices to bound the minimax risk for the instance $\sJ_0^\NE$, since for any $\wh \pi \in \Pi^\NE \subset \Pi^\CE = \Pi^\CCE$, we have that $h\sups{M,\CCE}(\wh\pi) \leq h\sups{M, \CE}(\wh \pi) \leq h\sups{M,\NE}(\wh \pi)$. (Recall the definition of $h\sups{M, \CCE}, h\sups{M, \CE}, h\sups{M, \NE}$ in the proof of \cref{prop:tabmg-dec}.) 
    The combination of \cref{prop:tabmg-dec} and \cref{prop:constrained-to-offset} yields that, for any $\Mbar \in \MM$, \[\deccpac[\vep](\sJ_0^\NE, \Mbar) \leq O \prn*{\vep \cdot \sqrt{KH^3\log(H) S \sum_{k=1}^K A_k} }.\]

  Because of the constraint that $\Mbar \in \MM$ in the DEC upper bound, we need a \emph{proper} estimation algorithm, i.e., one with $\wh \MM = \MM$ (in the context of \cref{ass:estimation}). To do so, we use the approach of \emph{layer-wise} estimators from \cite{foster2021statistical}. Note that the model class $\MM$ has the product structure $\MM = \MM_1 \times \cdots \times \MM_H$, where each $\MM_h$ is the set of transition kernels $\MS_h \times \MA \ra \Delta(\MS_{h+1})$, which is a convex set, thus satisfying Assumption 7.2 of \cite{foster2021statistical}. Furthermore, by gridding the transition densities into multiples of $\vep^2$, we have that $\MN(\MM_h, \vep) \leq (1/\vep^2)^{S^2 A}$, and therefore, by Proposition 7.1 and Lemma A.16 of \cite{foster2021statistical}, there is an estimation algorithm $\AlgEst$ with $\wh \MM = \MM$ and which has estimation error $\Est(T, \delta) \leq O(S^2 AH) \cdot \polylog(S, H, \delta^{-1}, T)$. 
  Therefore, \cref{thm:constrained-upper} combined with \cref{prop:ma-to-pm} gives that there is an algorithm with
  \begin{align}
    \RiskDM \leq & \sqrt{KH^3 \log(H) S \sum_{k=1}^K A_k} \cdot \sqrt{\frac{\Est(T, \delta)}{T}} \cdot \polylog(T, 1/\delta) \nonumber\\
    \leq &  \sqrt{\frac{\max_k A_k \cdot AS^3H^4 }{T}} \cdot \polylog(T, \delta^{-1}, A, S, H)\nonumber.
  \end{align}
\end{proof}

\paragraph{Lower bounds} It is straightforward to see that for any $k$, each of the instances $\sJ^\NE, \sJ^\CE, \sJ^\CCE$ embeds an instance corresponding the class of single-player MDPs on state spaces $\MS_h$, action space $\MA_k$, and horizon $H$: in particular, take the subclass of $\MM$ whose transitions and rewards only depend on player $k$'s action at each step. Then it follows from the proof of Proposition 5.8 of \cite{foster2021statistical} that $\deccpac(\sJ^\NE) \geq \deccpac(\sJ^\CE) \geq \deccpac(\sJ^\CCE) \geq \Omega(\vep \sqrt{SH \cdot \max_k A_k})$. Therefore,  \cref{thm:constrained-lower} (with $\vepslowerT = \frac{c\sqrt{SH \cdot \max_k A_k}}{KT \log T}$, for sufficiently small $c > 0$) together with \cref{prop:ma-to-pm} gives that for any of the instances $\sJ^\CCE, \sJ^\CE, \sJ^\NE$, and any algorithm, there is a model for which $\E[\RiskDM] \geq \til\Omega(SH \cdot {\max_k A_k} / (KT))$.

\subsubsection{A separation between multi-agent DEC and single-agent DEC}
\label{app:single_multiple_separation}
In the previous subsections, we bounded the multi-agent DEC, and thereby the minimax risk (via an application of \cref{thm:constrained-upper} and \cref{prop:ma-to-pm}), for several bandit problems. In all cases, our upper bound on the multi-agent DEC (for CCE, CE, and Nash instances) followed via an application of \cref{thm:single-multiple-ch} to upper bound the multi-agent DEC by the single-agent DEC of the model classes $\til \MM_k$ defined in \eqref{eq:define-til-mk}. The next (straightforward) proposition shows that this approach is not tight in general, indicating that the multi-agent DEC represents a fundamental complexity measure that is distinct from existing ones.
\begin{proposition}
  \label{prop:multi-single-separation}
For any $K, A \in \bbN$, there is a $K$-player \ma NE instance $\sJ = \instma$ so that $\decoreg(\sJ) = 0$ but $\decoreg(\til \MM_k) \geq \Omega(A/\gamma)$ for all $\gamma > 0$, where $\til \MM_k$ are defined as in \cref{eq:define-til-mk}.
\end{proposition}
\begin{proof}[\pfref{prop:multi-single-separation}]
  Fix $K, A \in \bbN$, and set $\Pi_k = \{0, 1, \ldots, A \}$ for each $k$. Let $\MR := [-1,1]$ and $\Ocirc := \MA$.  Define $\MF \subseteq (\Pi \ra \MR^K)$ to be the class of all tuples $(f_1, \ldots, f_K)$ with $f_k : \Pi \ra \MR^K$ with the property that for all $\pi \in \Pi$, if there is any $k$ so that $\pi_k = 0$, then $f_{k'}(\pi) = 0$ for all $k' \in [K]$. Set $\MM := \MM_\MF$, and define $\Dev, \Sw$ as in \cref{def:ne-instance}.

  Define $\pi_0 = (0, \ldots, 0)$. Since $\hm(\pi_0) = 0$ for all $M \in \MM$, it follows that $\decoreg[\gamma](\sJ) = 0$. On the other hand, it is straightforward to see that each class $\til \MM_k$ embeds a standard multi-armed bandit instance with $A$ arms, meaning that by Proposition 5.3 of \cite{foster2021statistical}, we have that $\decoreg(\til \MM_k) \geq \Omega(A/\gamma)$ for all $\gamma > 0$. 
\end{proof}

}

\colt{
  \part{Main results}
  \label{part:main}

  \section{Organization of appendix}
  \label{sec:organization}

  \section{Preliminaries}
  \label{sec:add-prelim}

  \section{Equivalence of \MAFrameworkShort and \FrameworkShort frameworks}
  \label{sec:relations}

  \section{Upper and lower bounds on minimax rates}
  \label{sec:bounds}

  \section{\MAFrameworkShort: From Multi-Agent to Single-Agent}
  \label{sec:single-multiple}

  \section{\MAFrameworkShort: On the Curse of Multiple Agents}
  \label{sec:curse}

  \part{Examples}
  \label{part:examples}
  \section{\MAFrameworkShort: Examples of instances}
  \label{app:ma_examples}
  
}

  \part{Proofs}
  \label{part:proofs}
  \section{Technical tools}
  \label{app:technical}
  \subsection{Information theory}

In this section we collect several technical lemmas which are used in our proofs.
\begin{lemma}
  \label{lem:fdiv-factor}
  Let $(\MX, \sX)$, $(\MI, \mathscr{I})$ be  measure spaces. Suppose that for each $i \in \MI$, there are distributions $P_i, P_i' \in \Delta(\MX)$, and $Q \in \Delta(\MI)$. Suppose further that there is a measurable function $\varphi : \MX \ra \MI$ so that, for each $i \in \MI$, $\BP_{x \sim P_i}(\varphi(x) = i) = \BP_{x \sim P_i'}(\varphi(x) = i) = 1$. %
  Then for any $f$-divergence $\Dgen{\cdot}{\cdot}$, it holds that
  \begin{align}
\Dgen{\E_{i \sim Q}[P_i]}{\E_{i \sim Q}[P_i']} = \E_{i \sim Q}[\Dgen{P_i}{P_i'}]\nonumber.
  \end{align}
\end{lemma}
\begin{proof}[\pfref{lem:fdiv-factor}]
  That $\Dgen{\E_{i \sim Q}[P_i]}{\E_{i \sim Q}[P_i']} \leq \E_{i \sim Q}[\Dgen{P_i}{P_i'}]$ follows from convexity of $\Dgen{\cdot}{\cdot}$. To establish the opposite direction, %
  our assumption on the function $\varphi$ together with the data processing inequality yields 
  \begin{align}
\Dgen{\E_{i \sim Q}[P_i]}{\E_{i \sim Q}[P_i']} \geq & \Dgen{\E_{i \sim Q}[\Indic{i} \times P_i]}{\E_{i \sim Q}[\Indic{i} \times P_i']} = \E_{i \sim Q}[\Dgen{P_i}{P_i'}]\nonumber,
  \end{align}
  where the final inequality follows from, e.g., \citet{polyanskiy2014lecture}.
\end{proof}

\begin{lemma}[e.g., \citet{polyanskiy2014lecture}]
  \label{lem:fdiv-conditioning}
  Let $(\MX, \sX)$ and $(\MY, \sY)$ be measure spaces, and let $\MX \times \MY$ be equipped with the product sigma-algebra $\sX \otimes \sY$. Let $(x,y)$ be a pair of random variables on $\MX \times \MY$, distributed according to some distribution $\BP_{x,y}$. 
  For any $f$-divergence $\Dgen{\cdot}{\cdot}$,  it holds that
  \begin{align}
\E_{x \sim \BP_x} \left[ \Dgen{\BP_{y | x}}{\BP_y} \right] = \E_{y \sim \BP_y} \left[ \Dgen{\BP_{x | y}}{\BP_x} \right]\nonumber.
  \end{align}
\end{lemma}

\begin{lemma}[Lemma B.5 of \cite{foster2022complexity}]
  \label{lem:hellinger-variational-bounded}
  Let $\BP, \BQ$ be probability distributions on a measure space $(\MX, \mathscr{X})$. For any $\alpha \geq 1$, let $\MG_\alpha := \{ g : \MX \ra \BR \ : \ \| g\|_\infty \leq \alpha \}$. Then
  \begin{align}
\frac 12 \hell{\BP}{\BQ} \leq \sup_{g \in \MG_\alpha} \left\{ 1 - \E_\BP[e^g] \cdot \E_\BQ[e^{-g}] \right\} + 4 e^{-\alpha}\nonumber.
  \end{align}
\end{lemma}

\begin{lemma}[e.g., \citet{foster2022complexity}]
  \label{lem:min-hell-exp}
  Consider measure spaces $(\MX, \mathscr{X}),\ (\MY, \mathscr{Y})$, and let $(x,y)$ be a pair of random variables distributed according to some distribution $\BP_{x,y}$ on $(\MX \times \MY, \mathscr{X} \otimes \mathscr{Y})$. Then
  \begin{align}
\E_{x \sim \BP_x}\left[ \hell{\BP_{y|x}}{\BP_y}\right] \leq 4 \cdot \inf_{\BQ \in \Delta(\MY)} \E_{x \sim \BP_x}\left[\hell{\BP_{y|x}}{\BQ}\right]\nonumber.
  \end{align}
\end{lemma}
\begin{proof}[\pfref{lem:min-hell-exp}]
  Consider any $\BQ \in \Delta(\MY)$. Using the fact that the Hellinger distance satisfies the triangle inequality, we have
  \begin{align}
    \E_{x \sim \BP_x}\left[ \hell{\BP_{y|x}}{\BP_y}\right] \leq & \E_{x \sim \BP_x} \left[ 2 \cdot \hell{\BP_{y|x}}{\BQ} + 2 \cdot \hell{\BQ}{\BP_y} \right]\nonumber\\
    \leq & 2 \cdot \E_{x \sim \BP_x} \left[ \hell{\BP_{y|x}}{\BQ} \right] + 2 \cdot \hell{\BQ}{\E_{x \sim \BP_x}[\BP_{y|x}]}\nonumber\\
    \leq & 4 \cdot \E_{x \sim \BP_x} \left[ \hell{\BP_{y|x}}{\BQ} \right]\nonumber,
  \end{align}
  where the final inequality follows from convexity of the squared Hellinger distance.
\end{proof}

  \begin{lemma}[Donsker-Varadhan; see \cite{polyanskiy2014lecture}]
    \label{lem:dv}
   Let $(\MX, \mathscr{X})$ be a measure space, and let $\BP, \BQ$ be probability measures on $(\MX, \mathscr{X})$. Then
    \begin{align}
\kld{\BP}{\BQ} = \sup_{h : \MX \ra \BR} \left\{ \E_{X \sim \BP}[h(X)] - \log  \E_{X \sim \BQ}[\exp(h(X))]\right\}\nonumber,
    \end{align}
    where the supremum is over all (measurable) functions $h : \MX \ra \BR$ satisfying $\E_{X \sim \BQ}[\exp(h(X))] < \infty$. 
  \end{lemma}

  \begin{lemma}
  \label{lem:pq-tvd-cond}
Let $\BP, \BQ$ be probability measures on some probability space $(\Omega, \CF)$. Consider some event $\ME \in \CF$ so that $\BP(\ME) \geq 1-\delta$, for some $\delta \in (0,1)$. Suppose also that for all events $\MF \in \CF$, we have $\BP(\ME \cap \MF) = \BQ(\ME \cap \MF)$. Then $\tvd{\BP}{\BQ} \leq \delta$.
\end{lemma}
\begin{proof}[\pfref{lem:pq-tvd-cond}]
  Choosing $\ME' = \Omega$ gives $\BQ(\ME) =\BP(\ME)\geq 1-\delta$. Then for any event $\MF \in \CF$, we have
  \begin{align}
    |\BP(\MF) - \BQ(\MF)| \leq & |\BP(\MF \cap \ME) - \BQ(\MF \cap \ME)| + |\BP(\MF \cap  \ME^\c) - \BQ(\MF \cap \ME^\c)| \nonumber\\
    = & |\BP(\MF \cap \ME^\c) - \BQ(\MF \cap \ME^\c)|\leq \delta\nonumber.
  \end{align}
\end{proof}

  \subsection{Concentration inequalities}
    \begin{lemma}[Lemma A.4 of \cite{foster2021statistical}]
    \label{lem:chernoff-martingale}
    Let $(X_t)_{t \in [T]}$ be any sequence of real-valued random variables adapted to a filtration $\CF\^t$. Then with probability at least $1-\delta$,
    \begin{align}
\sum_{t=1}^T X_t \leq \sum_{t=1}^T \log\left( \E\left[ e^{X_t} | \CF\^{t-1}\right] \right) + \log (1/\delta)\nonumber.
    \end{align}
  \end{lemma}

\subsection{Topological lemmas}
The below lemma is a special case of the Berge maximum theorem.
\begin{lemma}
  \label{lem:berge-theorem}
Let $\MU, \MV$ be compact subsets of Euclidean space, and consider any continuous function $G : \MU \times \MV \ra \BR$. Define $\MC : \MV \ra \MP(\MU)$ by $\MC(v) := \argmin_{u \in \MU} \{ G(u,v)\}$. Then $\MC$ is upper hemicontinuous.
\end{lemma}
\begin{proof}[\pfref{lem:berge-theorem}]
  Consider any sequences $u_n \ra u \in \MU$, $v_n \ra v \in \MV$ so 
  that $u_n \in \MC(v_n)$ for all $n$. We wish to show that $u \in
  \MC(v)$, i.e., $G(u,v) \leq G(u',v)$ for all $u' \in \MU$ (which suffices to prove upper hemicontinuity by compactness of $\MV$; see \cite[Lemma 6.2.6]{beer1993topologies}). 
  \dfcomment{would be good to add reference to show this is sufficient
  (under compactness)}\noah{added} To do so, fix any $u' \in \MU$ and $\ep > 0$. There exists $N$ so that for $n \geq N$, we have $|G(u_n, v_n) - G(u,v)| \leq \ep$ and $|G(u', v_n) - G(u', v)| \leq \ep$, by continuity of $G$. Then
  \begin{align}
G(u,v) \leq G(u_n, v_n) + \ep \leq G(u', v_n) + \ep \leq G(u', v) + 2\ep,\nonumber
  \end{align}
  and by taking $\ep \ra 0$ we get that $G(u,v) \leq G(u',v)$. 
\end{proof}

The next lemma is a straightforward consequence of Kakutani's fixed
point theorem. In its statement, we write $\MX_{-k} := \prod_{k' \neq
  k} \MX_{k'}$ and $\MX = \prod_{k \in [K]} \MX_k$.
\dfcomment{I don't think $\powerset{\cdot}$ is defined until later in the
  appendix - move to prelims? also, is there a reason to use $2^{\cU}$ in lemma above
  instead of $\powerset{\cU}$?}\noah{fixed both}
\begin{lemma}
  \label{lem:kakutani}
  Suppose that $\MX_1, \ldots,\MX_K$ are nonempty, compact, and convex
  subsets of Euclidean space. Suppose that for each $k \in [K]$ we are
  given an upper hemicontinuous function $F_k : \MX \ra \powerset{\MX_k}$ so that, for all $x \in \MX$, $F_k(x)$ is nonempty, closed, and convex. Then there is some $x \in \MX$ so that
  \begin{align}
  x \in   F_1(x) \times \cdots \times F_K(x) \nonumber.
  \end{align}
\end{lemma}
\begin{proof}[\pfref{lem:kakutani}]
Define $F : \MX \ra \powerset{\MX}$ by $F(x) := F_1(x) \times \cdots
\times F_K(x)$. It is evident that for each $x \in \MX$, $F(x)$ is
nonempty, closed, and convex. Furthermore, we claim that $F$ is upper
hemicontinuous. To see this, consider any sequences $x_n \ra x$  and
$y_n \ra y$ so that $y_n \in F(x_n)$ for each $n \in \BN$. Writing
$y_n = (y_{n,1}, \ldots, y_{n,K})$ and $y = (y_1, \ldots, y_K)$, by
the product structure of $F(x_n)$, we have that $y_{n,k} \in F_k(x_n)$
for each $k \in [K]$. By upper hemicontinuity of $F_k$ and the fact
that $y_{n,k} \ra y_k$, it holds that $y_k \in F_k(x)$. Thus $y \in
F(x)$. By Kakutani's fixed point theorem \cite[Lemma 20.1]{osborne1994course}, it holds that $F$ has a fixed point, namely some $x \in \MX$ so that $x \in F(x)$. 
\end{proof}

\subsection{Minimax theorem}
\begin{theorem}[Sion's minimax theorem]
  \label{thm:sion}
  Let $\MX, \MY$ be convex subsets of topological vector spaces, with $\MX$ compact. Let $F : \MX \times \MY \ra \BR$ be a function such that (a) the mapping $y \mapsto F(x,y)$ is concave and upper semicontinuous for all $x \in \MX$, and (b) the mapping $x \mapsto F(x,y)$ is convex and lower semicontinuous for all $y \in \MY$. Then
  \begin{align}
\inf_{x \in \MX} \sup_{y \in \MY} F(x,y) = \sup_{y \in \MY} \inf_{x \in \MX} F(x,y)\nonumber.
  \end{align}
\end{theorem}

\section{Proofs for \cref{sec:relations}}
\begin{proof}[\pfref{prop:ma-to-pm}]
Consider an instance $\MAI = (\MM, \Act, \Ocirc, \MR, \{ \Dev\}_\ag, \{ \Sw\}_\ag)$ of the \MAFrameworkShort framework. %

  For all models $M$ and decisions $\act \in \Act$, define\footnote{The addition of $K$ in the definition of $\til{f}\sups{M}(\act)$ is for convenience, so as to ensure that if $\hm(\act) \in [0,K]$ for all $M, \act$, then the same holds for $\til{f}\sups{M}(\act)$.}
  \begin{align}
\til{f}\sups{M}(\act) := K - \hm(\pi) =  K - \sum_{\ag=1}^\Ag \sup_{\dev \in \Dev} \left\{ \fm_\ag(\Sw(\dev, \act)) - \fm_\ag(\act) \right\}\nonumber.
  \end{align}
Now fix any $M \in \MM$. By \cref{ass:existence-eq}, there is some $\act^\st \in \Act$ so that $\hm(\act^\st) = 0$. By \cref{ass:nonneg-dev}, it holds that $\hm(\act) \geq 0$ for all $\act \in \Act$. Then %
  \begin{align}
    \sup_{\act' \in \Act} \til{f}\sups{M}(\act') - \til{f}\sups{M}(\act) = K - (K - \hm(\act)) = \hm(\act) \label{eq:rew-equivalence}.
  \end{align}
  Note that the instance $\HRI = (\MM, \Act, \MO, \{ \til{f}\sups{M}\}_M)$ is well-defined since models in $M$ are probability kernels $M : \Act \ra \MO = \Ocirc \times \MR^\Ag$, and the observation space in the instance $\HRI$ is by definition $\MO$. Thus, the first claimed point is follows from \cref{eq:rew-equivalence} since for any $\Mbar$, we have: %
  \begin{align}
\deccpac(\MAI, \Mbar) &= \inf_{p,q \in \Delta(\Pi)} \sup_{M \in \MH_{q,\vep}(\Mbar)} \E_{\pi \sim p}[\hm(\pi)] \\&= \inf_{p,q \in \Delta(\Pi)} \sup_{M \in \MH_{q,\vep}(\Mbar)} \E_{\pi \sim p}\left[\sup_{\pi' \in \Pi} \til{f}\sups{M}(\pi') - \til{f}\sups{M}(\pi)\right] = \deccpac(\HRI, \Mbar)\nonumber.
  \end{align}

  Finally, we note that since the decision and (full) observation spaces of $\sI, \sJ$ are identical, the space of algorithms $(p,q)$ and distributions $\BP^{\Mstar, (p,q)}$ are identical in the two frameworks. It follows from the definitions of $\mf M(\sI, T)$ and $\mf M(\sJ, T)$ that they are equal. %
\end{proof}

\begin{proof}[\pfref{prop:pm-to-ma}]
  Consider an instance $\sI = \instpm$ and some $V \in \BN$. We first specify the instance $\sJ$ by defining each of its components:
  \begin{itemize}
  \item Define $\Sigma_1 = \Pi$ and $\Sigma_2 = \{ 0, 1, \ldots, V \}$, $\til \Pi_k = \Delta(\Sigma_k)$ for $k \in \{1,2\}$, and $\til\Pi :=\til \Pi_1 \times \til \Pi_2$. \dfcomment{Should we define $\til\Pi_k=\Delta(\Sigma_k)$ here to follow the NE instance convention?}\noah{did so}
  \item We define $\Dev, \Sw$ for $\ag \in [2]$ in the standard fashion for NE instances, per \cref{def:ne-instance}; in particular, set $\Dev = \til \Pi_k = \Delta(\Sigma_k)$ for each $\ag$ and $\Sw(\dev, \act) = (\dev, \act_{-\ag})$.  \dfcomment{doesn't \cref{def:ne-instance} set $\Dev=\Pi_k=\Delta(\Sigma_k)$?}\noah{changed accordingly}
  \item Define $\Ocirc := \MO \cup \{ \perp \}$, $\MR = [-1,1]$, and set $\til \MO = \Ocirc \times \MR^2$.
  \item The model class $\til \MM$ is indexed by tuples $(M, v) \in \MM \times \{ 1, 2, \ldots, V \} = \MM \times [V]$. In particular, for each such tuple $(M, v)$, we have a model $\til M\subs{M,v}\in\til\cM$, which is defined as explained below. As the instance $\sJ$ we are constructing corresponds to that of computing mixed Nash equilibria in a game whose pure action sets are $\Sigma_1, \Sigma_2$, we call elements of $\Sigma_1 \times \Sigma_2$ \emph{pure decisions}.  \dfcomment{I guess this is actually using the CCE notation from Definition 1.2, just specialized to zero-sum nash (e.g., definition 1.1. doesn't mention pure decisions)}\noah{added clarifying sentence}
    \begin{itemize}
    \item For pure decisions of the form $(\sigma_1, 0) \in \Sigma_1 \times \Sigma_2$, the distribution of $(\ocirc, r_1, r_2) \sim \til M\subs{M, v}((\sigma_1, 0))$ is given by:
      \begin{align}
\ocirc \sim M(\sigma_1) \in \MO \subset \Ocirc, \qquad r_1 = r_2 = 0\nonumber.
      \end{align}
    \item For pure decisions of the form $(\sigma_1, i) \in \Sigma_2 \times \Sigma_2$ with $i > 0$, the distribution of $(\ocirc, r_1, r_2) \sim \til M\subs{M, v}((\sigma_1, i))$ is given by:
      \begin{align}
        \ocirc = \perp, \qquad r_2 = -r_1 = \begin{cases}
          -1 &: i \neq v \\
          \gm(\sigma_1) &: i = v.
        \end{cases}\nonumber
      \end{align}
    \item For general decisions $\pi \in \til \Pi$, we can write $\pi = \pi_1 \times \pi_2$ for $\pi_\ag \in \Delta(\Sigma_\ag)$ for $\ag \in [2]$. Then the distribution $\til M\subs{M, v}(\pi)$ is the distribution of $\til M\subs{M,v}(\sigma)$, for $\sigma = (\sigma_1, \sigma_2)$ is distributed as: $\sigma_\ag \sim \pi_\ag$ for $\ag \in [2]$. 
    \end{itemize}
  \end{itemize}

For reference later in the proof, we state a basic technical lemma, which is an immediate consequence of the construction of $\til \MM$.
  \begin{lemma}
    \label{lem:ne-lowerbound}
    For any $\pi = \pi_1 \times \pi_2 \in \til \Pi$, and any $M' = \til M\subs{M, v} \in \til \MM$, it holds that
    \begin{align}
\hmprime(\pi) \geq   \pi_2(\Sigma_2 \backslash \{ v \}) \cdot \E_{\sigma_1 \sim \pi_1}[\gm(\sigma_1)] +  \pi_2(\Sigma_2 \backslash \{ 0,v \})\nonumber.
    \end{align}
  \end{lemma}
  \begin{proof}[\pfref{lem:ne-lowerbound}]
    By considering the deviation $\dev[2] = v$, we have
    \begin{align}
      \hmprime(\pi) \geq & \sup_{\dev[2] \in \Dev[2]} \{ \fmprime_2(\Sw[2](\dev[2], \pi)) - \fmprime_2(\pi) \} \nonumber\\
      \geq & \fmprime_2( \pi_1 \times \indic_v) - \fmprime_2(\pi_1 \times \pi_2)\nonumber\\
      = & \E_{\sigma_1 \sim \pi_1}[\gm(\sigma_1)] + \pi_2(\Sigma_2 \backslash \{0,v \}) - \pi_2(v) \cdot \E_{\sigma_1 \sim \pi_1}[\gm(\sigma_1)]\nonumber\\
      =& \pi_2(\Sigma_2 \backslash \{ v \}) \cdot \E_{\sigma_1 \sim \pi_1}[\gm(\sigma_1)] + \pi_2(\Sigma_2 \backslash \{0,v\})\nonumber.
    \end{align}
  \end{proof}

  \paragraph{Bounding $\mf M(\sI, T)$ by $\mf M(\sJ, T)$}
  Consider any algorithm $(\til p, \til q)$ which achieves $\mf M(\sJ, T)$. We have $\til p : \prod_{t=1}^T (\til \Pi \times \til \MO) \ra \Delta(\til \Pi)$ and $\til q = (\til q\^1, \ldots, \til q\^T)$, with each $\til q\^t : \prod_{i=1}^{t-1} (\til \Pi \times \til \MO) \ra \Delta(\til \Pi)$ (we refer to \cref{sec:add-prelim} for background on how algorithms in the \maf and \hrf are formalized).

  Given $(\til p, \til q)$, we define an algorithm $(p,q)$ for the instance $\sI$ as follows. For any model $M \in \MM$, the algorithm attempts to simulate the interaction of $(\til p, \til q)$ with $\til M\subs{M, v}$ by only interacting with $M \in \MM$. The algorithm will store internal state, denoted by $(\til \pi\^i, (\til o\^i, r_1\^i, r_2\^i))$, for each $i \in [T]$, which store the ``simulated'' decisions and observations taken with respect to $\til M\subs{M, v}$. As a result of this internal state, our description below does not explicitly identify the probability kernels $p(\cdot | \cdot), \ q\^t(\cdot | \cdot)$.  Since these kernels take as input the entire history, there exist kernels $p(\cdot | \cdot), q\^t(\cdot | \cdot)$ which produce exactly the same distribution over trajectories as the below algorithm, but writing them down explicitly is somewhat cumbersome.%
  
  In particular, the distributions $q\^t$ (for $t \in [T]$) and $p$ are defined (implicitly) as follows: %
  \begin{enumerate}
  \item For $t = 1, 2, \ldots, T$:
    \begin{enumerate}
  \item Draw $\til\pi\^t \sim \til q\^t(\cdot | (\til \pi\^1, (\til o\^1, r_1\^1, r_2\^t)), \ldots, (\til \pi\^{t-1}, (\til o\^{t-1}, r_1\^{t-1}, r_2\^{t-1})))$, so that $\pi\^t \in \til \Pi$. %
  \item Draw $(\sigma_1\^t, \sigma_2\^t) \sim \til\pi\^t$.
  \item The distribution $q\^t$ is defined (implicitly) by taking the decision $\sigma_1\^t \in \Sigma_1 = \Pi$.
  \item For use in choosing future decisions: as a function of the observation $o\^t$ received after $\sigma_1\^t$ is played, define
    \begin{align}
      (\til o\^t, r_1\^t, r_2\^t) = \begin{cases}
        (o\^t, 0, 0) &: \sigma_2\^t = 0 \\
        (\perp, 1, -1) &: \sigma_2\^t > 0.
      \end{cases}\nonumber
    \end{align}
  \end{enumerate}
\item   Finally, the distribution $p$ is defined as the distribution of $\wh \sigma_1$, where $\wh \pi \sim \til p(\cdot | (\til \pi\^1, (\til o\^1, r_1\^t, r_2\^t)), \ldots, (\til \pi\^T, (\til o\^T, r_1\^T, r_2\^T)))$ and $\wh \sigma_1 \sim \wh \pi_1$.
\end{enumerate}
To analyze this algorithm, for each $M \in \MM$, we introduce a model $\til M\subs{M, 0}$ which is defined identically to $\til M\subs{M, v}$ for any $v \in [V]$ except that $\til M\subs{M, 0}((\sigma_1, i))$ outputs $(\perp, 1, -1)$ a.s.~for any $\sigma_1 \in \Sigma_1, i \in [V]$. It is straightforward to see that if there is some underlying model $M \in \MM$ so that $o\^t \sim M(\sigma_1\^t)$ when the algorithm $(p,q)$ defined above is used, then the distribution of $\{ (\til \pi\^t, (\til o\^t, r_1\^t, r_2\^t)) \}_{t=1}^T$ defined above is exactly the distribution of the history under $\BP\sups{\til M\subs{M, 0}, (\til p, \til q)}$. We next appeal to the following claim, which states that we can pass from this distribution to the distribution $\BP\sups{\til M\subs{M, v^\st}, (\til p, \til q)}$ for some $v^\st \in [V]$:
  \begin{lemma}
    \label{lem:ustar-hidden}
There is an absolute constant $C>0$ so that for any choice of algorithm $(\til p, \til q)$ and model $M \in \MM$, there exists $v^\st \in [V]$ so that:
    \begin{enumerate}
    \item $\tvd{\BP\sups{\til M\subs{M, 0}, (\til p, \til q)}}{
        \BP\sups{\til M\subs{M, v^\st}, (\til p, \til q)}} \leq C \sqrt{T \log(T)/V}$.
    \item $\E\sups{\til M\subs{M, v^\st}, (\til p, \til q)} \E_{\wh \pi \sim \til p} [\wh \pi_2(v^\st)] \leq C \sqrt{T \log(T)/V}$. 
      \end{enumerate}
  \end{lemma}

The proof of \cref{lem:ustar-hidden} is provided following in the sequel. Let $\delta \ldef C \sqrt{T \log(T)/V}$, where $C$ is the constant from \cref{lem:ustar-hidden}. If $\delta > 1$, then it is immediate that $\mf M(\sI, T) \leq \mf M(\sJ, T) + O(\delta)$, so we may assume henceforth that $\delta \leq 1$. We then have:
  \begin{align}
    \E\sups{M, (p,q)} \E_{\wh \sigma_1 \sim p} [ \gm(\wh \sigma_1)] &= \E\sups{\til M\subs{M, 0}, (\til p, \til q)} \E_{\wh \pi \sim \til p} \E_{\wh \sigma_1 \sim \wh \pi_1} [\gm(\wh \sigma_1)]\nonumber\\
    &\leq \E\sups{\til M\subs{M, v^\st}, (\til p, \til q)} \E_{\wh \pi \sim \til p} \E_{\wh \sigma_1 \sim \wh \pi_1} [\gm(\wh \sigma_1)] + \delta\nonumber\\
      &\leq \E\sups{\til M\subs{M, v^\st}, (\til p, \til q)} \E_{\wh \pi \sim \til p} \left[ \min \left\{ \frac{h\sups{\til M\subs{M, v^\st}}(\wh \pi)}{1 - \wh \pi_2(v^\st)}, 1 \right\} \right] + \delta \nonumber\\
    &\leq (1 + \sqrt \delta ) \cdot  \E\sups{\til M\subs{M, v^\st}, (\til p, \til q)} \E_{\wh \pi \sim \til p} [ h\sups{\til M\subs{M, v^\st}}(\wh \pi)] + 2 \sqrt \delta \nonumber\\
    &\leq \E\sups{\til M\subs{M, v^\st}, (\til p, \til q)} \E_{\wh \pi \sim \til p} [ h\sups{\til M\subs{M, v^\st}}(\wh \pi)] + 6 \sqrt \delta \nonumber,
  \end{align}
  \dfcomment{is the third line supposed to have min rather than max?}\noah{yep, fixed}
  where the first inequality follows from the first point of \cref{lem:ustar-hidden}, the second inequality follows from \cref{lem:ne-lowerbound}, the second-to-last inequality uses the second point of \cref{lem:ustar-hidden} together with Markov's inequality to conclude that $\BP\sups{\til M\subs{M, v^\st}, (\til p, \til q)}(\wh \pi_2(v^\st) \geq \sqrt \delta) < \sqrt \delta $, and the final inequality uses that $\hmprime(\pi) \leq 4$ for all $M' \in \til \MM, \pi \in \til \Pi$. 

  Taking a supremum over all models $M \in \MM$, we conclude that
  \begin{align}
\sup_{M \in \MM} \E\sups{M, (p,q)}\E_{\wh \sigma_1 \sim p}[\gm(\wh \sigma_1)] \leq & \sup_{\til M \in \til \MM} \E\sups{\til M, (\til p, \til q)} \E_{\wh \pi \sim \til p}[h\sups{\til M}(\wh \pi)] + 6\sqrt{\delta}\nonumber.
  \end{align}

  \paragraph{Bounding $\mf M(\sJ, T)$ by $\mf M(\sI, T)$}
  Consider any algorithm $(p,q)$ which achieves $\mf M(\sI, T)$. We have $p : \prod_{t=1}^T (\Pi \times \MO) \ra \Delta(\Pi)$, and $q = (q\^1, \ldots, q\^T)$, with each  $q\^t : \prod_{i=1}^{t-1} (\Pi \times \MO) \ra \Delta(\Pi)$.

  We define an algorithm $(\til p, \til q)$ for the instance $\sJ$ as follows. Given $\pi\^t \in \til \MO, o\^t \in \Ocirc, r_1\^t, r_2\^t \in \MR$ for each $t \in [T]$, we define
  \begin{align}
\til p(\cdot | (\pi\^1, (o\^1, r_1\^1, r_2\^1)), \ldots, (\pi\^T, (o\^T, r_1\^T, r_2\^T))) \in \Delta(\til\Pi)\nonumber
  \end{align}
  to be the distribution obtained by sampling $(\sigma_1\^t, \sigma_2\^t) \sim \pi\^t$ for each $t$, and taking the pure decision $(\wh \sigma, 0)$, where $\wh \sigma$ is distributed according to $p(\cdot | (\sigma_1\^1, o\^1), \ldots, (\sigma_1\^T, o\^T))$. Similarly, define
  \begin{align}
\til q\^t(\cdot | (\pi\^1, (o\^1, r_1\^1, r_2\^1)), \ldots, (\pi\^{t-1}, (o\^{t-1}, r_1\^{t-1}, r_2\^{t-1}))) \in \Delta(\til\Pi)\nonumber
  \end{align}
  to be the distribution obtained by sampling $(\sigma_1\^i, \sigma_2\^i) \sim \pi\^i$ for each $i<t$, and taking the pure decision $(\sigma_1\^t, 0)$, where $\sigma_1\^t$ is distributed according to $q\^t(\cdot | (\sigma_1\^1, o\^1), \ldots, (\sigma_1\^{t-1}, o\^{t-1}))$. Since each $\til q\^t$ is supported only on (pure) decisions in $\Sigma_1 \times \{ 0 \}$, for any model $M \in \MM$ and any $v \in [V]$, letting $M' = \til M\subs{M, v}$, the distribution of $\left\{ (\sigma_1\^t, o\^t))\right\}_{t=1}^T$ under $\BP\sups{M', (\til p, \til q)}$ %
  is the same as the distribution of $\left\{ (\pi\^t, o\^t) \right\}_{t=1}^T$ under $\BP\sups{M, (p,q)}$. Thus, we have %
  \begin{align}
    \E\sups{M, (p,q)} \E_{\wh\pi \sim p} \left[ \gm(\wh\pi) \right] =& \E\sups{M', (\til p, \til q)} \E_{(\wh \sigma_1, 0) \sim \til p}[\gm(\wh \sigma_1)]\nonumber\\
    =&  \E\sups{M', (\til p, \til q)} \E_{(\wh \sigma_1, 0) \sim \til p}\left[\sup_{\dev[2] \in \Dev[2]} \{ \fmprime_2(\Sw[2](\dev[2], (\wh \sigma_1, 0))) - \fmprime_2((\wh \sigma_1, 0)) \}\right] \nonumber\\
    =& \E\sups{M', (\til p, \til q)} \E_{(\wh \sigma_1, 0) \sim \til p} \left[ \hmprime((\wh \sigma_1, 0)) \right]\nonumber,
  \end{align}
where above we have shortened  $p = p(\cdot | (\pi\^1, o\^1), \ldots, (\pi\^T, o\^T))$ to denote the random variable under $\BP\sups{M, (p,q)}$ and $\til p = \til p(\cdot | (\pi\^1, (o\^1, r_1\^1, r_2\^1)), \ldots, (\pi\^T, (o\^T, r_1\^T, r_2\^T)))$ to denote the random variable under $\BP\sups{M', (\til p, \til q)}$. This establishes that $\mf M(\sI, T) \leq \mf M(\sJ, T)$.

  \paragraph{Bounding $\deccpac[\vep](\sI)$ by $\deccpac[\vep'](\sJ)$} Consider any reference model $\Mbar \in \co(\MM)$. Given $\vep > 0$, set $\vep' := \vep + \sqrt{6/V}$. 
  We will upper bound $\deccpac[\vep](\sI, \Mbar)$ by $\deccpac[\vep'](\sJ, \til M)$ for some $\til M \in \co(\til\MM)$. For some distribution $\nu \in \Delta(\MM)$, we can write $\Mbar(\pi) = \E_{M \sim \nu}[M(\pi)]$ for all $\pi \in \Pi$. Define $\mu := \nu \times \Unif([V]) \in \Delta(\MM \times [V])$, and $\til M(\pi) := \E_{(M, v) \sim \mu}[\til M\subs{M,v}(\pi)]$ for all $\pi \in \Pi$. Choose some $\til p, \til q \in \Delta(\til \Pi)$ so that
  \begin{align}
\deccpac[\vep'](\sJ, \til M) = \sup_{M' \in \til \MM} \left\{ \E_{\pi \sim \til p}[\hmprime(\pi)] \ | \ \E_{\pi \sim \til q} \left[ \hell{M'(\pi)}{\til M(\pi)}\right] \leq (\vep')^2 \right\}\nonumber.
  \end{align}
  Define $p \in \Delta(\Pi)$ to be the distribution of $\sigma_1$ where $\pi =\pi_1 \times \pi_2 \sim \til p$ and $\sigma_1 \sim \pi_1$. Similarly define $q \in \Delta(\Pi)$ to be the distribution of $\sigma_1$ where $\pi=\pi_1 \times \pi_2 \sim \til q$ and $\sigma_1 \sim \pi_1$. Now choose $v^\st \in [V]$ as follows:
  \begin{align}
v^\st := \argmin_{v \in [V]} \left\{ \E_{\pi \sim \til p}[\pi_2(v)] + \E_{\pi \sim \til q}[\pi_2(v)] \right\}\nonumber,
  \end{align}
  where we have used the convention that $\pi = \pi_1 \times \pi_2$ above. Then we have
  \begin{align}
\E_{\pi \sim \til p}[\pi_2(v^\st)] + \E_{\pi \sim \til q}[\pi_2(v^\st)] \leq \frac{2}{V}\nonumber.
  \end{align}
  Consider any model $M \in \MM$, and let $M' := \til M_{M, v^\st}$. We now compute \dfcomment{min instead of max below?}\noah{yep, fixed}
  \begin{align}
    \E_{\sigma_1 \sim p}[\gm(\sigma_1)] =& \E_{\pi \sim \til p} \E_{\sigma_1 \sim \pi_1}[\gm(\sigma_1)]\nonumber\\
    \leq & \E_{\pi \sim \til p} \left[ \min \left\{ \frac{\hmprime(\pi) - \pi_2(\Sigma_2 \backslash \{0, v^\st\})}{\pi_2(\Sigma_2 \backslash \{ v^\st\})},\ 1 \right\}\right]\nonumber\\
    \leq & 2V^{-1/2} + \E_{\pi \sim \til p} \left[ \frac{\hmprime(\pi)}{1 - V^{-1/2}} \right]\nonumber\\
    \leq & 2V^{-1/2} + (1 + 2V^{-1/2}) \cdot \E_{\pi \sim \til p}[\hmprime(\pi)]\nonumber,
  \end{align}
  where the first inequality uses \cref{lem:ne-lowerbound} and the second-to-last inequality uses Markov's inequality to conclude that $\BP_{\pi \sim \til p}[\pi_2(v^\st) > V^{-1/2}] \leq 2V^{-1/2}$. Furthermore, we have
  \begin{align}
    \E_{\pi \sim \til q} \left[ \hell{M'(\pi)}{\til M(\pi)} \right] \leq & \E_{\pi \sim \til q} \left[ \E_{(\sigma_1, \sigma_2) \sim \pi} \left[ \hell{M'((\sigma_1, \sigma_2))}{\til M((\sigma_1, \sigma_2))} \right] \right]\nonumber\\
    \leq & \E_{\pi \sim \til q} \left[ \pi_2(0) \cdot \E_{\sigma_1 \sim \pi_1}\left[ \hell{M'((\sigma_1, 0))}{\til M((\sigma_1, 0))} \right] \right] \nonumber\\
                                                                     &+ \E_{\pi \sim \til q} \left[ \pi_2(v^\st) \cdot \E_{\sigma_1 \sim \pi_1}\left[ \hell{M'((\sigma_1, v^\st))}{\til M((\sigma_1, v^\st))} \right] \right]\nonumber\\
    &+\E_{\pi \sim \til q}\left[ \pi_2(\Sigma \backslash \{0, v^\st\}) \cdot \hell{\Ber(0)}{\Ber(1/V)}\right] \nonumber\\
 \leq & \E_{\sigma_1 \sim q} \left[ \hell{M'((\sigma_1, 0))}{\til M((\sigma_1, 0))}\right] + 2 \cdot \E_{\pi \sim \til q}[\pi_2(v^\st)] + 2/V \nonumber\\
    \leq & \E_{\sigma_1 \sim q}\left[ \hell{M(\pi)}{\Mbar(\pi)} \right]  + 6/V\nonumber,
  \end{align}
  \dfcomment{would be good to add a little more detail on the $\Dhels{\Ber(0)}{\Ber(1/V)}$ calculation}\noah{added}
  where the first equality uses convexity of the squared hellinger distance, the second inequality uses that $\til M$ is a mixture of $\til M\subs{M, v}$ with $v \sim \Unif([V])$, and the third inequality uses that $\hell{\Ber(0)}{\Ber(1/V)} \leq 2 \cdot \tvd{\Ber(0)}{\Ber(1/V)} = 2/V$. 
  Thus, it follows that
  \begin{align}
    \deccpac[\vep](\sI, \Mbar) \leq & \sup_{M \in \MM} \left\{ \E_{\sigma_1 \sim p}[\gm(\pi)] \ | \ \E_{\sigma_1 \sim q} \left[ \hell{M(\sigma_1)}{\Mbar(\sigma_1)} \right] \leq \vep^2 \right\}\nonumber\\
    \leq & \sup_{M' \in \til \MM} \left\{ 2V^{-1/2} + (1+ 2V^{-1/2}) \cdot \E_{\pi \sim \til p}[\hmprime(\pi)] \ | \ \E_{\pi \sim \til q} \left[ \hell{M'(\pi)}{\til M(\pi)} \right] \leq \vep^2 + 6/V \right\}\nonumber\\
    \leq & 2V^{-1/2} + (1 + 2V^{-1/2}) \cdot \deccpac[\vep + (6/V)^{-1/2}](\sJ, \til M)\nonumber\\
    \leq & 6V^{-1/2} + \deccpac[\vep + (6/V)^{-1/2}](\sJ, \til M)\nonumber.
  \end{align}

\paragraph{Bounding $\deccpac(\sJ)$ by $\deccpac(\sI)$}
  Next consider any reference model $\til M \in \co(\til \MM)$. We will upper bound $\deccpac[\vep](\sJ, \til M)$ by $\deccpac[\vep](\sI, \Mbar)$ for some $\Mbar \in \co(\MM)$. For some distribution $\mu \in \Delta(\MM \times [V])$, we have $\til M(\pi) = \E_{(M, v) \sim \mu}[ \til M\subs{M,v}(\pi)]$ for all $\pi \in \til\Pi$. Define $\Mbar$ by letting $\nu\in\Delta(\MM)$ to be the marginal of $\mu$ over $\MM$, and then: $\Mbar(\pi) := \E_{M \sim \nu}[M(\pi)]$ for each $\pi \in \Pi$. Choose some $p,q \in \Delta(\Pi)$ so that
  \begin{align}
\deccpac[\vep](\sI, \Mbar) = \sup_{M \in \MM} \left\{ \E_{\pi \sim p} [\gm(\pi)] \ | \ \E_{\pi \sim q}[\hell{M(\pi)}{\Mbar(\pi)}] \leq \vep^2 \right\}\nonumber.
  \end{align}
  Let $\til p \in \Delta(\til \Pi)$ be the distribution of $(\sigma_1, 0)$ where $\sigma_1 \sim p$, and $\til q \in \Delta(\til \Pi)$ be the distribution of $(\sigma_1, 0)$ where $\sigma_1 \sim q$. By definition of the models $\til M\subs{M,v}$, for any $M' = \til M\subs{M, v} \in \til \MM$, we have:
  \begin{align}
    \E_{\pi \sim \til p}[h\sups{M'}(\pi)] =& \E_{\pi \sim \til p} \left[ \sup_{\dev[1] \in \Dev[1]} \{ f\sups{M'}_1(\Sw[1](\dev[1], \pi)) - f\sups{M'}_1(\pi) \} + \sup_{\dev[2] \in \Dev[2]} \{ f\sups{M'}_2(\Sw[2](\dev[2], \pi)) - f\sups{M'}_2(\pi) \}  \right] \nonumber\\
    = & \E_{\pi \sim \til p} \left[ \sup_{\dev[2] \in \Dev[2]} \{ f\sups{M'}_2(\Sw[2](\dev[2], \pi)) \} \right]\nonumber\\
    =& \E_{\sigma_1 \sim p}[\gm(\sigma_1)] \nonumber. 
  \end{align}
  Furthermore, since $\til q$ is supported entirely on $\Sigma_1 \times \{ 0 \}$ and all models in $\til \MM$ have $r_1 = r_2 = 0$ a.s.~under such policies, it holds that
  \begin{align}
\E_{\pi \sim \til q} \left[ \hell{M'(\pi)}{\Mtil(\pi)} \right] = \E_{\sigma_1 \sim q}\left[ \hell{M(\pi)}{\Mbar (\pi)} \right]\nonumber,
  \end{align}
  which certifies that
  \begin{align}
\deccpac[\vep](\sJ, \til M) \leq \sup_{M' \in \til \MM} \left\{ \E_{\pi \sim \til p}[\hmprime(\pi)]\ | \ \E_{\pi \sim \til q} \left[ \hell{M'(\pi)}{\til M(\pi)} \right] \leq \vep^2 \right\} \leq \deccpac[\vep](\sI, \Mbar)\nonumber.
  \end{align}

  for each $\pi \in \Pi$, $\Mbar(\pi)$ to be the distribution of $\ocirc$ when $(\ocirc, r_1, r_2) \sim \til M\subs{M, v}((\pi, 0))$ and $(M, v) \sim \mu$. 
\end{proof}

\begin{proof}[\pfref{lem:ustar-hidden}]
 We denote a history drawn according to any of the distributions $\til M\subs{M, v}$ (for $v \geq 0$) by $\{ (\til \pi\^t, (\til o\^t, r_1\^t, r_2\^t)) \}_{t=1}^T$. Furthermore, we abbreviate $\til q\^t = \til q\^t(\cdot | (\til \pi^1, (\til o\^1, r_1\^t, r_2\^t)), \ldots, (\til \pi\^{t-1}, (\til o\^{t-1}, r_1\^{t-1}, r_2\^{t-1})))$ and $\til p = \til p(\cdot |  (\til \pi^1, (\til o\^1, r_1\^t, r_2\^t)), \ldots, (\til \pi\^{T-1}, (\til o\^{T-1}, r_1\^{T-1}, r_2\^{T-1}))))$. Define $M' := \til M \subs{M, 0}$ and choose
  \begin{align}
v^\st := \argmin_{v \in [V]} \left\{ \sum_{t=1}^T \E\sups{M', (\til p, \til q)} \E_{\til\pi\^t \sim \til q\^t} [\til\pi\^t_2(v)] + T \cdot  \E\sups{M', (\til p, \til q)} \E_{\wh \pi \sim \til p}[\wh \pi_2(v)] \right\}\nonumber.
  \end{align}
  Then the choice of $v^\st$ together with the fact that all $\pi \in \til \Pi$ satisfy $\sum_{v=1}^V \pi_2(v) \leq 1$
  ensures that
  \begin{align}
\sum_{t=1}^T \E\sups{M', (\til p, \til q)} \E_{\til\pi\^t \sim \til q\^t} [\til\pi\^t_2(v^\st)] + T \cdot  \E\sups{M', (\til p, \til q)} \E_{\wh \pi \sim \til p}[\wh \pi_2(v^\st)]  \leq \frac{2T}{V}\label{eq:ustar-alg-guarantee}.
  \end{align}
  Write $M'' = \til M\subs{M, v^\st}$. Next, using \cite[Lemma A.13]{foster2021statistical},\footnote{In particular, we apply this lemma to the sequence $X_1, \ldots, X_{2T}$, where, for odd values of $t$ we have $X_t = \til \pi\^t$, $X_{t+1} = (\til o\^t, r_1\^t, r_2\^t)$, and use that the conditional distribution of $\til \pi\^t$ given the history up to step $t-1$ is the same under the distributions $\BP\sups{M', (\til p, \til q)}$ and $\BP\sups{M'', (\til p, \til q)}$ since the algorithm $(\til p, \til q)$ is the same.} we have:
  \begin{align}
    \hell{\BP\sups{M', (\til p, \til q)}}{\BP\sups{M'', (\til p, \til q)}} =& O(\log T) \cdot \E\sups{M', (\til p, \til q)} \left[ \sum_{t=1}^T \hell{M'(\til \pi\^t)}{M''(\til \pi\^t)} \right]\nonumber\\
    \leq & O(\log T) \cdot \E\sups{M', (\til p, \til q)} \left[ \sum_{t=1}^T \E_{(\sigma_1\^t, \sigma_2\^t) \sim \til \pi\^t} [\hell{M'((\sigma_1\^t, \sigma_2\^t))}{M''((\sigma_1\^t, \sigma_2\^t))}] \right]\nonumber\\
    \leq &  O(\log T) \cdot \E\sups{M', (\til p, \til q)} \left[ \sum_{t=1}^T \E_{(\sigma_1\^t, \sigma_2\^t) \sim \til \pi\^t} [2 \cdot \One{\sigma_2\^t = v^\st}] \right]\nonumber\\
    \leq & O \left( \frac{T \log T}{V} \right)\nonumber,
  \end{align}
  where the final inequality uses \cref{eq:ustar-alg-guarantee}. Since total variation distance is bounded above by Hellinger distance, it follows that $\tvd{\BP\sups{M', (\til p, \til q)}}{\BP\sups{M'', (\til p, \til q)}} \leq  C \sqrt{T \log(T)/V}$ for some constant $C > 0$. Using this fact together with \cref{eq:ustar-alg-guarantee}, we see that
  \begin{align}
\E\sups{M'', (\til p, \til q)} \E_{\wh \pi \sim \til p} [\wh \pi_2(v^\st)] \leq \E\sups{M', (\til p, \til q)} \E_{\wh \pi \sim \til p} [\wh \pi_2(v^\st)] + C \sqrt{T \log(T)/V} \leq 2T/V + C \sqrt{T \log(T) / V}\nonumber,
  \end{align}
  where the second inequality above uses \cref{eq:ustar-alg-guarantee}. 
\end{proof}

\section{Proofs for \cref{sec:bounds}}

\subsection{Proofs from \cref{sec:upper}}
\label{sec:upper-proof}

\subsubsection{Further details for upper bound}
\label{sec:ub-details}

The upper bound from \pref{thm:constrained-upper-finitem} is derived by appealing to the \etdppac algorithm from \citet{foster2023tight}. In what follows, we give some background on the algorithm, as well as a more general upper bound. In brief, the \etdppac algorithm proceeds as follows: 
The algorithm uses an \emph{online estimation oracle}, denoted by $\AlgEst$ (defined formally in \cref{ass:estimation}), which is given as input a model class $\MM$ and attempts to estimate the true model $\Mstar \in \MM$ given data obtained from playing various decisions under $\Mstar$. To generate each successive datapoint at iteration $t$, which will be fed to the estimation oracle $\AlgEst$, the \etdppac algorithm solves the minimax problem in \cref{eq:decc-pm} to compute distributions $p\^t,q\^t$, where the model $\Mbar$ is set to be the output of the estimation oracle from the previous iteration. Then, a decision $\pi\^t$ is sampled from $q\^t$, and we observe the resulting observation $o\^t \sim \Mstar(\pi\^t)$. The tuple $(\pi\^t, o\^t)$ is then be fed to the estimation oracle, which produces its next estimate $\Mhat\^{t+1}$. The algorithm's output after $T$ iterations is given by a sample from one of the distributions $p\^{t^\st}$, where $t^\st \sim [T]$ is uniform. See \citet{foster2023tight} for further background.

\begin{assumption}[Estimation oracle for $\MM$]
  \label{ass:estimation}
  For each time $t \in [T]$, an online estimation oracle $\AlgEst$ for the class $\MM$ takes as input $\hist\^{t-1} = (\pi\^1, o\^1), \ldots, (\pi\^{t-1}, o\^{t-1})$ where $o\^i \sim \Mstar(\pi\^i)$ and $\pi\^i \sim q\^i$, for arbitrary (adaptive) choices of the distributions $q\^i \in \Delta(\Pi)$. Then, for some class $\wh \MM \subseteq \co(\MM)$, the oracle $\AlgEst$ returns an estimator $\wh M\^t \in \wh\MM$. We assume that if $\Mstar \in \MM$, the estimators produced by the algorithm satisfy
  \begin{align}
\EstHel \ldef  \sum_{t=1}^T \E_{\pi\^t \sim q\^t} \left[ \hell{\Mstar(\pi\^t)}{\wh M\^t(\pi\^t)} \right] \leq \Est(T, \delta)\nonumber,
  \end{align}
  with probability at least $1-\delta$, where $\Est(T, \delta)$ is a known upper bound. 
\end{assumption}
For most estimation oracles, the class $\wh \MM$ in \cref{ass:estimation} will be $\co(\MM)$, though in some cases it is possible to take it to be smaller (see \cref{prop:mg-risk} for an example). 
\begin{theorem}[\cite{foster2023tight}, Theorem 3.1;  Upper bound for \FrameworkShort]
  \label{thm:constrained-upper}
  Fix $\delta \in \left(0,\frac{1}{10}\right)$ and $T \in \BN$, and consider any instance $\sI = \instpm$. Suppose that \cref{ass:realizability,ass:estimation} hold for the model class $\MM$ and some class $\wh \MM \subseteq \co(\MM)$, and let $\EstBar := \Est \left( \frac{2T}{\lceil \log 2/\delta \rceil}, \frac{\delta}{4 \lceil \log 2/\delta\rceil} \right)$. Letting $\vepsupperT := 8 \sqrt{\frac{\lceil \log 2/\delta\rceil}{T} \cdot \EstBar}$, \etdppac, with access to the oracle $\AlgEst$,  guarantees that with probability at least $1-\delta$,
  \begin{align}
    \RiskDM  \leq %
\sup_{\Mbar \in \wh\MM}    \deccpac[\vepsupperT](\sI, \Mbar) \leq 
    \deccpac[\vepsupperT](\sI)\nonumber.
  \end{align}
  If further $\fm(\cdot) \in [0,R]$ for all $M \in \MM$ and some $R > 0$, then the expected risk is bounded as $\E[\RiskDM] \leq \deccpac[\vepsupperT](\sI) + R \delta$. 
\end{theorem}
We remark that \cref{thm:constrained-upper} is only stated in \cite{foster2023tight} for the case $\wh \MM = \co(\MM)$, but an inspection of the proof shows that  the same guarantee holds for an arbitrary subclass $\wh \MM \subseteq \co(\MM)$ in which $\AlgEst$ produces its predictions (with no modifications to the proof being necessary). 

\cref{thm:constrained-upper-finitem} follows from \cref{thm:constrained-upper} by noting that there exists an estimation oracle with $\Est(T,\delta)\leq 2\log(\abs{\cM}/\delta)$ for finite classes \citep{foster2023tight}.

\begin{remark}[Analogue for \MAFrameworkShort]
  \label{rem:ma-hr-rescaling}
Using the transformation of \cref{prop:ma-to-pm} (which does not change the model class of the instance, and therefore preserves estimation error guarantees), there is an analogue of \cref{thm:constrained-upper} for the multi-agent setting. In particular, for any instance $\sJ$ of \MAFrameworkShort, under \cref{ass:realizability-ma,ass:estimation}, there is an algorithm that ensures with probability $1-\delta$, $\RiskDM \leq \deccpac[\vepsupperT](\sJ)$. %
\end{remark}

\paragraph{Infinite model classes} As some of our applications in \cref{app:ma_examples} involve infinite model classes $\MM$, we next describe a simple way to bound the estimation error $\Est(T, \delta)$ for such classes, following the approach in \cite{foster2021statistical}.
\begin{definition}[Model class cover; \cite{foster2021statistical}, Definition 3.2]
  \label{def:model-cover}
  A model class $\MM' \subseteq \MM$ is an \emph{$\vep$-cover} for $\MM$ if for all $M \in \MM$, there is $M' \in \MM'$ so that $\sup_{\pi \in \Pi} \hell{M(\pi)}{M'(\pi)} \leq \vep^2$. Let $\MN(\MM, \vep)$ denote the size of the smallest such cover $\MM'$, and define
  \begin{align}
\est(\MM, T) := \inf_{\vep \geq 0} \left\{ \log \MN(\MM, \vep) + \vep^2 T \right\}\nonumber.
  \end{align}
\end{definition}
We will bound the estimation error $\Est(T, \delta)$ for a model class in terms of the quantity $\est(\MM, T)$; to do so, we need the following mild assumption.
\begin{assumption}
  \label{ass:kernel-B}
  Suppose that there is a kernel $\nu$ from $(\Pi, \sP)$ to $(\MO, \sO)$ so that $M(\pi) \ll \nu(\pi)$ for all $M \in \MM, \pi \in \Pi$, and let $m\sups{M}(\cdot | \pi)$ denote the density of $M(\cdot | \pi)$ with respect to $\nu(\cdot | \pi)$. Furthermore, suppose there is a constant $B \geq e$ so that
  \begin{enumerate}
  \item $\nu(\MO | \pi ) \leq B$ for all $\pi \in \Pi$.
  \item $\sup_{\pi \in \Pi} \sup_{o \in \MO} m\sups{M}(o | \pi) \leq B$ for all $M \in \MM$.
  \end{enumerate}
\end{assumption}
\cref{prop:estimation-infinite} below shows that the estimation error $\Est(T, \delta)$ scales with $\log B$. This quantity is typically small: for instance, it is a constant for standard multi-armed bandit problems (e.g., Bernoulli bandits and Gaussian bandits), and is polylogarithmic in the size of the state and action spaces for reinforcement learning problems with finite state and action spaces.
\begin{proposition}[Lemma A.16 of \cite{foster2021statistical}]
  \label{prop:estimation-infinite}
  Suppose \cref{ass:kernel-B} holds. Fix $T \in \bbN, \delta \in (0,e^{-1})$, and write $b_T = \log(2B^2 T)$.  Then there is an algorithm $\AlgEst$ that guarantees that, with probability $1-\delta$, we have
  \begin{align}
\Est(T) \leq O(b_T \cdot \est(\MM, T) + b_T^2 \log(\delta^{-1}))\nonumber,
  \end{align}
  i.e., we can take $\Est(T, \delta) = C \cdot (b_T \cdot \est(\MM, T) + b_T^2 \log(\delta^{-1}))$ for some universal constant $C$. 
\end{proposition}

\subsubsection{Proof of \cref{thm:constrained-lower}}
\begin{proof}[\pfref{thm:constrained-lower}]
  Fix $T \in \BN$ and an algorithm $(p,q) = \{ q\^t(\cdot | \cdot), p(\cdot | \cdot) \}_{t=1}^T$. For each model $M \in \cMall$, we use the abreviation $\BP\sups{M} \equiv \BP\sups{M, (p,q)}$, and write $\E\sups{M}$ for the corresponding expectation. %
  We also define
  \begin{align}
p\subs{M} = \E\sups{M}[p(\cdot | \hist\^T)], \qquad q\subs{M} = \E\sups{M}\left[ \frac{1}{T} \sum_{t=1}^T q\^t(\cdot | \hist\^{t-1}) \right]\nonumber.
  \end{align}
  Choose $\vepslowerT$ as in the theorem statement, and write $\vep = \vepslowerT$.  Choose $\Mbar \in \co(\MM)$ so that $\deccpac[\vep](\MM) = \deccpac[\vep](\MM, \Mbar)$.\footnote{If the supremum over $\Mbar$ is not achievable, then we may apply the argument that follows for a sequence that achieves the supremum.} We will prove a lower bound on the expected risk in terms of $\deccpac[\vep](\MM, \Mbar)$. 
  Define
  \begin{align}
M := \argmax_{M \in \MM} \left\{ \E_{\pi \sim \pmbar}[\gm(\pi)] \ | \ \E_{\pi \sim \qmbar}[\hell{M(\pi)}{\Mbar(\pi)}] \leq \vep^2 \right\}\nonumber,
  \end{align}
  where we recall that $C(T) := \log(T \wedge \abscont)$. Note that if the $\MH_{\qmbar, \vep}(\Mbar) = \emptyset$, then by definition $\deccpac(\MM, \Mbar) = 0$ and the result follows. Thus, we may assume that $\MH_{\qmbar, \vep}(\Mbar) \neq \emptyset$, and hence the choice of $M$ above is well-defined. Furthermore, the choice of $M$ ensures that
  \begin{align}
    \label{eq:pmbar-gm-lb}
    \E_{\pi \sim \pmbar}[\gm(\pi)] \geq \deccpac(\MM, \Mbar) = \deccpac(\MM).
  \end{align}

  By Lemma A.13 in \cite{foster2021statistical}, we have\footnote{In order to apply this result, we need to ensure that for all measurable sets $\ME \subseteq \MO$ and all $\pi \in \Pi$, we have $\frac{\Mbar(\ME | \pi)}{M(\ME | \pi)} \leq \abscont$. This follows from the definition of $\abscont$ in \cref{eq:define-abscont} and the fact that $\Mbar \in \co(\MM)$.}
  \begin{align}
\hell{\BP\sups{M}}{\BP\sups{\Mbar}} \leq C(T) \cdot T \cdot \E_{\pi \sim \qmbar} [\hell{M(\pi)}{\Mbar(\pi)}] \leq C(T) \cdot T \cdot \vep^2 \nonumber.
  \end{align}
  Using the data processing inequality, it follows that
  \begin{align}
\hell{\pm}{\pmbar} \leq C(T) \cdot T \cdot \vep^2 \leq \frac{1}{8R} \cdot \deccpac(\MM)\label{eq:pm-pmbar-hellub},
  \end{align}
  where the second inequality follows from the choice of $\vep = \vepslowerT$.

  Next, using Lemma A.11 in \cite{foster2021statistical} and the fact that $\gm(\pi) \in [0,R]$ for all $\pi$, we have
  \begin{align}
\E_{\pi \sim \pmbar}[\gm(\pi)] \leq & 3 \cdot \E_{\pi \sim \pm}[\gm(\pi)] + 4R \cdot \hell{\pm}{\pmbar}\nonumber.
  \end{align}
  Combining the above display with \cref{eq:pmbar-gm-lb} and \cref{eq:pm-pmbar-hellub} and rearranging, we see that
  \begin{align}
\frac{1}{6} \cdot \deccpac(\MM) \leq \E_{\pi \sim \pm}[\gm(\pi)] = \E\sups{M} \E_{\pi \sim p(\cdot | \hist\^T)}[\gm(\pi)] = \E\sups{M}[\RiskDM]\nonumber,
  \end{align}
  which gives the desired lower bound on expected risk.
  
\end{proof}

\subsection{Proofs from \cref{sec:gap}}

\subsubsection{Proof of \creftitle{prop:gap-bounding}}

\begin{proof}[\pfref{prop:gap-bounding}]
  Define $\Delta := \frac{\deccpac[\vepsupperT](\MM)}{8 \cdot \vepsupperT^2 \cdot C(T) \cdot T}$. If $\Delta \geq 1$, then we have $\vepslowerT \geq \vepsupperT$, so we may assume from here on that $\Delta < 1$. Choose
  \begin{align}
\alpha = \left\lceil \frac{\log 1/\Delta}{2 \log( \Creg/\creg)} \right\rceil \geq 1\nonumber,
  \end{align}
  which in particular is the smallest positive integer so that $(\Creg^2 / \creg^2)^\alpha \geq 1/\Delta$. Such $\alpha$ is well-defined by our assumption that $\Creg > \creg$ and since $1/\Delta > 1$. Applying \cref{ass:regularity} to $\vep = \vepsupperT \cdot (\creg/\Creg)^j$ for $0 \leq j < \alpha$, it follows that
  \begin{align}
    \deccpac[\vepsupperT](\MM) \leq & \creg^{2\alpha} \cdot \deccpac[\vepsupperT/\Creg^\alpha](\MM) \nonumber\\
    \leq & \Delta \cdot \Creg^{2\alpha} \cdot \deccpac[\vepsupperT/\Creg^\alpha](\MM) \nonumber\\
    \leq & \deccpac[\vepsupperT](\MM) \cdot \frac{\deccpac[\vepsupperT / \Creg^\alpha](\MM)}{8 \cdot (\vepsupperT / \Creg^\alpha)^2 \cdot C(T) \cdot T}\nonumber.
  \end{align}
  Hence $\vepslowerT \geq \vepsupperT / \Creg^\alpha$, and so
  \begin{align}
\deccpac[\vepslowerT](\MM) \geq \deccpac[\vepsupperT/\Creg^\alpha](\MM) \geq \frac{1}{\creg^{2\alpha}} \cdot \deccpac[\vepsupperT](\MM)\nonumber.
  \end{align}
  The definition of $\alpha$ and $\vepsupperT$ gives that
  \begin{align}
    \left( \frac{\Creg}{\creg} \right)^{2\alpha} \leq &  \frac{\Creg}{\creg} \cdot \frac{1}{\Delta} \leq \frac{\Creg}{\creg} \cdot \frac{ 8^3 \cdot \lceil \log 2/\delta \rceil \cdot \EstBar \cdot C(T)}{\deccpac[\vepsupperT](\MM)}\nonumber.
  \end{align}
  Our definition of $\beta$ ensures that $\creg^{2\alpha} \leq ((\Creg/\creg)^{2\alpha})^\beta$, meaning that, for some constant $C$,
  \begin{align}
\deccpac[\vepsupperT](\MM) \leq \creg^{2\alpha} \cdot \deccpac[\vepslowerT](\MM) \leq (C\cdot\Creg/\creg)^\beta \cdot \log^\beta 1/\delta \cdot \EstBar^\beta \cdot C(T)^\beta \cdot \deccpac[\vepsupperT](\MM)^{-\beta} \cdot \deccpac[\vepslowerT](\MM)\nonumber,
  \end{align}
  and rearranging yields:
  \begin{align}
\deccpac[\vepsupperT](\MM)\leq \left( C \log 1/\delta \cdot \EstBar \cdot C(T) \cdot \Creg/\creg\right)^{\frac{\beta}{1+\beta}} \cdot \deccpac[\vepslowerT](\MM)^{\frac{1}{1+\beta}}\nonumber.
  \end{align}
\end{proof}

\subsubsection{Proof of \creftitle{prop:gap-ub-easy} and \creftitle{prop:gap-inherent}}

\begin{proof}[\pfref{prop:gap-ub-easy}]
  We set $\Pi = [A]$ and $\MO = \{0,1\}$. For each $\delta > 0$, we define a model class $\MM_\delta \subset (\Pi \ra \Delta(\MO))$, as follows: $\MM_\delta = \{ M_{\delta, a}  : \ a \in [A]\}$, and define $M_{\delta, a}(\pi) = \Ber(1/2 + \delta \One{\pi = a })$. We now set $\MM = \bigcup_{i=2}^{L} \MM_{2^{-i}}$, from which it follows that $|\MM| \leq LA$. Define $\fm(\pi) = \E\sups{M, \pi}[r]$, where $r \sim M(\pi)$. Finally set $\sI = \instpm$. Note that the instance $\sI$ is actually a standard (non-hidden reward) DMSO instance in the sense of \citet{foster2021statistical}.

  Since the model class $\MM$ is a subclass of the class of all $A$-armed bandit problems, we have from Proposition 5.1 of \cite{foster2021statistical} and \cref{prop:constrained-to-offset} (which applies identically to \FrameworkShort instances in addition to \MAFrameworkShort instances) that $\deccpac(\sI) \leq O(\vep \sqrt{A})$. Furthermore, we have $\mf M(\sI, T) \leq O(\sqrt{TA})$ \citep{audibert2009minimax} (up to logarithmic factors, this bound is also a consequence of, e.g., \cref{thm:constrained-upper-finitem}).

  For each $\delta > 0$, write $\sI_\delta := (\MM_\delta, \Pi, \MO, \{ \fm(\cdot ) \}\subs{M})$. Also write $\Mbar_\delta = \frac{1}{A} \sum_{a=1}^A M_{\delta, a} \in \co(\MM_\delta)$. Since for all $\pi \in \Pi$ and $a \in [A]$, $\hell{\Mbar_\delta(\pi)}{M_{\delta, a}(\pi)} \leq 4\delta^2$, it is straightforward to see that $\deccpac[4\delta/\sqrt{A}](\sI_\delta) \geq \deccpac[4\delta/\sqrt{A}](\sI_\delta, \Mbar_\delta) \geq \Omega(\delta)$. Since increasing the size of the model class cannot decrease the DEC, it follows that, for all $\vep$ satisfying $1/\sqrt{A} > \vep > 2^{-L}$, $\deccpac[\vep](\sI) \geq \Omega(\vep \sqrt{A})$. Finally, since the rewards are observed in the instance $\sI$, we can use Theorem 2.1 of \cite{foster2023tight} to conclude that for $A$ at least some sufficiently large constant, and $T \leq 2^{L/2}$, $\mf M(\sI, T) \geq \Omega(\sqrt{A/T})$. 
\end{proof}

\begin{proof}[\pfref{prop:gap-inherent}]
Fix $L$ to be larger than some universal constant (whose value will be specified below), and consider any value for a constant $\Cscale \geq 1$. We define the following instance $\sI = (\MM, \Pi, \MO, \{ \fm(\cdot ) \}_{M \in \MM})$, with the individual components defined as follows:
  \begin{itemize}
  \item For $1 \leq \ell \leq L$, define $\alpha_\ell := 1/L$, $N_\ell = 2^\ell$, and $\delta_\ell = \frac{1}{(\Cscale N_\ell)^2}$. 
  \item Let $\MV := \prod_{\ell=1}^L [N_\ell]$, and set $\Pi := \MV$.
  \item Let $\MO = \prod_{\ell=1}^L ([N_\ell] \cup \{\perp_\ell\})$. For ease of notation we write $\MO_\ell := [N_\ell] \cup \{ \perp_\ell\}$.
  \item For $o_\ell \in \MO_\ell$ and $v_\ell \in [N_\ell]$, define
    \begin{align}
      P_{v_\ell}(o_\ell) := \begin{cases}
        1 - \delta(N_\ell-1) &: o_\ell = \perp_\ell\\
        \delta_\ell &: o_\ell \in [N_\ell] \backslash \{ v_\ell\} \\
        0 &: o_\ell = v_\ell.
      \end{cases}\nonumber
    \end{align}
    Then $P_{v_\ell} \in \Delta(\MO_\ell)$. 
  \item The class $\MM$ is indexed by tuples $v \in \MV$; in particular, for each $v = (v_1, \ldots, v_\ell) \in \MV$, there is a model $M_v$, defined as follows. For $\pi \in \Pi$, $M_v(\pi) \in \Delta(\MO)$ is the following distribution which does not depend on $\pi$: for $o = (o_1, \ldots, o_L) \in \MO$, 
    \begin{align}
M_v(\pi)(o) = \prod_{\ell=1}^L P_{v_\ell}(o_\ell)\nonumber.
    \end{align}
    Since the distribuiton $M_v(\pi)$ does not depend on $\pi$, we will often drop the argument $\pi$ and simply write $M_v \in \Delta(\MO)$. Accordingly, the Hellinger distance between observation distributions of two models $M, M' \in \co(\MM)$ will be denoted by $\hell{M}{M'}$. 
  \item For all $v \in \MV$ and $\pi = (\pi_1, \ldots, \pi_L) \in \Pi$, the value function $f : \Pi \ra [0,1]$ is defined as follows:
    \begin{align}
f\sups{M_v}(\pi) := \sum_{\ell=1}^L \alpha_\ell \cdot \left( 1 - \One{\pi_\ell = v_\ell} \right)\nonumber.
    \end{align}
    For convenience we write $f_\ell\sups{M_v}(\pi) := (1 - \One{\pi_\ell = v_\ell})$, so that $f\sups{M_v}(\pi) = \sum_{\ell=1}^L \alpha_\ell \cdot f_\ell\sups{M_v}(\pi)$. It is clear that for all $v \in \MV$ there is some $\pi$ (namely, any $\pi$ so that $\pi_\ell \neq v_\ell$ for all $\ell$) for which $f\sups{M_v}(\pi) = 1$, meaning that $g\sups{M_v}(\pi) = 1 - f\sups{M_v}(\pi)$. 
  \end{itemize}

  \paragraph{Upper bounding the minimax sample complexity}
  Fix some $T \in \BN$; we next upper bound $\mf M(\sI, T)$. Since the distribution over observations for all models in the class $\MM$ does not depend on the decision, to specify an algorithm $(p,q)$ we need only to specify the distribution $p$, which is a mapping from $T$-tuples of observations to distributions over decisions. To define $p$, we first define mappings $p_\ell : \MO_\ell^T \ra [N_\ell]$, as follows:
  \begin{align}
    p_\ell(o_{\ell,1}, \ldots, o_{\ell,T}) := \begin{cases}
      \Unif(\Pi) &: o_{\ell,1} = \cdots = o_{\ell, T} = \perp_\ell, \\
      \indic_{o_{\ell,t}} &: t \ldef \argmin \{ s \in [T] \ | \ o_{\ell, s} \neq \perp_\ell\}\; \text{exists}.
    \end{cases}\nonumber
  \end{align}
  In particular, $p_\ell$ outputs the first index of an observation which is not $\perp_\ell$; if no such index exists, then $p_\ell$ outputs the uniform distribution over $[N_\ell]$. Now we define
  \begin{align}
p(o_1, \ldots, o_T) := \left( p_\ell(o_{\ell,1}, \ldots, o_{\ell, T}) \right)_{\ell=1}^L \nonumber,
  \end{align}
  where we have written $o_t = (o_{1,t}, \ldots, o_{L,t})$ for each $T \in [T]$. 

  We now upper bound the risk of the algorithm $p$. We abbreviate the distribution over histories under a given model $M_v \in \MM$ by $\BP\sups{M_v}(\cdot)$, and write $\E\sups{M_v}[\cdot]$ for the corresponding expectation. For each $\ell \in [L]$, we have, for all $M_v \in \MM$,
  \begin{align}
\E\sups{M_v} \E_{\pi \sim p(o_1, \ldots, o_T)} \left[1 - f_\ell\sups{M_v}(\pi) \right] \leq (1 - \delta_\ell (N_\ell - 1))^T \cdot \frac{1}{N_\ell}\nonumber,
  \end{align}
  \dfcomment{Isn't $(1-\delta_\ell(N_\ell-1))^T$ the probability that \emph{no such $t$ exists?}}\noah{yep, fixed} 
  since the probability that there is no $t\in [T]$ so that $o_{\ell,t} \neq \perp_\ell$ is $(1-\delta_\ell(N_\ell-1))^T$, and on the complement of this event (so that such $t$ exists), $p_\ell(o_1, \ldots, o_T)$ puts all its mass on such $o_{\ell,t} \neq v_\ell$, so that $f_\ell\sups{M_v}(\pi)=1$. Hence
  \begin{align}
\E\sups{M_v} \E_{\pi \sim p(o_1, \ldots, o_T)} [g\sups{M_v}(\pi)] \leq \sum_{\ell=1}^L \frac{\alpha_\ell}{N_\ell} \cdot (1-\delta_\ell (N_\ell-1))^T \leq \sum_{\ell=1}^L \frac{\alpha_\ell}{N_\ell} \cdot \left( 1 - \frac{1}{2\Cscale^2 N_\ell} \right)^T.\nonumber
  \end{align}
  Given $T \leq N_L$, choose $\ell_\st = \ell_\st(T) \in [L]$ as large as possible so that $T \geq 2 \log(N_{\ell_\st}) \cdot 2\Cscale^2 N_{\ell_\st}$, which gives \arxiv{\noah{using $\alpha_\ell$ def mildly}}
  \begin{align}
    \E\sups{M_v} \E_{\pi \sim p(o_1, \ldots, o_T)} [g\sups{M_v}(\pi)] \leq & \sum_{\ell=1}^{\ell_\st} \frac{\alpha_\ell}{N_\ell} \cdot \exp \left( -\frac{2 \log N_{\ell_\st}}{T} \right)^T + \sum_{\ell=\ell_\st + 1}^L \frac{\alpha_\ell}{N_\ell} \nonumber\\
    \leq & \sum_{\ell=1}^{\ell_\st} \frac{\alpha_\ell}{N_\ell} \cdot \frac{1}{N_{\ell_\st}^2} + \sum_{\ell=\ell_\st+1}^L \frac{\alpha_\ell}{N_\ell} \leq \frac{1}{N_{\ell_\st}} \leq \frac{8\Cscale^2 \log T}{T}\nonumber,
  \end{align}
where the final inequality uses that our choice of $\ell_\st$ gives that $N_{\ell_\st} \geq \frac{T}{8\Cscale^2 \log T}$,

  \paragraph{Lower bounding the DEC} By the tensorization property of the squared Hellinger distance, we have, for any two models $M_v, M_u \in \MM$, %
  \begin{align}
    1 - \frac 12 \hell{M_u}{M_v} = &\prod_{\ell=1}^L \left(1 - \frac 12 \hell{P_{v_\ell}}{P_{u_\ell}} \right) = \prod_{\ell=1}^L \left( 1 - \frac 12 \cdot \One{u_\ell \neq v_\ell} \cdot 2\delta_\ell \right) \geq 1 - \sum_{\ell=1}^L \delta_\ell \cdot \One{u_\ell \neq v_\ell}\nonumber,
  \end{align}
  which implies that $\hell{M_u}{M_v} \leq 2 \sum_{\ell=1}^L \delta_\ell \cdot \One{u_\ell \neq v_\ell}$. Let $\wb{v} := (1, 1, \ldots, 1) \in \MV$, and set $\Mbar := M_{\wb{v}}$.

  Now consider any $2 \geq \vep \geq \sqrt{2\delta_L}$. Choose $\ell^\st =\ell^\st(\vep)$ to be the smallest possible value of $\ell \in [L]$ so that $\vep^2 \geq 2\delta_\ell$. For each $i \in [N_{\ell^\st}]$, define $\wb{v}^i \in \MV$ by:
  \begin{align}
    \wb{v}^i_\ell = \begin{cases}
      \wb{v}_\ell &: \ell \neq \ell^\st \\
      i &: \ell = \ell^\st,
    \end{cases}\nonumber
  \end{align}
  and write $M^i := M_{\wb{v}^i}$. Then for all $i \in [N_{\ell^\st}]$, we have $\hell{\Mbar}{ M^i} \leq 2 \delta_{\ell^\st} \leq \vep^2$. For any distribution $p \in \Delta(\Pi)$, there must be some $i^\st \in [N_{\ell^\st}]$ so that
  \begin{align}
\E_{\pi \sim p}[1 - f\sups{M^{i^\st}}_{\ell^\st}(\pi)] =    \BP_{\pi \sim p}(\pi_{\ell^\st} = i^\st) \geq 1/N_{\ell^\st}. \nonumber
  \end{align}
  Therefore, \arxiv{\noah{using $\alpha$ defn}}
\begin{align}
  \deccpac[\vep](\MM, \Mbar) \geq \frac{\alpha_{\ell^\st}}{N_{\ell^\st}} = \alpha_{\ell^\st} \Cscale \sqrt{\delta_{\ell^\st}} \geq \frac{\alpha_{\ell^\st} \Cscale}{\sqrt{8}} \cdot \vep = \frac{\Cscale}{\sqrt{8} \cdot L} \cdot \vep\label{eq:dec-alphastar-lb},
\end{align}
where the final inequality uses that $\vep^2 \leq 8\delta_\ell$ since 
$\delta_{\ell+1} = \delta_\ell/4$ for all $\ell < L$ and $\vep \leq 2$.

\paragraph{Upper bounding the DEC}
Next we upper bound $\deccpac[\vep](\MM)$ for $\vep \in (0,2)$; while not necessary for lower bounding $\vepslowerT$, an upper bound on the $\deccpac(\MM)$ serves to ensure that the class $\MM$ satisfies the regularity condition of \cref{ass:regularity}. This certifies that the instance $\sI$ we construct satisfies the assumptions that we use to upper and lower bounding minimax risk in terms of the DEC.

Consider any $\Mbar \in \co(\MM)$. We can write $\Mbar = \E_{v \sim \mu}[M_v]$ for some distribution $\mu \in \Delta(\MV)$. For each $\ell \in [L]$, let $\mu_\ell \in \Delta([N_\ell])$ be the marginal of $\mu$ on $[N_\ell]$ (recall that $\MV = \prod_{\ell=1}^L [N_\ell]$). Since $\hell{P_{u_\ell}}{P_{v_\ell}} = 2\delta_\ell$ for $u_\ell \neq v_\ell$, any two distinct values $v_\ell, v_\ell' \in [N_\ell]$ satisfying $\hell{\E_{u_\ell \sim \mu_\ell}[P_{u_\ell}]}{P_{v_\ell}} \leq \vep^2$ and $\hell{\E_{u_\ell \sim \mu_\ell}[P_{u_\ell}]}{P_{v_\ell'}} \leq \vep^2$ must in turn satisfy
\begin{align}
2 \delta_\ell = \hell{P_{v_\ell}}{P_{v_\ell'}} \leq 2 \cdot \hell{\E_{u_\ell \sim \mu_\ell}[P_{u_\ell}]}{P_{v_\ell}} + 2 \cdot \hell{\E_{u_\ell \sim \mu_\ell}[P_{u_\ell}]}{P_{v_\ell'}} \leq 4 \vep^2\label{eq:vl-vlp-separate}.
\end{align}
Now consider any $\vep \in (\sqrt{\delta_L},2)$. %
Define $\wb{\ell} = \wb\ell(\vep)$ to be the largest possible value of $\ell \in [L]$ so that $\vep^2 < \delta_{\wb\ell}/2$. By \cref{eq:vl-vlp-separate} it follows that for all $\ell \leq \wb{\ell}$, there is at most a single value of $v_\ell \in [N_\ell]$ so that  $\hell{\E_{u_\ell \sim \mu_\ell}[P_{u_\ell}]}{P_v} \leq \vep^2$. Denote this value of $v_\ell$ by $\wb{v}_\ell$ if such a $v_\ell$ exists; if not, choose $\wb{v}_\ell \in [N_\ell]$ arbitrarily.

By the data processing inequality, for any $v \in \MV$ and each $\ell \in [L]$, it holds that $\hell{\Mbar}{M_v} \geq \hell{\E_{u_\ell \sim \mu_\ell}[P_{u_\ell}]}{P_{v_\ell}}$. Thus, for each $M_v \in \MM$ so that $\hell{\Mbar}{M_v} \leq \vep^2$, we must have that $v_\ell = \wb{v}_\ell$ for all $\ell \leq \wb{\ell}$.

Now choose any $v^\st \in \MV$ so that $v^\st_\ell \neq \wb{v}_\ell$ for all $\ell \leq \wb{\ell}$, and define $p^\st \in \Delta(\Pi)$ as follows:
\begin{align}
p^\st := \Unif \left( \left\{ v \in \MV \ | \ v_\ell = v_\ell^\st \ \forall \ell \leq \wb{\ell} \right\} \right)\nonumber.
\end{align}
We may now compute: \arxiv{\noah{again using $\alpha_\ell$ defn}}
\begin{align}
  \deccpac[\vep](\MM, \Mbar) \leq & \sup_{M \in \MM} \left\{ \E_{\pi \sim p^\st}[\gm(\pi)] \ | \ \hell{\Mbar}{M} \leq \vep^2 \right\} 
  \leq  \sum_{\ell = \wb{\ell}+1}^L \frac{\alpha_\ell}{N_\ell} \leq \frac{1}{N_{\wb{\ell}+1}}\leq 2\Cscale \vep\nonumber,
\end{align}
where the final inequality uses that $\vep^2 \geq \delta_{\wb{\ell}}/8$ by definition of $\wb\ell$ and the fact that $\vep \geq \sqrt{\delta_L}$.

\paragraph{Bounding $\vepslowerT$}
Consider any $T \leq N_L/L^3$, which ensures that (for sufficiently large $L$),
\begin{align}
\frac{\Cscale}{8\sqrt{8} \cdot L \cdot C(T) \cdot T} \geq \frac{\sqrt{2}}{\Cscale \cdot N_L} = \sqrt{2\delta_L}\label{eq:ep0-lb},
\end{align}
where we have used that $C(T) \leq C_0 \cdot \log(T)$ for some universal constant $C_0$. 
Recall that $\vepslowerT$ is defined to be as large as possible so that $\vepslowerT^2 \cdot C(T) \cdot T \leq \frac{1}{8} \cdot \deccpac[\vepslowerT](\MM)$. Set
\begin{align}
\vep_0 := \frac{\Cscale  }{8\sqrt{8}\cdot L \cdot C(T) \cdot T} \geq \Omega \left( \frac{\Cscale}{T \log (T) \cdot L} \right)\nonumber,
\end{align}
which, using \cref{eq:ep0-lb} and \cref{eq:dec-alphastar-lb}, satisfies $\vep_0^2 \cdot C(T) \cdot T \leq \frac 18 \cdot \deccpac[\vep_0](\MM)$, and thus $\vepslowerT \geq \vep_0$. 
\end{proof}

\subsubsection{Proof of \creftitle{prop:fdiv-separation}}
\begin{proof}[\pfref{prop:fdiv-separation}]
Given any $\Cprob \geq 1$, fix $N =\lceil \sqrt{T/\Cprob}\rceil$. For real numbers $\delta, \beta \in (0,1)$, we will define instances $\sI^{\delta, \beta}$ of the \FrameworkShort framework. We will later choose $\sI_1, \sI_2$ to be such instances for certain choices of $\delta, \beta$. For some model classes $\MM^{\delta, \beta}$, each of size $N$, we will have, for all $\delta, \beta$, $\sI^{\delta, \beta} =  (\MM^{\delta, \beta}, \Pi, \MO, \{ \fm(\cdot) \}\subs{M})$, i.e., the instances $\sI^{\delta, \beta}$ share the same decision space, observation space, and value functions. We next define these components:
  \begin{itemize}
  \item $\Pi = [N]$ and $\MO = [N] \cup \{ \perp\}$.
  \item For all $\delta, \beta$, we have $\MM^{\delta, \beta} = \{ M_1^{\delta, \beta}, \ldots, M_N^{\delta, \beta} \}$. For $i \in [N]$ and $\pi \in \Pi$, $M_i^{\delta, \beta}(\pi) \in \Delta(\MO)$ is the following distribution, which does not depend on $\pi$:
    \begin{align}
      M_i^{\delta, \beta}(\pi)(j) = \begin{cases}
        1 - \delta(N-1) - \beta &: j = \perp \\
        \delta &: j \in [N] \backslash \{ i \} \\
        \beta &: j=i.
      \end{cases}\nonumber
    \end{align}
    Since the distribution $M_i^{\delta, \beta}(\pi)$ does not depend on $\pi$, we will often drop the argument $\pi$ and simply write $M_i^{\delta, \beta} \in \Delta(\MO)$. 
  \item For all $\delta, \beta$ and for all $\pi \in \Pi$, $i \in [N]$, the value function $f\sups{M_i^{\delta, \beta}} : \Pi \ra [0,1]$ is defined as follows:
    \begin{align}
      f\sups{M_i^{\delta, \beta}}(\pi) := 1 - \One{i = \pi}.\nonumber
    \end{align}
    Since the above value function does not depend on $\delta,\beta$, we will simply write $f^i(\pi) := f\sups{M_i^{\delta, \beta}}(\pi)$ and $g^i(\pi) :=g\sups{M_i^{\delta, \beta}}(\pi) = \One{i = \pi}$. 
  \end{itemize}
  In the model $M_i^{\delta, \beta}$, all decisions except decision $i$ are optimal. Furthermore, we will always have $\beta < \delta$, meaning that, under $M_i^{\delta, \beta}$, it is more likely to observe any given index in $[N] \backslash \{ i\}$ than it is to observe $i$.

  \paragraph{Upper bounding the minimax risk}
  Next, for $T \in \BN$, we upper bound $\mf M(\sI^{\delta, \beta}, T)$. Since the distribution over observations for all models in the classes $\MM^{\delta, \beta}$ does not depend on the decision, to specify an algorithm $(p,q)$ we need only to specify the distribution $p$, which is a mapping from $T$-tuples of observations to distributions over decisions. Furthermore, to specify the distribution over histories under a given model $M_i^{\delta, \beta}$, we write $\E\sups{M_i^{\delta, \beta}}[\cdot]$. Now consider the algorithm $p$ defined by:
  \begin{align}
    p(o_1, \ldots, o_T) := \begin{cases}
      \Unif(\Pi) &: o_1 = \cdots = o_T = \perp \\
      \indic_{o_t} &: t \ldef \argmin \left\{ s \in [T] \ | \ o_s \neq \perp \right\},\;\;\text{exists}.
    \end{cases}\label{eq:p-alg}
  \end{align}
  In particular, $p$ outputs the index of the first observation which is not $\perp$; if no such index exists, then $p$ outputs the uniform distribution over decisions. To upper bound the expected risk of $p$, note that, for any model $M_i^{\delta, \beta}$, we have
  \begin{align}
\E\sups{M_i^{\delta, \beta}} \E_{\pi \sim p(o_1, \ldots, o_T)}[g^i(\pi)] \leq & \frac{(1 - \delta(N-1) - \beta)^T}{N} + \frac{\beta}{\beta + \delta(N-1)} \leq \frac{(1 - \delta(N-1))^T}{N} + \frac{\beta}{\delta(N-1)}\label{eq:p-risk-ub},
  \end{align}
  where the first term on the right-hand side accounts for the case that $o_1 = \cdots = o_T = \perp$, and the second term gives the probability that, given that there exists $s$ such that $o_s \neq \perp$, the index $t$ of the first such observation satisfies $o_t = i$. \dfcomment{this should say $o_t=i$ right?}\noah{yep fixed}

  \paragraph{Lower bounding the minimax risk}
  We next lower bound the minimax risk for the instances $\sI^{\delta, \beta}$ in the following lemma; the proof is provided at the end of the section.
  \begin{lemma}
  \label{lem:ideltabeta-lb}
  Fix any real numbers $C \geq 1$ and $\ep \in (0,1)$, suppose $\delta \leq 1/N$, and write $\beta = \delta/C$. The minimax risk for the instance $\sI^{\delta, \beta}$ is bounded below as follows: for $S \leq 1/\delta^{1-\ep}$, $\mf M(\sI^{\delta, \beta}, S) \geq \frac{1}{2NC^{2/\ep}}$.
\end{lemma}

  \paragraph{Computing $\Dphishort$}
  It is now straightforward to compute the  $\Dphishort$-divergence between any two models in $\MM^{\delta, \beta}$. In particular, for $i \in [N]$, we have:
  \begin{align}
    \Dphi{M_i^{\delta, \beta}}{M_j^{\delta, \beta}} = \begin{cases}
      0 &: i = j \\
      \beta \cdot \phi(\delta/\beta) + \delta \cdot \phi(\beta/\delta) &: i \neq j.
      \end{cases} \label{eq:pairwise-dphi}
  \end{align}

  \paragraph{Choosing $\delta, \beta$}
  \dfedit{Let $\eps\in(0,1)$ and $T\in\bbN$ be as in the theorem statement.} 
  Since $\phi$ is assumed to be $(\alpha, \beta)$-bounded, we have that
  \begin{align}
\frac{\phi(N^{\ep/\alpha}) \cdot N^{-\ep/\alpha} + \phi(N^{-\ep/\alpha})}{\phi(2)/2 + \phi(1/2)} \leq \frac{\beta \cdot N^{\ep}}{\phi(2)/2 + \phi(1/2)} \leq \beta' \cdot N^{\ep}\label{eq:use-phi-bounded},
  \end{align}
  where we have written $\beta' := \max\left\{1, \frac{\beta}{\phi(2)/2 + \phi(1/2)}\right\}$. %
  
  For some constant $C_0 > 0$ to be specified below, we choose
  \begin{align}
    \delta_1 = \frac{C_0 \ln T}{(N-1)T}, \quad \beta_1 = \frac{\delta_1}{N^{\ep/\alpha}}; \qquad \delta_2 = \frac{\phi(N^{\ep/\alpha}) \cdot N^{-\ep/\alpha} + \phi(N^{-\ep/\alpha})}{\phi(2)/2 + \phi(1/2)} \cdot \delta_1, \quad \beta_2 = \frac{\delta_2}{2}\label{eq:define-delta-betas}.
  \end{align}
  The choices of $\delta_1, \beta_1, \delta_2, \beta_2$  ensure that
  \begin{align}
\beta_1 \cdot \phi(\delta_1/\beta_1) + \delta_1 \cdot \phi(\beta_1/\delta_1) = \beta_2 \cdot \phi(\delta_2/\beta_2) + \delta_2 \cdot \phi(\beta_2/\delta_2),\label{eq:fdiv_equality}
  \end{align}
  which, together with \cref{eq:pairwise-dphi}, ensures that for all $i, j \in [N]$, $\Dphi{M_i^{\delta_1, \beta_1}}{M_j^{\delta_1, \beta_1}} = \Dphi{M_i^{\delta_2, \beta_2}}{M_j^{\delta_2, \beta_2}}$. 

  \paragraph{Wrapping up}
  We set $\sI_1 = \sI^{\delta_1, \beta_1}$ and $\sI_2 = \sI^{\delta_2, \beta_2}$, and correspondingly set $\MM_1 = \MM^{\delta_1, \beta_1}$ and $\MM_2 = \MM^{\delta_2, \beta_2}$. 
  We define the one-to-one mapping $\sE : \MM_1 \ra \MM_2$ by the mapping that sends $M_i^{\delta_1, \beta_1} \mapsto M_i^{\delta_2, \beta_2}$ for all $i \in [N]$. It is clear that these definitions satisfy \cref{it:fm-equal} and \cref{it:dphi-equal} of the proposition statement. 
  
  From \cref{eq:p-risk-ub}, the expected risk of $p$ against a worst-case model in $\MM^{\delta_1, \beta_1}$ is bounded above as follows:
  \begin{align}
    \sup_{\Mstar \in \MM^{\delta_1, \beta_1}} \E\sups{\Mstar, p}[\RiskDM] \leq & \left( 1 - \frac{C_0 \ln T}{T} \right)^T + \frac{\beta_1}{\delta_1 \cdot (N-1)} \nonumber\\
    \leq &  \exp \left( -\frac{C_0 \ln T}{T} \right)^T + \frac{2}{N^{1+\ep/\alpha}} \nonumber\\
    \leq &  T^{-1} + 2 \cdot \left( \frac{\Cprob}{T} \right)^{1/2 + \ep/(2\alpha)} \nonumber,
  \end{align}
  where the final inequality holds as long as we choose $C_0 = 1$; recall that $N \ldef\lceil \sqrt{T/\Cprob}\rceil$. The above display establishes the upper bound of \cref{it:minimax-separation}.

  Next, for the lower bound, recall that \cref{eq:use-phi-bounded} gives that $\delta_2 \leq \beta' \cdot N^{\ep} \cdot \delta_1$, so
  \begin{align}
\frac{1}{\delta_2^{1-\ep}} \geq \frac{1}{(\beta' N^\ep \cdot \delta_1)^{1-\ep}} \geq \frac{(N^{1-\ep} T)^{1-\ep}}{2\beta' C_0 \ln T} \geq \frac{N^{1-4\ep} \cdot T}{2\beta' C_0 \Cprob^\ep \ln T} \geq \frac{T^{3/2 - 2\ep}}{2   \beta' C_0 \Cprob^{1/2 + \ep} \ln T}\nonumber, %
  \end{align}
  where the second-to-last inequality uses that $T^\ep \leq N^{2\ep} \cdot \Cprob^\ep$. %
  Thus, from \cref{lem:ideltabeta-lb} with $(\delta, \beta) = (\delta_2, \beta_2)$ (so that $C = 2$), we have that for all $T' \leq \frac{T^{3/2 - 2\ep}}{2   \beta' C_0 \Cprob^{1/2 + \ep} \ln T}$, 
  \begin{align}
\mf M(\sI^{\delta_2, \beta_2}) \geq \frac{1}{2^{1 + 2/\ep}} \cdot \frac{1}{N} \geq \frac{1}{2^{2+2/\ep}} \cdot \sqrt{\frac{T}{\Cprob}}\nonumber.
  \end{align}
Thus, taking $C_\phi = 2\beta' C_0$, the above inequality verifies the lower bound of \cref{it:minimax-separation}. 

\end{proof}

\begin{proof}[\pfref{lem:ideltabeta-lb}]
  Consider any algorithm $p : \MO^S \ra \Pi$. Note that the distributions of $M_1^{\delta, \delta}, \ldots, M_N^{\delta, \delta}$ are all identical. %
  Thus, there is some $i \in [N]$ so that 
  \begin{align}
\E\sups{M_i^{\delta, \delta}}\E_{\pi \sim p(o_1, \ldots, o_S)} [\One{\pi = i}] = \E\sups{M_i^{\delta, \delta}}\E_{\pi \sim p(o_1, \ldots, o_S)} [g^i(\pi)] \geq 1/N\nonumber.
  \end{align}
  For $\lambda \geq 0$, define
  \begin{align}
\MS_\lambda := \left\{ (o_1, \ldots, o_S) \in \MO^S  \ : \ |\{ t \in [S] \ : \ o_t = i \}| > \lambda \right\}\nonumber.
  \end{align}
Then the probability that $(o_1, \ldots, o_S) \in \MS_\lambda$ is bounded above as follows: %
  \begin{align}
\BP\sups{M_i^{\delta, \delta}}\left( (o_1, \ldots, o_S) \in \MS_\lambda \right) \leq {S \choose \lambda} \cdot \delta^\lambda \leq (S\delta)^\lambda \leq \delta^{\ep\lambda}\nonumber.
  \end{align}
  Choosing $\lambda = 2/\ep$ yields $\delta^{\ep\lambda} = \delta^2 \leq 1/N^2$, meaning that
  \begin{align}
\E\sups{M_i^{\delta, \delta}} \E_{\pi \sim p(o_1, \ldots, o_S)}\left[ \One{(o_1, \ldots, o_S) \not \in \MS_{2/\ep}} \cdot \One{\pi =i} \right] \geq 1/N - 1/N^2 \geq 1/(2N)\nonumber.
  \end{align}

  For any $(o_1, \ldots, o_S) \not \in \MS_\lambda$, we have that
  \begin{align}
\frac{M_i^{\delta, \beta}((o_1, \ldots, o_S))}{M_i^{\delta, \delta}((o_1, \ldots, o_S))} \geq (\beta/\delta)^\lambda \leq 1/C^\lambda\nonumber.
  \end{align}
  Thus,
  \begin{align}
    \E\sups{M_i^{\delta, \beta}}\E_{\pi \sim p(o_1, \ldots, o_S)}[g^i(\pi)] &\geq \E\sups{M_i^{\delta, \beta}} \E_{\pi \sim p(o_1, \ldots, o_S)} [\One{\pi = i} \cdot \One{(o_1, \ldots, o_S) \not\in \MS_{2/\ep}}]\nonumber\\
    &\geq \E\sups{M_i^{\delta, \delta}} \E_{\pi \sim p(o_1, \ldots, o_S)} [\One{\pi = i} \cdot \One{(o_1, \ldots, o_S) \not\in \MS_{2/\ep}}] \cdot 1/C^{2/\ep}\nonumber\\
    &\geq \frac{1}{2NC^{2/\ep}}\nonumber.
  \end{align}
\end{proof}

\subsection{Proofs from \cref{sec:ma-analogues}}

\begin{proof}[\pfref{thm:ma-fdiv-separation}]
  Given $\Cprob \geq 1$, fix $N = \lceil \sqrt{T/\Cprob} \rceil$. 
  Recall the definition of the instances $\sI^{\delta, \beta} = (\MM^{\delta, \beta}, \Pi, \MO ,\{ \fm(\cdot)\}_M)$ (for $\delta, \beta \in (0,1)$) of the \FrameworkShort framework defined in the proof of \cref{prop:fdiv-separation}, where we have $\Pi = [N]$ and $\MO = [N] \cup \{ \perp\}$. For each $\delta, \beta$, we now define $\sJ^{\delta, \beta} = (\til \MM^{\delta, \beta}, \til \Pi, \til \MO, \{ \Dev\}_k, \{ \Sw\}_k)$ to be the instance of the (2-player) \MAFrameworkShort framework constructed given the instance $\sI^{\delta, \beta}$ per the construction in the proof of \cref{prop:pm-to-ma} with a value of $V$ to be specified below. In particular, $\til \Pi, \til \MO, \Dev, \Sw$ do not depend on $\delta, \beta$. For clarity, we explicitly write out the definition of the components of $\sJ^{\delta, \beta}$ in terms of the components of $\sI^{\delta, \beta}$: \dfcomment{see comment in theorem 2.2 re: NE v. CCE notation}\noah{fixed}
  \begin{itemize}
  \item Define $\Sigma_1 = \Pi = [N]$ and $\Sigma_2 = \{0, 1, \ldots, V \}$, $\til \Pi_k = \Delta(\Sigma_k)$ for $k \in \{1,2\}$, and $\til \Pi = \til \Pi_1 \times \til \Pi_2$.
  \item Define $\Dev, \Sw$ for $k \in [2]$ so that $\sJ^{\delta, \beta}$ is an NE instance (\cref{def:ne-instance}); in particular, $\Dev = \til \Pi_k$ for each $k$ and $\Sw(\dev, \pi) = (\pi_k, \pi_{-k})$. 
  \item Define the pure observation space to be $\Ocirc := \MO \cup \{ \til \perp \}$, the reward space to be $\MR = [-1,1]$, and the full observation space to $\til \MO := \Ocirc \times \BR^2$.
  \item The model class $\til \MM^{\delta, \beta}$ is indexed by tuples $(M, v) \in \MM^{\delta, \beta} \times \{ 1, 2, \ldots, V \} = \MM^{\delta, \beta} \times [V]$. (Thus $|\til \MM^{\delta, \beta}| = NV$.) In particular, for each such tuple $(M, v)$, we have a model $\til M\subs{M,v}$, which is defined as follows:
    \begin{itemize}
    \item For pure decisions of the form $(\sigma_1, 0) \in \Sigma_1 \times \Sigma_2$ the distribution of $(\ocirc, r_1, r_2) \sim \til M\subs{M, v}((\sigma_1, 0))$ is given by:
      \begin{align}
\ocirc \sim M(\sigma_1) \in \MO \subset \Ocirc, \qquad r_1 = r_2 = 0\nonumber.
      \end{align}
    \item For pure decisions of the form $(\sigma_1, i) \in \Sigma_1 \times \Sigma_2$ with $i > 0$, the distribution of $(\ocirc, r_1, r_2) \sim \til M\subs{M, v}((\sigma_1, i))$ is given by:
      \begin{align}
        \ocirc = \til\perp, \qquad r_2 = -r_1 = \begin{cases}
          -1 &: i \neq v \\
         \gm(\sigma_1)  &: i = v,
        \end{cases}\label{eq:tilperp-rewards}
      \end{align}
      where we recall that $ \gm(\sigma_1) = \max_{\sigma_1' \in \Sigma_1} \{ \fm(\sigma_1') \} - \fm(\sigma_1)$. 
    \item For general decisions $\pi \in \til \Pi$, we can write $\pi = \pi_1 \times \pi_2$ for $\pi_\ag \in \til \Pi_k$ for $\ag \in [2]$. Then the distribution $\til M\subs{M, v}(\pi)$ is the distribution of $\til M\subs{M,v}(\sigma)$ where $\sigma = (\sigma_1, \sigma_2)$ is distributed as: $\sigma_\ag \sim \pi_\ag$ for $\ag \in [2]$. 
    \end{itemize}
  \end{itemize}
  Next, let $\delta_1, \delta_2$ be defined given $T, N, \ep, \phi, \alpha$, as in the proof of \cref{prop:fdiv-separation} (in particular, they are specified in \cref{eq:define-delta-betas}). We write $\sJ_1 = \sJ^{\delta_1, \beta_1}$ and $\sJ_2 = \sJ^{\delta_2, \beta_2}$, and correspondingly write $\MM_1 = \til \MM^{\delta_1, \beta_1}$ and $\MM_2 = \til \MM^{\delta_2, \beta_2}$. Moreover, we define the mapping $\sE : \MM_1 \ra \MM_2$ in an analogous manner to the definition in the proof of \cref{prop:fdiv-separation}. In particular, for each $\delta, \beta$, we have $\MM^{\delta, \beta} = \{ M_1^{\delta, \beta}, \ldots, M_N^{\delta, \eta} \}$. First define $\sE_0 : \MM^{\delta_1, \beta_1} \ra \MM^{\delta_2, \beta_2}$ by $\sE_0(M_i^{\delta_1, \beta_1}) = M_i^{\delta_2, \beta_2}$, for $i \in [N]$ (exactly as was done in the proof of \cref{prop:fdiv-separation}. Then for each model of the form $\til M\subs{M, v} \in \til M^{\delta_1, \beta_1} = \MM_1$ (so that $M \in \MM^{\delta_1, \beta_1}, v \in [V]$), define $\sE(\til M\subs{M, v}) := \til M\subs{\sE_0(M), v}$.  
  We are now ready to verify the individual claims of the theorem:

  \paragraph{Proof of \cref{it:fm-equal-ma}} Consider any $\til M\subs{M,v} \in \til \MM^{\delta_1, \beta_1}$ (so that $M \in \MM^{\delta_1, \beta_1}, v \in [V]$). %
  For any $(\sigma_1, \sigma_2) \in \Sigma_1 \times \Sigma_2$, we have, by definition of $\sE$, 
  \begin{align}
    f_2\sups{\til M\subs{M,v}}(\sigma_1, \sigma_2) = f_2\sups{\sE(\til M\subs{M,v})}(\sigma_1, \sigma_2) = \begin{cases}
      0 &: \sigma_2 = 0 \\
      -1 &: \sigma_2 \in [V] \backslash \{ v \} \\
      \gm(\sigma_1) &: \sigma_2 = v ,
    \end{cases}\nonumber
  \end{align}
  which establishes \cref{it:fm-equal-ma} since all instances are 2-player 0-sum instances.

  \paragraph{Proof of \cref{it:dphi-equal-ma}} Consider any two models $\til M\subs{M,v}, \til M\subs{M', v'} \in \til \MM^{\delta_1, \beta_1}$ (so that $M,  M' \in \MM^{\delta_1, \beta_1}$, and $v,v' \in [V]$). For any $\pi_1 \in \Pi_1$, we have that
  \begin{align}
    & \Dphi{\til M\subs{M,v}(\pi_1, 0)}{\til M\subs{M',v'}(\pi_1, 0)} = \Dphi{M(\pi_1)}{M'(\pi_1)} \nonumber\\
    =& \Dphi{\sE_0(M)(\pi_1)}{\sE_0(M')(\pi_1)} = \Dphi{\til M\subs{\sE_0(M), v}(\pi_1, 0)}{\til M\subs{\sE_0(M'), v'}(\pi_1, 0)}\nonumber\\
    =& \Dphi{\sE(\til M\subs{M,v})(\pi_1, 0)}{\sE(\til M\subs{M',v'})(\pi_1, 0)}\label{eq:se-pi20},
  \end{align}
  where the first and third equalities follow by definition of $\til \MM^{\delta, \beta}$ above, the second equality follows by \cref{it:dphi-equal} of \cref{prop:fdiv-separation} and the fact that our choice of $\delta_1, \beta_1, \delta_2, \beta_2$ is identical to that in the proof of  \cref{prop:fdiv-separation} (cf. \eqref{eq:define-delta-betas} and \eqref{eq:fdiv_equality}), and the fourth equality follows from definition of $\sE$.

  Next, for any $\sigma_1 \in \Sigma_1$ and $\sigma_2 \in \Sigma_2 \backslash \{ 0 \}$, note that the distributions $\til M\subs{M,v}(\sigma_1, \sigma_2)$ and $\sE(\til M\subs{M,v})(\sigma_1, \sigma_2) = \til M\subs{\sE_0(M), v}(\sigma_1, \sigma_2)$ are identical: the pure observation under both these distributions is $\til \perp$ a.s., and the rewards are given by \cref{eq:tilperp-rewards}, where we have noted that $\gm(\sigma_1) = g\sups{\sE_0(M)}(\sigma_1)$ for all $\sigma_1 \in \Sigma_1$. It follows that for any $\pi_1 \in \Pi_1$ and $\sigma_2 \in \Delta(\Sigma_2 \backslash \{ 0 \})$, the distributions $\til M\subs{M,v}(\pi_1, \pi_2)$ and $\sE(\til M\subs{M, v})(\pi_1, \pi_2)$ are identical. In a similar manner, we have that for any such $\pi_1, \pi_2$, the distributions $\til M\subs{M',v'}(\pi_1, \pi_2)$ and $\sE(\til M\subs{M',v'})(\pi_1, \pi_2)$ are identical. Therefore, 
  \begin{align}
\Dphi{\til M\subs{M,v}(\pi_1, \pi_2)}{\til M\subs{M',v'}(\pi_1, \pi_2)} = \Dphi{\sE(\til M\subs{M,v})(\pi_1, \pi_2)}{\sE(\til M\subs{M',v'})(\pi_1, \pi_2)}\label{eq:se-pi2nz}.
  \end{align}
  Now consider any joint decision $(\pi_1, \pi_2) \in \Pi_1 \times \Pi_2$. Let us write $\pi_2 = \pi_2(0) \cdot \indic_0 + (1-\pi_2(0)) \cdot \pi_2'$, where $\pi_2' \in \Delta(\Sigma_2\backslash \{0\})$. Since, for any model $\til M \in \til \MM^{\delta, \beta}$ (for any $\delta, \beta$), the distributions $\til M(\pi_1, 0)$ and $\til M(\pi_1, \pi_2')$ have disjoint support (namely, under the second, the pure observation is always $\til \perp$, and under the first, the pure observation is never $\til \perp$), it follows from \cref{lem:fdiv-factor} that for any two models $\til M, \til M' \in \MM^{\delta, \beta}$,\noah{todo check}
  \begin{align}
    \Dphi{\til M(\pi_1, \pi_2)}{\til M'(\pi_1, \pi_2)} =& \pi_2(0) \cdot \Dphi{\til M(\pi_1, 0)}{\til M'(\pi_1, 0)}\nonumber\\
    & + (1-\pi_2(0)) \cdot \Dphi{\til M(\pi_1, \pi_2')}{\til M'(\pi_1, \pi_2')}\label{eq:tilm-decompose}.
  \end{align}
  Then for the decision $(\pi_1, \pi_2) \in \Pi_1 \times \Pi_2$, with $\pi_2'$ defined as above, we have
  \begin{align}
    & \Dphi{\til M\subs{M,v}(\pi_1, \pi_2)}{\til M\subs{M',v'}(\pi_1, \pi_2)} \nonumber\\
    =& \pi_2(0) \cdot \Dphi{\til M\subs{M,v}(\pi_1, 0)}{\til M\subs{M',v'}(\pi_1, 0)} + (1-\pi_2(0)) \cdot \Dphi{\til M\subs{M,v}(\pi_1,\pi_2')}{\til M\subs{M',v'}(\pi_1, \pi_2')}\nonumber\\
    =& \pi_2(0) \cdot \Dphi{\sE(\til M\subs{M,v})(\pi_1, 0)}{\sE(\til M\subs{M',v'})(\pi_1, 0)} + (1-\pi_2(0)) \cdot \Dphi{\sE(\til M\subs{M,v})(\pi_1,\pi_2')}{\sE(\til M\subs{M',v'})(\pi_1, \pi_2')}\nonumber\\
    =&\Dphi{\sE(\til M\subs{M,v})(\pi_1, \pi_2)}{\sE(\til M\subs{M',v'})(\pi_1, \pi_2)} \nonumber,
  \end{align}
  where the first and third equalities use \cref{eq:tilm-decompose}, and the second equality uses \cref{eq:se-pi20} and \cref{eq:se-pi2nz}. The above display verifies \cref{it:dphi-equal-ma}.

  \paragraph{Proof of \cref{it:minimax-separation-ma}} For each $\delta, \beta \in (0,1)$, the construction of $\sJ^{\delta, \beta}$ given $\sI^{\delta, \beta}$ according to the construction in the proof of \cref{prop:pm-to-ma}, together with the conclusion of \cref{prop:pm-to-ma}, gives that, for all $T' \in \BN$,
  \begin{align}
\mf M(\sJ^{\delta, \beta}, T') \leq \mf M(\sI^{\delta, \beta}, T') \leq \mf M(\sJ^{\delta, \beta}, T') + O((T' \log(T')/V)^{1/4})\label{eq:sj-si-relate}.
  \end{align}
  Then \cref{it:minimax-separation} of \cref{prop:fdiv-separation}, together with our choice of $\delta_1, \beta_1, \delta_2, \beta_2$ to mimic that in the proof of \cref{prop:fdiv-separation}, yields that for all $T'$ with $T \leq T' \leq T^{3/2 - 2\ep} \cdot (C_\phi \Cprob^{1/2 + \ep} \ln T)^{-1}$
  \begin{align}
    \mf M(\sJ_1, T') = \mf M(\sJ^{\delta_1, \beta_1}, T') \leq & \mf M(\sI^{\delta_1, \beta_1}, T') \leq \frac{1}{T} + 2 \cdot \left( \frac{\Cprob}{T} \right)^{1/2 + \ep/(2\alpha)}\label{eq:minimax-m1}\\
    \mf M(\sJ_2, T') = \mf M(\sJ^{\delta_2, \beta_2}, T') \geq & \mf M(\sI^{\delta_2, \beta_2}, T') - O((T' \log(T')/V)^{1/4})\nonumber\\
    \geq &  2^{-2-2/\ep} \cdot \left( \frac{\Cprob}{T}\right)^{1/2} -O((T' \log(T')/V)^{1/4})\nonumber.
  \end{align}
  Choosing $V = T^{100} \cdot 2^{8 + 8/\ep}$ ensures that
  \begin{align}
 \mf M(\sJ_2, T') \geq  2^{-3-2/\ep} \cdot \left( \frac{\Cprob}{T}\right)^{1/2}\label{eq:minimax-m2}.
  \end{align}
  Together \cref{eq:minimax-m1} and \cref{eq:minimax-m2} verify \cref{it:minimax-separation-ma}. 
  
\end{proof}

\section{Proofs for \cref{sec:single-multiple}}
\label{sec:proofs_single_multiple}
Throughout this section, we consider an instance $\sJ = \instma$ of \MAFrameworkShort which is an NE instance (\cref{def:ne-instance}). It follows in particular that for any $M \in \MM, \pi \in \Pi$, we have
\begin{align}
\hm(\pi) = \sum_{k=1}^K \hm_k(\pi) =  \sum_{k=1}^K \sup_{\pi_k' \in \Pi_k} \fm_k(\pi_k', \pi_{-k}) - \fm_k(\pi)\nonumber.
\end{align}

\subsection{Bounds for general games with convex decision spaces}
\label{sec:ma-convex-hull}

\begin{proof}[\pfref{thm:single-multiple-ch}]
  For each $\ag \in [\Ag]$ and $\pi_{-\ag} \in \Pi_{-\ag}$, define
  \begin{align}
\til \MM_\ag(\pi_{-\ag}) := \{ \pi_\ag \mapsto \single{M}(\pi_\ag, \pi_{-\ag}) \ : \ M \in \MM \}\nonumber.
  \end{align}
  It is straightforward from the definition of $\til \MM_\ag$ in \cref{eq:define-til-mk} that for each $\ag \in [\Ag]$, $\til \MM_\ag = \bigcup_{\pi_{-\ag} \in \Pi_{-\ag}} \til \MM_\ag(\pi_{-\ag})$, and therefore that $\bigcup_{\pi_{-\ag} \in \Pi_{-\ag}} \co(\til \MM_\ag(\pi_{-\ag})) \subseteq \co(\til \MM_\ag)$.   For any $\wb\pi_{-\ag} \in \Pi_{-\ag}$ and $\Mbar \in \co(\MM)$, we denote the corresponding element of  $\co(\til \MM_\ag(\wb\pi_{-\ag}))$ by $(\Mbar, \wb\pi_{-\ag})$.  (In particular, $(\Mbar, \wb\pi_{-\ag})$ is the model that sends $\pi_\ag \mapsto \single{\Mbar}(\pi_\ag, \wb\pi_{-\ag})$.)
  It then suffices to prove the following stronger result: for each $\Mbar \in \co(\MM)$, 
  \begin{align}
\decoreg[\gamma](\sJ, \Mbar) \leq \sum_{\ag=1}^\Ag \sup_{\wb\pi_{-\ag} \in \Pi_{-\ag}} \decoreg[\gamma/\Ag](\til \MM_\ag, (\Mbar, \wb\pi_{-\ag}))\label{eq:dec-ch-refined}.
  \end{align}
Next, note that for any $M \in \MM, \pi_{-k} \in \Pi_{-k}$, the value function for the model $\pi_k \mapsto \single{M}(\pi_k, \pi_{-k})$ is given by $f\sups{\single{M}}(\pi_k) = \fm_k(\pi_k, \pi_{-k})$, for $\pi \in \Pi$ (this holds since the distribution of the reward under $\single{M}(\pi_k, \pi_{-k})$ is simply the distribution of agent $k$'s reward under $M(\pi_k, \pi_{-k})$). 
Then for any $\Mbar \in \co(\MM)$, $\wb\pi_{-k} \in \Pi_{-\ag}$, we have
\begin{align}
  & \decoreg[\gamma/K](\til\MM_k, (\Mbar, \wb\pi_{-\ag}))\nonumber\\
  =&  \inf_{p_k \in \Delta(\Pi_k)} \sup_{\substack{M \in \MM\\ \pi_{-k} \in \Pi_{-k}}} \E_{\pi_k \sim p_{k}} \left[ \max_{\pi_k' \in \Pi_k} \fm_k(\pi_k', \pi_{-k}) - \fm_k(\pi_k, \pi_{-k}) - \frac{\gamma}{K} \cdot \hell{M(\pi_k, \pi_{-k})}{\Mbar(\pi_k, \wb\pi_{-k})}\right]\nonumber\\
  \geq & \inf_{p_k \in \Delta(\Pi_k)} \sup_{\substack{M \in \MM\\ \pi_{-k} \in \Pi_{-k}}} \E_{\substack{\pi_k \sim p_{k} \\ a_k \sim \pi_k}} \left[ \max_{\pi_k' \in \Pi_k} \fm_k(\pi_k', \pi_{-k}) - \fm_k(a_k, \pi_{-k}) - \frac{\gamma}{K} \cdot \hell{M(a_k, \pi_{-k})}{\Mbar(a_k, \wb\pi_{-k})}\right]\nonumber\\
  =& \inf_{\pi_k \in \Pi_k} \sup_{\substack{M \in \MM\\ \pi_{-k} \in \Pi_{-k}}} \E_{a_k \sim \pi_k} \left[ \max_{\pi_k' \in \Pi_k} \fm_k(\pi_k', \pi_{-k}) - \fm_k(a_k, \pi_{-k}) - \frac{\gamma}{K} \cdot \hell{M(a_k, \pi_{-k})}{\Mbar(a_k, \wb\pi_{-k})}\right]\label{eq:dec-mtil-lb},
\end{align}
where the inequality uses joint convexity of the squared Hellinger distance,  
and the final inequality uses the fact that any distribution $p_k \in \Delta(\Pi_k)$ may be replaced by the singleton distribution for the decision $\til \pi_k := \E_{\pi_k \sim p_k}[\pi_k]$, without changing the value of the expression.

Thus
\begin{align}
 \decoreg[\gamma/K](\til\MM_k, (\Mbar, \wb\pi_{-\ag})) \geq& \inf_{\pi_k \in \Pi_k} \sup_{\substack{M \in \MM\\ \pi_{-k} \in \Pi_{-k}}} \E_{a_k \sim \pi_k} \left[ \hm_k(a_k, \pi_{-k}) - \frac{\gamma}{K} \cdot  \hell{M(a_k, \pi_{-k})}{\Mbar(a_k, \wb\pi_{-k})}\right]\nonumber.
\end{align}

\paragraph{Existence of fixed points}

 For each $k \in [K]$, define the set-valued function $\MC_k : \Pi \ra \powerset{\Pi_k}$ by
\begin{align}
\MC_k(\wb\pi) := \argmin_{\pi_k \in \Pi_k} \sup_{M \in \MM, \pi_{-k} \in \Pi_{-k}} \E_{a_k \sim \pi_k} \left[ \hm_k(a_k, \pi_{-k}) - \frac{\gamma}{K} \cdot \hell{M(a_k, \pi_{-k})}{\Mbar(a_k, \wb\pi_{-k})}\right]\nonumber.
\end{align}
Further, for $\pi_{-k} \in \Pi_{-k}, M \in \MM$, define the function $G_{M, \pi_{-k}} : \Pi_k \times \Pi_{-k} \ra \BR$ by
\begin{align}
G_{M, \pi_{-k}}(\pi_k, \wb\pi_{-k}) = \E_{a_k \sim \pi_k} \left[ \hm_k(a_k, \pi_{-k}) - \frac{\gamma}{K} \cdot \hell{M(a_k, \pi_{-k})}{\Mbar(a_k, \wb\pi_{-k})}\right]\nonumber.
\end{align}
\cref{ass:convexity_pols} gives that for all $a_k$, the map $\wb\pi_{-k} \mapsto \Mbar(a_k, \wb\pi_{-k})$ is linear. It follows by the dominated convergence theorem that for all $M, \pi_{-k}, a_k$, the function $\wb\pi_{-k} \mapsto \hell{M(a_k, \pi_{-k})}{\Mbar(a_k, \wb\pi_{-k})}$ is continuous. Hence $G_{M, \pi_{-k}}(\pi_k, \wb\pi_{-k})$ is continuous in $(\pi_k, \wb\pi_{-k})$, and the function
$$\til G_k(\pi_k, \wb\pi_{-k}) := \sup_{M \in \MM, \pi_{-k} \in \Pi_{-k}} G_{M, \pi_{-k}}(\pi_k, \wb\pi_{-k})$$
is also continuous in $(\pi_k, \wb\pi_{-k})$. 
Furthermore, since, for each $\wb\pi_{-k}$, the function $G_{M, \pi_{-k}}(\pi_k, \wb\pi_{-k})$ is linear in $\pi_k$ (\cref{ass:convexity_pols}), $\til G_k(\pi_k, \wb\pi_{-k})$ is convex in $\pi_k$. It follows that $\MC_k(\wb\pi) = \argmin_{\pi_k\in \Pi_k} \{ \til G_k(\pi_k, \wb\pi_{-k})\}$ is a closed, nonempty, and convex subset of $\Pi_k$ for all $\wb\pi$. Furthermore, by continuity of $\til G_k$ and  \cref{lem:berge-theorem}, we have that $\MC_k(\wb\pi)$ is upper hemicontinuous. By \cref{lem:kakutani}, it follows that the mapping $\pibar\mapsto\MC_1(\pibar) \times \cdots \times \MC_K(\pibar)$ has a fixed point, namely some $\wb\pi \in \Pi$ so that $\wb\pi \in \prod_{k \in [K]} \MC_k(\wb\pi)$.

\paragraph{Applying the fixed point strategy}

Let $\wb\pi \in \prod_{k \in [K]} \MC_k(\wb\pi)$ be a fixed point of $\MC_1 \times \cdots \times \MC_K$. Then
\begin{align}
  \decoreg[\gamma](\sJ, \Mbar) \leq & \sup_{M \in \MM} \left\{ \hm(\wb\pi) - \gamma \cdot \hell{M(\wb\pi)}{\Mbar(\wb\pi)} \right\}\nonumber\\
  = & \sup_{M  \in \MM} \left\{\sum_{k=1}^K \hm_k(\wb\pi) - \gamma \cdot \hell{M(\wb\pi)}{\Mbar(\wb\pi)} \right\}\nonumber\\
  \leq & \sum_{k=1}^K \sup_{M \in \MM}  \left\{ \hm_k(\wb\pi) - \frac{\gamma}{K} \cdot \hell{M(\wb\pi)}{\Mbar(\wb\pi)}\right\}\nonumber\\
  \leq & \sum_{k=1}^K \sup_{M \in \MM, \pi_{-k} \in \Pi_{-k}} \left\{ \hm_k(\wb\pi_k, \pi_{-k}) - \frac{\gamma}{K} \cdot \hell{M(\wb\pi_k, \pi_{-k})}{\Mbar(\wb\pi_k, \wb\pi_{-k})}\right\}\nonumber\\
  = & \sum_{k=1}^K \sup_{M \in \MM, \pi_{-k} \in \Pi_{-k}} \E_{a_k \sim \wb\pi_k} \left[ \hm_k(a_k, \pi_{-k}) - \frac{\gamma}{K} \cdot \hell{M(a_k, \pi_{-k})}{\Mbar(a_k, \wb\pi_{-k})}\right]\nonumber\\
  =& \sum_{k=1}^K \inf_{\pi_k \in \Pi_k} \sup_{M \in \MM, \pi_{-k} \in \Pi_{-k}} \E_{a_k \sim \pi_k} \left[ \hm_k(a_k, \pi_{-k}) - \frac{\gamma}{K} \cdot \hell{M(a_k, \pi_{-k})}{\Mbar(a_k, \wb\pi_{-k})}\right]\nonumber\\
  \leq & \sum_{k=1}^K \dec_{\gamma/K}(\til \MM_k, (\Mbar, \wb\pi_{-k}))\nonumber.
\end{align}
Above, we have used the following facts:
\begin{enumerate}
  \item The second equality uses \cref{ass:convexity_pols} to conclude that for all $a_k, \pi_k, \wb\pi_{-k}, M, \Mbar$, 
\begin{align}
\BP_{o \sim M(a_k, \pi_{-k})}(\varphi(o) = a_k) = 1, \quad \BP_{o \sim \Mbar(a_k, \wb\pi_{-k})}(\varphi(o) = a_k) = 1\nonumber,
\end{align}
thus allowing us to apply \cref{lem:fdiv-factor} to give that
\begin{align}
\hell{M(\wb\pi_k, \pi_{-k})}{\Mbar(\wb\pi_k, \wb\pi_{-k})} = \E_{a_k \sim \wb\pi_k} \left[ \hell{M(a_k, \pi_{-k})}{\Mbar(a_k, \wb\pi_{-k})}\right]\nonumber.
\end{align}
\item The third equality follows from the fact that $\wb\pi_k \in \MC_k(\wb\pi)$ for all $k \in [K]$.
\item The final inequality follows from \cref{eq:dec-mtil-lb}.
\end{enumerate}

\end{proof}

\subsection{Bounds for Markov games}

Here, we prove \cref{thm:single-multiple-mg}. The proof uses a number of technical lemmas which are stated and proven in the sequel.

\begin{proof}[\pfref{thm:single-multiple-mg}]
  As in the proof of \cref{thm:single-multiple-ch}, for each $k \in [K]$ and $\pi_{-k} \in \Pi_{-k}$, we define
    \begin{align}
\til \MM_\ag(\pi_{-\ag}) := \{ \pi_\ag \mapsto \single{M}(\pi_\ag, \pi_{-\ag}) \ : \ M \in \MM \}\nonumber.
    \end{align}
    For any $\wb\pi_{-k} \in \Pi_{-k}$ and $\Mbar \in \MM$, we denote the corresponding element of  $\til \MM_k(\wb\pi_{-k}) \subseteq \til \MM_k$ by $(\Mbar, \wb\pi_{-k})$. We will prove the following stronger result: there is some constant $C' > 0$ so that for each $\Mbar \in \MM$,
    \begin{align}
\decoreg[\gamma](\sJ, \Mbar) \leq  \frac{C'KH\log H}{\gamma} + \sum_{k=1}^K \sup_{\wb \pi_{-k} \in \Pi_{-k}} \decoreg[\gamma/(C'KH\log H)](\til\MM_k, (\ol M, \wb \pi_{-k}))\label{eq:mg-bound-each-mbar}.
    \end{align}
Fix any $\ep > 0$. For each $k \in [K]$, let $\Pi_k^\ep$ be a finite $\ep$-cover of $\Pi_k$ in the sense that for all $\pi_k \in \Pi_k$, there is some element $\pi_k^\ep \in \Pi_k^\ep$ so that, for all $M \in \MM, \pi_{-k} \in \Pi_{-k}$,
\begin{align}
\hell{M(\pi_k, \pi_{-k})}{M(\pi_k^\ep, \pi_{-k})} \leq \ep^2\nonumber.
\end{align}
Furthermore, we require that the mapping $\pi_k \mapsto \pi_k^\ep$ is measurable with respect to the Borel $\sigma$-algebra on $\Pi_k$. By finiteness of $\MS, \MA_k$, it is straightforward to see that such a finite cover $\Pi_k^\ep$ exists. The size of the cover $\Pi_k^\ep$ may depend on $|\MS|, |\MA|$, but this will not matter as $|\Pi_k^\ep|$ will not enter into our final bounds. (We introduce discretization here only to ensure that $\Pi_k^\ep$ is compact when applying \cref{lem:berge-theorem}.) \dfcomment{is discretization only being used to ensure compactness? it might be good to explain that this is the purpose.}\noah{did so}

We collect a few basic properties of $\Pi_k^\ep$ in the below lemma, proved at the end of the section:
\begin{lemma}
  \label{lem:pikep-properties}
  For any $\pi_k \in \Pi_k$, there is some $\pi_k^\ep \in \Pi_k^\ep$ so that the following holds.  For any $M, \Mbar \in \MM$, $\pi_{-k} \in \Pi_{-k}$,
  \begin{align}
    \hell{M(\pi_k^\ep, \pi_{-k})}{\Mbar(\pi_k^\ep, \wb\pi_{-k})} \geq & \frac 13 \cdot \hell{M(\pi_k, \pi_{-k})}{\Mbar(\pi_k, \wb\pi_{-k})} - 2\ep^2 \nonumber\\
    \left| \hm_k(\pi_k, \pi_{-k}) - \hm_k(\pi_k^\ep, \pi_{-k}) \right| \leq & \ep \nonumber.
  \end{align}
\end{lemma}

\paragraph{Existence of fixed points} Let $C > 0$ be the constant of \cref{prop:equiv-markov}, and write $\gamma ' = \gamma / (CKH\log H)$. 
For each $k \in [K]$, define the function $\MC_k : \Pi \ra {\Delta(\Pi_k^\ep)}$ by
\begin{align}
\MC_k(\ol \pi) = \argmin_{p_k \in \Delta(\Pi_k^\ep)} \sup_{M \in \MM, \pi_{-k} \in \Pi_{-k}} \E_{\pi_k \sim p_k} \left[ \hm_k(\pi_k, \pi_{-k}) - \gamma' \cdot \hell{M(\pi_k, \pi_{-k})}{\ol M(\pi_k, \ol\pi_{-k})}\right] + \ep \cdot \| p_k \|_2^2\nonumber,
\end{align}
where $\| p_k \|_2^2$ denotes the squared $\ell_2$ norm of $p_k$, interpreted as a vector in the Euclidean space $\BR^{|\Pi_k^\ep|}$. 

Further, for $\pi_{-k} \in \Pi_{-k},\ M \in \MM$, define the function $G_{M, \pi_{-k}} : \Delta(\Pi_k) \times \Pi_{-k} \ra \BR$ by
\begin{align}
G_{M, \pi_{-k}}(p_k, \ol \pi_{-k}) = \E_{\pi_k \sim p_k} \left[ \hm_k(\pi_k, \pi_{-k}) - \gamma' \cdot \hell{M(\pi_k, \pi_{-k})}{\ol M(\pi_k, \ol \pi_{-k})}\right]\nonumber.
\end{align}
We may view $\wb\pi_{-k}$ as an element of $\Delta(\MA_k)^{\MS \times [H]}$, which is a subset of Euclidean space (since $\MA_k, \MS$ are assumed to be finite). 
Since there are finitely many states and actions, it follows from the dominated convergence theorem that for all $M, \pi_k, \pi_{-k}$, the function $\ol \pi_{-k} \mapsto \hell{M(\pi_k, \pi_{-k})}{\ol M(\pi_k, \ol \pi_{-k})}$ is continuous. Hence $G_{M, \pi_{-k}}(p_k, \ol \pi_{-k})$ is continuous in $(p_k, \ol \pi_{-k})$. Hence the function
\begin{align}
\til G_k(p_k, \ol \pi_{-k}) := \sup_{M \in \MM, \pi_{-k} \in \Pi_{-k}} G_{M, \pi_{-k}}(p_k, \ol \pi_{-k}) + \ep \cdot \| p_k \|_2^2\nonumber
\end{align}
is also continuous. Furthermore, $G_{M, \pi_{-k}}(p_k, \ol \pi_{-k})$ is linear in $p_k$ (for fixed $\ol \pi_{-k}$), so $\til G_k(p_k, \ol \pi_{-k})$ is strongly convex in $p_k$ (for fixed $\ol \pi_{-k}$). Thus $\MC_k(\ol \pi) = \argmin_{p_k \in \Delta(\Pi_k^\ep)} \{ \til G_k(p_k, \ol \pi_{-k}) \}$ is a singleton for all $\ol \pi$. Furthermore, by continuity of $\til G_k$, compactness of $\Delta(\Pi_k^\ep)$ and $\Pi_{-k}$, and \cref{lem:berge-theorem}, we have that $\MC_k(\ol \pi)$ is upper hemicontinuous, which means, by single-valuedness, it is actually continuous.

Given $\ol M \in \MM$, $\wb\pi_{-k} \in \Pi_{-k}$, note that the pure observation distribution of the model $\pi_k \mapsto \Mbar(\pi_k, \wb\pi_{-k})$ is exactly that of an MDP, which we denote by $\Mbar_{\wb\pi_{-k}}$: it has horizon $H$, state space $\MS$, action space $\MA_k$, and rewards and transitions given by those of $\Mbar$ when each agent $k' \neq k$ acts according to $\wb\pi_{k', h}(\cdot | s)$ at each state $s$ and step $h$ (to be precise, the rewards of $\Mbar_{\wb\pi_{-k}}$ are given by the rewards of agent $k$ in $\Mbar$). Note that the space of randomized nonstationary policies of $\Mbar_{\wb\pi_{-k}}$ is $\Pi_k$ (using \cref{ass:pi-mg}). %

\dfcomment{It would be good to add a little explanation here, explaining that moving to $\pistar_k$ is how we avoid convexifying.}\noah{did so}

Since we do not assume convexity of $\Pi_k$, elements $p_k \in \Delta(\Pi_k^\ep)$ may not belong to $\Pi_k$. We next introduce a set of decisions in $\Pi_k$ which are ``equivalent'' to $p_k$ given a reference model $\Mbar$ and a reference decision $\wb\pi_{-k}$. In particular, for $\Mbar \in \MM, \wb\pi_{-k} \in \Pi_{-k}$, and $p_k \in \Delta(\Pi_k^\ep)$, let $\Pi\subs{\Mbar, \wb\pi_{-k}}^\st(p_k) \subset \Pi_k$ be the set of all policies $\pi_k^\st \in \Pi_k$ which satisfy \eqref{eq:define-pistar} of \cref{lem:make-pol-randomized} for $p_k$ and $\pi_k \mapsto \Mbar(\pi_k, \wb\pi_{-k})$. Note that $\Pi\subs{\Mbar, \wb\pi_{-k}}^\st(p_k)$ is a nonempty convex set: as a subset of $\Delta(\MA_k)^{\MS \times [H]}$, it is a product of sets (one for each factor of $\Delta(\MA_k)$), each of which is either a singleton or all of $\Delta(\MA_k)$. It is straightforward from the definition that the map $(p_k, \wb\pi_{-k}) \mapsto \Pi\subs{\Mbar, \wb\pi_{-k}}^\st(p_k)$ is upper hemicontinuous. Then \cref{prop:equiv-markov} gives that, for any $\Mbar$ and $\wb\pi_{-k}$ and $p_k$, if $\pi_k^\st \in \Pi\subs{\Mbar,\wb\pi_{-k}}^\st(p_k)$ is the corresponding policy in \cref{eq:define-pistar}, then for $\gamma > 0$, 
\begin{align}
  & \sup_{M \in \MM, \ \pi_{-k} \in \Pi_{-k}} \left\{\hm_k(\pi_k^\st, \pi_{-k}) - \frac{\gamma}{K} \cdot \hell{M(\pi_k^\st, \pi_{-k})}{\ol M(\pi_k^\st, \ol \pi_{-k})}\right\}\nonumber\\
  \leq & \frac{1}{\gamma'} + \sup_{M \in \MM, \pi_{-k} \in \Pi_{-k}} \E_{\pi_k \sim p_k} \left[ \hm_k(\pi_k, \pi_{-k}) - {\gamma'} \cdot \hell{M(\pi_k, \pi_{-k})}{\ol M(\pi_k, \ol \pi_{-k})}\right]\label{eq:bound-by-pistar}.
\end{align}

Since the mapping $\wb\pi \mapsto \MC_k(\wb\pi) \in \Delta(\Pi_k^\ep)$ is continuous, the composition $\MC_k^\st(\wb\pi) := \Pi^\st\subs{\Mbar,\wb\pi_{-k}}(\MC_k(\wb\pi))$ is upper hemicontinuous. Thus, by Kakutani's fixed point theorem \cite[Lemma 20.1]{osborne1994course}, the set-valued mapping $C^\st(\wb\pi) := \MC_1^\st(\wb\pi) \times \cdots \times \MC_K^\st(\wb\pi)$ has a fixed point.

\paragraph{Applying the fixed point strategy}
Let $\ol \pi \in \Pi$ be a fixed point for $C^{\star}$, so that $\ol \pi_k \in \MC_k^\st(\ol \pi)$ for each $k \in [K]$.  Then
\begin{align}
  & \decoreg[\gamma](\sJ, \Mbar) \nonumber\\
  \leq & \sup_{M \in \MM} \left\{ \sum_{k=1}^K \hm_k(\ol \pi) - \gamma \cdot \hell{M(\ol\pi)}{\ol M(\ol\pi)}\right\}\nonumber\\
  \leq & \sum_{k=1}^K \sup_{M \in \MM}  \left\{ \hm_k(\ol \pi_k, \ol \pi_{-k}) - \frac{\gamma}{K} \cdot \hell{M(\ol\pi_k, \ol \pi_{-k})}{\ol M(\ol \pi_k, \ol \pi_{-k})}\right\}\nonumber\\
  \leq & \sum_{k=1}^K \sup_{M \in \MM, \ \pi_{-k} \in \Pi_{-k}} \left\{\hm_k(\ol \pi_k, \pi_{-k}) - \frac{\gamma}{K} \cdot \hell{M(\ol \pi_k, \pi_{-k})}{\ol M(\ol\pi_k, \ol \pi_{-k})}\right\}\nonumber\\
  \leq & \frac{1}{\gamma'} +\sum_{k=1}^K \sup_{M \in \MM, \pi_{-k} \in \Pi_{-k}} \E_{\pi_k \sim \MC_k(\ol \pi)} \left[ \hm_k(\pi_k, \pi_{-k}) - {\gamma'} \cdot \hell{M(\pi_k, \pi_{-k})}{\ol M(\pi_k, \ol \pi_{-k})}\right]\label{eq:use-markovian}\\
  \leq & \frac{1}{\gamma'} + \sum_{k=1}^K \ep +  \inf_{p_k \in \Delta(\Pi_k^\ep)} \sup_{M \in \MM, \pi_{-k} \in \Pi_{-k}} \E_{\pi_k \sim p_k} \left[ \hm_k(\pi_k, \pi_{-k}) - \gamma' \cdot \hell{M(\pi_k, \pi_{-k})}{\ol M(\pi_k, \ol \pi_{-k})}\right]\label{eq:use-ck-choice}\\
  \leq &  \frac{1}{\gamma'} + \sum_{k=1}^K 2\ep + \gamma' \cdot 2\ep^2 +  \inf_{p_k \in \Delta(\Pi_k)} \sup_{M \in \MM, \pi_{-k} \in \Pi_{-k}} \E_{\pi_k \sim p_k} \left[ \hm_k(\pi_k, \pi_{-k}) - \frac{\gamma'}{3} \cdot \hell{M(\pi_k, \pi_{-k})}{\ol M(\pi_k, \ol \pi_{-k})}\right]\label{eq:transfer-to-piep}.
\end{align}
where \cref{eq:use-markovian} uses \eqref{eq:bound-by-pistar} and the fact that $\ol \pi_k \in \MC_k^\st(\wb\pi) = \Pi_{\ol M, \ol \pi_{-k}}^\st(\MC_k(\ol \pi))$ for each $k$, and \cref{eq:use-ck-choice} uses the definition of $\MC_k(\ol \pi)$. Finally, \cref{eq:transfer-to-piep} uses \cref{lem:pikep-properties}, as follows: given any distribution $p_k \in \Delta(\Pi_k)$, we consider the distribution $p_k^\ep \in \Delta(\Pi_k^\ep)$ which is given by pushing forward $p_k$ through the map $\pi_k \mapsto \pi_k^\ep$ (here we use that $\pi_k \mapsto \pi_k^\ep$ is measurable to ensure that $p_k^
ep$ is well-defined). Then by \cref{lem:pikep-properties}, for all $M \in \MM, \pi_{-k} \in \Pi_{-k}, \wb\pi_{-k} \in \Pi_{-k}$, we have
\begin{align}
  \E_{\pi_k \sim p_k^\ep}[\hm_k(\pi_k, \pi_{-k})] \leq & \E_{\pi_k \sim p_k}[\hm_k(\pi_k, \pi_{-k})] + \ep\nonumber\\
  - \gamma' \cdot \E_{\pi_k \sim p_k^\ep}[\hell{M(\pi_k, \pi_{-k})}{\Mbar(\pi_k, \wb\pi_{-k})}] \leq & -\frac{\gamma'}{3} \cdot \E_{\pi_k \sim p_k} [\hell{M(\pi_k, \pi_{-k})}{\Mbar(\pi_k, \wb\pi_{-k})}] + \gamma' \cdot 2\ep^2\nonumber.
\end{align}
By taking $\ep \ra 0$, we obtain that, for some constant $C > 0$, 
\begin{align}
  \decoreg[\gamma](\sJ, \Mbar) \leq & \frac{CKH\log H}{\gamma} + \sum_{k=1}^K \decoreg[\gamma/(CKH\log H)](\til\MM_k, (\ol M, \ol \pi_{-k}))  \nonumber\\
  \leq & \frac{CKH\log H}{\gamma} + \sum_{k=1}^K \sup_{\til \pi_{-k} \in \Pi_{-k}} \decoreg[\gamma/(CKH\log H)](\til\MM_k, (\ol M, \til \pi_{-k}))\nonumber,
\end{align}
thus verifying \cref{eq:mg-bound-each-mbar}. 
\end{proof}

\subsubsection{Supporting lemmas}

\begin{proof}[\pfref{lem:pikep-properties}]
  To establish the first property, we use the definition of $\Pi_k^\ep$ and the triangle inequality for Hellinger distance to conclude that
  \begin{align}
    &  \hell{M(\pi_k, \pi_{-k})}{\Mbar(\pi_k, \wb\pi_{-k})}\nonumber\\
    \leq & 3 \cdot \left( \hell{M(\pi_k^\ep, \pi_{-k})}{M(\pi_k, \pi_{-k})} + \hell{\Mbar(\pi_k^\ep, \wb\pi_{-k})}{\Mbar(\pi_k, \wb\pi_{-k})} + \hell{M(\pi_k^\ep, \pi_{-k})}{\Mbar(\pi_k^\ep, \wb\pi_{-k})} \right)\nonumber\\
    \leq & 3 \cdot \left( \hell{M(\pi_k^\ep, \pi_{-k})}{M(\pi_k, \pi_{-k})} + 2\ep^2 \right)\nonumber,
  \end{align}
  and rearranging gives the first claimed inequality of the lemma.

  To prove the second inequality, we note that for each $\pi_k \in \Pi_k$, the cover element $\pi_k^\ep \in \Pi_k^\ep$ satisfies the following: for all $M \in \MM, \pi_{-k} \in \Pi_{-k}$
  \begin{align}
\left| \hm_k(\pi_k, \pi_{-k}) - \hm_k(\pi_k^\ep, \pi_{-k})\right| = \left| \fm_k(\pi_k, \pi_{-k}) - \fm_k(\pi_k^\ep, \pi_{-k}) \right| \leq & \Dhel{M(\pi_k, \pi_{-k})}{M(\pi_k^\ep, \pi_{-k})} \leq \ep\nonumber.
  \end{align}
\end{proof}

\noah{there's no reason we need finite-support for this to hold, probably should change it if time to streamline things}

The following lemma shows that for any MDP $\Mbar$ and distribution $p\in\Delta(\PiRNS)$, there exists a corresponding randomized policy in $\PiRNS$ which induces identical occupancies in $\Mbar$.
\begin{lemma}
  \label{lem:make-pol-randomized}
  Consider any finite-horizon MDP $\ol M = (\MS, H, \MA, P, R, \mu)$ with finite state and action spaces $\MS, \MA$. Let $\PiRNS$ denote the set of randomized nonstationary policies of $\Mbar$.  Suppose $p \in \Delta(\PiRNS)$ is a distribution over $\PiRNS$ with finite support. Consider any policy $\pi^\st \in \PiRNS$ so that: %
  \begin{align}
\forall a \in \MA,\ s \in \MS \mbox{ s.t. } \sum_{\pi' \in \PiRNS} p(\pi') \cdot d_h\sups{\Mbar, \pi'}(s) > 0: \quad \pi_h^\st(a|s) = \sum_{\pi \in \PiRNS  : \ p(\pi) > 0}  \frac{p(\pi) \cdot d_h\sups{\ol M, \pi}(s)}{\sum_{\pi' \in \PiRNS} p(\pi') \cdot d_h\sups{\ol M, \pi'}(s)} \cdot  \pi_h(a|s)\label{eq:define-pistar}.
  \end{align}
  Then for all states $s \in \MS$, $d_h\sups{\ol M, \pi^\st}(s) = \sum_{\pi \in \PiRNS} p(\pi) \cdot d_h\sups{\ol M, \pi}(s)$, and for all $(s,a) \in \MS \times \MA$, $d_h\sups{\ol M, \pi^\st}(s,a) = \sum_{\pi \in \PiRNS} p(\pi) \cdot d_h\sups{\ol M, \pi}(s,a)$.

  As a consequence, it follows that $V_1^{\ol M, \pi^\st} = \sum_{\pi \in \PiRNS} p(\pi) \cdot V_1^{\ol M, \pi}$. 
\end{lemma}
\begin{proof}[\pfref{lem:make-pol-randomized}]
We drop the superscript $\ol M$ in all relevant quantities throughout the proof. We use induction on $h$, noting that the base case $h=1$ is immediate since $d_1^\pi$ is identical for all $\pi \in \PiRNS$. Fix $p\in\Delta(\PiRNS)$, and let $\pistar$ be chosen as in \eqref{eq:define-pistar}. Assuming that the statement of the lemma holds at step $h-1$, we compute
  \begin{align}
    d_h^{\pi^\st}(s) =& \sum_{\substack{s', a':\\ d_{h-1}^{\pi^\st}(s') > 0}} d_{h-1}^{\pi^\st}(s') \cdot \pi_{h-1}^\st(a' | s') \cdot P_{h-1}(s | s',a')\nonumber\\
    =& \sum_{\substack{s',a':\\ d_{h-1}^{\pi^\st}(s') > 0}} \left( \sum_{\pi'} p(\pi') \cdot d_{h-1}^{\pi'}(s') \right) \cdot \sum_{\pi} \frac{p(\pi) \cdot d_{h-1}^\pi(s')}{\sum_{\pi'} p(\pi') \cdot d_{h-1}^{\pi'}(s)} \cdot \pi_{h-1}(a'|s') \cdot P_{h-1}(s | s',a')\nonumber\\
    =& \sum_\pi p(\pi) \cdot \sum_{\substack{s', a':\\ d_{h-1}^{\pi^\st}(s') > 0}} d_{h-1}^\pi(s') \cdot \pi_{h-1}(a'|s') \cdot P_{h-1}(s | s',a')\nonumber\\
    =& \sum_\pi p(\pi) \cdot \sum_{s', a'} d_{h-1}^\pi(s') \cdot \pi_{h-1}(a' | s') \cdot P_{h-1}(s | s', a')\nonumber\\
    =& \sum_\pi p(\pi) \cdot d_h^\pi(s)\nonumber,
  \end{align}
  where the second-to-last inequality follows since if $d_{h-1}^{\pi^\st}(s') = 0$, then (using the inductive hypothesis) for all $\pi$, $p(\pi) \cdot d_{h-1}^\pi(s') = 0$. 
The above chain of equalities then  completes the inductive step. It then follows immediately from the definition of $\pi^\st$ that $d_h^{\pi^\st}(s,a) = \sum_{\pi \in \PiRNS} p(\pi) \cdot d_h^\pi(s,a)$.

  The final statement regarding the value functions follows since, for all policies $\pi$,
  \begin{align}
V_1^\pi = \sum_{h=1}^H \sum_{(s,a) \in \MS \times \MA} d_h^\pi(s,a) \cdot r_h(s,a)\nonumber.
  \end{align}
\end{proof}

The remaining lemmas establish certain technical properties for the policy $\pistar\in\PiRNS$ constructed in \cref{lem:make-pol-randomized}.
\begin{lemma}
  \label{lem:p-pistar-hellinger}
There is a constant $C > 0$ so that the following holds.    Consider any finite-horizon MDP $\ol M = (\MS, H, \MA, P\sups\Mbar, R\sups\Mbar, \mu\sups\Mbar)$ with finite state and action spaces $\MS, \MA$. Let $\PiRNS$ denote the set of randomized nonstationary policies of $\Mbar$, and let $p \in \Delta(\PiRNS)$ be a distribution of finite support. Consider any policy $\pi^\st \in \PiRNS$ satisfying  \eqref{eq:define-pistar} for $p$. Then for any MDP $\MM = (\MS, H, \MA, P\sups{M}, R\sups{M}, \mu\sups{M})$,
  \begin{align}
\E_{\pi \sim p} \left[ \hell{M(\pi)}{\ol M(\pi)}\right] \leq CH \log H \cdot \hell{M(\pi^\st)}{\ol M(\pi^\st)}\nonumber.
  \end{align}
\end{lemma}
\begin{proof}[\pfref{lem:p-pistar-hellinger}]
  For any $\pi \in \PiRNS$, a full observation $(r, \ocirc) \sim M(\pi)$ consists of the trajectory $(s_1, a_1, r_1, \ldots, s_H, a_H, r_H)$, where $s_1 \sim \mu\sups{M}$, $s_{h+1} \sim P_h\sups{M}(s_h, a_h)$ for $h \in [H-1]$, $r_h \sim R_h\sups{M}(s_h, a_h)$ for $h \in [H]$, and $a_h \sim \pi_h(s_h)$ for $h \in [H]$. We use the notation $\tau_{1:h}$ to denote the portion of a trajectory consisting of $(s_1, a_1, r_1, \ldots, s_h, a_h, r_h)$.

  We use $\BP\sups{M,\pi}$ to denote the distribution of the trajectory $\tau_H \sim M(\pi)$, and $\BP\sups{\Mbar,\pi}$ to denote the distribution of the trajectory $\tau_H \sim \Mbar(\pi)$. We use $\E\sups{M, \pi}[\cdot]$ and $\E\sups{\Mbar, \pi}[\cdot]$ to denote the corresponding expectations. 
By Lemma A.13 of \citet{foster2021statistical}, it holds that, for some constant $C > 0$,
  \begin{align}
    & \E_{\pi \sim p} \left[ \hell{M(\pi)}{\ol M(\pi)} \right]\nonumber\\
    \leq & C \log(H) \cdot \E_{\pi \sim p} \E\sups{\ol M, \pi} \left[ \sum_{h=1}^H  \hell{\BP\sups{M, \pi}(s_h|\tau_{1:h-1}) }{\BP\sups{\ol M, \pi}(s_h | \tau_{1:h-1})}\right] \nonumber\\
    & + C \log(H) \cdot \E_{\pi \sim p} \E\sups{\ol M, \pi} \left[ \sum_{h=1}^H  \hell{\BP\sups{M, \pi}(r_h|\tau_{1:h-1}, s_h, a_h) }{\BP\sups{\ol M, \pi}(r_h | \tau_{1:h-1}, r_h, a_h)}\right] \nonumber\\
    = & C \log(H) \cdot \hell{\mu\sups{M}}{\mu\sups{\Mbar}} + C \log(H) \cdot \E_{\pi \sim p} \E^{\ol M, \pi} \left[ \sum_{h=1}^{H-1} \hell{P_h\sups{M}(s_h, a_h)}{P_h\sups{\ol M}(s_h, a_h)} \right] \nonumber\\
    & + C \log(H) \cdot \E_{\pi \sim p} \E\sups{\ol M, \pi} \left[ \sum_{h=1}^H \hell{R_h\sups{M}(s_h, a_h)}{R_h\sups{\ol M}(s_h, a_h)} \right]\label{eq:split-p-mbar}.
  \end{align}
  By \cref{lem:make-pol-randomized} and the definition of $\pi^\st$, for each $h \in [H], s \in \MS, a \in \MA$, it holds that
  \begin{align}
\E_{\pi \sim p} \left[ d_h\sups{\Mbar, \pi}(s,a) \right] = d_h\sups{\Mbar, \pi^\st}(s,a)\nonumber.
  \end{align}
  Thus, we may replace the expectation over $\pi \sim p,\ (s_h, a_h) \sim \ol M(\pi)$ in (\ref{eq:split-p-mbar}) with $(s_h, a_h) \sim \ol M(\pi^\st)$, and obtain
  \begin{align}
    & \E_{\pi \sim p} \left[ \hell{M(\pi)}{\ol M(\pi)}\right]\nonumber\\
    \leq & C \log(H) \cdot \left(\hell{\mu\sups{M}}{\mu\sups{\Mbar}} +  \E\sups{\ol M, \pi^\st} \left[ \sum_{h=1}^{H-1} \hell{P_h\sups{M}(s_h, a_h)}{P_h\sups{\ol M}(s_h, a_h)}\right] \right.\nonumber\\
    & \left.+ \E\sups{\ol M, \pi^\st} \left[ \sum_{h=1}^H \hell{R_h\sups{M}(s_h, a_h)}{R_h\sups{\ol M}(s_h, a_h)}\right]\right)\nonumber.
  \end{align}
  By \citet[Lemma A.9]{foster2021statistical} and the data processing inequality, we have that: %
  \begin{align}
    \E\sups{\ol M, \pi^\st} \left[ \sum_{h=1}^{H-1} \hell{P_h\sups{M}(s_h, a_h)}{P_h\sups{\ol M}(s_h, a_h)}\right] \leq & 4H \cdot \hell{M(\pi^\st)}{\ol M(\pi^\st)},\nonumber\\
    \E^{\ol M, \pi^\st} \left[ \sum_{h=1}^H \hell{R_h\sups{M}(s_h, a_h)}{R_h\sups{\ol M}(s_h, a_h)}\right] \leq & 4H \cdot \hell{M(\pi^\st)}{\ol M(\pi^\st)},\nonumber\\
    \hell{\mu\sups{M}}{\mu\sups\Mbar} \leq & \hell{M(\pi^\st)}{\Mbar(\pi^\st)}\nonumber.
  \end{align}
  It then follows that, for some constant $C > 0$,
  \begin{align}
\E_{\pi \sim p}\left[ \hell{M(\pi)}{\ol M(\pi)} \right] \leq C \cdot H \log(H) \cdot \hell{M(\pi^\st)}{\ol M(\pi^\st)}\nonumber,
  \end{align}
  as desired.
\end{proof}

\begin{lemma}
  \label{prop:equiv-markov}
There is a constant $C > 0$ so that the following holds. Consider any model class $\MM$ consisting of MDPs of fixed horizon $H$, finite state space $\MS$, finite action space $\MA$, and cumulative rewards bounded by $[0,1]$.  Let $\PiRNS$ be the class of randomized nonstationary policies. Consider any $\ol M \in \MM$ and finite-support distribution $p \in \Delta(\PiRNS)$, and let $\pi^\st \in \PiRNS$ denote any policy satisfying \eqref{eq:define-pistar} for $\Mbar$ and $p$. Then for any $\gamma > 0$,
\begin{align}
  & \sup_{M \in \MM} \left\{ \fm(\pim) - \fm(\pi^\st) - \gamma \cdot \hell{M(\pi^\st)}{\ol M(\pi^\st)} \right\}\nonumber\\
\leq     & \frac{CH\log H}{\gamma} +  \sup_{M \in \MM} \E_{\pi \sim p} \left[ \fm(\pim) - \fm(\pi) - \frac{\gamma}{CH\log H} \cdot \hell{M(\pi)}{\ol M(\pi)} \right]  \label{eq:dec-nodist}.
  \end{align}
\end{lemma}
An immediate consequence of \cref{prop:equiv-markov} is that
\begin{align}
  &  \inf_{\pi \in \PiRNS} \sup_{M \in \MM} \left\{ \fm(\pim) - \fm(\pi) - \gamma \cdot \hell{M(\pi)}{\ol M(\pi)} \right\}\nonumber\\
  \leq & \frac{CH\log H}{\gamma} + \inf_{p \in \Delta(\PiRNS)} \sup_{M \in \MM} \E_{\pi \sim p} \left[ \fm(\pim) - \fm(\pi) - \frac{\gamma}{CH\log H} \cdot \hell{M(\pi)}{\ol M(\pi)} \right] \nonumber.
  \end{align}
  \begin{proof}[\pfref{prop:equiv-markov}]
    Consider any $\Mbar \in \MM$, finite-support $p \in \Delta(\PiRNS)$, and let $\pi^\st$ be defined as in the statement of the lemma. 
    \cref{lem:make-pol-randomized} gives that $\fmbar(\pi^\st) = \E_{\pi \sim p} \left[ \fmbar(\pi) \right]$. Let $C$ be the constant from Lemma \ref{lem:p-pistar-hellinger}, and let $C' = C + \frac 12$.  Then for any $\gamma > 0$,
  \begin{align}
    & \sup_{M \in \MM} \fm(\pim) - \fm(\pi^\st) - C'H \log H \cdot \gamma \cdot \hell{M(\pi^\st)}{\ol M(\pi^\st)} \nonumber\\
    \leq & \sup_{M \in \MM} \fm(\pim) - f\sups{\ol M}(\pi^\st) - C'H \log H \cdot \gamma \cdot \hell{M(\pi^\st)}{\ol M(\pi^\st)} + \frac{1}{2\gamma} + \frac{\gamma}{2} \cdot \hell{M(\pi^\st)}{\ol M(\pi^\st)}\nonumber\\
    = & \sup_{M \in \MM} \fm(\pim) - f\sups{\ol M}(\pi^\st) - CH \log H \cdot \gamma \cdot \hell{M(\pi^\st)}{\ol M(\pi^\st)} + \frac{1}{2\gamma}\nonumber\\
    \leq & \sup_{M \in \MM} \E_{\pi \sim p} \left[ \fm(\pim) -f\sups{\ol M}(\pi) - \gamma \cdot \hell{M(\pi)}{\ol M(\pi)} \right] + \frac{1}{2\gamma} \label{eq:switch-to-p}\\
    \leq & \sup_{M \in \MM} \E_{\pi \sim p} \left[ \fm(\pim) - \fm(\pi) - \gamma \cdot \hell{M(\pi)}{\ol M(\pi)} + \frac{1}{2\gamma} + \frac{\gamma}{2} \cdot \hell{M(\pi)}{\ol M(\pi)} \right] + \frac{1}{2\gamma} \nonumber\\
    =& \sup_{M \in \MM} \E_{\pi \sim p} \left[ \fm(\pim) - \fm(\pi) - \frac{\gamma}{2} \cdot \hell{M(\pi)}{\ol M(\pi)} \right] + \frac{1}{\gamma}\label{eq:p-min-complete}.
  \end{align}
  where (\ref{eq:switch-to-p}) uses Lemma \ref{lem:p-pistar-hellinger}. %
  The statement of the proposition follows by replacing $\gamma$ with $\gamma \cdot C' H \log H$. 
\end{proof}

\section{Proofs for upper bounds from  \cref{sec:curse}}
\label{sec:proofs_curse}
In this section we prove \cref{thm:curse_ub}, which gives an upper bound for learning equilibria for CCE and CE instances in the \maf in a way that avoids the curse of multiple agents, i.e., avoids exponential scaling with the number of players $K$. %
In \cref{sec:ma-exo-define}, we describe the algorithm (\cref{alg:maexo}) used to prove \cref{thm:curse_ub}, which is based on the idea of \emph{exploration-by-optimization}, used previously in \cite{foster2022complexity,lattimore2021mirror}. In \cref{sec:ma-ir,sec:ir-exo,sec:exo-alg} we prove \cref{thm:curse_ub}; our proofs roughly follow those of \cite{foster2022complexity}, but require some subtle modifications to account for the multi-agent nature of our problem, as well as the more general notion of deviation sets $\Dev$ that we study. 

\subsection{The multi-agent exploration-by-optimization objective}
\label{sec:ma-exo-define}
We begin by describing the algorithm, \emph{Multi-Agent Exploration-by-Optimization} (\maexo; \cref{alg:maexo}) used to prove \cref{thm:curse_ub}. The algorithm is a multi-agent counterpart to the exploration-by-optimization (\exoalg) algorithm given \cite{foster2022complexity}. At a high level, \maexo (as well as its precursor \exoalg) is a variant of \texttt{EXP3}, which applies the exponential weights algorithm to a sequence of reward estimators which act as importance-weighted estimates for the true reward function. However, unlike \texttt{EXP3} and \exoalg, \maexo does not apply exponential weights to agents' pure policies themselves, but rather to their potential deviations $\Dev$.

In particular, \maexo operates over $T$ rounds of interaction with the environment. At each round $t \in [T]$, the algorithm first computes, for each player $k$, a \emph{reference distribution} $q_k\^t \in \Delta(\Dev)$ over their deviation space $\Dev$, according to an exponential weights update given a sequence of vectors $\wh f_k\^1, \ldots, \wh f_k\^{t-1}$ constructed by the algorithm in previous rounds (\cref{line:qk-mwu}). Roughly speaking, for $s\leq{}t-1$, the entries $\wh f_k\^s(\dev)$, $\dev \in \Dev$, of these vectors can be interpreted as the potential gain in value that agent $k$ could receive by deviating to $\dev$, given adversarial choices of the other agents' decisions. Accordingly, the reference distribution $q_k\^t$ will put more mass on deviations which lead to larger gains in value.

Next, in \cref{line:solve-exo}, the players jointly solve an optimization problem. To define this optimization problem, we introduce some notation. 
For each $k \in [K]$, let $\MG_k$ denote the set of all functions $g_k : \Pi_k' \times \Sigma \times \MO \ra \BR$, and let $\MG = \MG_1 \times \cdots \times \MG_K$.
 Given $q \in \prod_{k=1}^K \Delta(\Dev)$, $\eta > 0$, $\pi \in \Pi$,  $g \in \MG$, $\pi^\st = (\pi_1^\st, \ldots, \pi_K^\st) \in \prod_{k=1}^K \Dev$, and $\ M \in \MM$, define
\begin{align}
 \Gamma_{q,\eta}(\pi,g; \pi^\st, M) := & \E_{\sigma \sim \pi} \left[ \sum_{k=1}^K \fm_k(\Sw(\pi_k^\st, \sigma)) - \fm_k(\sigma) \right]  \label{eq:define-gamma}\\
  &+ \frac{1}{\eta} \cdot \sum_{k=1}^K \E_{\sigma \sim \pi, o \sim M(\sigma)} \E_{\pi_k' \sim q_k} \left[ \exp \left( \frac{\eta}{\pi(\sigma)} \cdot \left( g_k(\pi_k'; \sigma, o) - g_k(\pi_k^\st; \sigma, o)\right)\right) - 1 \right]\nonumber.
\end{align}
With this definition, the optimization problem solved in \cref{line:solve-exo} of \maexo is as follows:
\begin{align}
(\pi\^t, g\^t) \gets \argmin_{\pi \in \Pi, g \in \MG} \sup_{M \in \MM, \pi^\st \in \prod_{k=1}^K \Dev} \Gamma_{q\^t, \eta}(\pi, g; \pi^\st, M)\label{eq:exo-solve-text}.
\end{align}
The interpretation of the objective \cref{eq:define-gamma} and the optimization problem \cref{eq:exo-solve-text} is as follows. Roughly speaking, for each $k \in [K], \dev \in \Dev, \sigma \in \Sigma, o \in \MO$, the value $g_k(\dev; \sigma, o)$ for $g\in\cG_k$ can be interpreted as an estimate of player $k$'s gain in value by deviating to $\dev$ under joint decision profile $\sigma$, under an unknown model $M$ which is ``consistent with'' the decision-observation pair $(\sigma, o)$.  
Then, by solving \cref{eq:exo-solve-text}, the algorithm wishes to find a joint decision $\pi\^t \in \Pi$ and estimator $g\^t = (g_1\^t, \ldots, g_K\^t) \in \MG$, which, for each player $k \in [K]$, satisfies the following two properties:
\begin{itemize}
\item First, corresponding to the first term in \cref{eq:define-gamma}, for a worst-case unknown model $M$ and an unknown deviation $\pi_k^\st$, it should not be possible for player $k$ to gain much value by deviating to $\pi_k^\st$ given the policy $\pi\^t$. Here $\pi_k^\st$ should be interpreted as the best deviation in hindsight at the termination of the algorithm.
\item Second, corresponding to the second term in \cref{eq:define-gamma}: $\pi\^t$ and $ g\^t_k$ should be chosen so that with high probability under $\sigma \sim \pi\^t$, $g_k\^t$ does not underestimate the value gain in  deviating to $\pi_k^\st$ as compared to a sample $\dev$ from the reference distribution $q_k\^t$. %
  The second term in \cref{eq:define-gamma} can be viewed as a term that regularizes the adversarial choice of $\pi_k^\st$, analogously to the term subtracting squared Hellinger distance in the offset DEC (see \cref{eq:define-deco}): in particular, if $\pi_k^\st$ has significantly high value under the estimate $g_k$, then this term will be very negative, canceling out the (potentially large) first term. 
\end{itemize}
\dfcomment{maybe add a brief comparison explaining how this differs from running independent single-agent exo?}\noah{did so}
Given $(\pi\ind{t},g\ind{t})$ computed in \eqref{eq:exo-solve-text}, \cref{alg:maexo} samples a decision $\sigma\^t \sim \pi\^t$ and receives an observation $o\^t$ from the true model. Finally, in \cref{line:define-fhat}, players construct their reward estimators $\wh f_k\^t$ (to be used in future iterations $t' > t$ to construct $q_k\^{t'}$) using $g_k\^t(\cdot;\sigma\^t, o\^t)$. \noah{todo if time, remove the importance weighting} %
Once all $T$ rounds conclude, the algorithm outputs the joint decision $\wh \pi$ which is the uniform average over the $T$ pure decisions $\sigma\ind{1},\ldots,\sigma\ind{T}$.
We remark that \cref{alg:maexo} is different from having each player run the exploration-by-optimization algorithm of \cite{foster2022complexity}: in the latter, agents each individually optimize their own objective, in contrast to the optimization problem in \cref{eq:exo-solve-text}, which is solved for all agents simultaneously. This feature of \maexo allows us to obtain a guarantee scaling with $\decoreg(\co(\sJ))$, which can be arbitrarily smaller than what one obtains by using the approach of \cite{foster2022complexity} (see \cref{prop:multi-single-separation}). 

In \cref{def:exo} below, we formalize the  value of the minimax objective \cref{eq:exo-solve-text} computed in the course of \cref{alg:maexo}.
\begin{algorithm}[ht]
  \setstretch{1.3}
  \begin{algorithmic}[1]
    \State \textbf{parameters}: Learning rate $\eta > 0$.
    \State Initialize $\wh f_k\^0(\dev) := 0$ for all $k \in [K]$, $\dev \in \Dev$. 
    \For{$t = 1, 2, \ldots T$}
    \State \label{line:qk-mwu} For each agent $k \in [K]$, define $q_k\^t \in \Delta(\Pi_k')$ via exponential weights update: for $\dev\in \Pi_k'$,
      \begin{align}
        q_k\^t(\dev) := \frac{\exp\left( \eta \sum_{i=1}^{t-1} \wh f_k\^i(\dev)\right)}{\sum_{\pi_k'' \in \Dev} \exp \left( \eta \sum_{i=1}^{t-1} \wh f_k\^i(\pi_k'') \right) }\nonumber.
      \end{align}
    \State Define $q\^t = q_1\^t \times \cdots \times q_K\^t$.  The players jointly solve the following objective: \hfill\algcommentlight{\eqref{eq:define-gamma}} \label{line:solve-exo} 
      \begin{align}
(\pi\^t, g\^t) \gets \argmin_{\pi \in \Pi, g \in \MG} \sup_{M \in \MM, \pi^\st \in \Pi'} \Gamma_{q\^t, \eta}(\pi,g;\pi^\st, M) \nonumber.
      \end{align}
    \State Sample $\sigma\^t \sim \pi\^t$, play $\sigma\^t$, and observe $o\^t \sim M^\st(\sigma\^t)$.
    \State Each player $k \in [K]$ constructs their reward estimator $\wh f_k\^t$ as follows: for $\dev \in \Pi_k'$,\label{line:define-fhat}
    \begin{align}
\wh f_k\^t(\dev) = \frac{g_k\^t(\dev; \sigma\^t, o\^t)}{\pi\^t(\sigma\^t)}\nonumber.
    \end{align}
    \EndFor
    \State \textbf{return} joint decision $\wh \pi := \frac{1}{T} \sum_{t=1}^T \indic_{\sigma\^t}$. \label{line:define-whpi}
\end{algorithmic}
\caption{Multi-Agent Exploration by Optimization (\maexo)}
\label{alg:maexo}
\end{algorithm}

\begin{definition}[Exploration-by-optimization objective]
  \label{def:exo}
  Consider any instance $\sJ = \instma$ satisfying \cref{ass:ce-convexity}. For any scale parameter $\eta > 0$ and distribution $q \in \prod_{k=1}^K \Delta(\Dev)$, define
  \begin{align}
\exo[\eta](\sJ, q) = \inf_{\pi \in \Pi, g \in \MG} \sup_{M \in \MM, \pi^\st \in \prod_{k=1}^K \Dev} \Gamma_{q, \eta}(\pi, g; \pi^\st, M)\nonumber,
  \end{align}
  and let $\exo(\sJ) := \sup_{q \in \prod_{k=1}^K \Delta(\Dev)} \exo(\sJ, q)$. 
\end{definition}

To prove \cref{thm:curse_ub}, we first (\cref{sec:exo-alg}) bound the performance of \cref{alg:maexo} in terms of $\exo[\eta](\sJ)$. Following this, in \cref{sec:ma-ir} and \cref{sec:ir-exo}, we will upper bound $\exo[\eta](\sJ)$ by $\decopac[\gamma](\sJ)$ for an appropriate choice of $\gamma$, using a quantity we call the \emph{multi-agent (parametrized) information ratio} as an intermediary. Finally, in \cref{sec:curse_ub_proof}, we put these pieces together and prove \cref{thm:curse_ub}. 

\subsection{Bounding the performance of \cref{alg:maexo}}
\label{sec:exo-alg}

The following result bounds the performance of \cref{alg:maexo} (namely, the quantity $\hmstar(\wh \pi)$) in terms of $\exo[\eta](\sJ)$. 
\begin{lemma}
  \label{lem:alg-perf-exo}
    For any $\eta > 0$, \cref{alg:maexo} ensures that for all $\delta > 0$, with probability at least $1-\delta$,
    \begin{align}
      \hmstar(\wh \pi) = \sum_{k=1}^K \max_{\dev \in \Dev} \fmstar_k(\Sw(\dev, \wh\pi)) - \fmstar_k(\wh \pi) \leq \exo[\eta](\sJ) +  \frac{2}{T\eta} \cdot \sum_{k=1}^K \log \left( \frac{K \cdot |\Dev|}{\delta} \right)   \nonumber.
    \end{align}
  \end{lemma}
  \begin{proof}[\pfref{lem:alg-perf-exo}]
    \dfcomment{this proof seems to use $p\ind{t}$ instead of $\pi\ind{t}$ for the randomization distribution}\noah{fixed}
    For any $\pi_k^\st \in \Pi_k'$ and player $k \in [K]$, we define player $k$'s \emph{regret} with respect to the deviation $\pi_k^\st \in \Dev$ as follows: %
    \begin{align}
\REG_k(\pi_k^\st) = \sum_{t=1}^T \E_{\sigma\^t \sim \pi\^t} [ \fmstar_k(\Sw(\pi_k^\st, \sigma\^t)) - \fmstar_k(\sigma\^t)] = T \cdot \left( \fmstar_k(\Sw(\pi_k^\st, \wh \pi)) - \fmstar_k(\wh \pi) \right)\nonumber,
    \end{align}
    where the second equality above uses the definition of $\wh \pi$ in \cref{line:define-whpi} of \cref{alg:maexo} and the second property in \cref{ass:ce-convexity}. Hence, it suffices to bound $\frac 1T \cdot \sum_{k=1}^K \max_{\pi_k^\st \in \Dev}\REG_k(\pi_k^\st)$ to establish the statement of the lemma.

    Throughout the proof we use the following convention: for functions $f_k : \Dev \ra \BR$ (for instance, the reward estimators $\wh f_k\^t$ defined in \cref{line:define-fhat} of \cref{alg:maexo}), we will view $f_k$ as a vector in $\BR^{|\Dev|}$, whose coordinates are the values of $f_k(\dev)$, for $\dev\in\Dev$. Furthermore, for each $\dev \in \Dev$, we write $e_{\dev} \in \BR^{|\Dev|}$ to denote the corresponding unit vector whose $\dev$-th entry is 1 and all other entries are 0.
    
    By adding and subtracting $ \sum_{t=1}^T \lng e_{\pi_k^\st} , \wh f_k\^t \rng$, we obtain
    \begin{align}
  \REG_k(\pi_k^\st)=    & \sum_{t=1}^T \E_{\sigma\^t \sim \pi\^t} [ \fmstar_k(\Sw(\pi_k^\st, \sigma\^t)) - \fmstar_k(\sigma\^t)]\nonumber\\
      =& \sum_{t=1}^T \E_{\sigma\^t \sim \pi\^t} [ \fmstar_k(\Sw(\pi_k^\st, \sigma\^t)) - \fmstar_k(\sigma\^t)] + \sum_{t=1}^T \lng e_{\pi_k^\st}, \wh f_k\^t \rng - \sum_{t=1}^T \lng e_{\pi_k^\st}, \wh f_k\^t \rng\label{eq:reg-add-subtract}.
    \end{align}
    By \cref{lem:mwu-negkl} and the definition of the multiplicative weights updates for $q_k\^t$ in \cref{line:qk-mwu} of \cref{alg:maexo}, it holds that
    \begin{align}
      & \sum_{t=1}^T \lng e_{\pi_k^\st} , \wh f_k\^t \rng \nonumber\\
      \leq & \sum_{t=1}^T \lng q_k\^{t+1} , \wh f_k\^t \rng - \frac{1}{\eta} \sum_{t=1}^T \kld{q_k\^{t+1}}{q_k\^t} + \frac{1}{\eta} \kld{e_{\pi_k^\st}}{q_k\^1}\nonumber\\
      \leq & \sum_{t=1}^t \lng q_k\^{t+1}, \wh f_k\^t \rng - \frac{1}{\eta} \sum_{t=1}^T \kld{q_k\^{t+1}}{q_k\^t} + \frac{\log|\Pi_k'|}{\eta} \label{eq:qk-mwu-updates}.
    \end{align}
    By \cref{lem:dv}, we have that for each $t \in [T]$,
    \begin{align}
\lng q_k\^{t+1}, \wh f_k\^t \rng - \frac{1}{\eta} \kld{q_k\^{t+1}}{q_k\^t} \leq \frac{1}{\eta} \log \left( \sum_{\dev \in \Pi_k'} q_k\^t(\dev) \cdot \exp(\eta \cdot \wh f_k\^t(\dev))\right)\nonumber.
    \end{align}
    Using the above together with (\ref{eq:qk-mwu-updates}) and (\ref{eq:reg-add-subtract}), we obtain
    \begin{align}
      \REG_k(\pi_k^\st) \leq &  \sum_{t=1}^T \E_{\sigma\^t \sim \pi\^t} [ \fmstar_k(\Sw(\pi_k^\st, \sigma\^t)) - \fmstar_k(\sigma\^t)] +  \frac{1}{\eta} \log \left( \sum_{\dev \in \Pi_k'} q_k\^t(\dev) \cdot \exp(\eta \cdot \wh f_k\^t(\dev))\right)\nonumber\\
      &+ \frac{\log|\Pi_k'|}{\eta} - \sum_{t=1}^T \lng e_{\pi_k^\st}, \wh f_k\^t \rng\label{eq:reg-xt}.
    \end{align}
   Let $\CF\^t$ denote the $\sigma$-algebra generated by $(\sigma\^1, o\^1, \ldots, \sigma\^t, o\^t)$ (where the random variables $\sigma\^s, o\^s$ are drawn as in \cref{alg:maexo}). Note that $\CF\^t$ is a filtration, and write $\E_t[\cdot] = \E[\cdot | \CF\^t]$. For each $\pi_k^\st \in \Pi_k'$, we define a sequence of random variables, denoted $\crl{X_t(\pi_k^\st)}_{t \in [T]}$, by
    \begin{align}
X_t(\pi_k^\st) :=  \log \left( \sum_{\dev \in \Pi_k'} q_k\^t(\dev) \cdot \exp(\eta \cdot \wh f_k\^t(\dev))\right) - \lng e_{\pi_k^\st}, \eta \cdot \wh f_k\^t \rng\nonumber.
    \end{align}
    By \cref{lem:chernoff-martingale} and the union bound, with probability at least $1-\delta/K$, it holds that for all $\pi_k^\st \in \Pi_k'$,
    \begin{align}
\sum_{t=1}^T X_t(\pi_k^\st) \leq \sum_{t=1}^T \log \E_{t-1}[e^{ X_t(\pi_k^\st)}] + \log\left( \frac{K \cdot |\Pi_k'|}{\delta} \right)\label{eq:use-chernoff-martingale}.
    \end{align}
Note that $\pi\^t, q\^t$ are both measurable with respect to $\CF\^{t-1}$.    Then, for any $\pi_k^\st \in \Pi_k'$ and any $t \in [T]$, we may compute
    \begin{align}
      & \log \E_{t-1}[e^{X_t(\pi_k^\st)}]\nonumber\\
      &= \log \E_{t-1}\left[ \exp\left( \log \left( \sum_{\dev \in \Pi_k'} q_k\^t(\dev) \cdot \exp(\eta \wh f_k\^t(\dev))\right) - \eta \wh f_k\^t(\pi_k^\st)\right)\right]\nonumber\\
      &= \log \E_{\sigma\^t \sim \pi\^t} \E_{o\^t \sim M(\sigma\^t)} \left[ \E_{\dev \sim q_k\^t} \left[ \exp(\eta \wh f_k\^t(\dev) ) \right] \cdot \exp(-\eta \wh f_k\^t(\pi_k^\st))\right]\nonumber\\
      &= \log \E_{\sigma\^t \sim \pi\^t} \E_{o\^t \sim M(\sigma\^t)} \left[ \E_{\dev \sim q_k\^t} \left[ \exp\left(\frac{\eta}{\pi\^t(\sigma\^t)} \cdot g_k\^t(\dev;\sigma\^t, o\^t) \right) \right] \cdot \exp\left(-\frac{\eta}{\pi\^t(\sigma\^t)} \cdot g_k\^t(\pi_k^\st; \sigma\^t, o\^t)\right)\right]\nonumber\\
      &\leq \E_{\sigma\^t \sim \pi\^t} \E_{o\^t \sim M(\sigma\^t)} \E_{\dev \sim q_k\^t} \left[ \exp \left( \frac{\eta}{\pi\^t(\sigma\^t)} \cdot \left( g_k\^t(\dev ; \sigma\^t, o\^t) - g_k\^t(\pi_k^\st; \sigma\^t, o\^t)\right)\right)\right]-1\label{eq:xt-bound},
    \end{align}
    where the final inequality uses that  $\log(x) \leq x-1$ for all $x > 0$. 
    
    By \cref{eq:reg-xt}, \cref{eq:use-chernoff-martingale}, and \cref{eq:xt-bound}, and a union bound over $k \in [K]$, it follows that with probability at least $1-\delta$, for all $\pi_1^\st \in \Dev[1], \ldots, \pi_K^\st \in \Dev[K]$, letting $\pi^\st = (\pi_1^\st, \ldots, \pi_K^\st)$, 
    \begin{align}
      & \sum_{k=1}^K \REG_k(\pi_k^\st) \nonumber\\
      & \leq \sum_{k=1}^K \left( \frac{\log |\Dev|}{\eta} + \frac{\log \left( \frac{K \cdot |\Dev|}{\delta} \right)}{\eta} \right) + \sum_{t=1}^T \sum_{k=1}^K \E_{\sigma\^t \sim \pi\^t} \left[ \fmstar_k(\Sw(\pi_k^\st, \sigma\^t)) - \fmstar_k(\sigma\^t)\right] \nonumber\\
      &\quad+ \sum_{t=1}^T \sum_{k=1}^K \frac{1}{\eta} \left(\E_{\sigma\^t \sim \pi\^t} \E_{o\^t \sim M(\sigma\^t)} \E_{\dev \sim q_k\^t} \left[ \exp \left( \frac{\eta}{\pi\^t(\sigma\^t)} \cdot \left( g_k\^t(\dev ; \sigma\^t, o\^t) - g_k\^t(\pi_k^\st; \sigma\^t, o\^t)\right)\right)\right]-1\right)\nonumber\\
      & \leq \frac{2}{\eta} \cdot \sum_{k=1}^K \log \left( \frac{K \cdot |\Dev|}{\delta} \right)  + \sum_{t=1}^T \Gamma_{q\^t, \eta}(\pi\^t, g\^t; \pi^\st, \Mstar)\nonumber\\
      & \leq \frac{2}{\eta} \cdot \sum_{k=1}^K \log \left( \frac{K \cdot |\Dev|}{\delta} \right)  + \sum_{t=1}^T \sup_{\til \pi^\st \in \Pi', \til M \in \MM} \Gamma_{q\^t, \eta}(\pi\^t, g\^t; \til \pi^\st, \til M)\nonumber\\
      & \leq\frac{2}{\eta} \cdot \sum_{k=1}^K \log \left( \frac{K \cdot |\Dev|}{\delta} \right)  + T \cdot \exo[\eta](\sJ)\nonumber,
    \end{align}
    where the second inequality uses the definition of $\Gamma_{q\^t, \eta}(\pi\^t, g\^t; \pi^\st, \Mstar)$ in \cref{eq:define-gamma}, and the final equality follows since $\pi\^t, g\^t$ are chosen so as to minimize the multi-agent exploration-by-optimization objective (\cref{line:solve-exo} of \cref{alg:maexo}). 
  \end{proof}

    \begin{lemma}
      \label{lem:mwu-negkl}
      Consider any $d \in \bbN$, and let $f\^1, \ldots, f\^T \in \BR^d$ be an arbitrary sequence of vectors. 
For $\eta > 0$, let $q\^1, \ldots, q\^T \in \Delta^d$ denote the exponential weights update iterates with step size $\eta$ when the reward vectors are given by $f\^1, \ldots, f\^T$; in particular, for $t \in [T]$: %
    \begin{align}
q\^t(i) = \frac{\exp(\eta \sum_{s \leq t} f\^s(i))}{\sum_{j=1}^d \exp (\eta \sum_{s \leq t} f\^s(j))}\label{eq:qt-mwu}.
    \end{align}
    Then for any $q \in \Delta^d$, 
    \begin{align}
\sum_{t=1}^T \lng q, f\^t \rng \leq \sum_{t=1}^T \lng q\^{t+1} , f\^t \rng - \frac{1}{\eta} \sum_{t=1}^T \kld{q\^{t+1}}{q\^t} + \frac{1}{\eta} \kld{q}{q\^1}\nonumber.
    \end{align}
  \end{lemma}
  \begin{proof}[\pfref{lem:mwu-negkl}]
    By rearranging and telescoping, it suffices to show that, for each $t \in [T]$,
    \begin{align}
\lng q - q\^{t+1}, f\^t \rng = \frac{1}{\eta} \cdot \left( \kld{q}{q\^t} - \kld{q}{q\^{t+1}} - \kld{q\^{t+1}}{q\^t} \right)\nonumber.
    \end{align}
    To establish this inequality, we note that the multiplicative weight updates (\ref{eq:qt-mwu}) are equivalent to the following mirror descent updates with the negative entropy regularizer $\Phi(q) := \sum_{i=1}^d q_i \cdot \log q_i$:
    \begin{align}
\grad \Phi(p\^{t+1}) = \grad \Phi(q\^t) + \eta \cdot f\^t, \quad q\^{t+1} = \frac{p\^{t+1}}{\lng \mathbf{1}, p\^{t+1}\rng}\nonumber,
    \end{align}
    where $\mathbf{1} \in \BR^d$ denotes the all-ones vector. 
    Using the fact that for all $x,y,z \in \Delta^d$ (Eq.~(4.1) of \cite{bubeck2015convex})
    \begin{align}
\lng \grad \Phi(y) - \grad \Phi(x), x-z \rng = \kld{z}{y} - \kld{z}{x} - \kld{x}{y}\nonumber
    \end{align}
    with $z=q, y=q\^t, x = q\^{t+1}$, we obtain
    \begin{align}
      \frac{1}{\eta} \cdot \left( \kld{q}{q\^t} - \kld{q}{q\^{t+1}} - \kld{q\^{t+1}}{q\^t} \right) =& \frac{1}{\eta} \cdot \lng \grad \Phi(q\^t) - \grad \Phi(q\^{t+1}), q\^{t+1} - q \rng\nonumber\\
      =& \frac{1}{\eta} \cdot \lng \grad \Phi(q\^t) - \grad \Phi(p\^{t+1}), q\^{t+1} - q \rng \nonumber\\
      =& \lng f\^t, q-q\^{t+1}  \rng\nonumber,
    \end{align}
    where in the second equality we have used that $\grad \Phi(q\^{t+1}) = \grad \Phi(p\^{t+1}) -\log \left( \lng \mathbf{1}, p\^{t+1}\rng \right) \cdot \mathbf{1}$.
  \end{proof}

\subsection{The multi-agent parametrized information ratio}
\label{sec:ma-ir}
In this section, we introduce a multi-agent version of the parametrized information ratio of \cite[Definition 3.1]{foster2022complexity}, and upper bound this information ratio by the DEC of the convex hull of $\sJ$. In the following section, we will upper bound $\exo(\sJ)$ by this information ratio.

We first introduce some notation. We will wish to reason about the space of probability measures on $\MM \times \Dev[1] \times \cdots \times \Dev[K]$. Since $|\MM|$ may be infinite, to avoid measure-theoretic issues, we will slightly abuse notation by letting $\Delta(\MM \times \Dev[1] \times \cdots \times \Dev[K])$ denote the set of \emph{finitely supported} probability measures on $\MM \times \Dev[1] \times \cdots \times \Dev[K]$. This convention ensures that for any function $h : \MM \times \Dev[1] \times \cdots \times \Dev[K] \ra \BR$ and any $\mu \in \Delta(\MM \times \Dev[1] \times \cdots \times \Dev[K])$, $\E_{(M, \dev[1], \ldots, \dev[K]) \sim \mu}[h(M, \dev[1], \ldots, \dev[K])]$ is well-defined.

Consider any $k \in [K]$, a distribution $\mu \in \Delta(\MM \times \Dev[1] \times \cdots \times \Dev[K])$, and a distribution $\pi \in \Delta(\Sigma)$. Let $\BP$ denote the law of the process $(M, \pi_1^\st, \ldots, \pi_K^\st) \sim \mu$, $\sigma \sim \pi$, and $o \sim M(\sigma)$. We introduce the following distributions, depending on $\mu$ and $k$:
\begin{itemize}
\item Define the distribution $\mu_{\prior}^k \in \Delta(\Dev)$ by $\mu_{\prior}^k(\dev) = \BP( \pi_k^\st=\dev)$, for $\dev \in \Dev$.
\item For each $\sigma \in \Sigma$ and $o \in \MO$, define the distribution $\mu_{\pstr}^k \in \Delta(\Dev)$ by $\mu_{\pstr}^k(\dev; \sigma, o) = \BP(\pi_k^\st = \dev | (\sigma, o))$, for $\dev \in \Dev$. 
\end{itemize}
The distribution $\mu_{\prior}^k$ should be thought of as a prior distribution over the deviation $\pi_k^\st$, and the distribution $\mu_{\pstr}^k(\cdot ; \sigma, o)$ should be thought of as a posterior distribution over $\pi_k^\st$ after observing the pure decision $\sigma$ together with an observation $o \sim M(\sigma)$. 

\begin{definition}[Multi-agent information ratio]
  \label{def:ma-ir}
Given an instance $\sJ = \instma$ which is a generalized correlated equilibrium instance, the \emph{parametrized multi-agent information ratio} of the instance $\sJ$ is defined as
\begin{align}
  \infr[\gamma](\sJ) :=& \sup_{\mu \in \Delta(\MM \times \Pi_1' \times \cdots \times \Pi_K')} \inf_{\pi \in \Pi}  \E_{\sigma \sim \pi} \E_{(M, \pi_1^\st, \ldots, \pi_K^\st) \sim \mu} \left[ \sum_{k=1}^K \fm_k(\Sw(\pi_k^\st, \sigma)) - \fm_k(\sigma)\right]\nonumber\\
  & \qquad \qquad \qquad \qquad \qquad - \gamma \cdot \E_{\sigma \sim \pi} \E_{o | \sigma} \left[ \sum_{k=1}^K \hell{\mu_{\pstr}^k(\cdot ; \sigma, o) }{ \mu_{\prior}^k(\cdot)} \right]\nonumber.
\end{align}
In the above expression, when we write $\Sw(\pi_k^\st, \sigma)$ and $\fm_k(\sigma)$, we view $\sigma \in \Sigma$ as an element of $\Pi$ by associating it with the singleton distribution on $\sigma$, recalling that $\Pi = \Delta(\Sigma)$. 
\end{definition}

\cref{lem:inf-dec} upper bounds the multi-agent information ratio in terms of the multi-agent offset DEC of the convex hull of a given instance. 
\begin{lemma}
  \label{lem:inf-dec}
  Consider any instance $\sJ = \instma$ which satisfies \cref{ass:ce-convexity}, and for which $\co(\sJ)$ satisfies \cref{ass:nonneg-dev}.  Then for all $\gamma > 0$,
  \begin{align}
\infr[\gamma](\sJ) \leq K \cdot \decoreg[\gamma](\co(\sJ)) \nonumber.
  \end{align}
\end{lemma}
\begin{proof}[\pfref{lem:inf-dec}]
  We denote the pure decision sets of the instance $\sJ$ by $\Sigma_1, \ldots, \Sigma_K$, and the joint decision set as $\Sigma = \Sigma_1 \times \cdots \times \Sigma_K$. 
  Fix a prior $\mu \in \Delta(\MM \times \Pi_1' \times \cdots \times \Pi_K')$ and a distribution $\pi \in \Delta(\Sigma)$. Recall our notation from above: let $\BP$ denote the law of the process $\sigma \sim \pi, (M, \pi_1^\st, \ldots, \pi_K^\st) \sim \mu, o \sim M(\sigma)$. For each $k \in [K]$, let $\mu_{\prior}^k(\pi_k') = \BP(\pi_k^\st = \pi_k')$ and $\mu_{\pstr}^k(\pi_k' ; \sigma, o) = \BP(\pi_k^\st = \pi_k' | (\sigma, o))$. 

  Consider the value of the multi-agent information ratio given the choices for $\mu, \pi$:
  \begin{align}
 \E_{\sigma \sim \pi} \E_{(M, \pi_1^\st, \ldots, \pi_K^\st) \sim \mu} \left[ \sum_{k=1}^K \fm_k(\Sw(\pi_k^\st, \sigma)) - \fm_k(\sigma)\right] - \gamma \cdot \E_{\sigma \sim \pi} \E_{o | \sigma} \left[ \sum_{k=1}^K \hell{\mu_{\pstr}^k(\cdot ; \sigma, o)}{ \mu_{\prior}^k(\cdot)} \right]\nonumber.
  \end{align}
  For each $k \in [K]$, $\pi_k' \in \Pi_k'$, and $\pi \in \Pi$, define $\ol M_{\pi_k'}^k(\pi) := \E_\mu[M(\pi) | \pi_k^\st = \pi_k']$. Further define $\ol M(\pi) = \E_\mu[M(\pi)]$. Note that $\ol M_{\pi_k'}^k(\sigma) = \BP_{o | \sigma, \pi_k'}$ and $\ol M(\sigma) = \BP_{o | \sigma}$. 

To proceed, note that for each fixed $\sigma \in \Sigma$,
  \begin{align}
     \E_{o | \sigma}  \left[ \hell{\mu_{\pstr}^k(\cdot ; \sigma, o)}{\mu_{\prior}(\cdot)}\right]=& \E_{o |\sigma} \left[ \hell{\BP_{\pi_k^\st | \sigma, o}}{\BP_{\pi_k^\st}} \right]\nonumber\\
    =&  \E_{o |\sigma} \left[ \hell{\BP_{\pi_k^\st | \sigma, o}}{\BP_{\pi_k^\st | \sigma}} \right]\nonumber\\
    = &  \E_{\pi_k^\st \sim \mu} \left[ \hell{\BP_{o | \sigma, \pi_k^\st}}{\BP_{o | \sigma}} \right]\nonumber\\
    =& \E_{\pi_k^\st \sim \mu} \left[ \hell{\ol M_{\pi_k^\st}^k(\sigma)}{\ol M(\sigma)} \right]\label{eq:dpo-pri},
  \end{align}
  where the second equality follows
  since $\sigma \sim \pi$ and $(\pi_1^\st, \ldots, \pi_K^\st) \sim \mu$ are (marginally) independent, and the third equality holds %
  by \cref{lem:fdiv-conditioning}. 
  Furthermore, we have that 
  \begin{align}
    \E_{(M, \pi_1^\st, \ldots, \pi_K^\st) \sim \mu} \E_{\sigma \sim \pi}  \left[ \fm_k(\Sw(\pi_k^\st, \sigma)) - \fm_k(\sigma) \right] =& \E_{\sigma \sim \pi} \E_{\pi_k^\st \sim \mu} \E_{\mu} \left[ \fm_k(\Sw(\pi_k^\st, \sigma)) - \fm_k(\sigma) | \pi_k^\st \right] \nonumber\\
     =& \E_{\sigma \sim \pi} \E_{\pi_k^\st \sim \mu} \left[ f_k\sups{\ol M_{\pi_k^\st}^k}(\Sw(\pi_k^\st, \sigma)) - f_k\sups{\ol M_{\pi_k^\st}^k}(\pi) \right] \nonumber\\
    \leq & \E_{\pi_k^\st \sim \mu} \max_{\pi_k' \in \Pi_k'} \E_{\sigma \sim \pi} \left[ f_k\sups{\ol M_{\pi_k^\st}^k}(\Sw(\pi_k', \sigma)) - f_k\sups{\ol M_{\pi_k^\st}^k}(\sigma) \right]\label{eq:subopt-barm}.
  \end{align}
  Next, for any $\Mbar_1 \in \co(\MM)$, we have, for $\gamma > 0$, 
  \begin{align}
    & \decoreg[\gamma](\co(\sJ), \Mbar_1)\nonumber\\
    =&  \inf_{p \in \Delta(\Pi)} \sup_{M \in \co(\MM)} \E_{\pi \sim p} \left[ \sum_{k=1}^K \left( \max_{\dev \in \Dev} \fm_k(\Sw(\dev, \pi)) - \fm_k(\pi)\right) - \gamma \cdot \hell{M(\pi)}{\Mbar_1(\pi)}\right]\nonumber\\
    =& \inf_{p \in \Delta(\Pi)} \sup_{M \in \co(\MM)} \E_{\pi \sim p} \left[  \sum_{k=1}^K \max_{\dev \in \Dev} \E_{\sigma \sim \pi} \left[  \fm_k(\Sw(\dev, \sigma)) - \fm_k(\sigma)\right] - \gamma \cdot \hell{\E_{\sigma \sim \pi}[M(\sigma)]}{\E_{\sigma \sim \pi}[\Mbar_1(\sigma)]} \right]\nonumber\\
    \geq& \inf_{p \in \Delta(\Pi)} \sup_{M \in \co(\MM)} \E_{\pi \sim p} \left[  \sum_{k=1}^K \max_{\dev \in \Dev} \E_{\sigma \sim \pi} \left[  \fm_k(\Sw(\dev, \sigma)) - \fm_k(\sigma)\right] - \gamma \cdot \E_{\sigma \sim \pi}[\hell{M(\sigma)}{\Mbar_1(\sigma)}] \right]\nonumber\\
    \geq & \inf_{p \in \Delta(\Pi)} \sup_{M \in \co(\MM)} \sum_{k=1}^K   \max_{\dev \in \Dev}  \E_{\pi \sim p} \E_{\sigma \sim \pi} \left[  \fm_k(\Sw(\dev, \sigma)) - \fm_k(\sigma)\right] - \gamma  \cdot \E_{\pi \sim p} \E_{\sigma \sim \pi} \left[\hell{M(\sigma)}{\Mbar_1(\sigma)} \right]\nonumber\\
    \geq & \inf_{\pi \in \Delta(\Sigma)} \sup_{M \in \co(\MM)} \sum_{k=1}^K \max_{\dev \in \Dev}   \E_{\sigma \sim \pi} \left[    \fm_k(\Sw(\dev, \sigma)) - \fm_k(\sigma)\right] - \gamma \cdot \E_{\sigma \sim \pi} \left[\hell{M(\sigma)}{\Mbar_1(\sigma)} \right]\nonumber,
  \end{align}
  where the second equality follows from %
  \cref{ass:ce-convexity}, the first inequality follows from convexity of squared Hellinger distance, %
  the second inequality follows from Jensen's inequality, and the final inequality follows by replacing any $p \in \Delta(\Pi)$ with the decision $\wb\pi := \E_{\pi \sim p}[\pi] \in \Delta(\Sigma) = \Pi$. 
  
By the above display, the following holds: for any $\Mbar_1 \in \co(\MM)$, there is some $\pi \in \Pi$ so that, for each $\Mbar_2 \in \co(\MM)$, %
  \begin{align}
\sum_{k=1}^K \max_{\pi_k' \in \Pi_k'} \E_{\sigma \sim \pi} \left[ f_k\sups{\ol M_2}(\Sw(\pi_k', \sigma)) - f_k\sups{\ol M_2}(\sigma)\right] -\gamma \cdot   \E_{\sigma \sim \pi} \left[ \hell{\ol M_2(\sigma)}{\ol M_1(\sigma)}\right] \leq \decoreg[\gamma](\co(\sJ)) \label{eq:dec-mnu-qty}.
  \end{align}
  Since we have assumed that $\co(\sJ)$ satisfies \cref{ass:nonneg-dev}, the following holds: for each $k \in [K]$, $\pi \in \Pi$, and $\ol M_1 \in \co(\MM)$, we have (again using \cref{ass:ce-convexity})
  \begin{align}
    \max_{\pi_k' \in \Pi_k'} \E_{\sigma \sim \pi} \left[ f_k\sups{\ol M_1}(\Sw(\pi_k', \sigma)) - f_k\sups{\ol M_1}(\sigma) \right] = \max_{\dev \in \Dev} f_k\sups{\Mbar_1}(\Sw(\dev, \pi)) - f_k\sups{\Mbar_1}(\pi) \geq 0.\label{eq:mnu-nonneg}
  \end{align}
Then, by \cref{eq:mnu-nonneg} and \cref{eq:dec-mnu-qty}, for each $k \in [K]$, we have that for any $\Mbar_1 \in \co(\MM)$, there is $\pi \in \Pi$ so that for each $\Mbar_2 \in \co(\MM)$, 
  \begin{align}
\max_{\pi_k' \in \Pi_k'} \E_{\sigma \sim \pi} \left[ f_k\sups{\ol M_2}(\Sw(\pi_k', \sigma)) - f_k\sups{\ol M_2}(\sigma) \right] - \gamma \cdot \E_{\sigma \sim \pi} \left[ \hell{\ol M_2(\sigma)}{\ol M_1(\sigma)}\right] \leq \decoreg[\gamma](\co(\sJ))\label{eq:dec-mnu-k}.
  \end{align}

Next,  choose $\wb\pi \in \Pi$, given $\ol M_1 = \ol M$ to ensure that (\ref{eq:dec-mnu-qty}) holds for all $\Mbar_2 \in \co(\MM)$. Then for each $k \in [K]$ and each $\pi_k^\st \in \Pi_k'$, choosing $\ol M_2 = \ol M_{\pi_k^\st}^k$ in (\ref{eq:dec-mnu-k}),
  \begin{align}
\max_{\pi_k' \in \Pi_k'} \E_{\sigma \sim \wb\pi} \left[ f_k\sups{\ol M_{\pi_k^\st}^k}(\Sw(\pi_k', \sigma)) - f_k\sups{\ol M_{\pi_k^\st}^k}(\sigma) \right] - \gamma \cdot \E_{\sigma \sim \wb\pi} \left[ \hell{\ol M_{\pi_k^\st}^k(\sigma)}{\ol M(\sigma)}\right] \leq \decoreg[\gamma](\co(\sJ))\nonumber.
  \end{align}
  Taking expectation over $\pi_k^\st \sim \mu$ and using (\ref{eq:dpo-pri}) and (\ref{eq:subopt-barm}), we obtain
  \begin{align}
    & \E_{(M, \pi_k^\st) \sim \mu} \E_{\sigma \sim \wb\pi} \left[ \fm_k(\Sw(\pi_k^\st, \sigma)) - \fm_k(\sigma) \right] - \gamma \cdot \E_{\sigma \sim \wb\pi} \E_{o |\sigma} \left[ \hell{\mu_{\pstr}^k(\cdot ; \sigma, o)}{\mu_{\prior}(\cdot)}\right] \nonumber\\
    \leq & \E_{\pi_k^\st \sim \mu} \left[ \max_{\pi_k' \in \Pi_k'} \E_{\sigma \sim \wb\pi} \left[ f_k\sups{\ol M_{\pi_k^\st}^k}(\Sw(\pi_k', \sigma)) - f_k\sups{\ol M_{\pi_k^\st}^k}(\sigma) \right] - \gamma \cdot \E_{\sigma \sim \wb\pi} \left[ \hell{\ol M_{\pi_k^\st}^k(\sigma)}{\ol M(\sigma)}\right] \right]\nonumber\\
    \leq & \decoreg[\gamma](\co(\sJ))\nonumber.
  \end{align}
  Note that the choice of $\wb\pi$ depends only on $\Mbar$, and in particular it does not depend on $k$. Therefore, we may sum the above display over $k \in [K]$, to obtain
  \begin{align}
    &  \E_{(M, \pi_1^\st, \ldots, \pi_K^\st) \sim \mu} \E_{\sigma \sim \wb\pi} \left[\sum_{k=1}^K \fm_k(\Sw(\pi_k^\st, \sigma)) - \fm_k(\sigma) \right] - \gamma \cdot \E_{\sigma \sim \wb\pi} \E_{o |\sigma} \left[\sum_{k=1}^K  \hell{\mu_{\pstr}^k(\cdot ; \sigma, o)}{\mu_{\prior}(\cdot)}\right]\nonumber\\
    \leq & K \cdot \decoreg[\gamma](\co(\sJ))\nonumber.
  \end{align}
Using that the choice of $\mu \in \Delta(\MM \times \Dev[1] \times \cdots \times \Dev[\Ag])$ is arbitrary, we obtain that $\infr[\gamma](\sJ) \leq \decoreg[\gamma](\co(\sJ))$, as desired.
\end{proof}

\subsection{Relating the multi-agent information ratio and exploration-by-optimization objective}
\label{sec:ir-exo}

In this section, we prove the following result, which upper bounds $\exo[\eta](\sJ)$ by the multi-agent information ratio of $\sJ$, at scale $1/(8\eta)$. 
\begin{lemma}
  \label{lem:exo-ir}
Consider any instance $\sJ = \instma$ satisfying \cref{ass:ce-convexity}. Then  for all $\eta > 0$,
  \begin{align}
\exo(\sJ) \leq \infr[1/(8\eta)](\sJ)\nonumber.
  \end{align}
\end{lemma}
\begin{proof}[\pfref{lem:exo-ir}]
  Throughout the proof, we will denote the (finite) pure decision sets, as guaranteed by \cref{ass:ce-convexity}, by
 $\Sigma_1, \ldots, \Sigma_K$, and the joint decision set by $\Sigma := \Sigma_1 \times \cdots \times \Sigma_K$.  Additionally, we write $\Devall := \prod_{k=1}^k \Dev$ to denote the product of the deviation sets $\Dev$. 
  We can write
  \begin{align}
\exo(\sJ) = \sup_{q \in \prod_{k=1}^K \Delta(\Pi_k')} \inf_{\pi \in \Delta(\Sigma),\ g \in \MG} \sup_{\mu \in \Delta(\MM \times \Pi')}  \E_{(M, \pi^\st) \sim \mu}  \left[ \Gamma_{q, \eta}(\pi, g; \pi^\st, M) \right]\nonumber.
  \end{align}
  For $\alpha \geq \max\{1,1/\eta\}$ and $\vep \in (0,1)$, define
  \begin{align}
\MG_\alpha = \{ (g_1, \ldots, g_k) \in \MG \ : \ \| g_k \|_\infty \leq \alpha \ \forall k \in [K]\}, \qquad \MP_\vep = \{ \pi \in \Pi \ : \ \pi(\sigma) \geq \ep |\Sigma|^{-1} \ \forall \sigma \}\nonumber.
  \end{align}
  We will now use Sion's minimax theorem (\cref{thm:sion}), with $\MX = \MP_\vep \times \MG_\alpha$ and $\MY = \Delta(\MM \times \Devall)$, to interchange the $\inf_{\pi \in \Pi, g \in \MG}$ and the $\sup_{\mu \in \Delta(\MM \times \Devall)}$ in the definition ot $\exo[\eta](\sJ)$ above. We first check that its preconditions hold:
  \begin{itemize}
  \item Let the set $\MP_\vep$ have the standard topology induced from $\Pi$, so that $\MP_\vep$ is compact, and let $\MG_\alpha$ have the product topology. Tychanoff's theorem yields that $\MG_\alpha$ is compact, and thus $\MX = \MP_\vep \times \MG_\alpha$ is compact. It is also clearly convex.
  \item Let us give $\MY = \Delta(\MM \times \Devall)$ (which we recall is the space of finitely supported distributions on $\MM \times \Devall$) the weak topology, which is the coarsest topology so that the functional $\mu \mapsto \int \phi d \mu$ is continuous for all bounded functions $\phi : \MM \times \Devall \ra \BR$. 
  \item To establish the remaining preconditions, we need that the mapping $(\pi, g, \mu) \mapsto \E_{(M, \pi^\st) \sim \mu}[\Gamma_{q,\eta}(\pi, g;\pi^\st, M)]$ is uniformly bounded for $(\pi, g) \in \MP_\vep \times \MG_\alpha$ and $\mu \in \Delta(\MM \times \Devall)$. This follows immediately from the definition of $\Gamma_{q,\eta}(\pi, g;\pi^\st, M)$ and the domains $\MP_\vep$ and $\MG_\alpha$.
  \item Clearly, the map $\mu \mapsto \E_{(M,\pi^\st)}[\Gamma_{q,\eta}(\pi, g;\pi^\st, M)]$ is linear, and thus concave, for each $\pi,g$. Moreover, it is continuous by boundedness of $\Gamma_{q,\eta}(\pi,g;\pi^\st, M)$, and the fact that $\Delta(\MM \times \Devall)$ has the weak topology.
  \item By \cref{lem:joint-convexity}, the map $(\pi, g) \mapsto \E_{(M, \pi^\st) \sim \mu}[\Gamma_{q,\eta}(\pi, g;\pi^\st, M)]$ is convex in $(\pi, g)$ for any fixed $\mu$. Furthermore, it is continuous by definition of the product topology and since $\pi(\sigma)$ is uniformly bounded below for $\pi \in \MP_\vep$. 
  \end{itemize}
Having verified all of the conditions for \cref{thm:sion} to apply, we now have:
  \begin{align}
\exo(\sJ)\leq & \sup_{q \in \prod_{k=1}^K \Delta(\Dev)} \inf_{\pi \in \MP_\vep, g \in \MG_\alpha} \sup_{\mu \in \Delta(\MM \times \Devall)} \E_{(M, \pi^\st) \sim \mu} [ \Gamma_{q, \eta}(\pi, g; \pi^\st, M)]\nonumber\\
    = & \sup_{q \in \prod_{k=1}^K \Delta(\Pi_k')} \sup_{\mu \in \Delta(\MM \times \Pi')} \inf_{\pi \in \MP_\vep, g \in \MG_\alpha} \E_{(M,\pi^\st) \sim \mu} \left[ \Gamma_{q,\eta}(\pi,g; \pi^\st, M)\right]\label{eq:maexo-gamma-ub},
  \end{align}
  where the inequality follows since we are restricting to smaller sets $\MG_\alpha \subset \MG$ and $\MP_\vep \subset \Pi$ in the infimum, and the equality uses \cref{thm:sion}. 
  Given $q \in \prod_{k=1}^K \Delta(\Pi_k')$, $\mu \in \Delta(\MM \times \Pi')$, $\pi \in \MP_\vep$, consider the value of
  \begin{align}
    & \inf_{g \in \MG_\alpha} \E_{(M,\pi^\st) \sim \mu} [\Gamma_{q,\eta}(\pi,g;\pi^\st, M)] \nonumber\\
    =&\E_{(M,\pi^\st) \sim \mu} \E_{\sigma \sim \pi} \left[ \sum_{k=1}^K \fm_k(U(\pi_k^\st, \sigma)) - \fm_k(\sigma) \right] \nonumber\\
& + \frac{1}{\eta} \inf_{g \in \MG_\alpha} \sum_{k=1}^K \E_{(M, \pi^\st) \sim \mu} \E_{\sigma \sim \pi, o \sim M(\sigma)} \E_{\pi_k' \sim q_k} \left[ \exp\left( \frac{\eta}{\pi(\sigma)} \cdot \left( g_k(\pi_k';\sigma, o) - g_k(\pi_k^\st; \sigma, o)\right) \right) -1\right]\nonumber\\
    =& \E_{(M,\pi^\st) \sim \mu} \E_{\sigma \sim \pi} \left[ \sum_{k=1}^K \fm_k(U(\pi_k^\st, \sigma)) - \fm_k(\sigma) \right] \nonumber\\
    &+ \frac{1}{\eta}  \sum_{k=1}^K\inf_{g_k \in \MG_{k,\alpha}} \E_{(M, \pi^\st) \sim \mu} \E_{\sigma \sim \pi, o \sim M(\sigma)} \E_{\pi_k' \sim q_k} \left[ \exp\left( \frac{\eta}{\pi(\sigma)} \cdot \left( g_k(\pi_k';\sigma, o) - g_k(\pi_k^\st; \sigma, o)\right) \right) -1\right]\label{eq:split-into-gks},
  \end{align}
  where we have used $\MG_{k,\alpha}$ to denote $\{ g_k  \in \MG_k \ : \ \| g_k \|_\infty \leq \alpha \}$, so that $\MG_\alpha = \MG_{1,\alpha} \times \cdots \times \MG_{K,\alpha}$. 
  
  Let $\BP$ be the law of the process $(M, \pi^\st) \sim \mu$, $\sigma \sim \pi$, $o \sim M(\sigma)$, and define, for $k \in [K]$, $\mu_{\prior}^k(\pi'_k) = \BP(\pi_k^\st = \pi_k')$, and $\mu_{\pstr}^k(\pi_k';\sigma, o) = \BP(\pi_k^\st = \pi_k' | (\sigma, o))$. For each $k \in [K]$, the term corresponding to agent $k$ in the second term of \cref{eq:split-into-gks} above can be rewritten as follows, using the definition of the posterior distribution $\mu_{\pstr}^k(\pi_k';\sigma,o)$:
  \begin{align}
    & \inf_{g_k \in \MG_{k,\alpha}} \E_{(M, \pi^\st) \sim \mu}\E_{\sigma \sim \pi, o \sim M(\sigma)} \E_{\pi_k' \sim q_k} \left[ \exp\left( \frac{\eta}{\pi(\sigma)} \cdot \left( g_k(\pi_k';\sigma, o) - g_k(\pi_k^\st; \sigma, o)\right) \right) -1\right]\nonumber\\
    =& \inf_{g_k \in \MG_{k,\alpha}} \E_{\sigma \sim \pi} \E_{o | \sigma} \left[ \E_{\pi_k' \sim q_k} \left[ \exp \left( \eta \cdot \frac{g_k(\pi_k' ; \sigma, o)}{\pi(\sigma)} \right) \right] \cdot \E_{\pi_k^\st \sim \mu_{\pstr}^k(\cdot ; \sigma, o)} \left[\exp \left( -\eta \cdot \frac{g_k(\pi_k^\st; \sigma, o)}{\pi(\sigma)}\right) \right] -1 \right]\nonumber.
  \end{align}
  Given any $g_k \in \MG_{k, \eta \alpha}$, we have that $(\pi_k', \sigma, o) \mapsto \frac{\pi(\sigma)}{\eta} \cdot g_k(\pi_k';\sigma, o)$ and $(\pi_k^\st, \sigma, o) \mapsto \frac{\pi(\sigma)}{\eta} \cdot g_k(\pi_k^\st; \sigma, o)$ both belong to $\MG_{k, \alpha}$, meaning that the above quantity is upper bounded by
  \begin{align}
\inf_{g_k \in \MG_{k,\eta\alpha}} \E_{\sigma \sim \pi} \E_{o | \sigma} \left[ \E_{\pi_k' \sim q_k} \left[ \exp(g_k(\pi_k';\sigma, o)) \right] \cdot \E_{\pi_k^\st \sim \mu_{\pstr}^k(\cdot ; \sigma, o)}[\exp( - g_k(\pi_k^\st; \sigma, o))] - 1\right]\nonumber.
  \end{align}
  This expression is equal to
  \begin{align}
V_k(\pi, q, \mu) := \E_{\sigma \sim \pi} \E_{o | \sigma} \inf_{g_k : \Dev \ra \BR, \| g_k \|_\infty \leq \alpha \eta } \left\{ \E_{\pi_k' \sim q_k} [\exp(g_k(\pi_k'))] \cdot \E_{\pi_k^\st \sim \mu_{\pstr}^k(\cdot ; \sigma, o)}[\exp(-g_k(\pi_k^\st))] - 1 \right\}\nonumber.
  \end{align}
  By \cref{lem:hellinger-variational-bounded}, we have that for all $\pi,q,\mu$,
  \begin{align}
    V_k(\pi,q,\mu) =& -\E_{\sigma \sim \pi} \E_{o | \sigma} \sup_{g_k : \Pi_k' \ra \BR, \| g_k \|_\infty \leq \alpha \eta} \left\{ -\E_{\pi_k' \sim q_k} [\exp(g_k(\pi_k'))] \cdot \E_{\pi_k^\st \sim \mu_{\pstr}^k(\cdot ; \sigma, o)}[\exp(-g_k(\pi_k^\st))] + 1 \right\}\nonumber\\
    \leq & -\frac{1}{2} \cdot \E_{\sigma \sim \pi} \E_{o | \sigma} \left[\hell{\mu_{\pstr}^k(\cdot ; \sigma, o)}{q_k}\right] + 4 e^{-\alpha \eta} \label{eq:variational-exp-hell}.
  \end{align}
  Combining (\ref{eq:maexo-gamma-ub}), (\ref{eq:split-into-gks}), and (\ref{eq:variational-exp-hell}), we obtain the following upper bound:
  \begin{align}
\exo(\sJ)\leq &  \sup_{q \in \prod_{k=1}^K \Delta(\Pi_k')} \sup_{\mu \in \Delta(\MM \times \Pi')} \inf_{\pi \in \MP_\vep} \left\{  \E_{(M,\pi^\st) \sim \mu} \E_{\sigma \sim \pi} \left[ \sum_{k=1}^K \fm_k(U(\pi_k^\st, \sigma)) - \fm_k(\sigma) \right]\right.\nonumber\\
    &- \left. \frac{1}{2\eta} \sum_{k=1}^K \E_{\sigma \sim \pi} \E_{o | \sigma} \left[ \hell{\mu_{\pstr}^k(\cdot ; \sigma, o)}{q_k} \right] + \frac{4K}{\eta}\cdot e^{-\alpha \eta} \right\} \nonumber.
  \end{align}
  Since $\fm_k \in [0,1]$ for all $k,M$ and $\hell{\cdot}{\cdot} \in [0,2]$, it follows that we may replace the $\inf_{\pi \in \MP_\vep}$ in the above expression with $\inf_{\pi \in \Pi}$ and pay an additive cost of $K\vep \cdot (1 + 1/\eta)$, and so
    \begin{align}
\exo(\sJ) \leq &  \sup_{q \in \prod_{k=1}^K \Delta(\Pi_k')} \sup_{\mu \in \Delta(\MM \times \Pi')} \inf_{\pi \in \Pi} \left\{   \E_{(M,\pi^\st) \sim \mu} \E_{\sigma \sim \pi} \left[ \sum_{k=1}^K \fm_k(U(\pi_k^\st, \sigma)) - \fm_k(\sigma) \right]\right.\nonumber\\
    &- \left. \frac{1}{2\eta} \sum_{k=1}^K \E_{\sigma \sim \pi} \E_{o | \sigma} \left[ \hell{\mu_{\pstr}^k(\cdot ; \sigma, o)}{q_k} \right] + \frac{4}{\eta}\cdot e^{-\alpha \eta} + K \vep \cdot (1 + 1/\eta) \right\} \nonumber.
    \end{align}
    Since the above holds for any $\vep \in (0,1)$ and $\alpha \geq \max\{1,1/\eta\}$, we may take the limits $\vep \ra 0, \alpha \ra \infty$ to get
    \begin{align}
\exo(\sJ) \leq &  \sup_{q \in \prod_{k=1}^K \Delta(\Pi_k')} \sup_{\mu \in \Delta(\MM \times \Pi')} \inf_{\pi \in \Pi} \left\{   \E_{(M,\pi^\st) \sim \mu} \E_{\sigma \sim \pi} \left[ \sum_{k=1}^K \fm_k(U(\pi_k^\st, \sigma)) - \fm_k(\sigma) \right]\right.\nonumber\\
    &- \left. \frac{1}{2\eta} \sum_{k=1}^K \E_{\sigma \sim \pi} \E_{o | \sigma} \left[ \hell{\mu_{\pstr}^k(\cdot ; \sigma, o)}{q_k} \right]\right\}\nonumber.
    \end{align}
    Next, for any choice of $q_k \in \Delta(\Dev)$, we have
    \begin{align}
      \E_{\sigma \sim \pi}\E_{o | \sigma} \left[\hell{\mu_{\pstr}^k(\cdot ; \sigma, o)}{\mu_{\prior}^k}\right]=& \E_{\sigma \sim \pi}\E_{o | \sigma} \left[\hell{\mu_{\pstr}^k(\cdot ; \sigma, o)}{\E_{\sigma \sim \pi} \E_{o|\sigma}[\mu_{\pstr}(\cdot ; \sigma, o)]}\right]  \nonumber\\
      \leq & 4 \cdot \E_{\sigma \sim \pi} \E_{o | \sigma} \left[ \hell{\mu_{\pstr}^k(\cdot ; \sigma, o)}{q_k} \right]\nonumber,
    \end{align}
    where the equality uses that, for $\pi_k' \in \Pi_k'$, $\mu_{\prior}^k(\pi_k') =\E_{\sigma \sim \pi} \E_{o | \sigma} [\mu_{\pstr}(\pi_k' ; \sigma, o)]$ (by Bayes' rule), and the inequality uses \cref{lem:min-hell-exp}. 
    
    Hence, we have
    \begin{align}
\exo(\sJ) \leq &  \sup_{\mu \in \Delta(\MM \times \Pi')} \inf_{\pi \in \Pi} \left\{   \E_{(M,\pi^\st) \sim \mu} \E_{\sigma \sim \pi} \left[ \sum_{k=1}^K \fm_k(U(\pi_k^\st, \sigma)) - \fm_k(\sigma) \right]\right.
      - \left. \frac{1}{8\eta} \sum_{k=1}^K \E_{\sigma \sim \pi} \E_{o | \sigma} \left[ \hell{\mu_{\pstr}^k(\cdot ; \sigma, o)}{\mu_{\prior}^k} \right]\right\}\nonumber\\
      =& \infr[1/(8\eta)](\sJ)\nonumber,
    \end{align}
    as desired.%
  \end{proof}

  \begin{lemma}
    \label{lem:joint-convexity}
For any fixed $\eta > 0$, $q \in \prod_{k=1}^K \Delta(\Dev)$, $M \in \MM$ and $\pi^\st \in \Devall$, the map $(\pi, g) \mapsto \Gamma_{q, \eta}(\pi, g; \pi^\st, M)$ is jointly convex with respect to $(\pi, g) \in \Pi \times \MG$. 
\end{lemma}
\begin{proof}[\pfref{lem:joint-convexity}]
  Fix any $\eta, q, M, \pi^\st$ as in the statement of the lemma. 
  Recall the definition of $\Gamma_{q,\eta}(\pi, g; \pi^\st, M)$ in \cref{eq:define-gamma}. %
  Since convexity is preserved under summation, it suffices to show that, for each $k$, the map from $\Pi \times \MG_k \ra \BR$, given by
  \begin{align}
     (\pi, g_k) \mapsto &  \E_{\sigma \sim \pi} \left[ \fm_k(\Sw(\pi_k^\st, \sigma)) - \fm_k(\sigma) \right] \nonumber\\
    & \quad + \frac{1}{\eta} \cdot \E_{\sigma \sim \pi, o \sim M(\sigma)} \E_{\dev \sim q_k} \left[ \exp \left( \frac{\eta}{\pi(\sigma)} \cdot (g_k(\dev; \sigma, o) - g_k(\pi_k^\st; \sigma, o))\right) - 1 \right]\nonumber
  \end{align}
  is convex. This follows directly from Lemma C.1 of \cite{foster2022complexity}. 
\end{proof}

\subsection{Putting everything together: Proof of \cref{thm:curse_ub}}
\label{sec:curse_ub_proof}
The proof of \cref{thm:curse_ub} is a straightforward consequence of the lemmas proven previously in this section. 
\begin{proof}[\pfref{thm:curse_ub}]
  Consider an instance $\sJ$ as in the statement of \cref{thm:curse_ub}. By \cref{lem:inf-dec} and \cref{lem:exo-ir}, we have that, for any $\eta > 0$,
  \begin{align}
\exo[\eta](\sJ) \leq \infr[1/(8\eta)](\sJ) \leq K \cdot \decoreg[1/(8\eta)](\co(\sJ))\nonumber.
  \end{align}
  On the other hand, \cref{lem:alg-perf-exo} gives that for any $\eta, \delta > 0$, \cref{alg:maexo} run with the value $\eta$ gives that with probability at least $1-\delta$,
  \begin{align}
\RiskDM = \hmstar(\wh \pi) \leq \exo[\eta](\sJ) + \frac{2K}{T \eta} \cdot \log \left( \frac{K \cdot \max_k |\Dev|}{\delta} \right)\nonumber.
  \end{align}
  Minimizing over $\eta > 0$ and substituting $\gamma = 1/(8\eta)$ yields that there is a value of $\eta$ for which \cref{alg:maexo} yields risk upper bounded as
  \begin{align}
\RiskDM = \hmstar(\wh \pi) \leq K \cdot \inf_{\gamma > 0} \left\{ \decoreg[\gamma](\co(\sJ)) + \frac{16 \gamma}{T} \cdot \log \left( \frac{K \cdot \max_k |\Dev|}{\delta} \right)\right\}\nonumber,
  \end{align}
  which yields the claimed statement of \cref{thm:curse_ub}. 
\end{proof}

\section{Proofs for lower bounds from \cref{sec:curse}}
\label{sec:proofs_lbs_curse}

\subsection{Proof of \cref{prop:ne-query-complexity}}

\begin{proof}[\pfref{prop:ne-query-complexity}]
Fix $K \in \bbN$, and consider the $K$-player NE instance $\sJ$ of \cref{ex:cce}, where $\MA_k = \{1,2\}$ for each $k \in [K]$. Certainly we have $|\Dev[k]| = |\MA_k| = 2$ for all $k$. By \cref{prop:nf-dec-bound}, we have $\decoreg[\gamma](\sJ^\NE) \leq O(K/\gamma)$ for all $\gamma > 0$. Finally, \citet{rubinstein2016settling} implies that there is no algorithm which draws $2^{o(K)}$ samples (each of which requires querying the true payoff function $ a \mapsto (\fmstar_1(a), \ldots, \fmstar_K(a))$ once) and outputs a $c_0$-approximate Nash equilibrium with probability at least $2/3$, where $c_0>0$ is a sufficiently small universal constant; this yields the third claimed statement of \cref{prop:ne-query-complexity}. 
\end{proof}

\subsection{Proof of \cref{thm:2p0s-poldep}}
In this section, we prove \cref{thm:2p0s-poldep}. 
Before proving the result, we introduce some notation that will be useful in the remainder of the section.
\begin{itemize}
\item For integers $N \geq N' \geq 0$, we let ${[N] \choose N'}$ denote the set of all subsets of $[N] = \{1, 2, \ldots, N\}$ of size $N'$.
\item For positive integers $n \leq n'$ let $[n,n'] = \{ n, n+1, \ldots, n' \}$. 
\item For sets $\MX, \MY$, $\MX \sqcup \MY$ denotes the \emph{disjoint union} of $\MX$ and $\MY$; it is formally defined as $\{ (x,0) \ : \ x \in \MX \} \cup \{ (y,1) \ : \ y \in \MY\}$.
\item For finite sets $\MX, \MY$, we let $\MX^\MY$ denote the set of all functions $\phi : \MY \ra \MX$. Note that, in the case of $\MY = [n]$ for some $n \in \BN$, the sets $\MX^{n}$ (which is the $n$-fold product of $\MX$) and $\MX^{[n]}$ are in bijection. We will at times slightly abuse notation by identifying these two sets.
\item For a finite set $\MX$, let $\Unif(\MX)$ denote the uniform distribution over $\MX$. 
\end{itemize}
\begin{proof}[\pfref{thm:2p0s-poldep}]
  Fix $\eps>0$ and $N \in \BN$; by increasing the constant $C_0$ in the statement of the theorem, it is without loss of generality to assume that $N$ is a multiple of 3. Set $N_1 = N/3$ and $N_2 = 2N/3 = N-N_1$. Define
  \begin{align}
    k=N, \qquad q=n=\frac{2k}{\ep} = \frac{2N}{\ep},\label{eq:kn_choice}
  \end{align}
  which ensures that $q^k \geq {N \choose N_1}$ for sufficiently large $N$. %
  We write $\sT_1 := {[N] \choose N_1}$ and $\sT_2 := {[N] \choose N_2}$. %
  We will now define a random function  $\til\Phi : \sT_1 \cup \sT_2 \ra [q]^{[n]} \sqcup [q+1, 2q]^{[n+1,2n]}$ so that $\til\Phi$ maps $\sT_1$ to $[q]^{[n]}$ and $\sT_2$ to $[q+1,2q]^{[n+1,2n]}$. %
  We will show that with positive probability, $\til\Phi$ satisfies certain conditions. %
  \begin{enumerate}
  \item First, let $\Gamma :\sT_1 \cup \sT_2 \ra [q]^{k} \sqcup [q]^k$ denote a random function, defined as follows: $\Gamma$ maps $\sT_1$ to the first copy of $[q]^k$ (uniformly at random), and $\sT_2$ to the second copy of $[q]^k$ (uniformly at random). In particular, for each $\MS \in \sT_1 \cup\sT_2$, $\Gamma(\MS)$ are independent and chosen uniformly over their respective copies of $[q]^k$. \dfcomment{Not super clear what ``uniformly at random'' refers to here. Are we randomizing on a per-element basis or in function space?}\noah{sure exactly what you mean by ``in function space'' but I added a clarifying sentence}
  \item We next define a mapping $\Sigma: [q]^k \sqcup [q]^k \ra [q]^{[n]} \sqcup [q+1,2q]^{[n+1,2n]}$ which maps the first copy of $[q]^k$ into $[q]^{[n]}$ and the second copy of $[q]^k$ into $[q+1,2q]^{[n+1,2n]}$ according to the Reed-Solomon code of \cref{thm:rs}. (Here we have identified each of $[q]^{[n]}$ and $[q+1,2q]^{[n+1,2n]}$ with $[q]^n$ in the natural way.) 
  \item We then set $\til\Phi = \Sigma \circ \Gamma$. %
    \dfcomment{it's slightly confusing to say we set based on some choice of $\Gamma$ here when we say it is randomized above. would it make more sense to just introduce the randomization below to prove that a good $\Gamma$ exists?}\noah{changed notation to refer to $\til \Phi$ as the random function and then $\Phi$ is the good one we choose}
  \end{enumerate}
  We next argue that there is some choice of $\Gamma$ for which the resulting $\til\Phi$  satisfies the following  \cref{cond:injective,cond:unif-hit}. %
  \begin{condition}
    \label{cond:injective}
 For each $i \in \{1,2\}$, for all sets $\MT, \MT' \in \sT_i$ with $\MT \neq \MT'$, it holds that $\ham{\til\Phi(\MT)}{\til\Phi(\MT')} \geq q-k+1$. %
  \end{condition}

  \begin{condition}
    \label{cond:unif-hit}
    For any subset $\MQ \subset [N]$ with $|\MQ| \leq \sqrt{N}$,
      \begin{align}
    \forall a_1 \in [n], a_2 \in [q], \quad   \BP_{\MT \sim \Unif(\sT_1)} \left( \til\Phi(\MT)(a_1) = a_2 | \MQ \subset \MT \right) \leq 2/q \label{eq:qcond-u1}\\
    \forall a_1 \in [n+1,2n], a_2 \in [q+1,2q], \quad \BP_{\MT \sim \Unif(\sT_2)} \left( \til\Phi(\MT)(a_1) = a_2 | \MQ \subset \MT \right) \leq 2/q\label{eq:qcond-u2}.
  \end{align}
\end{condition}
To see that there exists such a choice for $\Gamma$, we make the following observations.
\begin{enumerate}
\item Since $q^k > 10 \cdot {N \choose N_1}^2$ \dfedit{whenever $N$ is sufficiently large (by \eqref{eq:kn_choice})}, with probability at least $1 - {{N \choose N_1}^2}/{q^k} > 9/10$, the function $\Gamma$ is injective. Conditioned on being injective, \cref{thm:rs} gives that \cref{cond:injective} holds, since the action of $\Sigma$ on each of the copies of $[q]^k$ is defined to be that of a Reed-Solomon code. Thus, \cref{cond:injective} holds with probability at least $9/10$ over the choice of $\Gamma$. 
\item %
Consider any fixed choice of $\MT \in \sT_1$.  Note that, for each coordinate $a_1 \in [n]$, the mapping $\MT \mapsto \til\Phi(\MT)(a_1) = \Sigma(\Gamma(\MT))(a_1)$, for $\MT \in \sT_1$, is distributed as a uniformly random function from $\sT_1 \to [q]$ (with respect to the randomness in $\Gamma$). This fact follows from the final sentence of \cref{thm:rs} and the fact that $\Gamma$ is a uniformly random function. 
  Thus, by \cref{lem:extractor} with $N_0 = N_1$ and a union bound over all $n$ possible values of $a_1$,  with probability $1 - n \cdot N^{\sqrt{N} + 1} \cdot 2^{-{5N/6 \choose N/6}/(Cq^2)}$ over the choice of $\Gamma$, for any subset $\MQ \subset [N]$ of size $|\MQ| \leq \sqrt{N}$, \eqref{eq:qcond-u1} holds. %
  Similarly, an application of \cref{lem:extractor} with $N_0 = N_2$ yields that with probability  $1 - n \cdot N^{\sqrt{N} + 1} \cdot 2^{-{5N/6 \choose N/6}/(Cq^2)}$ over the choice of $\Gamma$, for any subset $\MQ \subset [N]$ of size $|\MQ| \leq \sqrt{N}$, \eqref{eq:qcond-u2} holds. Note that our choices of $q, N, \ep$ ensure that, for the constant $C$ in \cref{lem:extractor}, as long as $N$ is sufficiently large,
  \begin{align}
    3 \log q \leq 6 \log(N/\ep) \leq 12 \log(N) \leq N/6 - C \leq \log{5N/6 \choose N/6} - C\nonumber,
  \end{align}
  meaning that it is valid to apply \cref{lem:extractor}. Finally, let us note that our choices for $N, q$ ensure that as long as $N$ is sufficiently large,
  \begin{align}
    {5N/6 \choose N/6} > Cq^2 \cdot \left( \log(2n ) + \log(N^{\sqrt{N}+1})+ 5 \right),\nonumber
  \end{align}
  and therefore, \cref{eq:qcond-u1} and \cref{eq:qcond-u2} hold for all $\MQ \subset [N]$ with $|\MQ| \leq \sqrt{N}$, with probability at least $1-2^{-5}$. In particular, \cref{cond:unif-hit} holds with probability at least $1-2^{-5}$ over the random choice of $\Gamma$. 
\end{enumerate}
Summarizing the above points, with probability at least $1-1/10 - 2^{-5} > 0$ over the choice of $\Gamma$, \cref{cond:injective,cond:unif-hit} both hold. We pick any such $\Gamma$ for which both conditions hold, and set $\Phi = \Sigma \circ \Gamma$. 

We are now ready to define the $2$-player instance $\sJ = \instma$.

\paragraph{Policy space} Let $\Pi_{1} = \{1, 2, \ldots, 2n\}$ and $\Pi_{2} = \{ 0, 1, \ldots, 2q \}$, and write $\Pi = \Pi_1 \times \Pi_2$ to denote the joint policy space. %

\paragraph{Deviation sets and switching functions}
The deviation sets $\Dev$ and switching function $\Sw$ are set as in \cref{def:ne-instance} to make $\sJ$ a 2-player NE instance. To be concrete, we have $\Dev = \Pi_k$ for each $k$, and $\Sw(\dev, \pi) = (\dev, \pi_{-k})$.

\paragraph{Model class $\MM$} The class $\MM$ is indexed by $\sT_1 \cup \sT_2$. Given a set $\MT \in \sT_1 \cup \sT_2$, we write the corresponding model as $M_\MT$. We will often consider the decomposition $\MM = \MM_1 \sqcup \MM_2$, where $\MM_1 \ldef \{ M_\MT \ :\ \MT \in \sT_1\}$ and $\MM_2 \ldef \{ M_\MT \ : \ \MT \in \sT_2 \}$. For each $M_\MT \in \MM$, we need to specify the distributions $o = (r_1, r_2, \ocirc) \sim M_\MT(\pi)$, for each $\pi \in \Pi$. 
To do so, we first define a mapping $\MB^\st : \sT_1 \cup \sT_2 \ra \powerset{[2n] \times [2q]} \subset \powerset{\Pi}$, as follows: recall that $\Phi$ maps $\sT_1$ to $[q]^{[n]}$ and $\sT_2$ to $[q+1,2q]^{[n+1,2n]}$. Then for $\MT \in \sT_1 \cup \sT_2$, define $\MB^\st(\MT) \subset [2n] \times [2q] \subset \Pi$ by %
    \begin{align}
      \MB^\st(\MT) = \begin{cases}
        \{ (i, \Phi(\MT)(i)) : i \in [n]\} &: M_\MT \in \MM_1 \\
        \{ (i, \Phi(\MT)(i)): i \in [n+1,2n]\} &: M_\MT \in \MM_2.
      \end{cases}\nonumber
    \end{align}
    Note that here we view, for each set $\MT$ in the domain of $\Phi$, $\Phi(\MT)$ as a function mapping either $[n] \ra [q]$ (for $M_\MT \in \MM_1$) or $[n+1,2n] \ra [q+1,2q]$ (for $M_\MT \in \MM_2$).

    We set the reward space to be $\MR = [0,1]$, and the pure observation space to be $\Ocirc = [N]$. Now, for each $M_\MT \in \MM$ and $\pi \in \Pi$, the full observation $o = (r_1, r_2, \ocirc) \sim M_\MT(\pi)$ is drawn as follows:
  \begin{itemize}
  \item The pure observation $\ocirc \in \Ocirc$ is simply a uniformly random element of the set $\MT$. 
  \item The rewards are deterministic, i.e., we have $r_k = f_k\sups{M_\MT}(\pi)$ for each $k \in [K]$, a.s. Moreover, we define
      \begin{align}
    f_1\sups{M_\MT}(\pi) = -f_2\sups{M_\MT}(\pi) = \begin{cases}
      0 &: \pi \in \Pi_1 \times \{ 0\} \\
      1 &: \pi \in (\Pi_1 \times \{ 1, 2, \ldots, 2q \}) \backslash \MB^\st(\MT)\\
      -\delta &: \pi \in \MB^\st(\MT),
    \end{cases}\label{eq:rewards-bstar}
  \end{align}
  where we set $\delta := 10^{-3}$. 
\end{itemize}

\paragraph{Establishing the claimed statements} It is immediate from definition of $\Pi$ that $|\Pi| = 2n \cdot (2q+1) = O(N^2/\ep^2)$, thus establishing the first claimed statement of the theorem.  
Next, \cref{lem:com-dec} below bounds $\decoreg[\gamma](\co(\sJ))$, establishing the second claimed statement.
  \begin{lemma}
    \label{lem:com-dec}
For any $\gamma > 0$, It holds that $\decoreg[\gamma](\co(\sJ)) \leq \ep$. 
\end{lemma}
The proof of \cref{lem:com-dec} uses that $\Phi$ satisfies \cref{cond:injective}.  Finally, the third claimed statement is established by the following lemma.
\begin{lemma}
  \label{lem:rs-sc-lb}
  There is a constant $C > 0$ so that the following holds. For any algorithm that has at most $T \leq \sqrt{N} / C$ rounds of interaction, there is some model $\Mstar \in \MM$ so that
  \begin{align}
    \E\sups{\Mstar}[\RiskDM] = \E\sups{\Mstar}\left[ \hmstar(\wh \pi) \right] > \delta/100 = 10^{-5}\nonumber.
  \end{align}
  Recall that above $\wh\pi$ denotes the output policy of the algorithm.
\end{lemma}
The proof of \cref{lem:rs-sc-lb} uses that $\Phi$ satisfies \cref{cond:unif-hit}. 
It remains to prove \cref{lem:com-dec,lem:rs-sc-lb}; we do so in the remainder of this section.
\end{proof}

\begin{proof}[\pfref{lem:com-dec}]
    For $\Mbar \in \co(\MM)$ and  $i \in [N]$, let $\ol M[i] \in [0,1]$ denote the probability $\BP_{(r_1, r_2, \ocirc) \sim \ol M(\pi)}[\ocirc = i]$, for an arbitrary decision $\pi \in \Pi$ (note that the choice of decision does not affect the distribution over the pure observation $\ocirc$). 
    For any $\ol M \in \co(\MM)$, we define the set $\MT(\ol M) \subset [N]$ as follows:
  \begin{align}
\MT(\ol M) := \left\{ i \in [N] \ : \  \ol M[i] \geq 1/N \right\}\nonumber,
  \end{align}

 Now fix any $\ol M \in \co(\MM)$. %
  Define $p^\st \in \Delta(\Pi)$ as follows, as a function of $\ol M$: 
  \begin{align}
    p^\st = \begin{cases}
      \Unif( \{(1,0), \ldots, (n,0) \}) &: |\MT(\ol M)| \geq N/2 \\
      \Unif(\{(n+1,0), \ldots, (2n,0)\}) &: |\MT(\ol M)| < N/2.
    \end{cases}\nonumber
  \end{align}
  We have that 
  \begin{align}
    \decoreg[\gamma](\co(\sJ), \ol M) \leq &  \sup_{M \in \co(\MM)} \E_{\pi \sim p^\st} \left[ \hm(\pi) - \gamma \cdot \hell{M(\pi)}{\ol M(\pi)}\right] \nonumber\\
    = & \sup_{M \in \co(\MM)} \E_{\pi \sim p^\st} \left[\sum_{k=1}^2 \max_{\pi_k' \in \Pi_k} \fm_k(\pi_k', \pi_{-k}) - \fm_k(\pi) - \gamma \cdot \hell{M(\pi)}{\ol M(\pi)}\right]\nonumber\\
    =& \sup_{M \in \co(\MM)} \E_{\pi \sim p^\st} \left[ \max_{\pi_2' \in \Pi_2} \fm_2(\pi_1, \pi_2') - \fm_2(\pi) - \gamma \cdot\hell{M(\pi)}{\ol M(\pi)}\right]\label{eq:sup-mcom},
  \end{align}
  where the final equality follows because for all $M \in \co(\MM)$ and all $\pi$ in the support of $p^\st$, $\max_{\pi_1' \in \Pi_1} \fm_1(\pi_1', \pi_2) = 0 = \fm_1(\pi)$.

  Fix $\nu \in \Delta(\MM)$ so that $M = \ol M_{\nu}(\pi) := \E_{M \sim \nu}[M(\pi)]$ attains the supremum in \eqref{eq:sup-mcom}. %
  We consider the following possibilities:

\paragraph{Case 1} Suppose first that $|\MT(\Mbar)| \geq N/2$. We consider the following sub-cases:
    \begin{enumerate}
    \item First suppose that $\nu(\MM_1) \leq \frac{1}{1+\delta}$, where we recall that $\delta\ldef 10^{-3}$ (\eqref{eq:rewards-bstar}). Then for all $\pi_1 \in [n]$ and $\pi_2 \in [2q]$, it holds that \dfcomment{explain how this uses $\cB^{\star}(\cT)$ definition}\noah{added below}
      \begin{align}
f_2^{\ol M_\nu}((\pi_1, \pi_2)) \leq \frac{1}{1+\delta} \cdot \delta - \left( 1 - \frac{1}{1+\delta} \right) =0\label{eq:f2mbar-nu},
      \end{align}
      since $f_2^{M_\MT}((\pi_1, \pi_2))$ is only positive when $\pi \in \MB^\st(\MT)$, which happens with probability at most $\frac{1}{1+\delta}$ under $M_\MT \sim \nu$, as $\pi_1 \in [n]$; moreover, when it is positive, it is $\delta$, and when it is not positive, it is $-1$. 
      Using \cref{eq:f2mbar-nu}, since for all decisions $\pi$ in the support of $p^\st$ (which have $\pi_1 \in [n]$ in this sub-case), the expression in \cref{eq:sup-mcom} is bounded above by 0. 
    \item Next suppose that there is some model $M_\MT \in \MM_1$ so that $\nu(M_\MT) \geq 14/15$. %
      Thus, we must have
$
\sum_{j \in \MT(\ol M) \backslash \MT} \ol M_\nu[j] \leq \sum_{j \in [N] \backslash \MT} \Mbar_\nu[j] \leq \frac{1}{15}. 
$
On the other hand, since $|\MT(\Mbar)| \geq N/2$, we have $\sum_{j \in \MT(\ol M) \backslash \MT} \ol M[j] \geq \frac{1}{N} \cdot \frac{N}{6} = 1/6$. Thus, for any decision $\pi \in \Pi$,
\begin{align}
\hell{\ol M_\nu(\pi)}{\ol M(\pi)} \geq (\tvd{\ol M_\nu(\pi)}{\ol M(\pi)})^2 \geq (1/6 -1/15)^2 =1/100\nonumber,
      \end{align}
      Thus, as long as $\gamma \geq 100$, since $\fm_2(\pi_1, \pi_2') \leq 1$ for all $\pi_1, \pi_2'$,
      if we recall that $M = \ol M_\nu$ is chosen to maximize the expression in \cref{eq:sup-mcom}, we have that this expression is bounded above by 0.
\item \label{it:eps-contradiction} In the remaining case, we must have $\nu(\MM_1) \geq \frac{1}{1+\delta}$, yet for each $M_\MT \in\MM_1$, $\nu(M_\MT) < 14/15$.  Suppose for the purpose of contradiction that
  \begin{align}
\E_{\pi \sim p^\st} \left[ \max_{\pi_2' \in \Pi_2} f_2\sups{\ol M_\nu}(\pi_1, \pi_2') - f_2\sups{\ol M_\nu}(\pi) \right]=\E_{\pi \sim p^\st} \left[ \max_{\pi_2' \in \Pi_2} f_2\sups{\ol M_\nu}(\pi_1, \pi_2') \right] > \ep \label{eq:deviation-epn}.
  \end{align}
  Write $\MI = \{ \pi_1 \in [n] : \max_{\pi_2' \in \Pi_2} f_2\sups{\ol M_\nu}(\pi_1, \pi_2') \geq 0 \}$; since $\max_{\pi_2' \in \Pi_2} f\sups{\Mbar_\nu}_2(\pi_1, \pi_2') \leq 1$ for all $\pi_1$, \cref{eq:deviation-epn} tells us that $|\MI| \geq \ep n$.  By construction, for each $\pi_1 \in \MI$, there is at most one value of $\pi_2 \in [q]$ so that $f_2\sups{\ol M_\nu} (\pi_1, \pi_2) \geq 0$; let this value of $\pi_2$ be denoted by $\pi_2^\st(\pi_1)$, if such $\pi_2$ exists given $\pi_1$, and otherwise set $\pi_2^\st(\pi_1) = -1$.

  Note that if $f_2\sups{\ol M_\nu}(\pi_1, \pi_2) \geq 0$ for any $\pi_1 \in \Pi_1$ \dfcomment{for $\pi_1\in\brk{n}$?}\noah{added clarification}, then we must have that
$ 
    \nu( \{ M_\MT \in \MM_1 : \Phi(\MT)(\pi_1) = \pi_2 \}) \geq \frac{1}{1+\delta} > 1-\delta.
    $
    \dfcomment{where is $\MI$ defined? this seems to be the first place it is used}\noah{moved defn beforehand} 
Therefore, for all $\pi_1$, if $\pi_2^\st(\pi_1) > 0$, then 
\begin{align}
\nu(\{M_\MT \in \MM_1 : \Phi(\MT)(\pi_1) = \pi_2^\st(\pi_1) \}) > 1-\delta \label{eq:t-pi-delta}.
\end{align}

  For each $M_\MT \in \MM_1$, define
  \begin{align}
\zeta(\MT) := \left| \left\{ \pi_1 \in \MI \ : \ \Phi(\MT)(\pi_1) \neq \pi_2^\st(\pi_1) \right\} \right|\nonumber.
  \end{align}
  We have that
  \begin{align}
|\MI| - \sum_{M_\MT \in \MM_1} \nu(M_\MT) \zeta(\MT)  = \sum_{\pi_1 \in \MI} \sum_{M_\MT \in \MM_1} \nu(M_\MT) \cdot \One{\Phi(\MT)(\pi_1) = \pi_2^\st(\pi_1)} \geq |\MI| \cdot (1-\delta)\nonumber,
  \end{align}
  where the inequality uses \cref{eq:t-pi-delta}. 
  Thus, by Markov's inequality, for some subset $\MM_1' \subset \MM_1$, it holds that $\nu(\MM_1 \backslash \MM_1') \leq \sqrt{\delta}$ and for all $M_\MT \in \MM_1'$, $\zeta (\MT) \leq |\MI|\cdot \sqrt{\delta}$. Since $\nu(\MM_1) \geq 1-\delta$, it follows that $\nu(\MM_1') \geq 1 - \delta - \sqrt{\delta} \geq 1-2\sqrt{\delta}$. Since $1-2\sqrt{\delta} > 14/15$ by our choice of $\delta = 10^{-3}$, there must be at least two distinct elements of $\MM_1'$, which we denote by $M_{\MT_1}$ and  $M_{\MT_2}$.

  To proceed, by definition of $\MM_1'$, it holds that
  \begin{align}
\left| \left\{ \pi_1 \in \MI \ : \ \Phi(\MT_1)(\pi_1) = \Phi(\MT_2)(\pi_1) = \pi_2^\st(\pi_1) \right\} \right| \geq |\MI| \cdot (1-2\sqrt \delta) \geq |\MI|/2 \geq n\ep / 2\nonumber.
  \end{align}
  It follows that $\ham{\Phi(\MT_1)}{\Phi(\MT_2)} \leq  n -n\ep/2 = n(1-\ep/2)$, which contradicts \cref{cond:injective}, since $n(1-\ep/2)=\frac{2N}{\ep} \cdot (1 - \ep/2) <  \frac{2N}{\ep} - N + 1= q-k+1$. Thus, \cref{eq:deviation-epn} is false, and therefore the expression in \cref{eq:sup-mcom} corresponding to choosing $M =\Mbar_\nu$ is bounded above by $\ep$. %
    \end{enumerate}
    \paragraph{Case 2} Now suppose that $|\MT(\ol M)| < N/2$.  In this case an argument symmetric to that in the case that $|\MT(\ol M)| \geq N/2$ may be applied to establish the same upper bound on the multi-agent DEC. (In particular, the roles of $\MM_1, \MM_2$ are swapped; the symmetry arises from the fact that sets in $\sT_1$ have size $N/3 = N/2 - N/6$ whereas sets in $\sT_2$ have size $2N/3 = N/2 + N/6$.) Below we expand on the details for completeness.
    \begin{enumerate}
    \item If $\nu(\MM_2) \leq \frac{1}{1+\delta}$, then for all $\pi_1 \in [n+1,2n]$ and $\pi_2 \in [2q]$, $f\sups{\Mbar_\nu}_2((\pi_1, \pi_2)) \leq 0$, meaning that, since for all decisions $\pi$ in the support of $p^\st$, $\pi_2 \in [n+1,2n]$, the expression in \cref{eq:sup-mcom} is non-positive.
    \item Next suppose there is some model $M_\MT \in \MM_2$ so that $\nu(M_\MT) \geq 14/15$. We must have that $\sum_{j \in \MT \backslash \MT(\Mbar)} \Mbar[j] \leq |\MT \backslash \MT(\Mbar)| \cdot \frac{1}{N}$. On the other hand, %
      since for each $i \in \MT$ we have $M_\MT[i] = 3/(2N)$, we have $\sum_{j \in \MT \backslash \MT(\Mbar)} \Mbar_\nu[j] \geq \frac{14}{15} \cdot \frac{3}{2N} \cdot |\MT\backslash\MT(\Mbar)|  \geq \frac{7}{5N} \cdot |\MT\backslash \MT(\Mbar)|$. Thus, for any $\pi \in \Pi$, since $|\MT(\Mbar)| \leq N/2$ and $|\MT| = 2N/3$ (as $M_\MT \in \MM_2$),
      \begin{align}
       &  \hell{\Mbar_\nu(\pi)}{\Mbar(\pi)} \geq \left( \tvd{\Mbar_\nu(\pi)}{\Mbar(\pi)}\right)^2 \geq \left(|\MT\backslash \MT(\Mbar)| \cdot \left( \frac{7}{5N} - \frac{1}{N} \right) \right)^2 \nonumber\\
        & \geq \left( \frac{N}{6} \cdot \left( \frac{7}{5N} - \frac{1}{N} \right)\right)^2 =\frac{1}{225}\nonumber.
      \end{align}
      Thus, as long as $\gamma \geq 225$, since $\fm_2(\pi_1, \pi_2') \leq 1$ for all $\pi_1, \pi_2'$, the expression in \cref{eq:sup-mcom} for $ M=  \Mbar_\nu$ is bounded above by 0.
    \item In the remaining case, we must have $\nu(\MM_2) \geq \frac{1}{1+\delta}$, yet for each $M_\MT \in \MM_2$, $\nu(M_\MT) < 14/15$. In this case, the expression in \cref{eq:sup-mcom} for $M = \Mbar_\nu$ is bounded above by $\ep$, via an argument identical to the one in \cref{it:eps-contradiction} above where one replaces all intances of $\MM_1$ with $\MM_2$. 
    \end{enumerate}
    Summarizing, we have shown that \cref{eq:sup-mcom} is bounded above by $\ep$ for an arbitrary choice of $\Mbar$, which completes the proof of the lemma. 
\end{proof}

\begin{proof}[\pfref{lem:rs-sc-lb}]
 Fix any $T \leq \sqrt{N} / C$ (for a constant $C$ to be specified below), and consider any algorithm $(p,q) = \{ (q\^t(\cdot | \cdot), p(\cdot | \cdot) \}_{t=1}^T$. Recall that, for any model $M$, $\hist\^T$ denotes the history of interaction between the algorithm $(p,q)$ and the model $M$, and is defined by $\hist\^T = (\pi\^1, o\^1), \ldots, (\pi\^T, o\^T)$. $\hist\^T$ is associated with the measure space $(\Omega\^T, \mathscr{F}\^T)$.  For each model $M \in \MM$, we use the abbreviate $\BP\sups{M} \equiv \BP\sups{M, (p,q)}$ as the law of $\hist\ind{T}$, and write $\E\sups{M}$ for the corresponding expectation. We will show the stronger statement that the algorithm $(p,q)$ has large risk for a \emph{uniformly random} model $\Mstar \in \MM$; in particular,
  \begin{align}
\E_{\Mstar \sim \Unif(\MM)} \E\sups{\Mstar} \left[\hmstar(\wh\pi) \right]  > \delta/100\label{eq:unif-mstar-lb}. 
  \end{align}
  Clearly \cref{eq:unif-mstar-lb} implies the statement of \cref{lem:rs-sc-lb}.

  In order to prove \cref{lem:rs-sc-lb}, we first prove a few intermediate results. To start, we define an additional model $M_0$: the distribution of $(r_1, r_2, \ocirc) \sim M_0(\pi)$ are as follows:
\begin{itemize}
\item The rewards $r_1, r_2$ are given as in \cref{eq:rewards-bstar} with $\MB^\st = \emptyset$; in particular, $r_k = f_k\sups{M_0}(\pi)$ are deterministic with
  \begin{align}
    f_1\sups{M_0}(\pi) = -f_2\sups{M_0}(\pi) = \begin{cases}
      0 &: \pi \in \Pi_1 \times \{0 \} \\
      1 &: \pi \in \Pi_2 \times \{1, 2, \ldots, 2q \}.
    \end{cases}\nonumber
  \end{align}
\item The pure observation $\ocirc\in [N]$ is a uniformly random element of $[N]$. %
\end{itemize}
Next, recall that we write, for $\pi \in \Pi, k \in \{1,2\}, M \in \MM$, $\hm_k(\pi) = \max_{\dev \in \Dev} \fm_k(\Sw(\dev,\pi)) - \fm_k(\pi)$. 

\cref{lem:m0-halfprob} below shows that for each $i\in\crl{1,2}$, under the model $M_0$, with constant probability either all models in $\MM_1$ or all models in $\MM_2$ have high risk with respect to the algorithm's output policy $\wh \pi$.
\begin{lemma}
  \label{lem:m0-halfprob}
  There is some $i \in \{1,2\}$ (depending on the algorithm $(p,q)$) so that 
  \begin{align}
\BP\sups{M_0} \left( \forall M \in \MM_i: \hm(\wh\pi) \geq \delta \right) \geq \frac{1}{2}\nonumber.
  \end{align}
\end{lemma}
The proof of \cref{lem:m0-halfprob} is provided at the end of this section. 
Since $M_0$ is not in $\MM$, \cref{lem:m0-halfprob} is not enough to prove \cref{lem:rs-sc-lb}; we will next use a series of change-of-measure arguments to reason about the history of interaction when the true model is a uniformly random model in $M$. In particular, for each model $M_\MT \in \MM$, we define an intermediate model $M_{\MT, 0}$: the distribution of $(r_1, r_2, \ocirc) \sim M_{\MT, 0}(\pi)$ is as follows:
\begin{itemize}
\item The rewards $(r_1, r_2)$ are given identically to the rewards under $M_0(\pi)$ (in particular, they are deterministic).
\item The pure observation $\ocirc$ is a uniformly random element of $\MT$. %
\end{itemize}
\cref{lem:m0-mt0-close} below shows that under a history drawn from $M_{\MT, 0}$ for a uniformly random $\MT \sim \Unif(\sT_i)$, with high probability the algorithm will not query any decision belonging to $\MB^\st(\MT) \subset \Pi$; 
furthermore, the distribution of the history $\hist\^T$ is close under $M_0$ and under $M_{\MT, 0}$, again for a uniformly random $\MT \sim \Unif(\sT_i)$:
\begin{lemma}
  \label{lem:m0-mt0-close}
  For each $i \in \{1,2\}$, the following holds:
  \begin{align}
\E_{\MT \sim \Unif(\sT_i)} \E\sups{M_{\MT, 0}} \left[\One{\{\pi\^1, \ldots, \pi\^T \} \cap \MB^\st(\MT) \neq \emptyset} \right]\leq \frac{2T}{q} + \frac{1}{100}\label{eq:intersect-ub}.
  \end{align}
  Furthermore, for any measurable subset $\MF \in \CF\^T$ of histories,
  \begin{align}
\left| \E\sups{M_0}\left[\One{\hist\^T \in \MF}\right] - \E_{\MT \sim \Unif(\sT_i)} \E\sups{M_{\MT, 0}} \left[ \One{\hist\^T \in \MF} \right] \right| \leq \frac{1}{100}\label{eq:lem-m0mt0-close}.
  \end{align}
\end{lemma}
The proof of \cref{lem:m0-mt0-close} is provided at the end of this section.

Next, \cref{lem:mt0-mt-close} shows that if, for some model $M_\MT$, the algorithm does not query any decision in $\MB^\st(\MT)$ with high probability, then the distribution of histories under $\BP\sups{M_{\MT,0}}$ and $\BP\sups{M_\MT}$ are close.
\begin{lemma}
  \label{lem:mt0-mt-close}
  Fix some model $M_\MT \in \MM$ so that $\BP\sups{M_{\MT,0}}(\{ \pi\^1, \ldots, \pi\^T\} \cap \MB^\st(M_\MT) \neq \emptyset) \leq \eta$ for some $\eta > 0$. Then $\tvd{\BP\sups{M_{\MT,0}}}{\BP\sups{M_\MT}} \leq \eta$.
\end{lemma}
The proof of \cref{lem:mt0-mt-close} is provided at the end of this section. Given the above lemmas, we now establish \cref{eq:unif-mstar-lb}.   Suppose for the purpose of contradiction that $\E_{M \sim \Unif(\MM)} \E\sups{M} \left[ \hm(\wh \pi) \right] \leq \delta/100$. Then by Markov's inequality, $\E_{M \sim \Unif(\MM)} \E\sups{M} \left[ \One{\hm(\wh \pi) \geq \delta} \right] \leq 1/100$. Since $\Unif(\MM)$ is the uniform average of $\Unif(\MM_1)$ and $\Unif(\MM_2)$, it follows that for each $i \in \{1,2\}$,
  \begin{align}
   \E_{M \sim \Unif(\MM_i)} \E\sups{M} \left[ \One{\hm(\wh \pi) \geq \delta }\right] \leq 1/50.\label{eq:t-mmi-ub}
  \end{align}

  We next note that \cref{lem:m0-halfprob} gives that for some $i^\st \in \{1,2\}$,
  \begin{align}
\BP\sups{M_0} \left( \forall M \in \MM_{i^\st}:  \hm(\wh \pi) \geq \delta \right) \geq \frac 12 \nonumber.
  \end{align}
  By the conclusion \cref{eq:lem-m0mt0-close} of \cref{lem:m0-mt0-close}, it follows that
  \begin{align}
\E_{\MT \sim \Unif(\sT_{i^\st})} \E\sups{M_{\MT,0}} \left[ \indic\crl*{\forall M \in \MM_{i^\st}: \hm(\wh \pi) \geq \delta }\right] \geq 1/2 - 1/100\label{eq:hybrid-allm-lb}.
  \end{align}

  Next, by the statement \cref{eq:intersect-ub} of \cref{lem:m0-mt0-close} and using that $2T \leq \sqrt{N}$ and $\sqrt{N}/q \leq 1/\sqrt{N} \leq 1/100$ %
  for sufficiently large $N$, 
  \begin{align}
   \E_{\MT \sim \Unif(\sT_{i^\st})} \E\sups{M_{\MT, 0}} \left[ \One{\{\pi\^1, \ldots, \pi\^T\} \cap \MB^\st(M_\MT) \neq \emptyset} \right] \leq \frac{\sqrt{N}}{q} + \frac{1}{100}  \leq \frac{1}{50}
    \label{eq:m0-mt0-close}. 
  \end{align}
  Now, for $\eta = 1/7$, let us write $\chi(\MT) := \One{ \E\sups{M_{\MT,0}}\left[\indic\crl*{ \{\pi\^1, \ldots, \pi\^T\} \cap \MB^\st(M_\MT) \neq \emptyset }\right] > \eta} \in \{0,1\}$; \eqref{eq:m0-mt0-close} together with Markov's inequality give that $\E_{\MT \sim \Unif(\sT_{i^\st})}[\chi(\MT)] \leq 1/7$.

  Next, \cref{lem:mt0-mt-close} gives that, for all $\MT \in \sT_1 \cup \sT_2$,
  \begin{align}
     \BP\sups{M_\MT} \left( \forall M \in \MM_{i^\st}: \ \hm(\wh \pi) \geq \delta \right) 
    \geq \BP\sups{M_{\MT,0}} \left( \forall M \in \MM_{i^\st}: \ \hm(\wh \pi) \geq \delta \right) - \chi(\MT) - \eta \nonumber,
  \end{align}
  and taking expectation over $\MT \sim \Unif(\sT_{i^\st})$ and using that $\E_{\MT \sim \Unif(\sT_{i^\st})}[\chi(\MT)] \leq 1/7$ and the choice of $\eta=1/7$ gives that
  \begin{align}
    & \BP_{\MT \sim \Unif(\sT_{i^\st})} \E\sups{M_\MT}\left[\One{\forall M \in \MM_{i^\st}:\ \hm(\wh \pi) \geq \delta }\right]\nonumber\\
    \geq & \BP_{\MT \sim \Unif(\sT_{i^\st})} \E\sups{M_{\MT,0}} \left[\One{ \forall M \in \MM_{i^\st}: \ \hm(\wh \pi) \geq \delta} \right] - 2/7\geq 1/2 - 1/100 - 2/7\nonumber,
  \end{align}
  where the final inequality follows by \eqref{eq:hybrid-allm-lb}. In particular, using that $M_\MT \in \MM_{i^\st}$ if $\MT \in \sT_{i^\st}$, we have
  \begin{align}
\E_{\MT \sim \Unif(\sT_{i^\st})} \E\sups{M_\MT} \left[ \One{h\sups{M_\MT}(\wh \pi) \geq \delta }\right] \geq 1/2 - 1/100 - 2/7 > 1/5\nonumber,
  \end{align}
which contradicts \eqref{eq:t-mmi-ub}, thus completing the proof.  
\end{proof}

\begin{proof}[\pfref{lem:m0-halfprob}]
We write $\til\Pi := \Pi_1\times(\Pi_2\setminus\crl{0}) = \Pi_1 \times \{1, 2, \ldots, 2q \} \subset \Pi$.  First, we claim that for all $\pi \in \til\Pi$, and all $M \in \MM$, it holds that $\hm(\pi) = \hm_1(\pi) + \hm_2(\pi) \geq \delta$. To see this, consider any $M = M_\MT \in \MM$, and  we consider the following two cases:
  \begin{itemize}
  \item If $\pi \not \in \MB^\st(\MT)$, then $f_2\sups{M}((\pi_1, 0)) - f_2\sups{M}(\pi) = 0 - (-1) = 1$.
  \item If $\pi \in \MB^\st(\MT)$, then there must be some $\pi_1' \in \Pi_1$ with $(\pi_1', \pi_2) \not \in \MB^\st(\MT)$, and so $f_1\sups{M}((\pi_1', \pi_2)) - f_1\sups{M}(\pi) = 1 - (-\delta) = 1+\delta$.
  \end{itemize}
  \dfcomment{I am a little confused as to where these bullets are invoked below. is the point just to say that we can now restrict our attention to $\pi_2=0$ below?}\noah{yes, added clarification sentence} 
Next, note that
  \begin{align}
\max\left\{\BP\sups{M_0} \left( \wh \pi \in \til \Pi \cup ([n] \times \{0 \})\right) , \BP\sups{M_0} \left( \wh \pi \in \til \Pi \cup ([n+1, 2n] \times \{0\}) \right)\right\} \geq 1/2\nonumber.
  \end{align}
  Let us first suppose that $\BP\sups{M_0} \left( \wh \pi \in \til \Pi\cup ([n] \times \{0 \})\right) \geq 1/2$. Note that if $M=M_\MT \in \MM_1$ and $\pi \in [n] \times \{0\}$, then $h\sups{M}(\pi) \geq h_2\sups{M}(\pi) = f_2\sups{M}((\pi_1, \Phi(\MT)(\pi_1))) - f_2\sups{M}(\pi) = \delta - 0 = \delta$. Moreover, the two bullet points above establish that if $\wh \pi \in \til \Pi$, then $\hm(\wh \pi) \geq 1 > \delta$. Thus, in this case, we have established that $\BP\sups{M_0} \left(\forall M \in \MM_1, \  \hm(\wh\pi) \geq \delta \right) \geq 1/2$.

  In the other case, where $\BP\sups{M_0} \left( \wh \pi \in \til \Pi \cup ([n+1, 2n] \times \{0\}) \right)\geq1/2$, it follows inb a symmetric manner that, $\BP\sups{M_0} \left(\forall M \in \MM_2,\ \hm(\wh\pi) \geq \delta \right) \geq 1/2$.
\end{proof}

\begin{proof}[\pfref{lem:m0-mt0-close}]
  Fix any $i \in \{1,2\}$.  For a model $M \in \{M_0\} \cup \bigcup_{\MT \in \sT_i} \{ M_{\MT, 0} \}$, consider a draw of $\hist\^T = (\pi\^1, (r_1\^1, r_2\^1, \ocirc\^1), \ldots, \pi\^T, (r_1\^T, r_2\^T, \ocirc\^T)) \sim \BP\sups{M}$, where we have written out the full observations $o\^t = (r_1\^t, r_2\^t, \ocirc\^t)$. Since the distribution of the pure observations $\ocirc\^t \sim M(\pi)$ does not depend on the policy $\pi$, the distribution of $\hist\^T$ is identical to the following one: first, $\ocirc\^1, \ldots, \ocirc\^T$ are drawn i.i.d.~from $M(\pi_0)$ (for an arbitrary decision $\pi_0$), and then the decisions $\pi\^t$ are chosen adaptively, $\pi\^t \sim q\^t(\cdot | \hist\^{t-1})$, with the rewards $r_1\^t, r_2\^t$ being determined by $\pi\^t$.

For any $\MT \in \sT_i$, and for any $t,t' \in [T]$ with $t \neq t'$, we have $\BP\sups{M_{\MT, 0}}[\ocirc\^t = \ocirc\^{t'}]=1/|\MT| \leq 3/N$. Thus
  \begin{align}
\BP\sups{M_{\MT, 0}} \left( \exists t \neq t' \ : \ \ocirc\^t = \ocirc\^{t'} \right) \leq T^2 \cdot 3/N \leq 1/100\label{eq:birthday},
  \end{align}
  where the final inequality follows since $T \leq \sqrt{N/300}$ (as long as the constant $C$ in the statement of \cref{lem:rs-sc-lb} is sufficiently large). Let $\ME \in \CF\^T$ denote the event that for all $t \neq t'$, $\ocirc\^t \neq \ocirc\^{t'}$. The inequality \cref{eq:birthday} gives that
  \begin{align}
    \E_{\MT \sim \Unif(\sT_i)} \E\sups{M_{\MT, 0}} \left[ \One{\ME} \right] \geq 1-1/100
    \label{eq:e-prob-lb}.
  \end{align}
  In a similar manner, we also have that
  \begin{align}
\E\sups{M_0}[\One{\ME}] \geq 1-1/100\label{eq:e-prob-m0}.
  \end{align}
Now, we may compute
  \begin{align}
    & \E_{\MT \sim \Unif(\sT_i)} \E\sups{M_{\MT, 0}} \left[ \One{\pi\^t \in \MB^\st(\MT) } \ | \ \ME \right]\nonumber\\
    =& \sum_{\MT \in \sT_i} \frac{1}{|\sT_i|} \sum_{\omega_1, \ldots, \omega_T \in [N]}  \BP\sups{M_{\MT, 0}} \left(\ocirc\^{1:T} = \omega_{1:T} \ | \ \ME \right) \cdot \E\sups{M_{\MT, 0}} \left[ \One{\pi\^t \in \MB^\st(\MT)} \ | \ \ocirc\^{1:T} = \omega_{1:T}\right]\nonumber\\
    =& \sum_{\MT \in \sT_i} \frac{1}{|\sT_i|} \sum_{\substack{\omega_1, \ldots, \omega_T \in \MT \\ \text{distinct}}} \frac{1}{N_i (N_i-1) \cdots (N_i-T+1)} \cdot \E\sups{M_{ 0}} \left[ \One{\pi\^t \in \MB^\st(\MT)} \ | \ \ocirc\^{1:T} = \omega_{1:T}\right]  \nonumber\\
    =& \sum_{\substack{\omega_1, \ldots, \omega_T \in [N] \\ \text{distinct}}}  \frac{1}{N(N-1) \cdots (N-T+1)} \E_{\MT \sim \Unif(\sT_i)} \E\sups{M_0}\left[ \One{\pi\^t \in \MB^\st(\MT)} \ | \ \ocirc\^{1:T} = \omega_{1:T},\ \{ \omega_1, \ldots, \omega_{T}\} \subset \MT\right] \label{eq:mt0-m0-relate}\\
    \leq & 2/q,\nonumber
  \end{align}
  where:
  \begin{itemize}
  \item The second equality uses that the distribution of $\hist\^T$ conditioned on $o\^{1:T}$ is identical under $\BP\sups{M_0}$ and $\BP\sups{M_{\MT, 0}}$.
  \item The third equality switches the order of summation and uses that $1/|\sT_i| = {N \choose N_i}^{-1} = \frac{N_i!}{N(N-1)\cdots (N-N_i+1)}$, as well as the fact that 
    the number of sets $\MT$ containing any tuple $\omega_1, \ldots, \omega_T \in [N]$ of distinct integers is $\frac{(N-T)(N-T-1) \cdots (N-N_i+1)}{(N_i-T)!}$.
  \item The final inequality uses the fact that, for fixed $\omega_1, \ldots, \omega_T$, the distribution of $\MT \sim \Unif(\sT_i) | \{ \omega_1, \ldots, \omega_T \} \subset \MT$ is independent of the distribution of $\hist\^T \sim \BP\sups{M_0} | \ocirc\^{1:T} = \omega_{1:T}$. Moreover, the definition of $\MB^\st(\MT)$ in terms of $\Phi(\MT)$ and the fact $\Phi$ satisfies \cref{cond:unif-hit} means that, for any fixed $\pi = (\pi_1, \pi_2) \in \Pi$ with $\pi_1 > 0$,
    \begin{align}
\BP_{\MT \sim \Unif(\sT_i)}(\pi \in \MB^\st(\MT) | \{ \omega_1, \ldots, \omega_T \} \subset \MT ) = \BP_{\MT \sim \Unif(\sT_i)} (\Phi(\MT)(\pi_1) = \pi_2 | \{ \omega_1, \ldots, \omega_T\}\subset \MT) \leq 2/q\nonumber,
    \end{align}
    where we take $\Phi(\MT)(\pi_1) = -1$ if $\pi_1$ is not in the domain of $\Phi(\MT)$. (Here we have also used that $T \leq \sqrt{N}$.) In particular, the above inequality holds with the random choice of $\pi\^t \sim \BP\sups{M_0} | \ocirc\^{1:T} = \omega_{1:T}$ replacing $\pi$. 
  \end{itemize}
Taking a union bound over all $T$ values of $t \in [T]$ and applying \cref{eq:e-prob-lb}, the first claim \cref{eq:intersect-ub} of the lemma follows.

To show the second claim \cref{eq:lem-m0mt0-close} of the lemma, we note that for any fixed subset $\MF \in \CF\^T$ (not depending on $\MT$) the chain of equalities ending in \cref{eq:mt0-m0-relate} implies that
\begin{align}
\E_{\MT \sim \Unif(\sT_i)} \E\sups{M_{\MT, 0}} [\One{\hist\^T \in \MF} \ | \ \ME] = \E\sups{M_0}[\One{\hist\^T \in \MF} \ | \ \ME] \nonumber,
\end{align}
\eqref{eq:lem-m0mt0-close} follows from the above equality combined with \cref{eq:e-prob-lb} and \cref{eq:e-prob-m0}. 
\end{proof}

\begin{proof}[\pfref{lem:mt0-mt-close}]
Let $\ME$ denote the event that $\{\pi\^1, \ldots, \pi\^T\} \cap \MB^\st(M_\MT) =\emptyset$.  Consider any subset $\MF \subset \CF\^T$ of histories. Then
  \begin{align}
\BP\sups{M_{\MT,0}} (\ME \cap \MF) = \BP\sups{M_\MT} (\ME \cap \MF)\nonumber,
  \end{align}
  which follows since for any decition $\pi \not \in \MB^\st(M_\MT)$, the distribution over the full observation $o \sim M_\MT(\pi)$ and $o \sim M_{\MT, 0}(\pi)$ is identical. The statement of the lemma then follows from \cref{lem:pq-tvd-cond}. 
\end{proof}

\subsection{Supplementary lemmas}

The following lemma, which is an elementary fact from coding theory, states the dimension and distance properties of the \emph{Reed-Solomon code}. To present it, we recall the definition of Hamming distance: for $q,n \in \BN$, and $w,w' \in [q]^n$, we let $\ham{w}{w'} = | \{ i \in [n] : w_i \neq w_i' \} |$ to be the number of positions at which $w,w'$ differ. 
\begin{lemma}[Reed-Solomon code; Section 5.2 of \cite{guruswami2022essential}]
  \label{thm:rs}
  Fix any integers $n,q,k$ satisfying $q \geq k$. Then there is a mapping $\Phi : [q]^k \ra [q]^n$ so that for any two vectors $v,v' \in [q]^k$ with $v \neq v'$, it holds that $\ham{\Phi(v)}{\Phi(v')} \geq n-k+1$.

Furthermore, $\Phi$ may be chosen so that if $X \in [q]^k$ is uniformly random, then for each $i \in [n]$, the value $\Phi(X)_i \in [q]$ is uniformly random. 
\end{lemma}

 \cref{lem:extractor} below shows that a certain type of \emph{randomness extractor} exists.
\begin{lemma}
  \label{lem:extractor}
  There is a sufficiently large constant $C \geq 1$ so that the following holds. 
  Consider any positive integers $N, N_0, R, q$ with $N_0 \leq 2N/3$, $R \leq N/6 \leq N_0 - N/6$, and $3 \log q \leq \log{5N/6 \choose N/6} -C$.  Let $\Psi : {[N] \choose N_0} \ra [q]$ be a uniformly random function. Then with probability at least $1 - N^{R+1} \cdot 2^{-{5N/6 \choose N/6} /(Cq^2)}$ over the choice of $\Psi$,  for all subsets $\MQ \subset [N]$ of size $|\MQ| \leq R$, and all $j \in [q]$,
  \begin{align}
    \BP_{\MT \sim \Unif{[N] \choose N_0} }\left( \Psi(\MT) = j | \MQ \subset \MT \right) \leq \frac{2}{q}.\nonumber %
  \end{align}
\end{lemma}
We clarify that the distribution of the uniformly random function $\Psi : {[N] \choose N_0} \ra [q]$ in the above lemma statement is given as follows: for each $\MS \in {[N] \choose N_0}$, $\Psi(\MS)$ is an independent random variable, distributed uniformly on $[q]$. 

\dfcomment{it would be good to clarify definition of uniformly random function above}\noah{did so}

\begin{proof}[\pfref{lem:extractor}]
  Since $R \leq N/6\leq N_0 - N/6$ and $N_0 \leq 2N/3$, for any subset $\MQ \subset [N]$ of size $|\MQ| \leq R$, the distribution of $\MT \sim \Unif{[N] \choose N_0} | \MQ \subset \MT$ puts mass at most $1/{5N/6 \choose N/6}$ on any subset $\MT$ (such a distribution is known as a \emph{flat $k$-source} for some $k \geq \log {5N/6 \choose N/6}$). By \citet[Proposition 6.12]{vadhan2012pseudorandomness} with $\vep = 1/q$, for a sufficiently large constant $C$, as long as $3 \log q \leq \log{5N/6 \choose N/6} - C$, with probability at least $1-2^{-{5N/6 \choose N/6}/(Cq^2)}$ over the choice of $\Psi$, it holds that, for any fixed $\MQ$ of size at most $R$, the distribution of $\Psi(\MT)$, with $\MT \sim \Unif{[N] \choose N_0} | \MQ \subset \MT$, is $1/q$-close (in total variation distance) to uniform on $[q]$, which in particular implies that $\Psi(\MT) = j$ with probability at most $2/q$ for any $j \in [q]$ (again under $\MT \sim \Unif{[N] \choose N_0}|\MQ \subset \MT$). 

  Taking a union bound over all $\sum_{r=0}^R {N \choose R} \leq N^{R+1}$ possible sets $\MQ$, we obtain that $\Psi$ satisfies the desired property with probability at least $1 - N^{R+1} \cdot 2^{-{5N/6 \choose N/6} /(Cq^2)}$.
\end{proof}

\end{document}